\documentclass[11pt]{article}

\usepackage[utf8]{inputenc}

\usepackage{mathrsfs}
\usepackage{amsmath,amssymb,amsthm,mathtools}
\usepackage{bm,dsfont}
\usepackage{natbib}

\usepackage{microtype}
\usepackage{graphicx}
\usepackage{wrapfig}
\usepackage{subcaption}

\usepackage{booktabs,array,multirow,threeparttable,longtable,makecell}
\usepackage{float}
\usepackage{capt-of}
\usepackage{enumitem}
\usepackage{nicefrac,pifont}
\usepackage{tikz}
\usetikzlibrary{arrows.meta}
\usepackage{eso-pic}
\usepackage{xcolor}
\usepackage{indentfirst}
\usepackage{setspace}
\usepackage[left=1in,right=1in,top=1in,bottom=1in]{geometry}

\usepackage[
  colorlinks=true,
  linkcolor=red,
  anchorcolor=blue,
  citecolor=blue,
  urlcolor=blue
]{hyperref}
\usepackage[capitalize,noabbrev,nameinlink]{cleveref}
\crefname{equation}{Eq.}{Eqs.}
\Crefname{equation}{Eq.}{Eqs.}

\crefname{appendix}{Appendix}{Appendices}
\Crefname{appendix}{Appendix}{Appendices}

\usepackage{paper_macros}

\newcommand{\disableaddcontentsline}{%
  \let\savedaddcontentsline\addcontentsline
  \renewcommand{\addcontentsline}[3]{}%
}
\newcommand{\enableaddcontentsline}{%
  \let\addcontentsline\savedaddcontentsline
}

\title{\huge From Spectra to Joint Schedules in LLM Pre-training:\\
\(3+3(+2)\) Scaling-Law Regimes}

\author{
Yichen Wang\thanks{Department of Computer Sciences, University of Wisconsin--Madison, USA. E-mails:
\href{mailto:yichen.wang@wisc.edu}{\texttt{yichen.wang@wisc.edu}} and
\href{mailto:yudongchen@cs.wisc.edu}{\texttt{yudongchen@cs.wisc.edu}}}
~~~
Fanghui Liu\thanks{School of Mathematical Sciences, Institute of Natural Sciences and MOE-LSC, Shanghai Jiao Tong University, China. E-mail:
\href{mailto:fanghui.liu@sjtu.edu.cn}{\texttt{fanghui.liu@sjtu.edu.cn}} (Corresponding author)}
~~~
Yudong Chen\footnotemark[1]
}
\date{}

\hypersetup{
  pdftitle={From Spectra to Joint Schedules in LLM Pre-training: 3+3(+2) Scaling-Law Regimes},
  pdfauthor={Yichen Wang, Fanghui Liu, Yudong Chen}
}

\begin{document}
\disableaddcontentsline
\maketitle

\begin{abstract}
Power-law learning curves are often treated as fixed properties of a model and
its data, although learning-rate and batch-size schedules can change the
observed loss.  We study this dependence in noisy online SGD with linear
random features. Conditional on the representation, an exact Volterra
equation separates two response components: a forcing term that propagates
unresolved target error and a memory kernel that propagates stochastic-error
injections. We prove that either component follows a power law if and only if
its cumulative weighted spectral mass has the corresponding low-spectrum
scaling; individual eigenvalues and target coefficients need not obey
coordinatewise power laws. Under a joint schedule, intrinsic time
\(T_t=\sum_{s<t}\eta_s\) controls optimization progress, while
\(r_t=B_t/\eta_t\) controls noise injection. Their interaction yields sharp
conditions under which a schedule preserves, changes, or destroys the clean
power law, together with a memory ceiling on noise reduction. The power-law
random-feature model realizes this mechanism in \(3+3(+2)\) propagation
regimes with phase-dependent compute rates. Controlled nanoGPT experiments show that (1) learning-rate and batch-size schedules with matched \(B/\eta\) paths are
nearly equivalent in intrinsic time, (2) a forcing-memory surrogate accurately
predicts loss across schedules, and (3) its fitted exponents across real-world datasets identify the regime of LLMs in \(3+3(+2)\) map.
The source code for reproducing our experiments is available on \url{https://github.com/yichenblue/spectra-to-schedules-in-pretraining}.

\end{abstract}

\section{Introduction}

Power-law learning curves in scaling laws are widely used to forecast training progress and allocate model size, data, and compute in large language models (LLMs) pre-training
\citep{hestness2017deep,kaplan2020scaling,hoffmann2022training}.
Such extrapolation is powerful as the fitted law can transfer beyond
the scale and training configuration on which it was measured. 

Yet it is often unclear what is expected to transfer: \emph{the existence of a power law}, \emph{its exponent}, or \emph{the mechanism that generates it}. These notions need not coincide.  An apparent power law may terminate when a finite range of slow modes is exhausted; and even with the representation and target fixed, changing the learning-rate or batch-size schedule (from LLM pre-training) can change both optimization progress and the accumulation of stochastic error.  A scaling exponent is therefore not, in general, a fixed property of a model and its data. Recent schedule-aware scaling laws make explicit that the learning curve can depend on the training schedule rather than being a fixed attribute of the model--data pair \citep{tissue2024scaling,luo2025multi,qiu2025scaling}. Functional scaling laws (FSL) \citep{li2025functional} formalize this dependence as an intrinsic-time functional under learning-rate and batch-size schedules, separately.

Rather than treating {\bf universality} as a single yes-no question, we specify
two distinct problems studied in this paper. The first is one of \emph{origin}: when do the underlying learning dynamics produce power-law response components? The second is one of \emph{transfer}: once such components exist, when does a joint learning-rate/batch-size schedule preserve, change, or destroy their law in the observed loss? 

It appears technically impossible to quantitatively study the above two problems in a trained language model. The representation learning, target alignment,
stochasticity, and scheduling evolve together, so their separate roles are
difficult to identify from the loss curve alone. We therefore study a \emph{proxy} model, linear random features trained by noisy online stochastic gradient descent (SGD)
\citep{rahimi2007random,liu2021random,mei2022generalization} to answer the above two questions. 

This model is deliberately simple. The value of the proxy lies not in its literal similarity to an LLM, but in whether the response coordinates it reveals remain predictive outside the proxy. We therefore test two consequences in controlled 300M Nano-GPT pre-training \citep{Karpathy2022}: i) whether the loss can be fully characterized by the theory-derived coordinates, i.e., the intrinsic time and the ratio between batch-size and learning rate; 2) whether a theory-derived forcing-memory surrogate fitted on one schedule predicts held-out schedules in LLM pre-training.
\Cref{fig:llm_response_summary} verifies these tests.

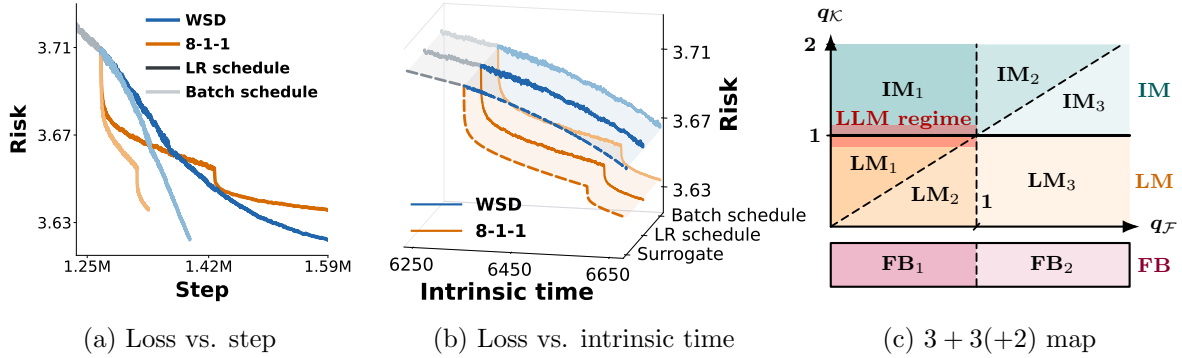
\begin{figure}[!t]
  \setlength{\abovecaptionskip}{2pt}
  \centering

  \begin{subfigure}[b]{0.28\linewidth}
    \centering
    \includegraphics[width=\linewidth]{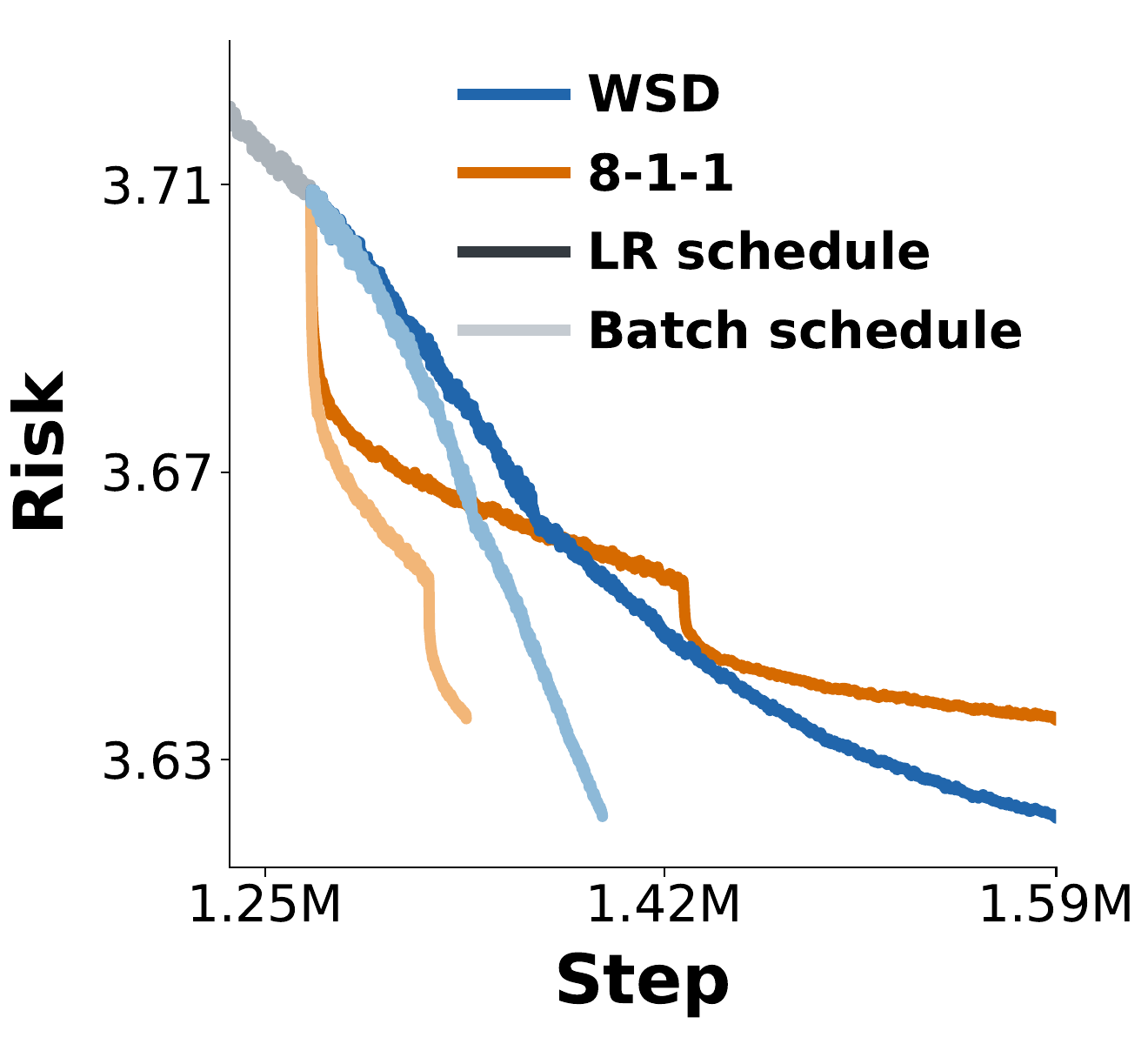}
\caption{Loss vs. step}    \label{fig:llm_response_summary_step}
  \end{subfigure}%
  \hspace*{0.004\linewidth}%
  \begin{subfigure}[b]{0.36\linewidth}
    \centering
    \includegraphics[width=\linewidth]{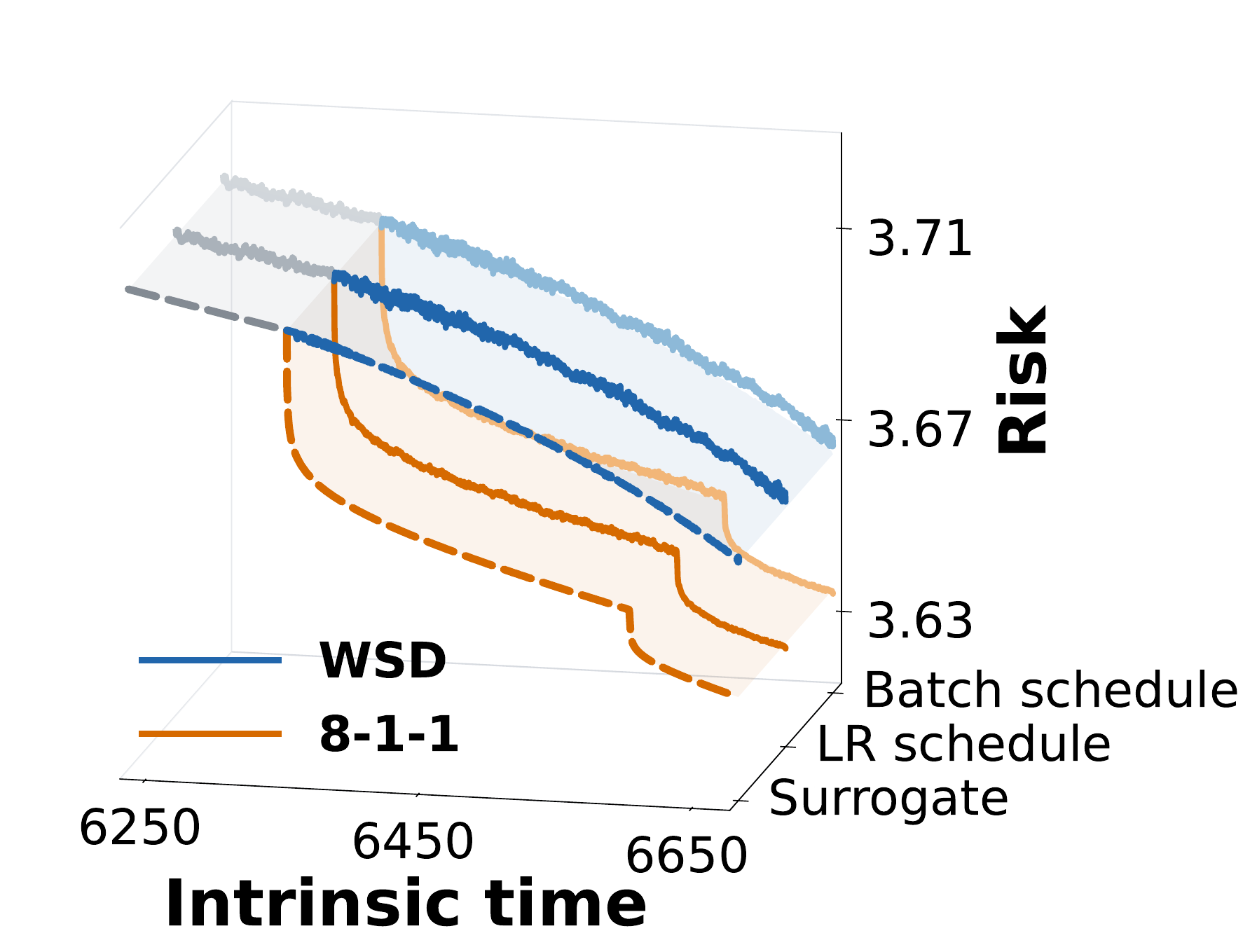}
    \caption{Loss vs. intrinsic time}
    \label{fig:llm_response_summary_intrinsic_time}
  \end{subfigure}%
  \hspace*{-0.03\linewidth}%
  \begin{subfigure}[b]{0.35\linewidth}
    \centering
    \raisebox{5pt}{%
      \resizebox{0.90\linewidth}{!}{%
        \begin{tikzpicture}[
          x=2.0cm,
          y=1.25cm,
          font=\scriptsize,
          line cap=round,
          line join=round,
          >=Latex,
          subregime label/.style={
            font=\scriptsize\bfseries,
            inner xsep=1.4pt,
            inner ysep=.55pt
          }
        ]
          % Finite-bulk regimes.
          \fill[purple!27] (0,-.66) rectangle (1,-.18);
          \fill[purple!11] (1,-.66) rectangle (2.05,-.18);
          \draw[black,line width=.82pt]
            (0,-.66) rectangle (2.05,-.18);

          % Long-memory and integrable-memory regimes.
          \fill[orange!34] (0,0)--(0,1)--(1,1)--cycle;
          \fill[orange!22] (0,0)--(1,0)--(1,1)--cycle;
          \fill[orange!10] (1,0) rectangle (2.05,1);
          \fill[teal!31] (0,1) rectangle (1,2);
          \fill[teal!19] (1,1)--(1,2)--(2,2)--cycle;
          \fill[teal!8]
            (1,1)--(2.05,1)--(2.05,2)--(2,2)--cycle;

          % Empirical LLM response regime.
          \fill[red!65,opacity=.48]
            (0,.88) rectangle (1.0,1.12);

          % Regime boundaries.
          \draw[black,line width=1.18pt] (0,1)--(2.05,1);
          \draw[black,line width=.72pt,densely dashed]
            (1,-.66)--(1,-.18);
          \draw[black,line width=.72pt,densely dashed]
            (1,0)--(1,2);
          \draw[black,line width=.72pt,densely dashed]
            (0,0)--(2,2);

          % Axes.
          \draw[->,line width=.82pt] (0,0)--(2.12,0)
            node[right,font=\scriptsize\bfseries]
            {$\bm{q}_{\mathcal F}$};
          \draw[->,line width=.82pt] (0,0)--(0,2.12)
            node[above=.5pt,font=\scriptsize\bfseries]
            {$\bm{q}_{\mathcal K}$};

          \draw[line width=.7pt] (.98,-.035)--(1.02,.035)
            node[
              above=2pt,
              xshift=3pt,
              font=\scriptsize\bfseries
            ] {$\bm 1$};

          \draw[line width=.7pt] (-.025,1)--(.025,1);
          \node[
            left=3pt,
            font=\scriptsize\bfseries,
            inner sep=.4pt
          ] at (0,1) {$\bm 1$};

          \draw[line width=.7pt] (-.025,2)--(.025,2)
            node[left=3pt,font=\scriptsize\bfseries]
            {$\bm 2$};

          % Regime-family labels.
          \node[
            font=\scriptsize\bfseries,
            text=purple!72!black
          ] at (2.22,-.42) {$\bm{\mathrm{FB}}$};

          \node[
            font=\scriptsize\bfseries,
            text=orange!80!black
          ] at (2.22,.50) {$\bm{\mathrm{LM}}$};

          \node[
            font=\scriptsize\bfseries,
            text=teal!72!black
          ] at (2.22,1.50) {$\bm{\mathrm{IM}}$};

          % Subregime labels.
          \node[subregime label]
            at (.50,-.42) {$\bm{\mathrm{FB}}_1$};
          \node[subregime label]
            at (1.52,-.42) {$\bm{\mathrm{FB}}_2$};

          \node[subregime label]
            at (.30,.68) {$\bm{\mathrm{LM}}_1$};
          \node[subregime label]
            at (.72,.34) {$\bm{\mathrm{LM}}_2$};
          \node[subregime label]
            at (1.52,.50) {$\bm{\mathrm{LM}}_3$};

          \node[subregime label]
            at (.50,1.50) {$\bm{\mathrm{IM}}_1$};
          \node[subregime label]
            at (1.29,1.66) {$\bm{\mathrm{IM}}_2$};
          \node[subregime label]
            at (1.74,1.36) {$\bm{\mathrm{IM}}_3$};

          \node[
            font=\scriptsize\bfseries,
            text=red!70!black
          ] at (.50,1.17) {LLM regime};
        \end{tikzpicture}%
      }%
    }

    \par\vspace{3pt}
    \caption{\(3+3(+2)\) map}
    \label{fig:llm_response_summary_phase}
  \end{subfigure}%
  \caption{
    \textbf{Proxy-derived response coordinates in 300M nanoGPT training.}
    (a) We control the same ratio path \(r_t=B_t/\eta_t\) by changing batch-size and learning rate schedules equivalently, e.g., the WSD \citep{hu2024minicpm} and 8-1-1 \citep{bi2024deepseek}, leading to different trajectories against optimizer step.
(b) After reparameterization by intrinsic time, each trajectory coincides. A surrogate fitted only to the 8-1-1 trajectory also predicts the held-out WSD trajectory without refitting.
(c) The fitted exponent satisfies
\(q_{\mathcal K}\approx1\) and \(q_{\mathcal F} < 1\), placing the LLM response near the \(\mathrm{IM}_1/\mathrm{LM}_1\) boundary of our \(3+3(+2)\) scaling law regimes.}
  \label{fig:llm_response_summary}
\end{figure}

\subsection{Contributions and findings}

The prediction risk of our proxy model can be precisely characterized by Volterra recursion \citep{paquette20244+}. Its forcing term \(F\) propagates unresolved target error
without stochastic feedback, whereas its memory kernel \(K\) describes how
long the effect of one stochastic-error injection survives.  The resulting
mechanism has two successive stages:
\[
\underbrace{\text{weighted low-spectrum mass}}_{\text{spectral origin}}
\quad\longrightarrow\quad
\underbrace{(F,K)}_{\text{forcing and memory}}
\quad\xrightarrow[\;{\color{red}T_t:=\sum_{s<t}\eta_s},\;{\color{blue}r_t:=B_t/\eta_t}\;]{\text{joint schedule}}\quad
\underbrace{R_{\sigma,t}}_{\text{observed loss}} \,,
\]
where $(T_t,r_t)$ are the two natural schedule coordinates.  Intrinsic time \(T_t\) measures optimization progress, motivated by \cite{li2025functional}; while \(r_t:=B_t/\eta_t\) under the batch size $B_t$ and the learning rate $\eta_t$ controls stochastic-error injection per unit progress.  The weighted low spectrum determines whether \(F\) and \(K\) follow power laws; the joint schedule then accumulates the memory response along the training trajectory.  Consequently, the observed curve of the loss is not determined by componentwise power laws separately, but governed by the competition between clean target-error decay and the accumulated stochastic response. Our contributions are summarized as below.

\begin{itemize}[leftmargin=2em,itemsep=.25em,topsep=.3em]

\item \textbf{Spectral origin of component power laws.}
We prove separate if-and-only-if criteria for power laws in learnable forcing
and one-injection memory.  The decisive quantities are their cumulative
weighted spectral masses near zero, rather than coordinatewise power laws for
individual eigenvalues or target coefficients.  Irregular spectra or targets
can therefore produce canonical power-law responses, while power-law data are neither sufficient nor necessary.

\item \textbf{Concrete \(3+3(+2)\) scaling law regimes.}
We consider specific power-law eigenvalues and target coefficients, formulating a 3+3(+2)-regime map: 3 long-memory (LM), 3 integrable-memory (IM), and 2 finite-bulk propagation regimes. We derive an exact conditional formula for the
finite-width noisy--clean gap in the finite-bulk regimes, asymptotic formulas
for the full deterministic-equivalent risk across the phase map, and
phase-dependent optimal rates under data and feature-compute budgets.

\item \textbf{Transformation under joint learning-rate-batch-size schedules.}
We show that the noisy-clean gap is obtained by accumulating
schedule-dependent stochastic injections through the memory kernel.  This
gives sharp boundaries between schedules that preserve, change, or destroy
the clean power law, together with a memory ceiling beyond which reducing
only late-stage noise cannot improve the decay rate.  Below this ceiling, the
gap identifies the asymptotic behavior of the ratio
\(r_t=B_t/\eta_t\), though not learning rate and batch size separately.

\item \textbf{Proxy for LLM pre-training prediction.}
Controlled nanoGPT (124M and 300M) experiments provide three tests of the theory:
the loss is determined by the intrinsic time and the ratio path \(r_t:=B_t/\eta_t\), which provides a possible way to tune learning-rate or batch-size schedules under a fixed ratio $r_t$. Besides, following the FSL fit-then-transfer protocol of
\citet{li2025functional}, a forcing--memory surrogate fitted on one schedule (e.g., 8-1-1) predicts held-out schedules without refitting (e.g., WSD). Across OpenWebText and FineWeb and across model scales, the fitted LLM response remains near \(q_{\mathcal K}\approx1\) and
\(q_{\mathcal F}<1\).

\end{itemize}

\subsection{Related work}

\noindent{\bf Power-law existence and exponent universality.}
Recent work argues for shared exponents or derives a universal \(1/3\) time
law
\citep{liu2026neuraluniversality,liu2026onethird}.
Other studies show that scaling laws can change with the observation
scale, terminate when a finite spectral range is exhausted, or vary across
spectral phases
\citep{xiao2024rethinking,maloney2022solvable,
bahri2024explaining,lin2024scaling,paquette20244+}.
These results motivate separating the existence of a component power law from
the invariance of the exponent observed in the total loss.

\noindent{\bf Spectral models and random-feature dynamics.}
Solvable spectral theories commonly prescribe power laws for eigenvalues or
target coefficients and then derive the resulting learning curve
\citep{maloney2022solvable,bordelon2024dynamical,
lin2024scaling,paquette20244+}.
Our componentwise criteria instead characterize power laws through cumulative
weighted spectral mass and therefore do not require coordinatewise
power-law sequences.  Random-feature models provide a setting in which
spectral, target, finite-width, and stochastic effects can be separated
\citep{mei2022generalization,xiao2022precise,bach2024high,
ba2022high,moniri2023theory}, while high-dimensional SGD limits provide the
dynamical foundation for the exact recursion used here
\citep{paquette2021sgd,paquette2025homogenization,atanasov2025two}.
The closest PLRF result is the constant-schedule, clean-label \(4+3\)
compute-optimal classification of \citet{paquette20244+}.  We introduce label
noise and joint schedules; our \(3+3(+2)\) map classifies forcing--memory
propagation rather than compute-optimal bottlenecks.

\noindent{\bf Schedules and transferable response laws.}
Schedule optimization has been studied from complementary starting points.
\Citet{bordelon2026theory} assume spectral and source power laws, whereas
functional-scaling-law analyses begin from a specified response law
\citep{li2025functional,li2026optimal,wang2026fast}.
In particular, \citet{li2025functional} fit one learning-rate schedule at
fixed batch size and predict other schedules without refitting.  
We follow this setting but extend the response analysis to jointly varying learning rate and batch size. Our LLM experiments correspondingly test both changes in the ratio path \(B_t/\eta_t\) and different learning-rate-batch-size
factorizations of the same path \citep{meterez2026quadratic,karkada2026predicting,hu2022universality}. Budget-aware schedule design is also studied through batch--sequence--token (BST) scaling, which jointly adjusts batch size, sequence length, and stepsize in stochastic conditional-gradient methods
\citep{islamov2026role}. Related norm-based methods such as Scion support hyperparameter transfer across model widths
\citep{pethick2025training,pethick2026scale}.

\noindent{\bf From a proxy mechanism to LLM response.}
A tractable proxy becomes useful beyond its literal assumptions only when the
coordinates it reveals make falsifiable predictions outside the proxy.
Previous results on local quadratic models \citep{meterez2026quadratic}, kernel-based functional scaling laws \citep{li2025functional}, and Gaussian universality \citep{karkada2026predicting,hu2022universality} demonstrate their potential in prediction of large complex models. Such mechanism-first viewpoint motivates us to test it directly in controlled plain-SGD LLM training, supported by the following two reasons. 

First, whether a proxy model or an LLM is trained with plain SGD, the dynamics can always be cast as signal--noise learning related to forcing and memory terms. Second, paired learning-rate and batch-size factorizations
with matched \(B_t/\eta_t\) paths follow different trajectories in optimizer
step but nearly coincide at equal intrinsic time.

\section{Problem setup and exact loss dynamics}
\label{iclr:sec:model_dynamics}

The model follows the linear random-feature setting of
\citet{paquette20244+}, trained by online SGD, but differs in two ways motivated
by LLM-pretraining practice
\citep{brown2020language,hoffmann2022training,li2025functional,wang2026fast}: it
introduces label noise as a tractable source of persistent stochastic error; it
allows the learning rate and batch size to vary jointly over time.

\subsection{Linear random features trained by online SGD}

We consider Gaussian data $\bm x=\bm\Lambda^{1/2}\bm z \in \mathbb{R}^d$, with $\bm z\sim\mathcal N(\bm 0,\bm I_d)$ and $\bm\Lambda=\operatorname{diag}(\lambda_1,\ldots,\lambda_d)
\succ\bm 0$. It can be naturally extended to kernel
feature maps under the hypercontractivity condition of
\citet{mei2022generalizationkernel}. The data generation process is
\begin{equation*}
y=f_\star(\bm x)+\varepsilon,
\qquad
f_\star(\bm x):=\langle\bm x,\bm\theta^\star\rangle,
\qquad
\varepsilon\perp\!\!\!\perp\bm x,
\qquad
\mathbb E[\varepsilon]=0,
\qquad
\mathbb E[\varepsilon^2]=\sigma^2\,,
\end{equation*}
where \(\bm\theta^\star\in\mathbb R^d\) is the target and the label noise $\varepsilon$ is with zero-mean and bounded variance $\sigma^2$. We use a linear random feature model of width \(m\), $f_{\bm a}(\bm x)
:=\langle\bm W^{\!\top}\bm x,\bm a\rangle$, where
\(\bm W\in\mathbb{R}^{d\times m}\) has i.i.d.\
\(\mathcal N(0,1/m)\) entries and is fixed; only
\(\bm a\in\mathbb{R}^m\) is trained by SGD.
Starting from \(\bm a_0=\bm0\), iteration \(t\) draws a fresh mini-batch of
size \(B_t\) and updates with learning rate \(\eta_t\):
\begin{equation}
\label{eq:joint_sgd_schedule_coordinates}
\bm a_{t+1}=\bm a_t-
\frac{\eta_t}{B_t}\sum_{i=1}^{B_t}
\bm W^{\!\top}\bm x_t^i
\bigl(f_{\bm a_t}(\bm x_t^i)-y_t^i\bigr),
\qquad T_t:=\sum_{s<t}\eta_s,
\qquad r_t:=\frac{B_t}{\eta_t}.
\end{equation}
Batches and label noises are independent across iterations.  Here \(T_t\) is
the optimization clock, while \(1/r_t\) is the variance injected
per unit intrinsic time.
Conditional on \(\bm W\), the excess risk is defined as
\begin{equation*}
R_{\sigma,t}
:=\mathbb E\!\left[
\bigl(f_{\bm a_t}(\bm x)-f_\star(\bm x)\bigr)^2
\,\middle|\,\bm W\right]
=\mathbb E\!\left[
\|\bm\Lambda^{1/2}(\bm W\bm a_t-\bm\theta^\star)\|_2^2
\,\middle|\,\bm W\right].
\end{equation*}
The expectation averages all training mini-batches and label noises and an
independent test covariate, conditioning on \(\bm W\). All definitions and exact recursions in this section hold for finite
\(d\) and \(m\).  The asymptotic results later take
\(m\to\infty\) and \(d=d_m\to\infty\) jointly, with \(d/m\) bounded
below by a constant strictly larger than one; see more details in \cref{sec:noisy_label_scaling_laws}.

\subsection{Forcing, memory, and the exact Volterra equation}
\label{sec:volterra-fixed}

To identify the spectral quantities first, temporarily fix the schedule: \(\eta_t\equiv\eta\), \(B_t\equiv B\), and
\(T:=\eta t\).  The frozen representation induces
\(\widehat{\bm H}:=\bm\Lambda^{1/2}\bm W\bm W^{\!\top}\bm\Lambda^{1/2}\) with its eigenpairs \((\widehat\lambda_j,\widehat{\bm u}_j)\).
The initial energy aligned with the empirical mode $j$ is
\(\lvert\langle\widehat{\bm u}_j,
\bm\Lambda^{1/2}\bm\theta^\star\rangle\rvert^2\). For Gaussian covariates, the corresponding second-moment mode has one-step survival factor
\begin{equation}\label{eq:qeta}
q_\eta(\lambda)
:=1-2\eta\lambda+
\left(1+1/B\right)\eta^2\lambda^2.
\end{equation}
Its power \(q_\eta(\lambda)^t\) is the spectral filter that determines how
much squared error in mode \(\lambda\) survives \(t\) steps: on the
long-time scale \(\lambda=O(T^{-1})\),
\(q_\eta(\lambda)^t\approx e^{-2\eta\lambda t}=e^{-2T\lambda}\). Summing the exact mode recursions gives the exact risk recursion
\begin{equation}
R_{\sigma,t}
=F_{\bm W}(t)+
\sum_{s=0}^{t-1}K_{\bm W}(t-1-s)(R_{\sigma,s}+\sigma^2)\,.
\label{iclr:eq:exact_volterra}
\end{equation}
Here the forcing term $F_{\bm W}$ and the memory kernel $K_{\bm W}$ are defined as
\begin{align*}
F_{\bm W}(t)
&:=\sum_j|\langle\widehat{\bm u}_j,
\bm\Lambda^{1/2}\bm\theta^\star\rangle|^2
q_\eta(\widehat\lambda_j)^t,
\qquad
K_{\bm W}(t)
:=\frac{\eta^2}{B}
\sum_j\widehat\lambda_j^2q_\eta(\widehat\lambda_j)^t \,,
\end{align*}
where \(F_{\bm W}\) is the initial target error propagated without
stochastic feedback.  The memory kernel \(K_{\bm W}\) is the impulse response
of one SGD-variance injection. When \(\sigma=0\),
\cref{iclr:eq:exact_volterra} reduces to the clean-label recursion of
\citet{paquette20244+}.
The zero-mode forcing is the finite-width approximation floor
\(R_{\mathrm{app}}:=\sum_{\widehat\lambda_j=0}
|\langle\widehat{\bm u}_j,
\bm\Lambda^{1/2}\bm\theta^\star\rangle|^2\); write
\(F_{\bm W,>0}\) for the learnable part of the forcing.

\section{When do frozen spectral dynamics produce power laws?}
\label{iclr:sec:foundations}

Every fixed positive mode in \cref{eq:qeta} decays exponentially.  A power law
can nevertheless emerge after infinitely many modes with different time scales
are added.  This section makes that statement precise for \(F\) and \(K\); it
does not yet claim a power law for the total loss.

\subsection{From the moving spectral cutoff to a spectral criterion}

On the long-time scale \(T=\eta t\), modes with
\(\lambda\gg T^{-1}\) have relaxed, while modes with
\(\lambda\ll T^{-1}\) are nearly untouched: $q_\eta(\lambda)^t\approx e^{-2T\lambda}
\approx \bm 1\{\lambda\lesssim T^{-1}\}$. Hence the remaining error is governed by cumulative spectral weight below the
moving cutoff \(T^{-1}\), not by the decay of a single mode
\citep{paquette20244+,li2025functional}. Accordingly, for the eigenpairs of \(\widehat{\bm H}\), define the target-weighted and
memory-weighted empirical spectral measures
\[
\nu_{\bm W}^{\mathcal F}
:=
\sum_j
\left|\left\langle
\widehat{\bm u}_j,\bm\Lambda^{1/2}\bm\theta^\star
\right\rangle\right|^2
\delta_{\widehat\lambda_j},
\qquad
\nu_{\bm W}^{\mathcal K}
:=
\sum_j\widehat\lambda_j^2\delta_{\widehat\lambda_j}.
\]
Here \(\nu_{\bm W}^{\mathcal F}((0,x])\) is the target energy in learnable
directions slower than \(x^{-1}\), whereas
\(\nu_{\bm W}^{\mathcal K}((0,x])\) measures the strength of one variance
injection stored in those directions.  The zero atom
\(\nu_{\bm W}^{\mathcal F}(\{0\})=R_{\mathrm{app}}\) is the approximation
floor and is excluded from \(F_{\bm W,>0}\).

These measures give the exact spectral representations and their long-time
Laplace approximations:
\[
\begin{aligned}
F_{\bm W,>0}(t)
&=\int_{(0,\infty)}q_\eta(\lambda)^t\,
\nu_{\bm W}^{\mathcal F}(\mathrm d\lambda)
\approx
\int_{(0,\infty)}e^{-2T\lambda}\,
\nu_{\bm W}^{\mathcal F}(\mathrm d\lambda)
 =2T\int_0^\infty e^{-2Tx}
\nu_{\bm W}^{\mathcal F}((0,x])\,\mathrm dx,\\
\frac{B}{\eta^2}K_{\bm W}(t)
&=\int_{(0,\infty)}q_\eta(\lambda)^t\,
\nu_{\bm W}^{\mathcal K}(\mathrm d\lambda)
\approx
\int_{(0,\infty)}e^{-2T\lambda}\,
\nu_{\bm W}^{\mathcal K}(\mathrm d\lambda)
 =2T\int_0^\infty e^{-2Tx}
\nu_{\bm W}^{\mathcal K}((0,x])\,\mathrm dx.
\end{aligned}
\]

Accordingly, the following theorem characterizes exactly when either component
follows a temporal power law.  The fully quantified statement is given in
\cref{thm:noisy_43_effective_spectral_criterion}.

\begin{theorem}[Componentwise spectral criterion (informal)]
\label{iclr:thm:spectral_iff}
Fix a stable constant schedule, condition on \(\bm W\) at each width, and take
a joint width--time limit with \(m,t\to\infty\) and
\(T=\eta t\to\infty\). Under appropriate uniform spectral-window conditions, with
\(x=T^{-1}\downarrow0\), for
\(q_{\mathcal F},q_{\mathcal K}>0\),
\[
\begin{aligned}
\nu_{\bm W}^{\mathcal F}((0,x])\propto x^{q_{\mathcal F}}
\Longleftrightarrow
F_{\bm W,>0}(t)\propto T^{-q_{\mathcal F}},%\\[0.35em]
\qquad \nu_{\bm W}^{\mathcal K}((0,x])\propto x^{q_{\mathcal K}}
\Longleftrightarrow
\frac{B}{\eta^2}K_{\bm W}(t)\propto T^{-q_{\mathcal K}}.
\end{aligned}
\]
\end{theorem}
\Cref{iclr:thm:spectral_iff} shows that the decisive quantities are the
cumulative weighted spectral masses near zero, rather than pointwise
power-law formulas for the eigenvalues or target coefficients.
It can be empirically validated by \cref{iclr:fig:spectral_iff_rapid_target}.
In
\cref{iclr:fig:spectral_iff_rapid_target_forcing,iclr:fig:spectral_iff_rapid_target_memory},
each temporal response is compared with its corresponding cumulative spectral
mass. The forcing response and forcing mass both
decay faster than every inverse power of \(T\), whereas the memory response and
memory mass both scale as \(T^{-3/4}\), illustrating the two componentwise
equivalences in \cref{iclr:thm:spectral_iff}.  Separately,
\cref{iclr:fig:spectral_iff_rapid_target_loss} shows that the minibatch-SGD
loss initially follows the exponentially decaying forcing response and then
crosses over to the \(T^{-3/4}\) power-law memory tail.

\begin{figure}[t]
\setlength{\abovecaptionskip}{2pt}
\captionsetup[subfigure]{skip=0pt}
\centering
\begin{subfigure}[t]{0.3\linewidth}
  \centering
  \includegraphics[width=\linewidth]{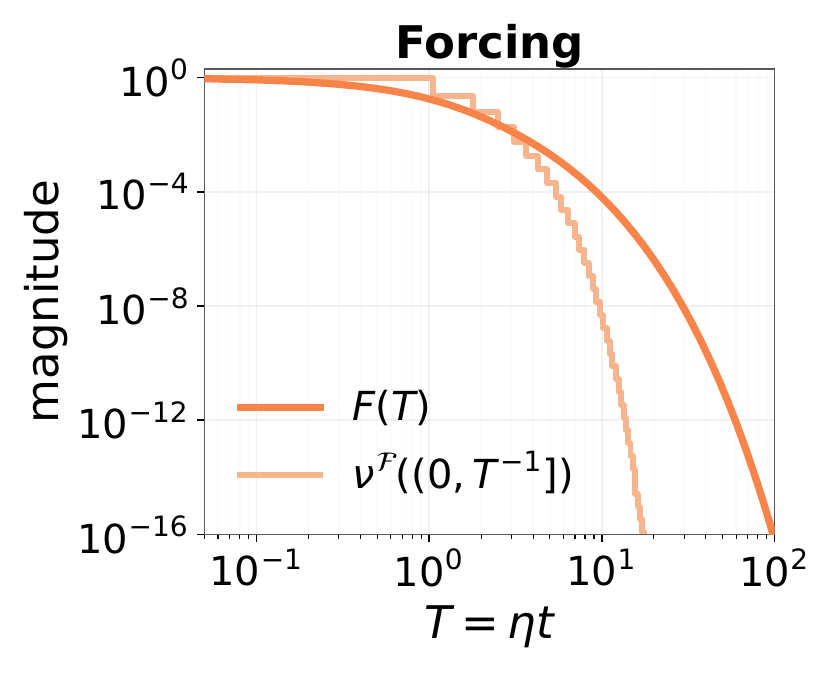}
  \caption{Forcing}
  \label{iclr:fig:spectral_iff_rapid_target_forcing}
\end{subfigure}\hfill
\begin{subfigure}[t]{0.3\linewidth}
  \centering
  \includegraphics[width=\linewidth]{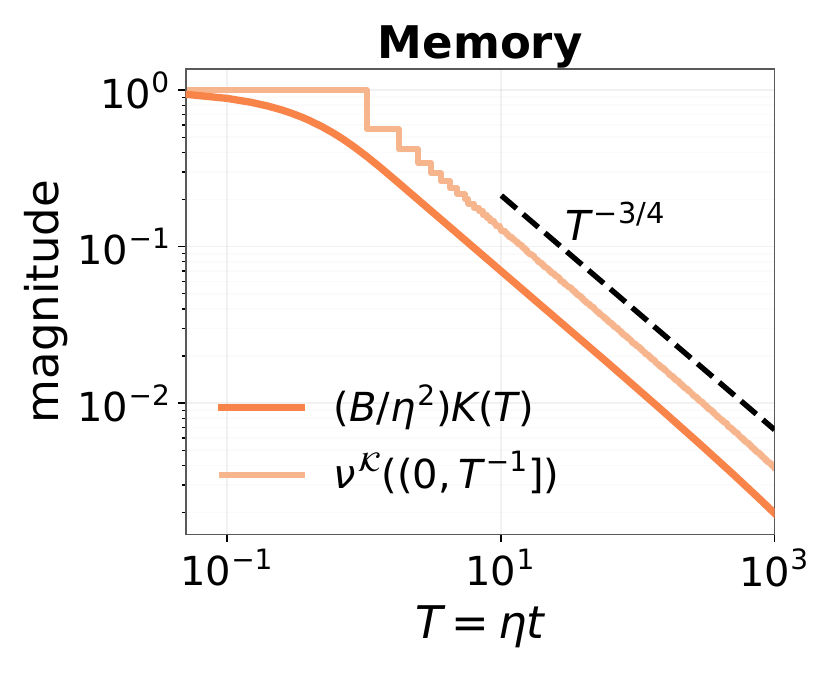}
  \caption{Memory}
  \label{iclr:fig:spectral_iff_rapid_target_memory}
\end{subfigure}\hfill
\begin{subfigure}[t]{0.3\linewidth}
  \centering
  \includegraphics[width=\linewidth]{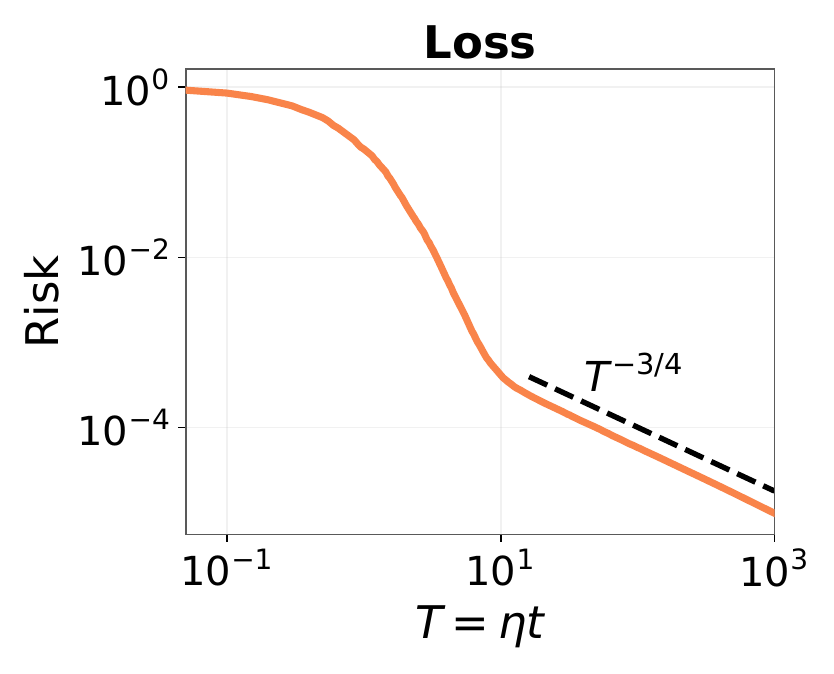}
  \caption{Loss}
  \label{iclr:fig:spectral_iff_rapid_target_loss}
\end{subfigure}
\caption{Finite-width rapid-target construction with
\(\lambda_j=j^{-0.8}\), \(\lvert\theta_j^\star\rvert^2=e^{-j}\), and
\(d=524{,}288\). (a) Normalized forcing and cumulative forcing mass
both decay faster than every inverse power of \(T\).  (b) Normalized
one-injection memory and cumulative memory mass both track \(T^{-3/4}\). (c) The minibatch-SGD loss curve averaged over 20 runs follows the
\(T^{-3/4}\) memory tail.}
\label{iclr:fig:spectral_iff_rapid_target}
\vspace*{-8pt}
\end{figure}

\subsection{When power laws appear---and when they do not}

Based on \Cref{iclr:thm:spectral_iff}, we provide several examples to see whether power laws appear or not. For instance, power-law data are neither sufficient nor necessary. See more details in \cref{app:condition_examples}.
\begin{itemize}
\item \textbf{Only power-law data are insufficient.}
    Zipf-like rank--frequency laws make power-law structure familiar in
    language data \citep{he2025pretrained,mikhaylovskiy2025zipf}. Such data-side scaling does not determine target alignment in the spectral
    dynamics studied here.  For \(\alpha>1/4\), take
    \(\lambda_j=j^{-2\alpha}\) and
    \(\lvert\theta_j^\star\rvert^2=e^{-j}\).
    The target places exponentially little energy in slow directions.  For the
    corresponding deterministic construction,
    \(-\log F_{\bm W,>0}(t)\asymp T^{1/(2\alpha+1)}\).
    The forcing therefore decays faster than every inverse power of \(T\),
    although the eigenspectrum is an exact power law.  The memory decay is
    unchanged because it does not depend on the target.
    \Cref{iclr:fig:spectral_iff_rapid_target} visualizes this deterministic
    construction.

    \item \textbf{Power-law data are not necessary either.}
    For example, the constructions
\(\lambda_j=e^{-j^\gamma}\),
\(\lvert\theta_j^\star\rvert\asymp j^{-\beta}\), \(0<\gamma<1\), and
\(\lambda_j=e^{-j}\),
\(\lambda_j\lvert\theta_j^\star\rvert^2\propto e^{-\tau j}\), both
produce power-law temporal orders on their corresponding scaling windows.
    
    \item \textbf{Irregular targets can give the canonical exponent.}
    For \(\lambda_j=j^{-2\alpha}\), with \(\alpha>1/4\) and
    \(2\alpha+2\beta>1\), write
    \(\lvert\theta_j^\star\rvert^2=j^{-2\beta}c_j\), where \(c_j\geq0\).
    If \(n^{-1}\sum_{j\leq n}c_j\to1\), then
    \(\sum_{\lambda_j\leq x}\lambda_j\lvert\theta_j^\star\rvert^2
    \propto x^{(2\alpha+2\beta-1)/(2\alpha)}\).
    Thus oscillatory, mass-compensated sparse, and frozen random-amplitude
    targets can retain the canonical forcing exponent.

    \item \textbf{A finite spectrum leads to exponential decay.}
It may display a long intermediate scaling window before entering this
asymptotic regime.
\end{itemize}

\noindent{\bf Non-identifiability: From forcing and memory to the observed loss.}
\Cref{iclr:thm:spectral_iff} is componentwise: its if-and-only-if conclusions
apply separately to \(F_{\bm W,>0}\) and \(K_{\bm W}\), not to the observed
loss.  \(F_{\bm W,>0}\) gives the learnable forcing, whereas \(K_{\bm W}\)
describes the survival of one stochastic injection.  The observed loss is
therefore not obtained by directly comparing these two raw components: the
memory contribution must first be accumulated over past injections and
propagated through feedback.
\section{PLRF propagation regimes: \texorpdfstring{\(3+3(+2)\)}{3+3(+2)} scaling laws}
\label{iclr:sec:schedules}

We now instantiate the general spectral criterion of
\cref{iclr:sec:foundations} in the canonical power-law random-feature (PLRF)
model of \citet{paquette20244+}. The resulting \(3+3(+2)\) phase map classifies
how learnable forcing and one-injection memory propagate, rather than assigning
eight universal exponents to the observed loss. The schedule must still
accumulate memory over past noise injections; \cref{sec:rv_joint_schedules}
carries out this step and determines the resulting learning curve.

Consider the following power-law setting in PLRF
\vspace*{-4pt}
\[
\lambda_j=j^{-2\alpha},
\qquad \theta_j^\star=j^{-\beta},
\qquad 2\alpha+2\beta>1 ~(\mbox{ensuring} \sum_j\lambda_j|\theta_j^\star|^2<\infty)\,,
\]
\vspace*{-4pt}
the moving cutoff
\(\lambda_{j_T}\asymp T^{-1}\) corresponds to
\(j_T\asymp T^{1/(2\alpha)}\).  Summing the unresolved tail gives
\[
F(T)\!\approx \!\! \sum_{j\gtrsim j_T}\!j^{-2(\alpha+\beta)}\!\!\propto\! T^{-q_{\mathcal F}}\!,\,\,\, {\color{red}q_{\mathcal F}\!:=\!\frac{2(\alpha\!+\!\beta)\!-\!1}{2\alpha}},\quad
\frac{B}{\eta^2}K(T)\!\approx\!\!\! \sum_{j\gtrsim j_T}\!j^{-4\alpha}\!\!\propto\! T^{-q_{\mathcal K}},\,\,\, {\color{blue}q_{\mathcal K}\!:=\!2\!-\!\frac{1}{2\alpha}}\,.
\]
This cutoff calculation gives the PLRF response exponents;
\cref{cor:noisy_43_direct_plrf_43_exponents} gives the corresponding rigorous
population-filter asymptotics.

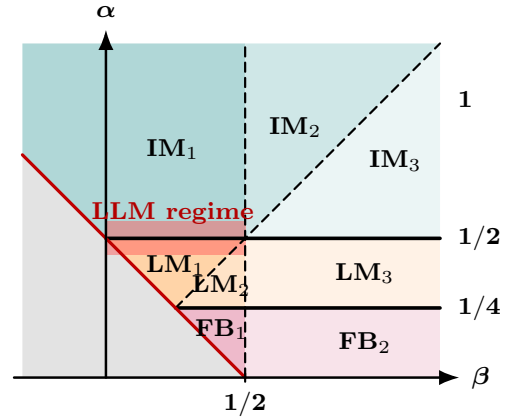
\begin{wrapfigure}{r}{0.4\textwidth}
\vspace{-1.0\baselineskip}
\setlength{\abovecaptionskip}{2pt}
\centering
\resizebox{\linewidth}{!}{%
\begin{tikzpicture}[
  x=3.2cm,y=3.2cm,font=\scriptsize,
  line cap=round,line join=round,>=Latex,
  subregime label/.style={font=\scriptsize\bfseries,
    inner xsep=1.4pt,inner ysep=.55pt},
  subregime label tight/.style={subregime label,
    inner xsep=.55pt,inner ysep=.20pt}
]
  % Excluded non-finite-energy region.
  \fill[gray!22] (-.3,0)--(.5,0)--(-.3,.8)--cycle;

  % One hue family per propagation regime; shade separates subregimes.
  \fill[purple!27] (.5,0)--(.25,.25)--(.5,.25)--cycle;
  \fill[purple!11] (.5,0) rectangle (1.2,.25);
  \fill[orange!34] (.25,.25)--(0,.5)--(.5,.5)--cycle;
  \fill[orange!22] (.25,.25)--(.5,.5)--(.5,.25)--cycle;
  \fill[orange!10] (.5,.25) rectangle (1.2,.5);
  \fill[teal!31] (0,.5)--(-.3,.8)--(-.3,1.2)--(.5,1.2)--(.5,.5)--cycle;
  \fill[teal!19] (.5,.5)--(1.2,1.2)--(.5,1.2)--cycle;
  \fill[teal!8] (.5,.5)--(1.2,.5)--(1.2,1.2)--cycle;

  % Empirical LLM response regime, matching the band in Figure 1.
  \fill[red!65,opacity=.48] (0.0,.44) rectangle (0.5,.56);

  % Structural boundaries and the finite-energy boundary.
  \draw[black,line width=1.18pt] (.25,.25)--(1.2,.25);
  \draw[black,line width=1.18pt] (0,.5)--(1.2,.5);
  \draw[black,line width=.72pt,densely dashed] (.5,0)--(.5,1.2);
  \draw[black,line width=.72pt,densely dashed] (.25,.25)--(1.2,1.2);
  \draw[red!75!black,line width=1.05pt] (.5,0)--(-.3,.8);

  % Coordinate axes; the alpha=1 optimizer guide remains omitted.
  \draw[->,line width=.82pt] (-.33,0)--(1.27,0)
    node[right,font=\scriptsize\bfseries] {$\bm\beta$};
  \draw[->,line width=.82pt] (0,0)--(0,1.25)
    node[above,font=\scriptsize\bfseries] {$\bm\alpha$};
  % Alpha coordinates are aligned along the right edge of the regime map.
  \node[right=2pt,font=\scriptsize\bfseries] at (1.2,.25)
    {$\mathbf{1/4}$};
  \node[right=2pt,font=\scriptsize\bfseries] at (1.2,.5)
    {$\mathbf{1/2}$};
  \node[right=2pt,font=\scriptsize\bfseries] at (1.2,1)
    {$\mathbf 1$};
  \draw[line width=.7pt] (.5,-.015)--(.5,.015)
    node[below=2pt,font=\scriptsize\bfseries] {$\mathbf{1/2}$};

  % Region and boundary labels.
  \node[subregime label tight] at (.415,.173) {$\bm{\mathrm{FB}}_1$};
  \node[subregime label] at (.93,.125) {$\bm{\mathrm{FB}}_2$};
  \node[subregime label] at (.25,.405) {$\bm{\mathrm{LM}}_1$};
  \node[subregime label tight] at (.415,.33) {$\bm{\mathrm{LM}}_2$};
  \node[subregime label] at (.93,.375) {$\bm{\mathrm{LM}}_3$};
  \node[subregime label] at (.24,.82) {$\bm{\mathrm{IM}}_1$};
  \node[subregime label] at (.68,.9) {$\bm{\mathrm{IM}}_2$};
  \node[subregime label] at (1.04,.78) {$\bm{\mathrm{IM}}_3$};
  \node[font=\scriptsize\bfseries,text=red!70!black]
    at (.24,.59) {LLM regime};
\end{tikzpicture}%
}
\caption{\textbf{PLRF regimes in \((\alpha,\beta)\).}
Colors match the response-coordinate map in
\cref{fig:llm_response_summary}; the red band marks the fitted LLM regime near
\(\alpha=1/2\).}
\label{iclr:fig:plrf_propagation_regimes}
\vspace{-1\baselineskip}
\end{wrapfigure}

\noindent {\bf Why use \((q_{\mathcal F},q_{\mathcal K})\) instead of
\((\alpha,\beta)\):} For PLRF, \(q_{\mathcal F}\) and
\(q_{\mathcal K}\) are algebraic functions of the microscopic source and
capacity parameters \((\alpha,\beta)\), i.e., a change of coordinates. The
purpose of the response coordinates is not to create new regimes by
relabeling the PLRF phase diagram.  Rather, \(q_{\mathcal F}\) and
\(q_{\mathcal K}\) are respectively the decay exponents of forcing and
one-injection memory, and are the only microscopic information entering the
subsequent schedule-transfer law.  They remain well defined for irregular
spectra and targets for which no coordinatewise \((\alpha,\beta)\) power laws
exist, and models with the same response exponents have the same leading
schedule behavior.  We therefore use \((\alpha,\beta)\) to describe the PLRF
origin of a regime and \((q_{\mathcal F},q_{\mathcal K})\) to describe its
dynamical propagation.  The finite-bulk regimes are kept separate because
their memory remains explicitly coupled to width and cannot be represented by
a width-independent \(q_{\mathcal K}>0\), shown in \cref{iclr:fig:plrf_propagation_regimes}.

\textbf{Long memory (LM)} has \(0<q_{\mathcal K}<1\), so its survival kernel is
not integrable; \textbf{integrable memory (IM)} has \(q_{\mathcal K}>1\).
Within the LM and IM families, the subscripts record the position of
\(q_{\mathcal F}\) relative to \(q_{\mathcal K}\) and \(1\).

For \(0<\alpha<1/4\), the squared spectral mass
\(\sum_{j\leq m}\lambda_j^2\asymp m^{1-4\alpha}\) grows with width, so no
width-independent decaying exponent \(q_{\mathcal K}>0\) exists.  This is the
\textbf{finite-bulk (FB)} regime.
The right panel of \cref{fig:llm_response_summary} organizes the propagation
regimes in the response coordinates \((q_{\mathcal F},q_{\mathcal K})\) that
govern schedule behavior.  \Cref{iclr:fig:plrf_propagation_regimes} pulls the
same LM/IM partition back to the PLRF parameters \((\alpha,\beta)\), with
finite bulk shown separately because no width-independent
\(q_{\mathcal K}\) exists there, making direct comparison with the \(4+3\)
taxonomy possible.

Relative to the compute-optimal \(4+3\) taxonomy of
\citet{paquette20244+}, the propagation viewpoint makes two structural
changes: Phase Ib splits across the finite-bulk and forcing--memory
boundaries, while that taxonomy's IVa--IVb crossover
\(\alpha=1-1/\sqrt2\) lies inside \(\mathrm{LM}_3\), rather than on a
propagation boundary.  We use the propagation taxonomy throughout.

The fitted \(q_{\mathcal K}\approx1\) corresponds to 
\(\alpha\approx1/2\), where the \(4+3\) theory predicts the \textbf{near-square-root}
compute-optimal width scaling reported by Chinchilla
\citep{paquette20244+,hoffmann2022training}; see
\cref{app:end_to_end_lm_experiments} for empirical validation.

\noindent{\bf Phasewise forcing.} The phase map becomes concrete through the deterministic-equivalent forcing. In the proved open source range
\(\beta<1+2\alpha\), and on the window
\(1\ll T\leq C_0m^{2\alpha}\),
\vspace*{-4pt}
\begin{equation}
\mathcal F(T,m)\asymp
\begin{cases}
m^{-2\alpha q_{\mathcal F}}+T^{-q_{\mathcal F}},
&\mathrm{FB}_1,\mathrm{LM}_1,\mathrm{LM}_2,\mathrm{IM}_1,\\
m^{-2\alpha}+T^{-q_{\mathcal F}},
&\mathrm{FB}_2,\mathrm{LM}_3,\\
m^{-2\alpha}+T^{-q_{\mathcal F}}
+m^{-1}T^{-(q_{\mathcal K}-1)},
&\mathrm{IM}_2,\mathrm{IM}_3.
\end{cases}
\label{iclr:eq:phase_forcing_profile}
\end{equation}

Here \(\mathcal F(T,m)\) denotes the deterministic-equivalent approximation
to \(F_{\bm W}(t)\) at intrinsic time \(T=\eta t\). Every row contains a finite-width floor and a target-aligned source transient;
only \(\mathrm{IM}_2\) and \(\mathrm{IM}_3\) contain the additional
feature-distortion transient.  The largest term determines the forcing on the
stated width--time window.

\cref{iclr:eq:phase_forcing_profile} describes the forcing
component, not the observed total risk.  This is why the \(3+3(+2)\) map
differs from the \(4+3\) compute-optimal taxonomy of
\citet{paquette20244+}: it classifies how forcing and memory propagate before
resource optimization.  \Cref{sec:rv_joint_schedules} supplies the missing
step by accumulating one-injection memory along the schedule and comparing
the resulting noise response with clean learning.

\vspace*{-4pt}
\section{How joint schedules transform the observed scaling laws}
\label{sec:rv_joint_schedules}
\vspace*{-4pt}

For schedules, we are inspired by \citet{li2025functional} that expresses schedule dependence in intrinsic time through a forgetting-kernel
convolution via stochastic derivative questions.
In our setting, the exponent \(q_{\mathcal K}\) characterizes how the contribution of one stochastic-error injection decays with its age.  A training schedule creates a stream of such injections.  At intrinsic time \(u\), the ratio
\(r(u)=B(u)/\eta(u)\) sets how much noise is injected, while the memory
profile \(k(T-u)\) sets how much survives until time \(T\).  The observed
loss is determined by their accumulation and its competition with clean
learning.

\vspace*{-4pt}
\subsection{Schedule response and preserve--change--destroy classification}
\vspace*{-4pt}

Recall \(T_t=\sum_{s<t}\eta_s\), and interpolate
\(r(u)=r_s=B_s/\eta_s\) for \(T_s\leq u<T_{s+1}\).
For a constant schedule, define the reference survival profile
\(k(\eta t):=(B/\eta^2)K_{\bm W}(t)\).
Under a varying schedule, the constant-schedule age-only kernel
\(K_{\bm W}(t-s)\) becomes the two-time kernel \(K_{t,s}\), the exact weight
with which stochastic error injected at step \(s\) survives until step \(t\).
Subtracting the clean recursion from the noisy one gives the first relation
below.  For a fixed infinite-spectrum dynamics with \(\sigma^2>0\), we assume
the continuum approximation and uniformly subcritical feedback in the
remaining two:
\vspace*{-4pt}
\[
R_{\sigma,t}-R_{0,t}
\!=\!\sum_{s<t}\!K_{t,s}\bigl[\sigma^2+R_{\sigma,s}-R_{0,s}\bigr],
\qquad
\sum_{s<t}\!K_{t,s}\!\asymp\!
\int_0^{T_t}\!\frac{k(T_t-u)}{r(u)}\,\mathrm du,
\qquad
\sup_t\!\sum_{s<t}\!K_{t,s}<1\,.
\]
\vspace*{-4pt}
Along the intrinsic-time grid, write \(T=T_t\),
\(R_\sigma(T):=R_{\sigma,t}\), and \(R_0(T):=R_{0,t}\).
The following theorem converts power-law \(k\) and \(r\) into the noise exponent
\(q_{\mathcal N}\), whose comparison with the clean exponent \(q_0\) determines
whether the clean loss law is preserved, changed, or destroyed.

\begin{theorem}[Schedule transformation of loss]
\label{thm:rv_joint_schedule}
\label{thm:rv_preserve_change_destroy}
Suppose \(k(T)\sim c_{\mathcal K}T^{-q_{\mathcal K}}\), with
\(c_{\mathcal K}>0\) and \(q_{\mathcal K}\neq1\), and let
\(r(T)\sim c_rT^\vartheta\) be eventually monotone, with \(c_r>0\).  On every
decaying branch,
\[
R_\sigma(T)-R_0(T)\asymp T^{-q_{\mathcal N}(\vartheta)},
\qquad
q_{\mathcal N}(\vartheta):=
\begin{cases}
\vartheta+q_{\mathcal K}-1,
&\mathrm{LM},\quad 1-q_{\mathcal K}<\vartheta<1,\\
q_{\mathcal K},
&\mathrm{LM},\quad \vartheta>1,\\
\min\{\vartheta,q_{\mathcal K}\},
&\mathrm{IM},\quad \vartheta>0.
\end{cases}
\]
At \(\vartheta=1\) in LM, the gap is
\(\asymp T^{-q_{\mathcal K}}\log T\).  If, in addition,
\(R_0(T)-R_{\mathrm{app}}\sim c_0T^{-q_0}\), then
\[
R_\sigma(T)-R_{\mathrm{app}}
\asymp T^{-\min\{q_0,q_{\mathcal N}(\vartheta)\}}\,, \quad q_0 > 0\,,
\]
up to the displayed LM logarithmic correction.
\end{theorem}

The theorem yields a phase classification by comparing the noise exponent \(q_{\mathcal N}(\vartheta)\) with the clean exponent \(q_0\).
\Cref{fig:rv_schedule_effects} summarizes this comparison. The case \(q_{\mathcal K}=1\) is excluded because it is the marginal LM/IM
boundary: cumulative memory grows as
\(\int_0^T k(v)\,\mathrm dv\asymp\log T\), so the gap acquires logarithmic
corrections rather than following the pure-power formulas above.

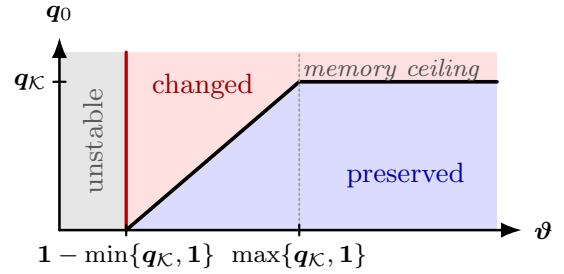
\begin{wrapfigure}{r}{0.45\textwidth}
\setlength{\abovecaptionskip}{2pt}
\vspace{-1.0\baselineskip}
\centering
\resizebox{\linewidth}{!}{%
\begin{tikzpicture}[
  x=2.8cm,y=2.4cm,font=\scriptsize,
  line cap=round,line join=round,>=Latex,
  response axis/.style={draw=black,line width=.82pt,-{Latex}},
  response boundary/.style={draw=black,line width=1.18pt},
  destroy boundary/.style={draw=red!70!black,line width=1.05pt},
  ceiling guide/.style={draw=gray!72,line width=.60pt,densely dotted},
  coordinate label/.style={font=\scriptsize\bfseries},
  boundary label/.style={font=\scriptsize\itshape,inner sep=0pt},
  ceiling label/.style={boundary label,text=gray!62!black}
]
  \fill[gray!20] (0.08,0.08) rectangle (0.35,0.92);
  \fill[red!12]
    (0.35,0.08)--(0.35,0.92)--(1.85,0.92)--(1.85,0.78)--
    (1.05,0.78)--cycle;
  \fill[blue!14]
    (0.35,0.08)--(1.05,0.78)--(1.85,0.78)--(1.85,0.08)--cycle;
  \draw[response axis] (0.08,0.08)--(1.95,0.08)
    node[right,coordinate label] {$\bm\vartheta$};
  \draw[response axis] (0.08,0.08)--(0.08,0.99)
    node[above=1pt,coordinate label] {$\bm q_0$};
  \draw[destroy boundary] (0.35,0.08)--(0.35,0.92);
  \draw[response boundary] (0.35,0.08)--(1.05,0.78)--(1.85,0.78);
  \draw[ceiling guide] (1.05,0.08)--(1.05,0.92);
  \draw[line width=.70pt] (0.35,0.055)--(0.35,0.105)
    node[below=2pt,coordinate label]
      {$\mathbf 1-\min\{\bm q_{\mathcal K},\mathbf 1\}$};
  \draw[line width=.70pt] (1.05,0.055)--(1.05,0.105)
    node[below=2pt,coordinate label]
      {$\max\{\bm q_{\mathcal K},\mathbf 1\}$};
  \draw[line width=.70pt] (0.055,0.78)--(0.105,0.78)
    node[left=2pt,coordinate label] {$\bm q_{\mathcal K}$};
  \node[rotate=90,gray!70!black,font=\footnotesize] at (0.215,0.50) {unstable};
  \node[red!60!black,font=\footnotesize] at (0.66,0.76) {changed};
  \node[blue!65!black,font=\footnotesize] at (1.48,0.34) {preserved};
  \node[ceiling label] at (1.42,0.835) {memory ceiling};
\end{tikzpicture}%
}
\caption{\textbf{LM/IM schedule response.}  The red line marks the destroy boundary, where the noisy--clean gap has no power decay; the gray region to its left is unstable.  The black
curve separates changed from preserved scaling and plateaus at the memory
ceiling.}
\label{fig:rv_schedule_effects}
\vspace{-1.0\baselineskip}
\end{wrapfigure}

Consequently, the leading clean-loss decay is \emph{preserved}, \emph{critical}, or
\emph{changed} according as \(q_{\mathcal N}(\vartheta)\) is above, equal to,
or below \(q_0\).  At \(\vartheta=1-q_{\mathcal K}\) in LM and
\(\vartheta=0\) in IM, the gap has no positive-power decay and therefore
\emph{destroys} convergence to \(R_{\mathrm{app}}\); below these edges,
uniform row stability is impossible.  Since
\(q_{\mathcal N}(\vartheta)\leq q_{\mathcal K}\), strict preservation is
possible only when \(q_0<q_{\mathcal K}\).  The complete boundary asymptotics and
tunable branches are given in
\cref{app:rv_joint_schedule_proof}.

The LLM experiment in \cref{subsec:llm_power_law_ratio_sweep} directly tests
the schedule family in \cref{thm:rv_joint_schedule}: it varies \(\vartheta\) in
\(r(T)\propto T^\vartheta\) at matched terminal intrinsic time, provides evidence for the predicted memory ceiling through the diminishing response at large \(\vartheta\).

\noindent{\bf What the loss can and cannot identify.}
Below the LM memory ceiling, a noisy--clean gap asymptotic
\(R_\sigma(T)-R_0(T)\sim cT^{-q}\), \(0<q<q_{\mathcal K}\), identifies
the ratio exponent \(\vartheta=1-q_{\mathcal K}+q\).  It identifies
\(r=B/\eta\), not learning rate and batch size separately; details are given in \cref{app:rv_joint_schedule_proof}.

\subsection{PLRF finite-bulk and full-risk asymptotics}

Here we start from the discrete-SGD memory kernel derived above, and then characterize when a joint ratio path \(r=B/\eta\) preserves, changes, or destroys a spectrally generated power law.

\Cref{thm:rv_joint_schedule} assumes a width-independent survival profile.
For \(0<\alpha<1/4\), an unresolved spectral band instead carries squared
mass \(m^{1-4\alpha}\).  On the finite-bulk window,
\begin{equation}
\label{eq:td_lrs_finite_bulk_gap_main}
R_{\sigma,t}-R_{0,t}
\asymp
\sigma^2m^{1-4\alpha}\sum_{s<t}\frac{\eta_s^2}{B_s}
=
\sigma^2m^{1-4\alpha}\int_0^{T_t}\frac{\mathrm du}{r(u)}.
\end{equation}
This is why FB is not a third branch of the LM/IM map: memory and
width remain coupled.

Let \(\mathcal F(T,m)\) be the forcing profile in
\cref{iclr:eq:phase_forcing_profile}, and write
\(\Delta T_s:=T_{s+1}-T_s\).  The PLRF deterministic-equivalent risk admits
the following forcing--memory representation:
\begin{equation}
\label{eq:td_lrs_plrf_full_risk}
\mathcal R_{\sigma,t}\asymp
\begin{cases}
\displaystyle
\mathcal F(T_t,m)
+\sum_{s<t}\frac{\Delta T_s}{r_s}
(1+T_t-T_{s+1})^{-q_{\mathcal K}}
[\mathcal F(T_s,m)+\sigma^2],
&\alpha>\frac14,\\[1ex]
\displaystyle
\mathcal F(T_t,m)
+\sigma^2m^{1-4\alpha}\sum_{s<t}\frac{\Delta T_s}{r_s},
&0<\alpha<\frac14.
\end{cases}
\end{equation}
The first term is the unresolved signal, while the second accumulates
surviving stochastic injections. In the LM/IM response,
\(\mathcal F(T_s,m)\) records mini-batch variance even with clean labels, whereas \(\sigma^2\) is the additional label-noise variance.  Precise source windows and empirical-to-DE transfer conditions are given in \cref{sec:joint_schedules}.

\noindent{\bf Forcing-memory surrogate for LLM.}
We use the \(\alpha>1/4\) branch because only in this regime does the
width-independent profile \(k(v)\asymp(1+v)^{-q_{\mathcal K}}\) exist, with
\(q_{\mathcal K}=2-1/(2\alpha)>0\). Followed by the FSL fit-then-transfer protocol \citet{li2025functional}, the \(\alpha>1/4\) branch of \cref{eq:td_lrs_plrf_full_risk} suggests a
7 parameter surrogate for LLM loss: 
\begin{equation}
\label{eq:llm_forcing_memory_surrogate}
\widehat L(T)
=L_\infty+A_{\mathcal F}(1+T)^{-q_{\mathcal F}}
+\int_0^T
\frac{A_0+A_1(1+u)^{-q_{\mathcal F}}}{r(u)}
\bigl(1+c_{\mathcal K}(T-u)\bigr)^{-q_{\mathcal K}}\,\mathrm du \,.
\end{equation}
Its seven parameters \((L_\infty,A_{\mathcal F},A_0,A_1,c_{\mathcal K},q_{\mathcal F},q_{\mathcal K})\) can be fitted to an observed loss curve.
\Cref{iclr:sec:LLM_exp,app:end_to_end_lm_experiments} fit these parameters on an 8-1-1 schedule and evaluate the resulting surrogate on a WSD schedule without refitting.

\noindent{\bf Relation to prior deterministic-equivalent and FSL.}
Starting from discrete SGD, we first derive an exact conditional Volterra recursion and then use the resolvent deterministic-equivalent approach of \citet{paquette20244+} to construct deterministic forcing and memory measures for general population spectra and targets. Their \(4+3\) analysis treats time-independent learning rate and batch size;
extending the same DE framework to time-varying joint schedules yields
\cref{eq:td_lrs_plrf_full_risk} after PLRF specialization. 

Our learning rate and batch size schedules in the discrete-SGD deterministic-equivalent route are inspired by FSL \citep{li2025functional,wang2026fast} which is built on a time-changed stochastic differential equations under an assumed population power-law spectrum. In their hard power-law regime, at the level of the response law, the correspondence is
\[
q_{\mathcal F}=s_{\mathrm{FSL}},\qquad
q_{\mathcal K}=2-\frac{1}{\beta_{\mathrm{FSL}}},\qquad
\frac{1}{r(T)}=\gamma_{\mathrm{FSL}}(T)\,.
\]
Thus, \cref{eq:llm_forcing_memory_surrogate} can be regarded as a joint-schedule parameterization of the FSL structure. 
Our additional results identify spectral conditions under which the forcing and memory terms exhibit power laws, locate these responses across the PLRF phases, and characterize how joint schedules preserve, change, or destroy the resulting learning curve.

\vspace*{-4pt}
\subsection{Optimal ratios and sharp resource rates}
\label{iclr:sec:plrf_de_design}
\vspace*{-4pt}

The forcing--memory representation also turns schedule design into
an optimization over the ratio path \(r(u)\).  In the LM/IM branches, for
fixed \(T\) and data budget \(D\), this means minimizing
\(\int_0^T(1+T-u)^{-q_{\mathcal K}}
[\mathcal F(u,m)+\sigma^2]/r(u)\,\mathrm du\), subject to
\(\int_0^T r(u)\,\mathrm du=D\).  Cauchy--Schwarz gives
\[
r^\star(u)
\propto
\sqrt{(1+T-u)^{-q_{\mathcal K}}[\mathcal F(u,m)+\sigma^2]},
\qquad
\int_0^T r^\star(u)\,\mathrm du=D.
\]
Hence more samples are allocated where an injection is both large and likely
to survive until time \(T\). This recovers the joint optimum of
\citet{bordelon2026theory} and its fixed-batch power-decay and WSD
factorizations \citep{bordelon2026theory,li2026optimal}.

\noindent{\bf A representative resource balance.}
In most LM/IM branches, the leading terms reduce to
\[
\mathcal R_\sigma(T)
\asymp T^{-p/(2\alpha)}+D^{-1}T^{1/(2\alpha)}, \quad \text{with} \, p=2\alpha+2\beta-1.
\]
Balancing signal and controlled noise gives
\[
T^\star\asymp D^{2\alpha/(1+p)}, \quad\mathcal R_\sigma^\star(D)\asymp D^{-p/(1+p)}.
\] 
Width floors, feature distortion, finite-bulk amplification, and the constraint
\(\mathfrak f\asymp mD\) produce the remaining phase-dependent compute rates.

\begin{corollary}[Phasewise optimal data and compute rates]
\label{cor:td_lrs_subregime_exponents}
Let \(p=2\alpha+2\beta-1>0\), fix \(\sigma^2>0\), and let
\(\mathcal R_\sigma^\star(D)\) and
\(\mathcal R_\sigma^\star(\mathfrak f)\) denote the optimal terminal
deterministic-equivalent risks under data and feature-compute budgets,
respectively. These optima have the orders in
\cref{tab:td_lrs_ratio_optimal_exponents}.
\end{corollary}

\begin{table}[t]
\centering
\caption{\textbf{Optimal risk rates across the PLRF propagation regimes.}
Each entry is an \(\asymp\)-order, with
\(p=2\alpha+2\beta-1\), data budget \(D\), and feature-compute budget
\(\mathfrak f\asymp mD\).}
\label{tab:td_lrs_ratio_optimal_exponents}
\small
\setlength{\tabcolsep}{4pt}
\renewcommand{\arraystretch}{1.12}
\begin{tabular}{@{}lll@{}}
\toprule
Regime and branch
& \(\mathcal R_\sigma^\star(D)\)
& \(\mathcal R_\sigma^\star(\mathfrak f)\)\\
\midrule
\(\mathrm{IM}_1,\ \beta<0\)
& \(D^{-p/(2\alpha)}\)
& \(\mathfrak f^{-p/(1+2\alpha)}\)\\
\(\mathrm{LM}_{1,2};\ \mathrm{IM}_1,\ \beta>0\)
& \(D^{-p/(1+p)}\)
& \(\mathfrak f^{-p/(2+p)}\)\\
\(\mathrm{LM}_3\)
& \(D^{-p/(1+p)}\)
& \(\mathfrak f^{-2\alpha p/[p+2\alpha(1+p)]}\)\\
\(\mathrm{IM}_{2,3},\ 1/2<\alpha\leq1\)
& \(D^{-p/(1+p)}\)
& \(\mathfrak f^{-p/(1+p+2\beta)}\)\\
\(\mathrm{IM}_{2,3},\ \alpha>1\)
& \(D^{-p/(1+p)}\)
& \(\mathfrak f^{-\alpha/(1+\alpha)}\)\\
\midrule
\(\mathrm{FB}_1\)
& \(D^{-p/(1+p)}\)
& \(\mathfrak f^{-p/(2+p)}\)\\
\(\mathrm{FB}_2\)
& \(D^{-2\alpha p/[p(1-2\alpha)+8\alpha^2]}\)
& \(\mathfrak f^{-2\alpha p/[p(2-2\alpha)+8\alpha^2]}\)\\
\bottomrule
\end{tabular}
\par\vspace{3pt}
\parbox{0.94\linewidth}{%
\footnotesize
\emph{Budget definitions.}
\(\mathcal R_\sigma^\star(D)\) is the optimal terminal risk subject to
\(\sum_{s<t}B_s\leq D\).
Similarly, \(\mathcal R_\sigma^\star(\mathfrak f)\) is the optimum subject to
\(m\sum_{s<t}B_s\leq\mathfrak f\), where
\(\mathfrak f\asymp mD\) is the feature-compute proxy.
}
\end{table}

Most LM/IM regimes share the data-optimal rate
\(D^{-p/(1+p)}\), but their compute-optimal rates differ because width,
feature distortion, and memory impose different bottlenecks.  The
\(\mathrm{IM}_1\), \(\beta<0\), and \(\mathrm{FB}_2\) branches are the two
exceptions already at fixed data.  The complete schedule constructions,
matching lower bounds, boundary qualifications, and the conditional
high-source \(\mathrm{FB}_2\) extension are given in
\cref{subsec:td_lrs_ratio_control_integration}.

\subsection{LLM experiments}
\label{iclr:sec:LLM_exp}

\begin{figure}[!t]
  \setlength{\abovecaptionskip}{2pt}
  \captionsetup[subfigure]{skip=0pt}
  \centering
  \begin{subfigure}[t]{0.3\linewidth}
    \centering
    \includegraphics[width=\linewidth]
      {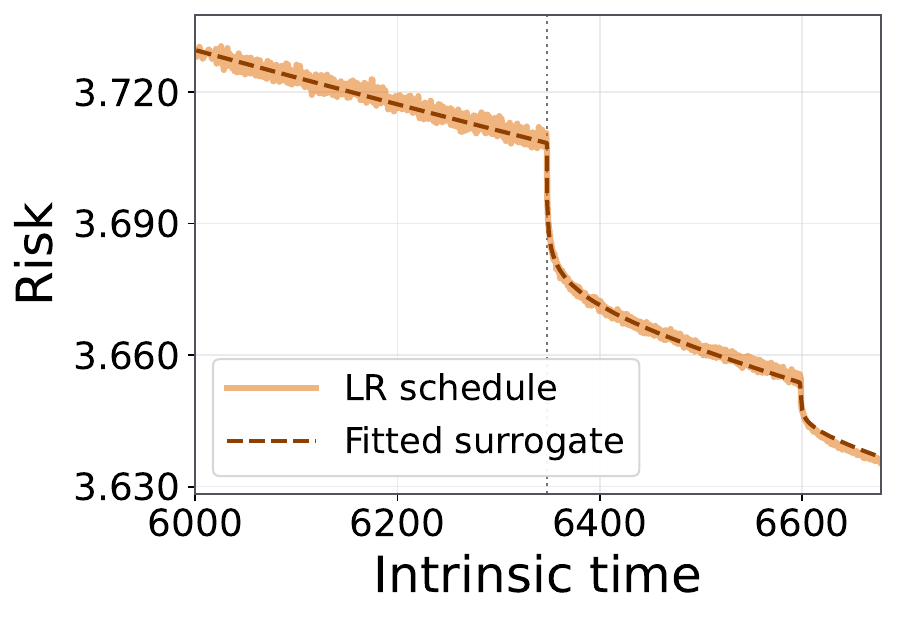}
    \caption{8-1-1 fit}
    \label{fig:llm_surrogate_811_fit}
  \end{subfigure}\hfill
  \begin{subfigure}[t]{0.3\linewidth}
    \centering
    \includegraphics[width=\linewidth]
      {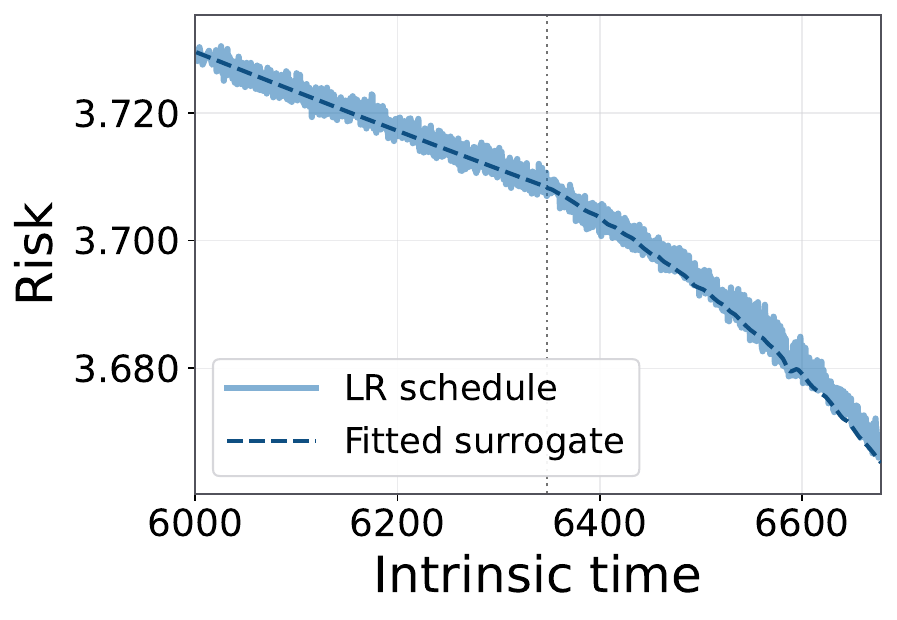}
    \caption{WSD zero-refit transfer}
    \label{fig:llm_surrogate_wsd_transfer}
  \end{subfigure}\hfill
  \begin{subfigure}[t]{0.3\linewidth}
    \centering
    \includegraphics[width=\linewidth]
      {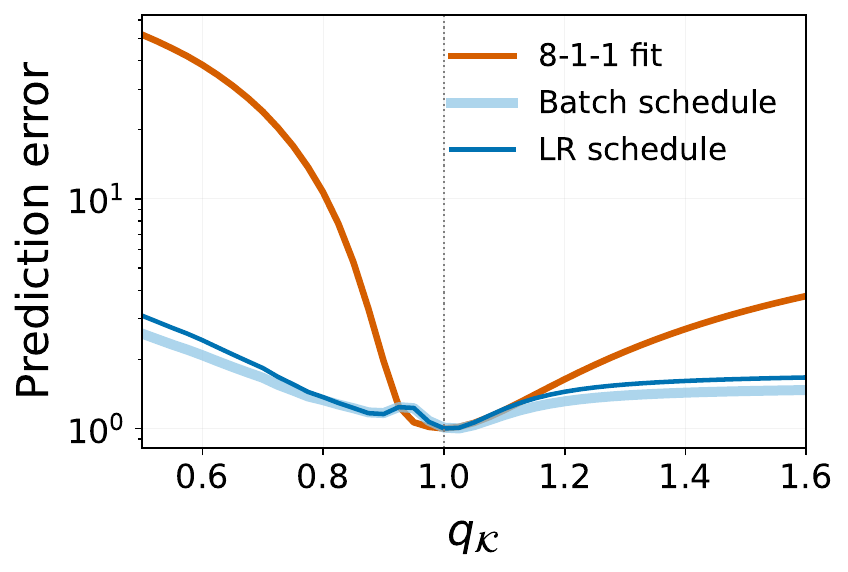}
    \caption{Prediction error}
    \label{fig:llm_qk_prediction_error}
  \end{subfigure}

  \caption{\textbf{LLM surrogate fit, zero-refit transfer, and effective \(q_{\mathcal K}\).}
  (a) For a \(300\)M SGD LLM, we fix \(q_{\mathcal K}=1\) in \cref{eq:llm_forcing_memory_surrogate} and fit the remaining 6 parameters to the 8-1-1 learning-rate schedule.
  (b) With all fitted parameters frozen, it predicts the WSD schedule without refitting.
  (c) Fixing \(q_{\mathcal K}\), fitting the remaining parameters only on
  8-1-1, and evaluating the WSD schedules without
  refitting yields prediction errors, each divided by its best
  value, that are minimized near \(q_{\mathcal K}\approx1\).}
  \label{fig:llm_surrogate_fit_transfer_profile}
\end{figure}

\Cref{fig:llm_response_summary} previews the LLM results.  We fit the 7 parameters of the surrogate
in \cref{eq:llm_forcing_memory_surrogate} only to the 8-1-1 trajectory and use
the frozen fit to predict WSD without refitting. \Cref{fig:llm_surrogate_811_fit,fig:llm_surrogate_wsd_transfer} show the fit to
8-1-1 and the zero-refit prediction of WSD, respectively.  The prediction errors in
\cref{fig:llm_qk_prediction_error} are minimized near
\(q_{\mathcal K}\approx 1\), placing the effective LLM response near the LM/IM
boundary previewed in \cref{fig:llm_response_summary_phase}.  Independent fits on OpenWebText \mbox{\citep{Gokaslan2019OpenWeb}}, FineWeb \mbox{\citep{penedo2024fineweb}}, and peS2o V2
\mbox{\citep{soldaini2023pes2o}} yield \(q_{\mathcal K}=1.017\), \(0.952\), and
\(0.995\), respectively.  \Cref{fig:nanogpt124m_cross_dataset_surrogate_transfer}
shows the independent FineWeb and peS2o V2 fits and their zero-refit schedule transfer.  Together with the OpenWebText result, their agreement near one supports a dataset-robust effective response coordinate.  Full experimental details are provided in
\cref{app:end_to_end_lm_experiments}.

\begin{figure}[!t]
  \centering
  \includegraphics[width=\linewidth]{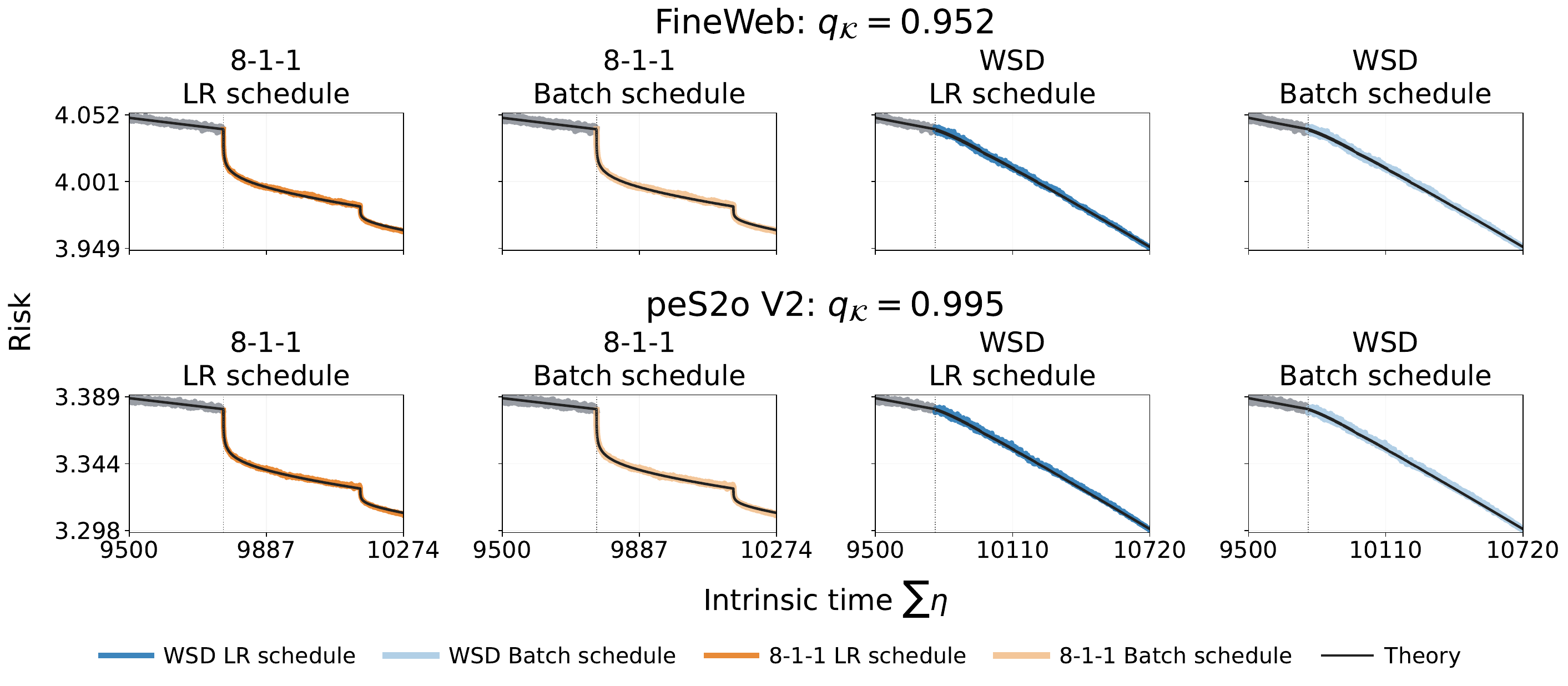}
  \caption{\textbf{Different datasets independently recover
  \(q_{\mathcal K}\approx1\).}
  Each surrogate is fitted only to fixed-batch 8-1-1 and transferred without
  refitting to the alternative factorization and both WSD trajectories.  The
  fitted \(q_{\mathcal K}\) are \(0.952\) on FineWeb and \(0.995\) on peS2o
  V2; together with the independently fitted OpenWebText value \(1.017\).}
  \label{fig:nanogpt124m_cross_dataset_surrogate_transfer}
  \vspace*{-4pt}
\end{figure}

\vspace*{-4pt}
\section{Conclusion and discussion}
\vspace*{-4pt}

Our results characterize a power-law learning curve as a dynamical response
jointly shaped by the spectrum, target alignment, noise, and schedule.
Target-weighted and squared-spectrum masses determine forcing and memory,
respectively, while \(B_t/\eta_t\) controls noise injection in intrinsic time.
We characterize when schedules preserve, change, or destroy clean power-law
decay. In the canonical PLRF model, these mechanisms yield the \(3+3(+2)\) map, an exact conditional finite-bulk gap law, and phase-dependent
deterministic-equivalent risk and resource-optimal rates.  Gaussian online-SGD
experiments support the exact conditional and DE predictions. LLM experiments show our proxy surrogate yields cross-schedule prediction without refitting, supporting the response-level
mechanism beyond frozen features. The factorization-collapse experiment
further distinguishes the optimizer coordinates: plain SGD trajectories nearly
collapse under \(B_t/\eta_t\), whereas Muon \citep{jordan2024muon,liu2025muon} trajectories nearly collapse under the coordinate \(B_t/\eta_t^2\). The fitted \(q_{\mathcal K}\approx1\) corresponds to an effective \(\alpha\approx1/2\), where the \(4+3\) theory predicts the near-square-root compute-optimal width scaling reported by Chinchilla \citep{paquette20244+,hoffmann2022training}.

\section*{Acknowledgments}
We thank Denny Wu, Lei Wu and Bruno Loureiro for their constructive discussions and suggestions.
Yudong Chen acknowledges support from a Vilas Associates Award.

\clearpage

\bibliography{references}
\bibliographystyle{ims}

\clearpage
\appendix
\crefalias{section}{appendix}
\enableaddcontentsline
\setcounter{tocdepth}{2}
\tableofcontents
\clearpage
\section*{Organization and scope of the appendix results}

The appendices are organized around the questions needed to verify the main
text.  \Cref{sec:noisy_label_scaling_laws} records the spectral
representations and the complete formal criteria for power-law learning
curves, while \cref{app:spectral_renewal_proofs} proves the general criteria
by translating cumulative slow-mode weight into long-time decay and then accounting for repeated noise feedback.  Width-dependent finite-bulk results are kept with the exact finite-width and deterministic-equivalent risk results in \cref{app:noisy_43_finite_bulk_empirical,subsec:app_td_lrs_joint_schedules}.
\Cref{sec:joint_schedules,app:rv_joint_schedule_proof} turn to schedules in
which learning rate and batch size vary together.  They derive the resulting
two-time risk equation, show how input
power laws pass to the noisy--clean
gap, and prove when that gap determines the schedule ratio.  The formal
joint-schedule statements and their propagation, finite-bulk, and optimization
proofs are now collected together in \cref{sec:joint_schedules}.
\Cref{sec:constant_schedule_noisy_43} places each constant-schedule result
beside its calculation and optimality proof.
Within \cref{app:condition_examples}, a stretched-exponential example shows
that temporal power laws need not come from a pointwise power-law spectrum.
Numerical checks at finite width appear beside the results they are
designed to illustrate, while
\cref{app:end_to_end_lm_experiments} collects end-to-end language-model
external-validity checks separately from the theorem-facing evidence.

\paragraph{Limit regimes: what grows and in what order.}
The results below use several asymptotic regimes because they answer different
questions rather than provide interchangeable approximations to the same
limit.  Exact finite-width dynamics describe a realized feature map over a
finite training horizon, while the fixed-width terminal limit describes what
remains after all positive modes have relaxed.  A fixed-time deterministic
equivalent replaces the random empirical spectrum at a prescribed horizon.  A
fixed-limit long-time law then studies an infinite-spectrum response as
\(T\to\infty\), whereas a joint width--time limit asks whether that response
remains valid along finite-width models with \(T=T_m\to\infty\).  Resource
limits additionally allow width, horizon, and schedule to vary with the
available data or feature compute.  \Cref{tab:limit_map} summarizes the limit
hierarchy used by the statements and proofs that follow.

\begin{table}[H]
\centering
\caption{Limit regimes of the principal results.}
\label{tab:limit_map}
\small
\setlength{\tabcolsep}{4pt}
\renewcommand{\arraystretch}{1.12}
\begin{tabular}{@{}
>{\raggedright\arraybackslash}p{0.35\textwidth}
>{\raggedright\arraybackslash}p{0.59\textwidth}@{}}
\toprule
\textbf{Limit regime}
& \textbf{Results in this paper}\\
\midrule

Exact finite width:\newline
\(d,m,\bm W\) fixed, \(t<\infty\)
& Exact finite-bulk kernel and noisy--clean gap
(\cref{lem:td_lrs_fb_bulk_kernel,%
cor:td_lrs_exact_finite_bulk_gap})\\
\cmidrule(lr){1-2}

Fixed-width terminal:\newline
\(d,m,\bm W\) fixed, \(T\to\infty\)
& Stable terminal-risk transfer and the noise-induced plateau
(\cref{thm:noisy_43_aggregate_volterra_transfer,%
prop:noisy_43_discrete_trace_optimum,%
prop:noisy_43_noise_plateau})\\
\cmidrule(lr){1-2}

Fixed-time DE:\newline
\(d,m\to\infty\), \(d/m\to c\), \(T\) fixed
& Positive spectral representation and DE risk recursion for varying schedules
(\cref{lem:noisy_43_de_positive_spectral_measure,%
prop:td_lrs_joint_de_recursion,%
lem:td_lrs_propagator})\\
\cmidrule(lr){1-2}

Fixed infinite-spectrum system: \(T\to\infty\)
& Schedule response, full classification, and gap-to-ratio identification
(\cref{thm:rv_joint_schedule,%
thm:noisy_43_aggregate_noise_criterion,%
thm:rv_discrete_transfer,%
thm:rv_discrete_integrable_transfer,%
thm:rv_exact_discrete_identification,%
cor:rv_physical_regime_map})\\
\cmidrule(lr){1-2}

Joint width--time:\newline
\(m,T_m\to\infty\), \(1\ll T_m\lesssim m^{2\alpha}\)
& Spectral iff criteria, uniform transfer, and LM/IM/FB risk laws
(\cref{iclr:thm:spectral_iff,%
cor:noisy_43_direct_plrf_43_exponents,%
thm:noisy_43_uniform_tauberian,%
thm:rv_spectral_memory_transfer,%
thm:td_lrs_joint_fsl,%
cor:td_lrs_joint_de_explicit_forcing,%
iclr:prop:finite_bulk_de,%
prop:td_lrs_joint_empirical_de_transfer})\\
\cmidrule(lr){1-2}

Resource limit:\newline
\(D\to\infty\) or \(\mathfrak f\to\infty\); \(m,T,r\) vary
& Optimal ratio control and phasewise data/compute rates
(\cref{thm:td_lrs_ratio_control,%
cor:td_lrs_subregime_exponents,%
lem:td_lrs_integer_reachability})\\

\bottomrule
\end{tabular}
\end{table}
In the canonical PLRF model, the smallest positive spectral scale is
\(m^{-2\alpha}\).  At fixed width, once \(T\gg m^{2\alpha}\), all positive
modes have been resolved and the remaining transient decays exponentially:
\[
m\ \text{fixed},\qquad T\gg m^{2\alpha}
\quad\Longrightarrow\quad
\text{exponential transient}.
\]
A temporal power law instead requires a joint window in which the filter scale
\(T^{-1}\) remains above the finite-width cutoff:
\[
m\to\infty,\qquad T=T_m\to\infty,\qquad
1\ll T_m\ll m^{2\alpha}.
\]
Thus the joint width--time limit defines the phenomenon studied here rather
than serving only as a technical device.  The order of limits also matters.
Writing \(R(T,m)\) schematically for the risk at width \(m\) and intrinsic time
\(T\), the two orders can differ for some spectral families:
\[
\lim_{m\to\infty}\lim_{T\to\infty}R(T,m)
\neq
\lim_{T\to\infty}\lim_{m\to\infty}R(T,m).
\]
The left-hand side first trains each finite-width model to its terminal
behavior.  The right-hand side first produces an infinite spectrum, so
progressively slower modes remain unresolved as \(T\) grows.

\paragraph{Analytical levels: what is exact and what is approximated.}
The limit hierarchy above specifies which variables grow and in what order.
A separate distinction concerns what is proved at each level.  Conditional on
the frozen features \(\bm W\), the finite-width risk dynamics are exact and
require no spectral approximation.  A deterministic equivalent then replaces
the random empirical spectrum by deterministic spectral measures as width and
dimension grow.  This approximation is initially a fixed-horizon statement:
validity for every fixed \(T\) does not by itself imply validity along
\(T=T_m\to\infty\).  Power-law asymptotics make this additional long-time
passage, either directly for empirical spectra or through the DE.  They
therefore require uniform control near the moving spectral scale \(T^{-1}\),
expressed through the corresponding spectral-control, no-escape, and
source-window conditions.  \Cref{tab:analytical_levels} records the formal
results at each level.

\begin{table}[!t]
\centering
\caption{Analytical levels and representative formal results.}
\label{tab:analytical_levels}
\small
\setlength{\tabcolsep}{4pt}
\renewcommand{\arraystretch}{1.12}
\begin{tabular}{@{}
>{\raggedright\arraybackslash}p{0.30\textwidth}
>{\raggedright\arraybackslash}p{0.64\textwidth}@{}}
\toprule
\textbf{Analytical level}
& \textbf{Formal results}\\
\midrule
Exact conditional finite-width dynamics
& Exact finite-bulk kernel and noisy--clean gap
(\cref{lem:td_lrs_fb_bulk_kernel,%
cor:td_lrs_exact_finite_bulk_gap}).\\
\cmidrule(lr){1-2}
Deterministic equivalent
& Positive DE spectral representation, varying-schedule DE recursion, and
conditional transfer to realized SGD
(\cref{lem:noisy_43_de_positive_spectral_measure,%
prop:td_lrs_joint_de_recursion,%
prop:td_lrs_joint_empirical_de_transfer}).\\
\cmidrule(lr){1-2}
Power-law asymptotics
& Spectral iff criteria, uniform Tauberian and schedule transfer, and LM/IM/FB
risk laws
(\cref{thm:noisy_43_effective_spectral_criterion,%
thm:noisy_43_uniform_tauberian,%
thm:rv_spectral_memory_transfer,%
thm:td_lrs_joint_fsl,%
iclr:prop:finite_bulk_de}).\\
\bottomrule
\end{tabular}
\end{table}

\clearpage

\section*{Notation guide}
\makeatletter
\addtocontents{toc}{\protect\contentsline{section}{Notation guide}{\thepage}{\@currentHref}}
\makeatother
Hats on \(\widehat{\bm H},\widehat\lambda_j,\widehat{\bm u}_j\) mark
empirical spectral objects after drawing \(\bm W\); hats on scalar functions
in \cref{app:rv_joint_schedule_proof} denote Laplace transforms.  Plain \(F,K,S,R\) denote \textbf{exact
dynamical quantities}, while \(\mathcal F,\mathcal K,\mathcal S,\mathcal R\)
denote their \textbf{deterministic-equivalent} counterparts.  Where present, subscript
\(\bm W\) makes conditioning on the frozen features explicit; on risks,
subscript \(0\) marks clean labels.  Spectral measures with subscript \(m\)
below are DE objects.

\begingroup
\small
\setlength{\tabcolsep}{3pt}
\renewcommand{\arraystretch}{1.04}
\begin{longtable}{@{}
>{\centering\arraybackslash}p{0.145\linewidth}
>{\raggedright\arraybackslash}p{0.325\linewidth}
>{\centering\arraybackslash}p{0.145\linewidth}
>{\raggedright\arraybackslash}p{0.325\linewidth}@{}}
\caption{Core notation; proof-local symbols are defined at first use.}
\label{tab:notation_guide}\\
\toprule
\textbf{Symbol} & \textbf{Meaning} & \textbf{Symbol} & \textbf{Meaning}\\
\midrule
\endfirsthead
\multicolumn{4}{@{}l}{\textit{Notation guide (continued)}}\\
\toprule
\textbf{Symbol} & \textbf{Meaning} & \textbf{Symbol} & \textbf{Meaning}\\
\midrule
\endhead
\midrule
\multicolumn{4}{r@{}}{\textit{Continued on the next page}}\\
\endfoot
\bottomrule
\endlastfoot
\multicolumn{4}{@{}l}{\textit{Model and spectrum}}\\*[1pt]
\(d\) & input dimension
& \(m\) & feature width\\
\(\alpha\) & spectral-decay exponent
& \(\beta\) & target-regularity exponent\\
\(p\) & target-energy tail exponent \(2\alpha+2\beta-1\)
& \(\bm x\) & covariate\\
\(\bm\Lambda\) & population covariance
& \(\lambda_j\) & population eigenvalue\\
\(\bm\theta^\star\) & teacher coefficients
& \(f_\star\) & teacher predictor\\
\(\varepsilon\) & label noise
& \(\sigma^2\) & label-noise variance\\
\(\bm W\) & frozen feature matrix
& \(\bm a_t\) & readout at update \(t\)\\
\(\widehat{\bm H}\) & \(\bm\Lambda^{1/2}\bm W\bm W^{\!\top}\bm\Lambda^{1/2}\)
& \(\widehat\lambda_j\) & empirical eigenvalue\\
\(\widehat{\bm u}_j\) & empirical eigenvector
& \(\bm M_m(z)\) & diagonal matrix-valued DE resolvent\\
\(\mathfrak m(z)\) & scalar resolvent fixed point
& & \\
\cmidrule(lr){1-4}
\multicolumn{4}{@{}l}{\textit{Dynamics and spectral objects}}\\*[1pt]
\(\eta_t\) & learning rate
& \(B_t\) & batch size\\
\(T_t\) & \(\sum_{s<t}\eta_s\)
& \(T\) & terminal intrinsic time\\
\(r_s\) & \(B_s/\eta_s\)
& \(r(u)\) & intrinsic-time ratio path\\
\(\vartheta\) & regular-variation exponent of \(r\)
& \(L_r\) & slow factor of \(r\)\\
\(q_\eta(\lambda)\) & one-step modal retention
& \(k(v)\) & unit-injection survival at age \(v\)\\
\(F_t\) & exact varying-schedule forcing
& \(K_{t,s}\) & exact two-time kernel\\
\(F_{\bm W}\) & exact conditional forcing
& \(K_{\bm W}\) & exact conditional kernel\\
\(\nu_{\bm W}^{\mathcal F}\) & target-weighted empirical spectral measure
& \(\nu_{\bm W}^{\mathcal K}\) & squared-eigenvalue empirical memory measure\\
\(\boldsymbol\mu_m\) & positive diagonal matrix-valued DE spectral measure
& \(\nu_m^{\mathcal K}\) & \(\lambda^2\)-weighted DE memory measure\\
\(\mu_m^{\mathcal F}\) & target-weighted DE forcing measure
& \(\mu_m^{\mathcal K}\) & DE trace spectral measure\\
\(\mathcal F_{\boldsymbol\eta,\boldsymbol B}(t,m)\)
& joint-schedule DE forcing
& \(\mathcal K_{\boldsymbol\eta,\boldsymbol B}(t,s,m)\)
& joint-schedule DE kernel\\
\(F_{\bm W,0}\) & exact zero-mode forcing
& \(F_{\bm W,>0}\) & exact learnable forcing\\
\(\mathcal F_0\) & DE null-space forcing component
& \(\mathcal F_{pp}\) & DE pure-point forcing component\\
\(\mathcal F_{ac}\) & DE absolutely-continuous forcing component
& \(\mathcal K_{pp}\) & DE pure-point memory component\\
\cmidrule(lr){1-4}
\multicolumn{4}{@{}l}{\textit{Risks, scaling laws, and budgets}}\\*[1pt]
\(R_{\sigma,t}\) & exact conditional noisy risk
& \(R_{0,t}\) & exact conditional clean risk\\
\(R_\sigma(T)\) & exact fixed-limit noisy risk
& \(R_0(T)\) & exact fixed-limit clean risk\\
\(\mathcal R_{\sigma,\boldsymbol\eta,\boldsymbol B}(t,m)\)
& joint-schedule DE noisy risk
& \(\mathcal R_0\) & clean DE risk; arguments suppressed\\
\(R_{\mathrm{app}}\) & exact zero-mode floor
& \(Z_t\) & exact noisy--clean gap\\
\(\widetilde R_\sigma(w)\) & exact risk generating series
& \(\mathcal R_\sigma^\star\) & admissible-class DE optimum\\
\(\mathcal F(T,m)\) & PLRF large-clock phase-order forcing profile
& \(q_{\mathcal F}\) & general learnable-forcing exponent; target-aligned
source coordinate in the PLRF map\\
\(L_{\mathcal F}\) & forcing slow factor
& \(q_{\mathcal K}\) & general one-injection memory exponent; PLRF coordinate
on the LM and IM branches\\
\(L_{\mathcal K}\) & memory slow factor
& \(q_0\) & complete floor-centered clean-loss exponent\\
\(L_0\) & centered clean-loss slow factor
& \(q_{\mathcal N}(\vartheta)\) & schedule-induced noisy--clean-gap exponent\\
\(D_t\) & processed samples by update \(t\)
& \(D\) & data budget\\
\(\mathfrak f\) & compute proxy, \(\mathfrak f\asymp mD\)
& \(\sim\) & ratio tends to one\\
\(\asymp\) & same order
& \(\lesssim\) & upper-order bound\\
\(\ll\) & negligible ratio
& \(\kappa_{\bm W}\) & exact total kernel mass\\
\(\Delta T_s\) & intrinsic-time increment, \(\eta_s\)
& & \\
\cmidrule(lr){1-4}
\multicolumn{4}{@{}l}{\textit{Abbreviations and regimes}}\\*[1pt]
\(\mathrm{PLRF}\) & power-law random-feature model
& \(\mathrm{FSL}\) & functional scaling law\\
\(\mathrm{SGD}\) & stochastic gradient descent
& \(\mathrm{FLOP}\) & floating-point operation\\
\(\mathrm{DE}\) & deterministic equivalent
& \(\mathrm{WSD}\) & warmup--stable--decay\\
\(\mathrm{FB}\) & finite-bulk regime
& \(\mathrm{LM}\) & long-memory regime\\
\(\mathrm{IM}\) & integrable-memory regime
& & \\
\end{longtable}
\endgroup

\begingroup

\section{Spectral representations and formal power-law criteria}
\label{sec:noisy_label_scaling_laws}

This appendix records the exact and deterministic-equivalent spectral
representations used in the paper and states the formal criteria for the
learnable forcing, one-injection memory, accumulated noise, and their transfer
to the observed loss.  The corresponding spectral
and renewal proofs are collected in \cref{app:spectral_renewal_proofs}, while
the width-dependent finite-bulk results are proved in
\cref{app:noisy_43_finite_bulk_empirical,subsec:app_td_lrs_joint_schedules}.
Throughout, \(a\asymp b\) hides dimension-independent positive constants.

\subsection{From exact dynamics to deterministic-equivalent risk}

We use the model, data distribution, online SGD update, conditional prediction
risk, and PLRF parameters defined in
\cref{iclr:sec:model_dynamics,iclr:sec:schedules}.  Thus
\(\lambda_j=j^{-2\alpha}\), \(\theta_j^\star=j^{-\beta}\), and
\(2\alpha+2\beta>1\).  We take \(d\geq c_0m\), \(c_0>1\), with
\(d/m\to c\in(1,\infty)\) when \(\alpha<1/2\) and
\(d/m\to c\in(1,\infty]\) when \(\alpha>1/2\).

For the constant-schedule spectral formulas below, set
\(\eta_t\equiv\eta\), \(B_t\equiv B=\Theta(1)\), and retain
\(\bm a_0=\bm0\).  All sampling, independence, and conditioning conventions
are those of \cref{iclr:sec:model_dynamics}. For this constant schedule, the stable learning-rate scale is
\begin{equation}
\label{eq:noisy_43_learning_rate}
\eta\asymp
\begin{cases}
1,&\alpha>\frac12,\\
m^{2\alpha-1},&0<\alpha<\frac12.
\end{cases}
\end{equation}
The hidden constant is chosen strictly below the threshold in
\cref{eq:noisy_43_stability}.

We first record the exact recursion stated in
\cref{iclr:sec:model_dynamics}, then construct the positive
deterministic-equivalent measures and assemble the corresponding risk.  The
formal power-law criteria follow in the next subsection.

\subsubsection{Exact conditional forcing--memory dynamics}
The operator \(\widehat{\bm H}\), filter \(q_\eta\), forcing
\(F_{\bm W}\), memory kernel \(K_{\bm W}\), and approximation floor
\(R_{\mathrm{app}}\) are defined in \cref{iclr:sec:model_dynamics}; the two
empirical spectral measures and their integral representations are given in
\cref{iclr:sec:foundations}.  We verify here the exact recursion stated in
\cref{iclr:eq:exact_volterra}.  Let
\(\bm e_t:=\bm\Lambda^{1/2}(\bm W\bm a_t-\bm\theta^\star)\) and
\(\rho_j(t):=\langle\widehat{\bm u}_j,\bm e_t\rangle\).
Let
\[
\mathscr F_t
:=
\sigma\!\left(
\bm W,
\{(\bm x_s^i,y_s^i):0\leq s<t,\ 1\leq i\leq B\}
\right),
\qquad
\mathbb E_t[\,\cdot\,]
:=
\mathbb E[\,\cdot\mid\mathscr F_t]\,.
\]
Then \((\bm a_t,\bm e_t,\rho_j(t))\) are \(\mathscr F_t\)-measurable and the
next mini-batch is independent of \(\mathscr F_t\).  Gaussian moments give
\[
\mathbb E_t[\rho_j^2(t+1)]
=
q_\eta(\widehat\lambda_j)\rho_j^2(t)
+\frac{\eta^2}{B}\widehat\lambda_j^2(\|\bm e_t\|_2^2+\sigma^2)\,.
\]
Iterating this scalar recursion and summing over \(j\) proves
\cref{iclr:eq:exact_volterra}.

\paragraph{Terminal risk and feedback stability.}
The exact Volterra equation also determines the fixed-width terminal level.
Define the ordinary generating series, interpreted as formal power series,
\[
\widetilde R_\sigma(w):=\sum_{t\geq0}R_{\sigma,t}w^t,
\qquad
\widetilde F_{\bm W}(w)
:=\sum_{t\geq0}F_{\bm W}(t)w^t,
\qquad
\widetilde K_{\bm W}(w)
:=\sum_{t\geq0}K_{\bm W}(t)w^t.
\]
Coefficient summation in \cref{iclr:eq:exact_volterra} gives
\[
\widetilde R_\sigma(w)
=
\frac{
\widetilde F_{\bm W}(w)
+\sigma^2w\widetilde K_{\bm W}(w)/(1-w)}
{1-w\widetilde K_{\bm W}(w)}.
\]
The reciprocal denominator
\((1-w\widetilde K_{\bm W}(w))^{-1}\) sums repeated reinjections of
previous prediction error and is the transfer function of the SGD feedback.

For fixed \((d,m,\bm W)\), the terminal risk can be computed exactly.  Assume
\(\eta(1+1/B)\lambda_{\max}(\widehat{\bm H})<2\) and define the exact kernel
mass
\[
\kappa_{\bm W}
:=\sum_{t\geq0}K_{\bm W}(t)
=\frac{\eta}{B}\sum_{\widehat\lambda_j>0}
\frac{\widehat\lambda_j}
{2-(1+1/B)\eta\widehat\lambda_j}\,.
\]
Since the kernel is nonnegative, \(\kappa_{\bm W}<1\) gives
\(\lvert w\widetilde K_{\bm W}(w)\rvert<1\) for
\(\lvert w\rvert\leq1\).  The feedback is therefore stable, with
input--output gain at most \((1-\kappa_{\bm W})^{-1}\).  The positive-spectrum
forcing then vanishes as \(t\to\infty\), and the exact conditional terminal
risk is
\begin{equation}
\lim_{t\to\infty}R_{\sigma,t}
=\frac{R_{\mathrm{app}}+\sigma^2\kappa_{\bm W}}
{1-\kappa_{\bm W}}\,.
\label{iclr:eq:terminal_plateau}
\end{equation}
Thus the terminal level combines frozen-representation error with persistent
training-label noise and its feedback amplification.  The centered-transfer
result is stated in \cref{thm:noisy_43_aggregate_volterra_transfer} and proved
in \cref{app:noisy_43_aggregate_volterra_transfer_proof}.

\subsubsection{Positive spectral measures for the deterministic equivalent}
The random-matrix step preserves the same forcing--memory picture: it replaces
the realized spectral projection by deterministic positive measures.  Target
weighting will produce the forcing measure, while the trace measure, together
with the squared-eigenvalue variance weight, will produce memory.

To construct these measures, write \(\lambda_j\) for the diagonal entries of
\(\bm\Lambda\); in PLRF, \(\lambda_j=j^{-2\alpha}\).  The diagonal
matrix-valued DE resolvent is
\begin{equation*}
\bm M_m(z)
:=
\operatorname{diag}\left(
\frac{1}{\lambda_j\mathfrak m(z)-z}:1\leq j\leq d
\right),
\end{equation*}
where \(\mathfrak m(z)\) is the analytic solution of
\begin{equation*}
\mathfrak m(z)
=
\left(
1+\frac1m\sum_{j=1}^d
\frac{\lambda_j}{\lambda_j\mathfrak m(z)-z}
\right)^{-1}.
\end{equation*}
Choose the Cauchy--Stieltjes branch of
\citet[Proposition~E.1]{paquette20244+}, characterized by
\(\operatorname{Im}\mathfrak m(z)<0\) on the upper half-plane and
\(\mathfrak m(z)\to1\) as \(|z|\to\infty\).

For the matrix-valued measure \(\boldsymbol\mu_m\) whose positivity and
uniqueness are established next, the two scalar projections relevant to
learning are
\begin{equation}
\label{eq:noisy_43_de_scalar_measures}
\mu_m^{\mathcal F}(A)
:=\left\langle
\boldsymbol\mu_m(A)\bm\Lambda^{1/2}\bm\theta^\star,
\bm\Lambda^{1/2}\bm\theta^\star
\right\rangle,
\qquad
\mu_m^{\mathcal K}(A):=\operatorname{tr}\boldsymbol\mu_m(A).
\end{equation}
The squared-spectrum weighting used by every DE memory kernel is
\begin{equation}
\label{eq:noisy_43_de_weighted_memory_measure}
\nu_m^{\mathcal K}(\mathrm d\lambda)
:=\lambda^2\mu_m^{\mathcal K}(\mathrm d\lambda).
\end{equation}
Its target-weighted zero atom is the DE approximation floor
\(\mathcal F_0(m):=\mu_m^{\mathcal F}(\{0\})\).

The next lemma turns the DE resolvent into genuine positive spectral measures,
so the contour formulas used later are ordinary spectral averages rather than
only formal complex integrals.  It also records the moments and functional
calculus needed below.

\begin{lemma}[Positive spectral measures of the PLRF deterministic equivalent]
\label{lem:noisy_43_de_positive_spectral_measure}
For every finite \(m,d\), there is a unique finite positive-semidefinite
diagonal matrix-valued Borel measure \(\boldsymbol\mu_m\), compactly supported
on \([0,\infty)\), such that
\begin{equation}
\label{eq:noisy_43_de_matrix_measure}
\bm M_m(z)
=
\int_{[0,\infty)}
\frac{\boldsymbol\mu_m(\mathrm d\lambda)}{\lambda-z},
\qquad z\in\mathbb C\setminus[0,\infty).
\end{equation}
It satisfies
\begin{equation}
\label{eq:noisy_43_de_matrix_measure_moments}
\boldsymbol\mu_m([0,\infty))=\bm I_d,
\qquad
\int\lambda\,\boldsymbol\mu_m(\mathrm d\lambda)=\bm\Lambda,
\qquad
\int\lambda^2\,\boldsymbol\mu_m(\mathrm d\lambda)
=
\bm\Lambda^2+\frac{\operatorname{tr}\bm\Lambda}{m}\bm\Lambda.
\end{equation}
The scalar measures in \cref{eq:noisy_43_de_scalar_measures} are finite and
nonnegative, with
\begin{equation}
\label{eq:noisy_43_de_scalar_measure_moments}
\begin{gathered}
\mu_m^{\mathcal F}([0,\infty))
=(\bm\theta^\star)^{\!\top}\bm\Lambda\bm\theta^\star,
\qquad
\mu_m^{\mathcal K}([0,\infty))
=d,\\
\int\lambda\,\mu_m^{\mathcal F}(\mathrm d\lambda)
=(\bm\theta^\star)^{\!\top}\bm\Lambda^2\bm\theta^\star,
\quad
\int\lambda\,\mu_m^{\mathcal K}(\mathrm d\lambda)
=\operatorname{tr}\bm\Lambda,
\quad
\int\lambda^2\,\mu_m^{\mathcal K}(\mathrm d\lambda)
=\operatorname{tr}(\bm\Lambda^2)
+\frac{(\operatorname{tr}\bm\Lambda)^2}{m}.
\end{gathered}
\end{equation}
If \(h\) is analytic on a neighborhood of
\(\operatorname{supp}\mu_m^{\mathcal K}\), then
\begin{equation}
\label{eq:noisy_43_de_functional_calculus}
-\frac{1}{2\pi\mathrm i}
\oint_{\mathscr C}h(z)\bm M_m(z)\,\mathrm dz
=
\int h(\lambda)\boldsymbol\mu_m(\mathrm d\lambda).
\end{equation}
\end{lemma}

The proof of \cref{lem:noisy_43_de_positive_spectral_measure} is in
\cref{app:noisy_43_positive_de_measure_proof}.

\subsubsection{Assembling the deterministic-equivalent risk}
We now return to the risk.  The DE retains the exact system's input--output
structure: its spectral transforms define forcing and memory, and stable
positive feedback keeps the noisy--clean gap within a fixed factor of
cumulative memory.

Let \(\mathscr C\) enclose the deterministic-equivalent spectral support and
no other singularities.  The DE forcing and memory are
\begin{align}
\mathcal F(t,m)
&:=
-\frac{1}{2\pi\mathrm i}
\oint_{\mathscr C}
\left\langle
\bm M_m(z)\bm\Lambda^{1/2}\bm\theta^\star,
\bm\Lambda^{1/2}\bm\theta^\star
\right\rangle
q_\eta(z)^t\,\mathrm dz ,
\label{eq:noisy_43_forcing_definition}\\
\mathcal K(t,m)
&:=
\frac{\eta^2}{B}\,
\operatorname{tr}\left[
-\frac{1}{2\pi\mathrm i}
\oint_{\mathscr C}
\bm M_m(z)z^2q_\eta(z)^t\,\mathrm dz
\right].
\label{eq:noisy_43_kernel_definition}
\end{align}
By \cref{eq:noisy_43_de_functional_calculus}, these contour expressions are
positive spectral transforms.
\begin{equation}
\label{eq:noisy_43_de_forcing_split}
\mathcal F_0(m)
=
-\underset{z=0}{\operatorname{Res}}
\left\langle
\bm M_m(z)\bm\Lambda^{1/2}\bm\theta^\star,
\bm\Lambda^{1/2}\bm\theta^\star
\right\rangle,
\qquad
\mathcal F_{>0}(t,m)
:=
\mathcal F(t,m)-\mathcal F_0(m),
\end{equation}
so that \(\mathcal F(t,m)=\mathcal F_0(m)+\mathcal F_{>0}(t,m)\).
Only the zero-versus-positive split is universal.  For general
\((\lambda_i,\theta_i^\star)\), we keep the positive-spectrum forcing intact;
the later pure-point/absolutely-continuous split is PLRF-specific.  The
temporal spectral and accumulation criteria below are first stated for exact
conditional quantities; the transfer theorem then covers both exact and DE
losses.

To assemble the components into risk, define the causal discrete convolution
for sequences \(a,b:\mathbb N_0\to\mathbb{R}\) by
\[
(a*b)(0):=0,
\qquad
(a*b)(t):=\sum_{s=0}^{t-1}a(t-1-s)b(s),
\quad t\geq1.
\]
The clean DE risk is the bounded response to the forcing under this kernel:
\begin{equation}
\label{eq:noisy_43_clean_volterra_definition}
\mathcal R_0=\mathcal F+\mathcal K*\mathcal R_0 .
\end{equation}
This is the zero-noise constant-schedule resolvent-DE surrogate for the SGD
dynamics in \cref{iclr:sec:model_dynamics}.

Its noisy counterpart adds cumulative label-noise injection:
\begin{equation}
\mathcal R_\sigma=\mathcal F+\mathcal K*\mathcal R_\sigma+\sigma^2\mathcal S,
\qquad
\mathcal S_t(m):=\sum_{s=0}^{t-1}\mathcal K(s,m),
\quad \mathcal S_0(m):=0.
\label{eq:noisy_43_clean_risk_volterra}
\end{equation}

The comparison requires two uniform stability conditions: every positive mode
contracts, and the total Volterra feedback remains below one:
\begin{equation}
\label{eq:noisy_43_stability}
\begin{aligned}
\eta\left(1+\frac1B\right)
\sup\{\lambda>0:\lambda\text{ lies in the DE spectral support of }
\bm M_m\}&\leq2-\delta,\\
\|\mathcal K\|_{\ell^1}
:=\sum_{t=0}^\infty\mathcal K(t,m)
&\leq\kappa<1.
\end{aligned}
\end{equation}
Here \(\delta>0\) and \(\kappa<1\) are fixed uniformly in \(m\).

Because the kernel is nonnegative, the direct injection gives the lower bound,
while repeated stable feedback amplifies it by at most
\((1-\kappa)^{-1}\).  Hence
\begin{equation}
\label{eq:noisy_43_scaling_reduction}
\mathcal R_0(t,m)+\sigma^2\mathcal S_t(m)
\leq
\mathcal R_\sigma(t,m)
\leq
\mathcal R_0(t,m)+\frac{\sigma^2}{1-\kappa}\mathcal S_t(m).
\end{equation}
This is the reduction used below: once the clean-risk and cumulative-memory
scales are known, the noisy risk follows up to constants fixed by stability.

\subsection{Spectral criteria for power-law learning dynamics}
\label{subsec:noisy_43_direct_spectral_criterion}

This subsection follows the temporal components to the observed loss.
Weighted empirical tails determine forcing and one-injection memory,
cumulative memory determines the effect of persistent noise, and stable
Volterra feedback transfers these components to the observed loss.

A positive function \(L\) is {\rm slowly varying} if \(L(cT)/L(T)\to1\) for every fixed
\(c>0\).

\paragraph{Forcing and memory: spectral tails determine temporal decay.}
Forcing has a power law exactly when
the target-weighted low-spectrum mass does, while one-injection memory has a
power law exactly when the squared-spectrum mass does.  A mode of size
\(\lambda\) is learned around \(\lambda^{-1}\), so at time \(T\) both component decays
are read near the moving cutoff \(T^{-1}\).

We study \(K_{\bm W}\) because persistent noise depends on its sum,
\(S_{\bm W,t}:=\sum_{u<t}K_{\bm W}(u)\), analyzed next.

Every fixed finite spectrum eventually decays exponentially, so the
equivalences use the following common joint-limit condition.  It excludes
step-boundary artifacts and moving spectral packets without presupposing a
target power law.  Proofs are in
\cref{app:noisy_43_effective_spectral_criterion_proof,app:noisy_43_effective_memory_criterion_proof}.

\begin{assumption}[Constant-schedule stability]
\label{ass:noisy_43_constant_schedule_stability}
A constant schedule \(\eta_t\equiv\eta\), \(B_t\equiv B\) satisfies
constant-schedule stability if there are constants \(\delta>0\) and
\(\kappa<1\), uniform over the width, such that:
\begin{enumerate}[label=(\alph*)]
\item \emph{Pointwise contraction.}
\[
\eta\left(1+\frac1B\right)
\lambda_{\max}(\widehat{\bm H})\leq2-\delta.
\]

\item \emph{Row stability.}
\[
\sup_{t\geq1}\sum_{s<t}K_{\bm W}(t-1-s)
=\sum_{h\geq0}K_{\bm W}(h)\leq\kappa.
\]
\end{enumerate}
For the DE recursion, pointwise contraction is imposed uniformly over
\(\lambda\in\operatorname{supp}\mu_m^{\mathcal K}\), and row stability
replaces \(K_{\bm W}(h)\) by \(\mathcal K(h,m)\).
\end{assumption}

\begin{assumption}[Uniform spectral window]
\label{ass:noisy_43_uniform_stable_spectral_window}
Retain the constant schedule, assume part \emph{(a)} of
\cref{ass:noisy_43_constant_schedule_stability}, and consider a joint
sequence with \(m\to\infty\), \(t\to\infty\), and
\(T=\eta t\to\infty\).  For the component under consideration, let
\(A_m(x)\) denote its cumulative weighted spectral mass below \(x\).  Fix
\(q>0\), a slowly varying function \(L\), and a normalization constant
\(c_A>0\), and set
\[
a_m:=c_A T^{-q}L(T).
\]
There exist \(C<\infty\) and \(\epsilon\in(0,q)\), independent of \(m,t\),
such that eventually, for every \(y>0\),
\begin{equation}
\label{eq:noisy_43_stable_spectral_window_envelope}
\frac{A_m(y/T)}{a_m}
\leq C\max\{y^{q-\epsilon},y^{q+\epsilon}\},
\qquad y>0.
\end{equation}
\end{assumption}

\begin{theorem}[Spectral criteria for power-law forcing and memory]
\label{thm:noisy_43_effective_spectral_criterion}
\label{thm:noisy_43_effective_memory_criterion}
Retain the constant schedule \(\eta_t\equiv\eta\), \(B_t\equiv B\).  For each width, let
\(\bm\Lambda=\operatorname{diag}(\lambda_1,\ldots,\lambda_d)\), with
\(\lambda_j\geq0\) and
\((\bm\theta^\star)^{\!\top}\bm\Lambda\bm\theta^\star<\infty\), so the
covariance and finite-energy target are otherwise arbitrary.  Condition
on \(\bm W\), let
\((\widehat\lambda_j,\widehat{\bm u}_j)\) be the empirical eigenpairs in
\cref{iclr:sec:model_dynamics}, and retain the empirical measures
\(\nu_{\bm W}^{\mathcal F},\nu_{\bm W}^{\mathcal K}\) above.
\begin{enumerate}
\item[\emph{Forcing.}]
Under zero initialization, the initial residual is
\(-\bm\Lambda^{1/2}\bm\theta^\star\).  Let \(q_{\mathcal F}>0\) and let
\(L_{\mathcal F}\) be slowly varying.  Assume
\cref{ass:noisy_43_uniform_stable_spectral_window} with
\[
A_m(x)=\nu_{\bm W}^{\mathcal F}((0,x]),
\qquad q=q_{\mathcal F},
\qquad L=L_{\mathcal F},
\qquad c_A=1.
\]
Then
\[
\nu_{\bm W}^{\mathcal F}((0,T^{-1}])
\sim T^{-q_{\mathcal F}}L_{\mathcal F}(T)
\]
if and only if
\[
F_{\bm W,>0}(t)
\sim
\Gamma(q_{\mathcal F}+1)(2T)^{-q_{\mathcal F}}L_{\mathcal F}(T)
\qquad (T\to\infty).
\]

\item[\emph{Memory.}]
Let \(q_{\mathcal K}>0\) and let \(L_{\mathcal K}\) be slowly varying.
Assume \cref{ass:noisy_43_uniform_stable_spectral_window} with
\[
A_m(x)=\nu_{\bm W}^{\mathcal K}((0,x]),
\qquad q=q_{\mathcal K},
\qquad L=L_{\mathcal K},
\qquad
c_A=\frac{2^{q_{\mathcal K}}}{\Gamma(q_{\mathcal K}+1)}.
\]
Then
\[
\nu_{\bm W}^{\mathcal K}((0,T^{-1}])
\sim
\frac{2^{q_{\mathcal K}}}{\Gamma(q_{\mathcal K}+1)}T^{-q_{\mathcal K}}L_{\mathcal K}(T)
\]
if and only if
\begin{equation}
\label{eq:noisy_43_effective_memory_tauberian}
K_{\bm W}(t)
\sim
\frac{\eta^2}{B}
T^{-q_{\mathcal K}}L_{\mathcal K}(T).
\end{equation}
\end{enumerate}
Each asymptotic equivalence above is understood locally uniformly under
constant-factor rescalings of \(T\).
\end{theorem}

Fixed-margin stability is standard
\citep{paquette2021sgd,paquette2025homogenization}; fixed-spectrum Potter
bounds follow from \citet[Theorem~1.5.6]{bingham1989regular}.

\begin{remark}[Joint-limit uniformity]
Empirical spectral control is available
\citep{misiakiewicz2024non,defilippis2024dimension}; the simultaneous
width--time uniformity used above is assumed.
\end{remark}

\paragraph{Cumulative memory: growth or convergence.}
Summing one-injection memory creates the dividing line
\(q_{\mathcal K}=1\).  Below it, cumulative memory has asymptotic scale
\((\eta/B)T^{1-q_{\mathcal K}}L_{\mathcal K}(T)\); above it, memory
converges with a remaining tail of the same order.  At equality, a fixed
infinite spectrum has a borderline integral.
The theorem also gives a direct spectral criterion for the terminal tail
without assuming a power law for \(K_{\bm W}\).  The proof is in
\cref{app:noisy_43_aggregate_noise_criterion_proof}.

Under the same constant schedule and conditional on \(\bm W\), retain
\(S_{\bm W,t}\) above and define its total mass by
\(S_{\bm W,\infty}:=\sum_{u=0}^{\infty}K_{\bm W}(u)\).

\begin{theorem}[When cumulative memory grows or converges]
\label{thm:noisy_43_aggregate_noise_criterion}
Under the same constant schedule, condition on \(\bm W\), retain \(q_\eta\)
and \(K_{\bm W}\).
Assume part \emph{(a)} of
\cref{ass:noisy_43_constant_schedule_stability}.

\emph{Cumulative-memory asymptotics.}
Let \(q_{\mathcal K}>0\), let \(L_{\mathcal K}\) be slowly varying, and
consider a joint limit in which \(m\to\infty\), \(t\to\infty\), and
\(T=\eta t\to\infty\).
Set
\[
A_m^{\mathcal K}(x)
:=\nu_{\bm W}^{\mathcal K}((0,x]),
\qquad
a_m^{\mathcal K}
:=\frac{2^{q_{\mathcal K}}}{\Gamma(q_{\mathcal K}+1)}
T^{-q_{\mathcal K}}L_{\mathcal K}(T),
\qquad
\widetilde A_m^{\mathcal K}(y)
:=\frac{A_m^{\mathcal K}(y/T)}{a_m^{\mathcal K}}.
\]
Suppose the memory cutoff asymptotic
\begin{equation}
\label{eq:noisy_43_aggregate_memory_cutoff}
\widetilde A_m^{\mathcal K}(y)\longrightarrow y^{q_{\mathcal K}}
\end{equation}
holds locally uniformly for \(y\in(0,\infty)\).  For the branch under
consideration, also suppose that there exist \(C<\infty\) and
\(\epsilon>0\) such that, eventually, for every \(y>0\),
\begin{equation}
\label{eq:noisy_43_aggregate_memory_endpoint_envelope}
\widetilde A_m^{\mathcal K}(y)
\leq C\max\{y^{q_{\mathcal K}-\epsilon},
y^{q_{\mathcal K}+\epsilon}\},
\qquad y>0.
\end{equation}
with constants uniform along the joint limit and
\[
0<\epsilon<\min\{q_{\mathcal K},1-q_{\mathcal K}\}
\quad\text{if }0<q_{\mathcal K}<1,
\qquad
0<\epsilon<q_{\mathcal K}-1
\quad\text{if }q_{\mathcal K}>1.
\]
Then, for \(0<q_{\mathcal K}<1\),
\begin{equation}
\label{eq:noisy_43_aggregate_noise_from_kernel}
S_{\bm W,t}
\sim
\frac{\eta}{B}
\frac{1}{1-q_{\mathcal K}}
T^{1-q_{\mathcal K}}L_{\mathcal K}(T),
\end{equation}
whereas for \(q_{\mathcal K}>1\),
\begin{equation}
\label{eq:noisy_43_terminal_noise_from_kernel}
S_{\bm W,\infty}-S_{\bm W,t}
\sim
\frac{\eta}{B}
\frac{1}{q_{\mathcal K}-1}
T^{1-q_{\mathcal K}}L_{\mathcal K}(T).
\end{equation}
Under \cref{ass:noisy_43_uniform_stable_spectral_window},
\cref{eq:noisy_43_aggregate_memory_cutoff} is equivalent to
\cref{eq:noisy_43_effective_memory_tauberian}.  All these hypotheses and conclusions are locally
uniform under constant-factor rescalings of \(T\).

\emph{Direct terminal-tail criterion.}
Without assuming a power law for \(K_{\bm W}\), consider the same
joint intermediate limit as in
\cref{thm:noisy_43_effective_memory_criterion}, and let
\(q_{\mathcal S}>0\) and \(L_{\mathcal S}\) be slowly varying.  Define
\[
A_m^{\rm tail}(x)
:=
\frac{\eta}{B}
\sum_{0<\widehat\lambda_j\leq x}
\frac{\widehat\lambda_j}
{2-(1+1/B)\eta\widehat\lambda_j}.
\]
Suppose that, for some \(\epsilon\in(0,q_{\mathcal S})\) and
\(C<\infty\), it eventually satisfies, for every \(y>0\),
\begin{equation*}
\frac{A_m^{\rm tail}(y/T)}{T^{-q_{\mathcal S}}L_{\mathcal S}(T)}
\leq
C\max\{y^{q_{\mathcal S}-\epsilon},
y^{q_{\mathcal S}+\epsilon}\},
\qquad y>0.
\end{equation*}
Under this envelope, the following two locally uniform families are
equivalent.  For \(c>0\), put \(t_c:=\lfloor ct\rfloor\) and
\(T_c:=\eta t_c\), so \(T_c\sim cT\) locally uniformly in \(c\).
Then, for every compact
\(J\subset(0,\infty)\),
\[
\sup_{c\in J}
\left|
\frac{A_m^{\rm tail}(1/T_c)}{T_c^{-q_{\mathcal S}}L_{\mathcal S}(T_c)}-1
\right|\longrightarrow0
\]
if and only if
\begin{equation*}
\sup_{c\in J}
\left|
\frac{S_{\bm W,\infty}-S_{\bm W,t_c}}
{\Gamma(q_{\mathcal S}+1)(2T_c)^{-q_{\mathcal S}}L_{\mathcal S}(T_c)}-1
\right|\longrightarrow0.
\end{equation*}

\end{theorem}

\begin{remark}[Exact finite-width accumulation]
The finite-width identities underlying the theorem are
\begin{equation}
S_{\bm W,t}=\frac{\eta}{B}\sum_{\widehat\lambda_j>0}
\frac{\widehat\lambda_j[1-q_\eta(\widehat\lambda_j)^t]}
{2-(1+1/B)\eta\widehat\lambda_j},
\qquad
S_{\bm W,\infty}-S_{\bm W,t}=\frac{\eta}{B}
\sum_{\widehat\lambda_j>0}
\frac{\widehat\lambda_jq_\eta(\widehat\lambda_j)^t}
{2-(1+1/B)\eta\widehat\lambda_j}.
\label{eq:noisy_43_terminal_noise_exact}
\end{equation}
\end{remark}

For a fixed infinite-spectrum kernel that satisfies
\(K(t)\sim(\eta^2/B)T^{-1}L_{\mathcal K}(T)\) at all sufficiently large
times, ordinary one-sequence summation gives
\((\eta/B)\int_1^T L_{\mathcal K}(u)\,\mathrm du/u\) when this integral
diverges, or the remaining tail
\((\eta/B)\int_T^\infty L_{\mathcal K}(u)\,\mathrm du/u\) when it converges.
For the DE kernel, retain the squared-spectrum measure in
\cref{eq:noisy_43_de_weighted_memory_measure} and set
\(A_{m,{\rm DE}}^{\mathcal K}(x):=\nu_m^{\mathcal K}((0,x])\).
The same accumulation and remaining-tail conclusions hold when this
cumulative measure satisfies the corresponding cutoff and endpoint-envelope
hypotheses and the DE stability margin holds.  The direct terminal-tail
criterion instead uses
\[
A_{m,{\rm DE}}^{\rm tail}(x)
:=\frac{\eta}{B}\int_{(0,x]}
\frac{\lambda}{2-(1+1/B)\eta\lambda}
\,\mu_m^{\mathcal K}(\mathrm d\lambda).
\]
The same DE proofs apply under the stated cutoff, endpoint-envelope, and
stability hypotheses.  The next result transfers these forcing and memory asymptotics to loss.

\paragraph{From components to the observed loss.}
Under either transfer condition, stable feedback introduces no scale beyond
the displayed forcing, cumulative-memory, and renewal terms.  The clean loss
follows the floor and forcing, while the noisy loss also follows cumulative
memory; renewal feedback adds one explicit kernel-sized term.  At a terminal
plateau, the same components govern centered convergence when the terminal
tail remains uniformly controlled under renewal and the two positive
post-feedback responses remain quantitatively noncancelling.  The theorem
covers both the exact conditional recursion and its deterministic equivalent.

The exact conditional loss and its deterministic equivalent obey the same
positive Volterra system, so we state their common transfer result once.  Use the
following neutral notation for the exact quantities in
\cref{iclr:eq:exact_volterra} and the DE quantities in
\cref{eq:noisy_43_forcing_definition,eq:noisy_43_kernel_definition,eq:noisy_43_clean_risk_volterra}:
\[
\begin{aligned}
(\mathsf F,\mathsf K;\mathsf R_\sigma,\mathsf R_0)
&=(F_{\bm W},K_{\bm W};R_{\sigma,\cdot},R_{0,\cdot}),
&&\text{exact conditional},\\
&=(\mathcal F,\mathcal K;\mathcal R_\sigma,\mathcal R_0),
&&\text{deterministic equivalent}.
\end{aligned}
\]
Suppress the explicit \(\bm W\)-conditioning and use \(m\) as the common
width index.  In either case, write
\[
\mathsf F(t,m)=\mathsf F_0(m)+\mathsf F_{>0}(t,m),
\qquad
\mathsf S_t(m):=\sum_{u=0}^{t-1}\mathsf K(u,m),
\qquad
\kappa_m:=\sum_{t\geq0}\mathsf K(t,m).
\]
Whenever the fixed-width terminal limit exists and \(\kappa_m<1\), denote it
by \(P_m:=\lim_{t\to\infty}\mathsf R_\sigma(t,m)\).  For the centered
response, also set
\[
\overline{\mathsf K}_m(t)
:=\kappa_m-\mathsf S_t(m),
\qquad
\mathsf G_m g
:=g+\sum_{n\geq1}\mathsf K^{(*n)}*g.
\]

\begin{assumption}[Long-tail conditions for transferring component laws to risk]
\label{ass:noisy_43_stable_feedback_transfer}
Assume part \emph{(b)} of
\cref{ass:noisy_43_constant_schedule_stability}.  Uniformly along the common
intermediate window, both \(\mathsf F_{>0}\) and \(\mathsf K\) are
long-tailed: each is asymptotically unchanged by any fixed time shift.
\end{assumption}

\begin{theorem}[How stable feedback transfers components to observed loss]
\label{thm:noisy_43_aggregate_volterra_transfer}
Under the same constant schedule, fix either of the two preceding rows and
assume \cref{ass:noisy_43_stable_feedback_transfer}.  The transfer has two
forms:
\begin{enumerate}[label=(\roman*)]
\item \emph{Forcing preservation.} The leading loss scales are
\begin{equation}
\label{eq:noisy_43_aggregate_loss_transfer}
\mathsf R_0(t,m)\asymp\mathsf F_0(m)+\mathsf F_{>0}(t,m),
\qquad
\mathsf R_\sigma(t,m)\asymp
\mathsf F_0(m)+\mathsf F_{>0}(t,m)+\sigma^2\mathsf S_t(m).
\end{equation}
These comparisons hold provided
\[
\left[\sum_{n\geq1}\mathsf K^{(*n)}(t,m)\right]
\sum_{s=0}^{t}\mathsf F_{>0}(s,m)
=o\!\left(\mathsf F_{>0}(t,m)\right)
\]
uniformly in the window.

\item \emph{Renewal feedback.}
The same comparisons hold after adding
\[
\left[\sum_{s\geq0}\mathsf F_{>0}(s,m)\right]\mathsf K(t,m)
\]
to both right-hand sides.  This conclusion holds when
\(\sum_{t\geq0}\mathsf F_{>0}(t,m)<\infty\) and the standard uniform
subexponential relations
\[
(\mathsf K*\mathsf K)(t,m)\asymp\mathsf K(t,m),
\qquad
(\mathsf K*\mathsf F_{>0})(t,m)\asymp
\|\mathsf K\|_{\ell^1}\mathsf F_{>0}(t,m)
+\left[\sum_{s\geq0}\mathsf F_{>0}(s,m)\right]\mathsf K(t,m)
\]
hold with summable domination over the convolution powers.
\end{enumerate}

If \(\mathsf F_{>0}(t,m)\to0\), the fixed-width terminal plateau exists and
satisfies
\begin{equation}
\label{eq:noisy_43_terminal_risk_exact}
P_m
=
\frac{\mathsf F_0(m)+\sigma^2\kappa_m}{1-\kappa_m}.
\end{equation}
For the centered response, the positive defective resolvent gives the exact
identity
\begin{equation}
\label{eq:noisy_43_centered_response_identity}
\mathsf R_\sigma(t,m)-P_m
=
(\mathsf G_m\mathsf F_{>0})(t)
-(P_m+\sigma^2)(\mathsf G_m\overline{\mathsf K}_m)(t).
\end{equation}
To reduce this identity to the original forcing and kernel tail, assume,
uniformly across the joint window, the terminal-tail renewal domination
\begin{equation}
\label{eq:noisy_43_terminal_tail_renewal_domination}
\sum_{n\geq1}
(\mathsf K^{(*n)}*\overline{\mathsf K}_m)(t)
\leq C_{\rm tail}\overline{\mathsf K}_m(t)
\end{equation}
and the quantitative noncancellation condition
\begin{equation}
\label{eq:noisy_43_centered_noncancellation}
\left|\mathsf R_\sigma(t,m)-P_m\right|
\geq c_{\rm nc}\left[
(\mathsf G_m\mathsf F_{>0})(t)
+(P_m+\sigma^2)(\mathsf G_m\overline{\mathsf K}_m)(t)
\right],
\qquad c_{\rm nc}>0.
\end{equation}
Here \(C_{\rm tail}<\infty\) and \(c_{\rm nc}>0\) are independent of
\(m,t\) throughout that window.
Then, in case \emph{(i)},
\begin{equation}
\label{eq:noisy_43_aggregate_centered_transfer}
\left|\mathsf R_\sigma(t,m)-P_m\right|
\asymp
\mathsf F_{>0}(t,m)
+\left(P_m+\sigma^2\right)
\left[\kappa_m-\mathsf S_t(m)\right].
\end{equation}
In case \emph{(ii)}, the right-hand side additionally contains
\(\left[\sum_{s\geq0}\mathsf F_{>0}(s,m)\right]\mathsf K(t,m)\).
\end{theorem}

The transfer theorem is generic; we now evaluate its spectral inputs in PLRF.

\paragraph{PLRF exponents.}
For PLRF, the abstract spectral criteria become two explicit coordinates.
With intrinsic time \(T=\eta t\), finite target energy gives
\(q_{\mathcal F}=(2\alpha+2\beta-1)/(2\alpha)\); on the LM and IM branches
\(\alpha>1/4\), one-injection memory has
\(q_{\mathcal K}=2-1/(2\alpha)\).  Target alignment therefore affects
forcing, whereas memory depends only on spectral decay.  The corollary records
the population-filter constants and admissible parameter range before we
explain the cutoff calculation.

\begin{corollary}[PLRF forcing and memory exponents]
\label{cor:noisy_43_direct_plrf_43_exponents}
Consider the PLRF specialization \(\lambda_i=i^{-2\alpha}\) and
\(\theta_i^\star=i^{-\beta}\).  If \(2\alpha+2\beta>1\), then the forcing
exponent is \(q_{\mathcal F}=(2\alpha+2\beta-1)/(2\alpha)\).
As \(T\to\infty\), the population target filter satisfies
\begin{equation*}
\sum_i\lambda_i(\theta_i^\star)^2e^{-2T\lambda_i}
\sim
\frac{\Gamma\!\left(1+\frac{2\alpha+2\beta-1}{2\alpha}\right)}
{2\alpha+2\beta-1}
(2T)^{-\frac{2\alpha+2\beta-1}{2\alpha}}.
\end{equation*}
If additionally \(\alpha>1/4\), then the memory exponent is
\(q_{\mathcal K}=2-1/(2\alpha)\).
As \(T\to\infty\), the population memory filter satisfies
\begin{equation*}
\sum_i\lambda_i^2e^{-2T\lambda_i}
\sim
\frac{\Gamma\!\left(3-\frac1{2\alpha}\right)}{4\alpha-1}
(2T)^{-2+1/(2\alpha)}.
\end{equation*}
\end{corollary}

These exponents come from the moving spectral cutoff.  At intrinsic time
\(T=\eta t\), the critical modes satisfy
\(q_\eta(\lambda)^t=e^{-2T\lambda}(1+o(1))\).
Consequently, in the common intermediate window,
\[
\mathcal F_{pp}(t,m)\asymp
\sum_i\lambda_i(\theta_i^\star)^2e^{-2T\lambda_i},
\qquad
\mathcal K_{pp}(t,m)\asymp
\frac{\eta^2}{B}\sum_i\lambda_i^2e^{-2T\lambda_i}.
\]
These are respectively unlearned target error and one surviving variance
injection.  For joint schedules, use elapsed intrinsic time and the injection
weight in
\cref{eq:td_lrs_joint_de_kernel_law}.

Since only modes with \(\lambda_iT\lesssim1\) remain, the filters reduce to
\[
\sum_{0<\lambda_i\lesssim T^{-1}}\lambda_i(\theta_i^\star)^2,
\qquad
\sum_{0<\lambda_i\lesssim T^{-1}}\lambda_i^2.
\]
Direct summation gives the displayed exponents; see
\cref{app:noisy_43_effective_spectral_criterion_proof,app:noisy_43_effective_memory_criterion_proof}.

Two boundaries delimit these formulas.  As
\(2\alpha+2\beta\downarrow1\), \(q_{\mathcal F}\downarrow0\).  At
\(\alpha=1/4\), the LM/IM memory formula gives way to FB.

For \(\alpha<1/4\),
\cref{lem:noisy_43_finite_bulk_spectral_event,lem:td_lrs_fb_bulk_kernel}
give the empirical width-scale band and its joint-schedule kernel.  The
resulting exact conditional noisy--clean gap is closed in
\cref{cor:td_lrs_exact_finite_bulk_gap}, without a target-forcing comparison.
Separately,
\cref{lem:noisy_43_de_finite_bulk_band_mass,prop:td_lrs_joint_de_finite_bulk_kernel}
give the DE kernel and gap; explicit DE total-risk asymptotics retain the forcing
range in \cref{lem:td_lrs_de_forcing_filters}.  For \(\alpha>1/4\),
\cref{lem:td_lrs_de_spectral_filters} proves LM/IM DE kernel transfer, with
that source restriction only for the explicit PLRF forcing profile.

\endgroup

\section{Proofs of the spectral power-law criteria}
\label{app:spectral_renewal_proofs}

This section supplies the proof path from frozen spectral structure to the
observed learning curve.  \Cref{app:noisy_43_positive_de_measure_proof}
first proves that the deterministic-equivalent resolvent defines positive
spectral measures.  \Cref{app:noisy_43_uniform_tauberian} then shows why
weighted spectral mass near zero is equivalent to long-time decay.
\Cref{app:noisy_43_component_criteria_proofs} applies this principle to the
learnable target error and the survival of one noise injection.
\Cref{app:noisy_43_aggregate_noise_criterion_proof} sums these
injections over training, and
\cref{app:noisy_43_aggregate_volterra_transfer_proof} shows how stable
recurrence carries the component rates to the observed loss.  Finite-bulk
results, whose spectral scale depends explicitly on width, are kept with the
corresponding exact finite-width and deterministic-equivalent risk results in
\cref{app:noisy_43_finite_bulk_empirical,subsec:app_td_lrs_joint_schedules}.

\subsection{Positive spectral representation of the deterministic equivalent}
\label{app:noisy_43_positive_de_measure_proof}

\begin{proof}[Proof of \cref{lem:noisy_43_de_positive_spectral_measure}]
The proof has two main steps: a Herglotz representation gives positivity and
support, while the expansion at infinity gives the required moments.  For
\(1\leq j\leq d\), set
\(r_{j,m}(z):=1/(\lambda_j\mathfrak m(z)-z)\).
The selected branch places this function in the positive Herglotz class.  By
\citet[Proposition~E.1]{paquette20244+},
\(\operatorname{Im}\mathfrak m(z)<0\) on the upper half-plane, hence
\(\operatorname{Im}(\lambda_j\mathfrak m(z)-z)<0\) and
\(\operatorname{Im}r_{j,m}(z)>0\).
Together with \(r_{j,m}(z)=-z^{-1}+O(|z|^{-2})\), the
Herglotz--Stieltjes representation gives a unique probability measure
\(\mu_{j,m}\) with
\[
r_{j,m}(z)
=
\int_{\mathbb{R}}\frac{\mu_{j,m}(\mathrm d\lambda)}{\lambda-z}.
\]

We next verify that this measure is supported on the nonnegative real axis and
has compact support.  For \(z=-s<0\), the fixed-point equation is equivalent to
\[
1
=
u+\frac1m\sum_{j=1}^d
\frac{\lambda_j u}{\lambda_j u+s}.
\]
The right-hand side increases strictly from zero past one on \((0,1)\), so it
has a unique solution \(u=\mathfrak m(-s)\in(0,1)\).  Its positive
\(u\)-derivative and the implicit-function theorem continue the branch across
the negative axis, excluding mass on \((-\infty,0)\); the far-field estimate
of \citet[Proposition~E.2]{paquette20244+} gives compact support.

Collecting the scalar measures now gives
\(\boldsymbol\mu_m(A):=\operatorname{diag}
\bigl(\mu_{1,m}(A),\ldots,\mu_{d,m}(A)\bigr)\).
This proves \cref{eq:noisy_43_de_matrix_measure} and positivity.  The
expansion at infinity,
\(\mathfrak m(z)=1+\operatorname{tr}(\bm\Lambda)/(mz)+O(|z|^{-2})\),
identifies the stated moment identities; in particular,
\[
r_{j,m}(z)
=
-\frac1z-\frac{\lambda_j}{z^2}
-\frac{\lambda_j^2+\lambda_j\operatorname{tr}(\bm\Lambda)/m}{z^3}
+O(|z|^{-4}).
\]
Comparing coefficients with the Stieltjes expansion gives
\[
\mu_{j,m}([0,\infty))=1,
\qquad
\int\lambda\,\mu_{j,m}(\mathrm d\lambda)=\lambda_j,
\qquad
\int\lambda^2\,\mu_{j,m}(\mathrm d\lambda)
=\lambda_j^2+\frac{\operatorname{tr}\bm\Lambda}{m}\lambda_j.
\]
This proves \cref{eq:noisy_43_de_matrix_measure_moments} and
\cref{eq:noisy_43_de_scalar_measure_moments}.  It remains to identify the zero
atom and the contour calculus.  The near-zero estimate of
\citet[Appendix~E]{paquette20244+} makes zero an isolated at-most-simple pole,
so
\[
-\operatorname*{Res}_{z=0}
\left\langle
\bm M_m(z)\bm\Lambda^{1/2}\bm\theta^\star,
\bm\Lambda^{1/2}\bm\theta^\star
\right\rangle
=
\mu_m^{\mathcal F}(\{0\}).
\]
Finally, substituting the measure representation into the contour integral and
applying Cauchy's formula componentwise gives
\cref{eq:noisy_43_de_functional_calculus}.
\end{proof}

\subsection{Uniform Tauberian transfer for moving spectra}
\label{app:noisy_43_uniform_tauberian}

\begin{theorem}[Uniform Tauberian equivalence for \(m\)-dependent spectra]
\label{thm:noisy_43_uniform_tauberian}
Let \(A_m:[0,\infty)\to[0,\infty)\) be bounded, nondecreasing, and
right-continuous, with \(A_m(0)=0\), and consider any joint sequence with
\(m\to\infty\) and \(T=T_m\to\infty\).
Fix \(\tau>0\), let \(L\) be slowly varying at infinity, and set
\(a_m:=T^{-\tau}L(T)\).
Assume the following two conditions.
\begin{enumerate}[label=(\alph*)]
\item \emph{Uniform spectral control.}
For some \(\epsilon\in(0,\tau)\) and \(C<\infty\), eventually
\begin{equation}
\label{eq:noisy_43_uniform_potter_envelope}
\frac{A_m(y/T)}{a_m}
\leq
C\max\{y^{\tau-\epsilon},y^{\tau+\epsilon}\},
\qquad y>0.
\end{equation}

\item \emph{Compatible training filter.}
Let \(h_m:(0,\infty)\to[0,\infty)\) satisfy, for every \(A<\infty\),
\begin{equation}
\label{eq:noisy_43_uniform_filter_local_identity}
\sup_{0<y\leq A}
\left|T h_m(y/T)-y\right|
\longrightarrow0,
\end{equation}
and at every point where \(A_m\) increases,
\begin{equation}
\label{eq:noisy_43_uniform_filter_lower_bound}
h_m(\lambda)\geq c_0\lambda
\end{equation}
for a constant \(c_0>0\) independent of \(m\).
\end{enumerate}
Then the following statements are equivalent:

\begin{enumerate}[label=(\roman*)]
\item For every compact \(K\subset(0,\infty)\),
\begin{equation}
\label{eq:noisy_43_uniform_tauberian_cutoff}
\sup_{c\in K}
\left|
\frac{A_m(c/T)}{a_m}-c^\tau
\right|
\longrightarrow0.
\end{equation}

\item For every compact \(K\subset(0,\infty)\),
\begin{equation}
\label{eq:noisy_43_uniform_tauberian_transform}
\sup_{c\in K}
\left|
\frac{\displaystyle
\int_0^\infty
e^{-2cT h_m(\lambda)}\,\mathrm dA_m(\lambda)}
{a_m}
-
\Gamma(\tau+1)(2c)^{-\tau}
\right|
\longrightarrow0.
\end{equation}
\end{enumerate}
\end{theorem}

\begin{proof}
Set \(B_m(y):=A_m(y/T)/a_m\).  Then
\begin{equation}
\label{eq:noisy_43_uniform_scaled_envelope}
B_m(y)
\leq
C\max\{y^{\tau-\epsilon},y^{\tau+\epsilon}\},
\qquad y>0.
\end{equation}

Assume first \cref{eq:noisy_43_uniform_tauberian_cutoff}.  Then
\(B_m(y)\to y^\tau\) locally uniformly on \((0,\infty)\).
Uniformly for \(c\) in a fixed compact subset of \((0,\infty)\),
\cref{eq:noisy_43_uniform_filter_local_identity} gives
\(e^{-2cT h_m(y/T)}\to e^{-2cy}\)
on compact \(y\)-intervals.  The \((0,\delta]\) contribution is uniformly
\(O(\delta^{\tau-\epsilon})\) by
\cref{eq:noisy_43_uniform_scaled_envelope}; moreover,
\cref{eq:noisy_43_uniform_filter_lower_bound} and integration by parts give,
uniformly for \(c\geq c_->0\),
\begin{align*}
\int_{(A,\infty)}
e^{-2cT h_m(y/T)}\,\mathrm dB_m(y)
&\leq
\int_{(A,\infty)}
e^{-2c_-c_0y}\,\mathrm dB_m(y)\\
&\leq
2c_-c_0\int_A^\infty
e^{-2c_-c_0y}B_m(y)\,\mathrm dy,
\end{align*}
which vanishes uniformly in \(m\) as \(A\to\infty\).  Passing to the limit on
\([\delta,A]\), then sending \(\delta\downarrow0\) and \(A\to\infty\), gives
locally uniformly in \(c>0\)
\[
\frac{\int_0^\infty e^{-2cT h_m(\lambda)}\,\mathrm dA_m(\lambda)}{a_m}
\longrightarrow
\int_0^\infty e^{-2cy}\tau y^{\tau-1}\,\mathrm dy
=
\Gamma(\tau+1)(2c)^{-\tau},
\]
which proves \cref{eq:noisy_43_uniform_tauberian_transform}.

Conversely, assume \cref{eq:noisy_43_uniform_tauberian_transform}.  The
envelope and Helly's theorem give, from every subsequence, one on which
\(B_m\) converges at continuity points to a nondecreasing \(B_\star\).  The
same endpoint estimates and \cref{eq:noisy_43_uniform_filter_local_identity}
give, for every \(c>0\),
\[
\int_0^\infty e^{-2cy}\,\mathrm dB_\star(y)
=
\Gamma(\tau+1)(2c)^{-\tau}.
\]
Uniqueness of Laplace--Stieltjes transforms gives \(B_\star(y)=y^\tau\), so
\(B_m(c)\to c^\tau\) for \(c>0\).
Monotonicity and continuity of \(c^\tau\) make convergence uniform on compact
subsets, proving \cref{eq:noisy_43_uniform_tauberian_cutoff}.
\end{proof}

\begin{remark}
\label{rem:noisy_43_uniform_tauberian_envelope}
The function \(A_m(x)\) is the cumulative weighted spectral mass below \(x\).
Its monotonicity and right-continuity mean that it represents a finite positive
measure, so contributions from different spectral modes cannot cancel.  The
joint limit examines the moving scale \(\lambda\asymp T^{-1}\), and
\(a_m=T^{-\tau}L(T)\) is the expected amount of mass at that scale.

The \emph{Uniform spectral control} assumption ensures that mass outside this
moving window cannot remain hidden from the cutoff asymptotic while still affecting
the training response.
For example, a group of spectral mass may be located at
\(\lambda_m=y_m/T\), where \(y_m\to\infty\).
It eventually lies above every fixed cutoff \(c/T\).  The envelope prevents
such moving mass from becoming large enough to change the filtered response.

The \emph{Compatible training filter} assumption treats \(h_m\) as the
effective decay rate of a spectral mode.  Its local condition makes the
training filter behave like \(e^{-2cy}\) for \(\lambda=y/T\), while its lower
bound suppresses faster modes away from the cutoff.  Together, these
assumptions ensure that the long-time response is determined by spectral mass
near \(T^{-1}\).
\end{remark}

\subsection{Spectral criteria for forcing and memory}
\label{app:noisy_43_component_criteria_proofs}

This subsection proves the forcing and memory criteria stated in
\cref{thm:noisy_43_effective_spectral_criterion,thm:noisy_43_effective_memory_criterion}.
After specifying their common constant-schedule limit, we derive both by applying
\cref{thm:noisy_43_uniform_tauberian} to the corresponding weighted spectral
measure.

The proof has two parallel applications of
\cref{thm:noisy_43_uniform_tauberian}: one to learnable forcing and the other
to one-injection memory.  Both use the same constant-schedule limit and
effective training filter.  Apply
\cref{ass:noisy_43_uniform_stable_spectral_window} with the component's
weighted spectral mass and normalization.  Part \emph{(a)} of
\cref{ass:noisy_43_constant_schedule_stability} verifies the
\emph{Compatible training filter} condition below, while the spectral-window
envelope is the \emph{Uniform spectral control} condition in
\cref{eq:noisy_43_uniform_potter_envelope}.  The two applications differ only
in their weighted spectral measure and normalization.

\begin{proof}[Proof of
\cref{thm:noisy_43_effective_spectral_criterion}]

\noindent\textbf{Learnable forcing.}
\label{app:noisy_43_effective_spectral_criterion_proof}

For the conditioned empirical spectrum, set
\(A_m(x):=\nu_{\bm W}^{\mathcal F}((0,x])\).  For every positive empirical
eigenvalue, set \(h_\eta(\lambda):=-\log q_\eta(\lambda)/(2\eta)\).
With \(x=\eta\lambda\), stability gives
\[
1-q_\eta(\lambda)
=x\left[2-\left(1+\frac1B\right)x\right]\geq\delta x.
\]
Since \(-\log q\geq1-q\),
\(h_\eta(\lambda)\geq\delta\lambda/2\) on the positive support.  Since
\(T=\eta t\), uniformly for
\(0<y\leq A\),
\[
T h_\eta(y/T)
=-
\frac{t}{2}
\log\!\left[
1-\frac{2y}{t}
+\left(1+\frac1B\right)\frac{y^2}{t^2}
\right]
=y+O\!\left(\frac{A^2}{t}\right).
\]
Hence \(h_\eta\) satisfies
\cref{eq:noisy_43_uniform_filter_local_identity,eq:noisy_43_uniform_filter_lower_bound}.
With \(T=\eta t\),
\[
F_{\bm W,>0}(t)
=
\sum_{\widehat\lambda_j>0}
\left|\left\langle
\widehat{\bm u}_j,\bm\Lambda^{1/2}\bm\theta^\star
\right\rangle\right|^2
e^{-2T h_\eta(\widehat\lambda_j)}
=
\int_0^\infty e^{-2T h_\eta(\lambda)}\,\mathrm dA_m(\lambda).
\]
Apply \cref{thm:noisy_43_uniform_tauberian} with
\(a_m=T^{-q_{\mathcal F}}L_{\mathcal F}(T)\) to obtain the stated cutoff and
forcing asymptotics, uniformly under constant-factor changes of \(T\).

\medskip\noindent\textbf{One-injection memory.}
\label{app:noisy_43_effective_memory_criterion_proof}

For the conditioned empirical spectrum, set
\(A_m(x):=\nu_{\bm W}^{\mathcal K}((0,x])\).  With the same map \(h_\eta\)
used in the forcing proof, the exact conditional kernel becomes
\[
\frac{B}{\eta^2}K_{\bm W}(t)
=
\int_0^\infty e^{-2T h_\eta(\lambda)}\,\mathrm dA_m(\lambda).
\]
The filter verification in the forcing proof is unchanged.  Apply
\cref{thm:noisy_43_uniform_tauberian} with
\[
a_m=\frac{2^{q_{\mathcal K}}T^{-q_{\mathcal K}}L_{\mathcal K}(T)}
{\Gamma(q_{\mathcal K}+1)}
\]
to obtain the claimed equivalence and constant-factor uniformity.
\end{proof}

\subsection{Accumulating memory over training}
\label{app:noisy_43_aggregate_noise_criterion_proof}

\begin{proof}[Proof of
\cref{thm:noisy_43_aggregate_noise_criterion}]
For the same constant schedule and every \(\lambda>0\),
\(1-q_\eta(\lambda)=\eta\lambda
[2-(1+1/B)\eta\lambda]\).
The quadratic identity
\[
q_\eta(\lambda)
=
\frac1{B+1}
+\left(1+\frac1B\right)
\left(\eta\lambda-\frac{B}{B+1}\right)^2
\]
and the stability margin show, on every positive empirical eigenvalue, that
\(0<q_\eta(\lambda)<1\) and
\[
2-\left(1+\frac1B\right)\eta\lambda\geq\delta,
\qquad
q_\eta(\lambda)^t
\leq e^{-t[1-q_\eta(\lambda)]}
\leq e^{-\delta T\lambda}.
\]
Substitution into the empirical kernel sum and exact geometric summation give
the two identities in
\cref{eq:noisy_43_terminal_noise_exact}.

We first prove the cumulative-memory asymptotics directly from those identities.
Retain \(A_m^{\mathcal K},a_m^{\mathcal K}\), and
\(\widetilde A_m^{\mathcal K}\) from the theorem.  On the rescaled spectral
support define the proof-local filters
\[
\phi^{\rm acc}_{m,t}(y)
:=
\frac{1-q_\eta(y/T)^t}
{y[2-(1+1/B)y/t]},
\qquad
\phi^{\rm tail}_{m,t}(y)
:=
\frac{q_\eta(y/T)^t}
{y[2-(1+1/B)y/t]}.
\]
Because \(T=\eta t\), the exact identities become
\begin{equation}
\label{eq:noisy_43_rescaled_accumulation_integral}
\frac{S_{\bm W,t}}{(\eta/B)Ta_m^{\mathcal K}}
=\int_{(0,\infty)}\phi^{\rm acc}_{m,t}(y)\,
\mathrm d\widetilde A_m^{\mathcal K}(y),
\qquad
\frac{S_{\bm W,\infty}-S_{\bm W,t}}{(\eta/B)Ta_m^{\mathcal K}}
=\int_{(0,\infty)}\phi^{\rm tail}_{m,t}(y)\,
\mathrm d\widetilde A_m^{\mathcal K}(y).
\end{equation}

On every compact interval \(0<a\leq y\leq A<\infty\),
\[
q_\eta(y/T)^t
=
\left[
1-\frac{2y}{t}
+\left(1+\frac1B\right)\frac{y^2}{t^2}
\right]^t
\longrightarrow e^{-2y}
\]
uniformly.  Hence
\[
\phi^{\rm acc}_{m,t}(y)
\longrightarrow\frac{1-e^{-2y}}{2y},
\qquad
\phi^{\rm tail}_{m,t}(y)
\longrightarrow\frac{e^{-2y}}{2y}
\]
uniformly there.  The cutoff asymptotic
\cref{eq:noisy_43_aggregate_memory_cutoff} makes the rescaled Stieltjes
measures converge on the same compact interval to \(\mathrm d(y^{q_{\mathcal K}})\).

It remains to make this compact limit uniform at zero and infinity.  The
preceding stability bounds and
\(1-q^t\leq t(1-q)\) for \(0\leq q<1\) give
\begin{equation}
\label{eq:noisy_43_aggregate_filter_endpoint_bounds}
0\leq\phi^{\rm acc}_{m,t}(y)
\leq\min\left\{1,\frac1{\delta y}\right\},
\qquad
0\leq\phi^{\rm tail}_{m,t}(y)
\leq\frac{e^{-\delta y}}{\delta y}.
\end{equation}
For \(0<q_{\mathcal K}<1\), the endpoint envelope therefore gives
\[
\int_{(0,a]}\phi^{\rm acc}_{m,t}\,
\mathrm d\widetilde A_m^{\mathcal K}
\lesssim a^{q_{\mathcal K}-\epsilon}.
\]
A dyadic decomposition at infinity gives
\[
\int_{(A,\infty)}\phi^{\rm acc}_{m,t}\,
\mathrm d\widetilde A_m^{\mathcal K}
\lesssim
\sum_{j\geq0}
\frac{\widetilde A_m^{\mathcal K}(2^{j+1}A)}{2^jA}
\lesssim
A^{q_{\mathcal K}+\epsilon-1}
\sum_{j\geq0}2^{j(q_{\mathcal K}+\epsilon-1)},
\]
which vanishes as \(A\to\infty\) because
\(q_{\mathcal K}+\epsilon<1\).

For \(q_{\mathcal K}>1\), Stieltjes integration by parts at zero instead
gives
\[
\int_{(0,a]}\phi^{\rm tail}_{m,t}\,
\mathrm d\widetilde A_m^{\mathcal K}
\lesssim
\int_{(0,a]}y^{-1}\,
\mathrm d\widetilde A_m^{\mathcal K}(y)
\lesssim a^{q_{\mathcal K}-\epsilon-1},
\]
which vanishes because \(q_{\mathcal K}-\epsilon>1\).  At infinity, a
dyadic decomposition combines the second bound in
\cref{eq:noisy_43_aggregate_filter_endpoint_bounds} with the polynomial
endpoint envelope; its resulting exponentially damped series vanishes
uniformly as \(A\to\infty\).

We may now pass to the compact limit and then send \(a\downarrow0\) and
\(A\to\infty\).  The limiting integrals are
\begin{align*}
\int_0^\infty\frac{1-e^{-2y}}{2y}\,
\mathrm d(y^{q_{\mathcal K}})
&=
\frac{\Gamma(q_{\mathcal K}+1)}
{2^{q_{\mathcal K}}(1-q_{\mathcal K})},
&&0<q_{\mathcal K}<1,\\
\int_0^\infty\frac{e^{-2y}}{2y}\,
\mathrm d(y^{q_{\mathcal K}})
&=
\frac{\Gamma(q_{\mathcal K}+1)}
{2^{q_{\mathcal K}}(q_{\mathcal K}-1)},
&&q_{\mathcal K}>1.
\end{align*}
Substituting the definition of \(a_m^{\mathcal K}\) into
\cref{eq:noisy_43_rescaled_accumulation_integral}
proves
\cref{eq:noisy_43_aggregate_noise_from_kernel,eq:noisy_43_terminal_noise_from_kernel}
with their displayed \(\eta/B\) constants.  The uniform convergence theorem
for slowly varying functions makes this rescaling explicit: for \(c\) in a
compact subset of \((0,\infty)\), put
\[
\quad t_c:=\lfloor ct\rfloor,
\qquad
T_c:=\eta t_c,
\qquad
a_{m,c}^{\mathcal K}
:=\frac{2^{q_{\mathcal K}}}{\Gamma(q_{\mathcal K}+1)}
T_c^{-q_{\mathcal K}}L_{\mathcal K}(T_c).
\]
Then \(T_c/T\to c\), and
\(a_{m,c}^{\mathcal K}/a_m^{\mathcal K}\to c^{-q_{\mathcal K}}\)
uniformly over such \(c\).  Repeating the preceding rescaling with
\(y=T_c\lambda\) gives the same compact limits and endpoint bounds uniformly
in \(c\).  This proves the stated constant-factor uniformity.

For the direct terminal-tail criterion, use the separate cumulative measure
\(A_m^{\rm tail}\) defined in the theorem.  Its transform under the forcing
proof's \(h_\eta\) is exactly
\[
S_{\bm W,\infty}-S_{\bm W,t}
=
\int_0^\infty e^{-2Th_\eta(\lambda)}
\,\mathrm dA_m^{\rm tail}(\lambda).
\]
The denominator in \(A_m^{\rm tail}\) is at least \(\delta\), and the
filter verification in the forcing proof is unchanged.  Thus
\cref{thm:noisy_43_uniform_tauberian}, with
\(a_m=T^{-q_{\mathcal S}}L_{\mathcal S}(T)\), gives both locally uniform
families.  Apply its cutoff family at \(T/T_c\) and its transform family at
\(T_c/T\).  Since \(T_c/T\to c\) and
\(L_{\mathcal S}(T_c)/L_{\mathcal S}(T)\to1\) locally uniformly in \(c\),
this is exactly the displayed direct equivalence.

Finally, replacing the empirical squared-spectrum measure by
\(\nu_m^{\mathcal K}(\mathrm d\lambda)\) gives the DE cumulative-memory
argument, while replacing \(A_m^{\rm tail}\) by
\(A_{m,{\rm DE}}^{\rm tail}\) gives the DE direct-tail argument.  Every
step uses the corresponding cutoff, endpoint-envelope, and stability
hypotheses stated after the theorem.
\end{proof}

\subsection{From forcing and memory asymptotics to observed loss under stable feedback}
\label{app:noisy_43_aggregate_volterra_transfer_proof}

\begin{proof}[Proof of \cref{thm:noisy_43_aggregate_volterra_transfer}]
Stability makes the nonnegative repeated-feedback series summable:
\[
\sum_{n\geq1}\|\mathsf K^{(*n)}\|_{\ell^1}
=\sum_{n\geq1}\kappa_m^n
=\frac{\kappa_m}{1-\kappa_m}
\leq\frac{\kappa}{1-\kappa}.
\]
The clean solution is
\(\mathsf F+\sum_{n\geq1}\mathsf K^{(*n)}*\mathsf F\).
Its constant-term response lies in
\([\mathsf F_0,\mathsf F_0/(1-\kappa)]\).

In case (i), split at a fixed fraction of \(t\): the long-tail assumptions
control bounded lags, summability of the feedback kernels controls the middle
part, and cumulative forcing controls the opposite endpoint.  Thus
\[
\mathsf F_{>0}(t,m)
+\sum_{n\geq1}(\mathsf K^{(*n)}*\mathsf F_{>0})(t,m)
\asymp\mathsf F_{>0}(t,m).
\]
In case (ii), the two uniform subexponential relations applied successively
to the convolution powers give
\[
(\mathsf K^{(*n)}*\mathsf F_{>0})(t,m)
\asymp
\kappa_m^n\mathsf F_{>0}(t,m)
+n\kappa_m^{n-1}
\left[\sum_{s\geq0}\mathsf F_{>0}(s,m)\right]\mathsf K(t,m).
\]
The assumed summable domination permits summation over \(n\), yielding
\[
\mathsf F_{>0}(t,m)
+\sum_{n\geq1}(\mathsf K^{(*n)}*\mathsf F_{>0})(t,m)
\asymp
\mathsf F_{>0}(t,m)
+\left[\sum_{s\geq0}\mathsf F_{>0}(s,m)\right]\mathsf K(t,m).
\]
This proves the clean comparison in
\cref{eq:noisy_43_aggregate_loss_transfer}; positivity of the common Volterra
system adds \(\sigma^2\mathsf S_t\), proving its noisy comparison.

If \(\mathsf F_{>0}(t,m)\to0\), taking the limit in the Volterra equation
and using \(\lim_{t\to\infty}\mathsf S_t(m)=\kappa_m\) gives
\cref{eq:noisy_43_terminal_risk_exact}.  Put
\(d_t:=P_m-\mathsf R_\sigma(t,m)\).
Subtracting the finite-time equation from its terminal fixed point gives
\[
d_t
=
-\mathsf F_{>0}(t,m)
+(\mathsf K*d)(t,m)
+\left(P_m+\sigma^2\right)
\overline{\mathsf K}_m(t).
\]
The stability bound makes the positive resolvent \(\mathsf G_m\) well
defined.  Iterating the last recursion gives the exact identity
\[
P_m-\mathsf R_\sigma(t,m)
=
-(\mathsf G_m\mathsf F_{>0})(t)
+(P_m+\sigma^2)
(\mathsf G_m\overline{\mathsf K}_m)(t),
\]
which is precisely
\cref{eq:noisy_43_centered_response_identity} after multiplication by
\(-1\).  Moreover,
\(P_m+\sigma^2=(\mathsf F_0(m)+\sigma^2)/(1-\kappa_m)\geq0\).
Together with positivity of \(\mathsf G_m\), this proves directly that the
two terms on the right-hand side of
\cref{eq:noisy_43_centered_response_identity} are nonnegative.

The forcing estimates already established above give, uniformly in the common
window,
\[
(\mathsf G_m\mathsf F_{>0})(t)\asymp\mathsf F_{>0}(t,m)
\qquad\text{in case \emph{(i)}},
\]
whereas in case \emph{(ii)},
\[
(\mathsf G_m\mathsf F_{>0})(t)
\asymp
\mathsf F_{>0}(t,m)
+\left[\sum_{s\geq0}\mathsf F_{>0}(s,m)\right]
\mathsf K(t,m).
\]
For the terminal response, positivity and
\cref{eq:noisy_43_terminal_tail_renewal_domination} give the explicit
two-sided comparison
\[
\overline{\mathsf K}_m(t)
\leq
(\mathsf G_m\overline{\mathsf K}_m)(t)
\leq
(1+C_{\rm tail})\overline{\mathsf K}_m(t).
\]
Consequently,
\((P_m+\sigma^2)(\mathsf G_m\overline{\mathsf K}_m)(t)
\asymp(P_m+\sigma^2)
[\kappa_m-\mathsf S_t(m)]\).
Finally, the exact response identity and
\cref{eq:noisy_43_centered_noncancellation} imply
that \(|\mathsf R_\sigma(t,m)-P_m|\) is comparable to the sum of the two
nonnegative terms in
\cref{eq:noisy_43_centered_response_identity}.
Substitution of the preceding positive-response estimates proves
\cref{eq:noisy_43_aggregate_centered_transfer} in case \emph{(i)} and its
stated additional renewal term in case \emph{(ii)}.
\end{proof}

\begingroup

\section{Joint learning-rate and batch-size schedules}
\label{sec:joint_schedules}
\label{app:time_dependent_lrs}

This appendix develops the schedule-dependent dynamics used in the main text
in five steps.  First,
\cref{subsec:td_lrs_joint_exact_dynamics,subsec:td_lrs_joint_de_recursion}
derive the exact conditional risk recursion for joint learning-rate and
batch-size schedules and its time-inhomogeneous PLRF deterministic equivalent.
Next,
\cref{subsec:td_lrs_common_schedule_setup,%
subsec:td_lrs_plrf_regimes_and_admissibility}
state the conditions under which intrinsic time controls how errors in each
spectral mode decay,
then define the PLRF regimes and admissible schedule families.  We then
establish the empirical finite-bulk spectrum, kernel, and exact noisy--clean
gap, followed by the deterministic-equivalent risk asymptotics in the
long-memory, integrable-memory, and finite-bulk regimes; see
\cref{app:noisy_43_finite_bulk_empirical,%
subsec:td_lrs_joint_de_ordinary,subsec:app_td_lrs_joint_schedules}.
\Cref{subsec:td_lrs_joint_empirical_transfer} states the additional comparison
conditions needed to transfer the DE total-risk orders to realized
finite-width SGD.  Finally,
\cref{subsec:td_lrs_ratio_control_integration} formulates the optimal
ratio-control problem, proves the matching data and feature-compute rates,
and constructs integer schedules that attain them.  Each proof is placed with
the result it establishes.

\subsection{Exact risk dynamics under joint schedules}
\label{subsec:td_lrs_joint_exact_dynamics}

This subsection derives the exact conditional risk dynamics for arbitrary
deterministic learning-rate and batch-size schedules.  The SGD update in
\cref{eq:td_lrs_joint_sgd_update} gives the modal recursion in
\cref{eq:td_lrs_joint_mode_recursion}, which sums to the closed risk equation
in \cref{eq:td_lrs_joint_volterra}.

Return to the deterministic schedule of \cref{sec:rv_joint_schedules}, with
integer \(B_t\geq1\) and \(\eta_t\geq0\).  Intrinsic time controls signal
learning and \(B_t/\eta_t\) controls noise per intrinsic-time increment,
although exact dynamics retain the individual steps:
\begin{equation}
\label{eq:td_lrs_joint_sgd_update}
\bm a_{t+1}
=
\bm a_t
-\frac{\eta_t}{B_t}\sum_{i=1}^{B_t}
\bm W^{\!\top}\bm x_t^i
\left(f_{\bm a_t}(\bm x_t^i)-y_t^i\right).
\end{equation}
Schedules are fixed before sampling.  With \((T_t,r_s)\) from
\cref{eq:joint_sgd_schedule_coordinates}, set
\(\Delta T_s:=T_{s+1}-T_s=\eta_s\)
with \(1/r_s:=0\) when \(\eta_s=0\), so step \(s\) has injection weight
\(\eta_s^2/B_s=\Delta T_s/r_s\).

To track the joint schedule exactly, let
\begin{equation*}
q_s(z)
:=
1-2\eta_sz+\left(1+\frac1{B_s}\right)\eta_s^2z^2,
\qquad
Q_{s,t}(z):=\prod_{u=s}^{t-1}q_u(z),
\qquad
Q_{t,t}(z):=1.
\end{equation*}
With
\(\widehat{\bm H}\widehat{\bm u}_j
=\widehat\lambda_j\widehat{\bm u}_j\),
\(\bm e_t:=\bm\Lambda^{1/2}(\bm W\bm a_t-\bm\theta^\star)\), and
\(\rho_j(t):=\langle\widehat{\bm u}_j,\bm e_t\rangle\), we can now derive
the risk equation directly.

Conditional on the frozen random features, the joint update gives
\[
\bm e_{t+1}
=
\bm e_t
-\frac{\eta_t}{B_t}\sum_{i=1}^{B_t}
\widehat{\bm H}\bm z_t^i
\left(\langle\bm z_t^i,\bm e_t\rangle-\varepsilon_t^i\right).
\]
For one fresh Gaussian covariate,
\[
\mathbb E[z_j\langle\bm z,\bm e_t\rangle]=\rho_j(t),
\qquad
\mathbb E[z_j^2\langle\bm z,\bm e_t\rangle^2]
=\|\bm e_t\|_2^2+2\rho_j^2(t).
\]
The \(B_t\) diagonal terms contribute
\(B_t\|\bm e_t\|_2^2+2B_t\rho_j^2(t)\), the
\(B_t(B_t-1)\) ordered cross terms contribute
\(B_t(B_t-1)\rho_j^2(t)\), and the independent label noises contribute
\(B_t\sigma^2\).  Their sum gives the exact modal recursion
\begin{equation}
\label{eq:td_lrs_joint_mode_recursion}
\mathbb E_t[\rho_j^2(t+1)]
=
q_t(\widehat\lambda_j)\rho_j^2(t)
+\frac{\eta_t^2}{B_t}\widehat\lambda_j^2
\left(\|\bm e_t\|_2^2+\sigma^2\right),
\end{equation}
and finite iteration gives
\[
\mathbb E[\rho_j^2(t)]
=Q_{0,t}(\widehat\lambda_j)\rho_j^2(0)
+\sum_{s<t}\frac{\eta_s^2}{B_s}\widehat\lambda_j^2
Q_{s+1,t}(\widehat\lambda_j)
\left(\mathbb E\|\bm e_s\|_2^2+\sigma^2\right).
\]
Define the propagated initial error and the weight of one variance injection by
\[
F_t:=\sum_j\rho_j^2(0)Q_{0,t}(\widehat\lambda_j),
\qquad
K_{t,s}:=\frac{\eta_s^2}{B_s}
\sum_j\widehat\lambda_j^2Q_{s+1,t}(\widehat\lambda_j).
\]
Summing the finite iterate over \(j\) yields the exact conditional risk
equation below; this finite-dimensional step needs no stability condition or
asymptotic spectral input:
\begin{equation}
\label{eq:td_lrs_joint_volterra}
R_{\sigma,t}
=
F_t
+\sum_{s=0}^{t-1}
K_{t,s}\left(R_{\sigma,s}+\sigma^2\right),
\qquad
R_{\sigma,t}:=\mathbb E\!\left[\|\bm e_t\|_2^2\mid\bm W\right].
\end{equation}
The fixed-batch recursion and its finite iterate require no second
calculation: set \(B_s\equiv B\) in the preceding formulas, so that
\(q_s=q_{\eta_s}\) and \(\eta_s^2/B_s=\eta_s^2/B\), and then sum over
\(j\).  Cauchy's functional calculus
with the exact resolvent \((\widehat{\bm H}-z\bm I_d)^{-1}\) gives the
corresponding empirical spectral formulas.

\subsection{Time-inhomogeneous PLRF deterministic equivalent}
\label{subsec:td_lrs_joint_de_recursion}

The positive representation in
\cref{lem:noisy_43_de_positive_spectral_measure} defines a deterministic
time-inhomogeneous surrogate whose exact mode-decay factors \(Q_{s,t}\) depend
on the full factorization \((\boldsymbol\eta,\boldsymbol B)\), not only on
\(T_t\).

For deterministic schedules
\(\boldsymbol\eta=(\eta_0,\eta_1,\ldots)\) and
\(\boldsymbol B=(B_0,B_1,\ldots)\),
with \(\eta_s\geq0\) and \(B_s\in\mathbb N\), retain \(q_s\) and \(Q_{s,t}\)
from the exact joint recursion and define
\begin{align}
\mathcal F_{\boldsymbol\eta,\boldsymbol B}(t,m)
&:=
\int_{[0,\infty)}
Q_{0,t}(\lambda)\,
\mu_m^{\mathcal F}(\mathrm d\lambda),
\label{eq:td_lrs_joint_de_forcing}\\
\mathcal K_{\boldsymbol\eta,\boldsymbol B}(t,s,m)
&:=
\frac{\eta_s^2}{B_s}
\int_{[0,\infty)}
\lambda^2Q_{s+1,t}(\lambda)\,
\mu_m^{\mathcal K}(\mathrm d\lambda),
\qquad 0\leq s<t.
\label{eq:td_lrs_joint_de_kernel}
\end{align}
The corresponding time-inhomogeneous PLRF resolvent-DE risk is specified by
\begin{equation}
\label{eq:td_lrs_joint_de_volterra}
\mathcal R_{\sigma,\boldsymbol\eta,\boldsymbol B}(t,m)
=
\mathcal F_{\boldsymbol\eta,\boldsymbol B}(t,m)
+
\sum_{s=0}^{t-1}
\mathcal K_{\boldsymbol\eta,\boldsymbol B}(t,s,m)
\left[
\mathcal R_{\sigma,\boldsymbol\eta,\boldsymbol B}(s,m)+\sigma^2
\right].
\end{equation}

\begin{proposition}[Joint-schedule PLRF resolvent-DE recursion]
\label{prop:td_lrs_joint_de_recursion}
For any deterministic \((\boldsymbol\eta,\boldsymbol B)\) with
\(\eta_s\geq0\) and \(B_s\in\mathbb N\), the DE recursion in
\cref{eq:td_lrs_joint_de_volterra} has a unique nonnegative solution at every
finite horizon.

Let \(\delta>0\).  If
\begin{equation}
\label{eq:td_lrs_joint_de_pointwise_stability}
\sup_{\substack{s\geq0\\
\lambda\in\operatorname{supp}\mu_m^{\mathcal K}}}
\eta_s\left(1+\frac1{B_s}\right)\lambda
\leq2-\delta,
\end{equation}
then
\begin{equation}
\label{eq:td_lrs_joint_de_row_mass_certificate}
\sup_{t\geq1}
\sum_{s<t}
\mathcal K_{\boldsymbol\eta,\boldsymbol B}(t,s,m)
\leq
\frac{\operatorname{tr}\bm\Lambda}{\delta}
\sup_s\frac{\eta_s}{B_s}.
\end{equation}
Consequently, if the right-hand side is at most \(\kappa<1\), then
\begin{equation}
\label{eq:td_lrs_joint_de_uniform_bound}
\sup_{t\geq0}
\mathcal R_{\sigma,\boldsymbol\eta,\boldsymbol B}(t,m)
\leq
\frac{
(\bm\theta^\star)^{\!\top}\bm\Lambda\bm\theta^\star
+\kappa\sigma^2
}{1-\kappa}.
\end{equation}
\end{proposition}

\begin{proof}
By \cref{eq:noisy_43_de_functional_calculus},
\cref{eq:td_lrs_joint_de_forcing,eq:td_lrs_joint_de_kernel} are the contour
formulas obtained by inserting \(Q_{0,t}\) and \(Q_{s+1,t}\), without
approximating growing-degree polynomials on the contour.

For every \(\lambda\geq0\),
\(q_s(\lambda)=(1-\eta_s\lambda)^2+\eta_s^2\lambda^2/B_s\geq0\).
Thus all mode-decay factors, forcing values, and kernel entries are nonnegative;
forward substitution gives the unique nonnegative finite-horizon solution.

Under \cref{eq:td_lrs_joint_de_pointwise_stability},
\[
1-q_s(\lambda)
=
\eta_s\lambda
\left[2-\left(1+\frac1{B_s}\right)\eta_s\lambda\right]
\geq
\delta\eta_s\lambda.
\]
It also follows that \(0\leq q_s(\lambda)\leq1\).  Hence
\begin{align*}
\sum_{s<t}
\frac{\eta_s^2}{B_s}\lambda^2Q_{s+1,t}(\lambda)
&\leq
\frac1\delta
\left(\sup_u\frac{\eta_u}{B_u}\right)
\lambda
\sum_{s<t}[1-q_s(\lambda)]Q_{s+1,t}(\lambda)\\
&=
\frac1\delta
\left(\sup_u\frac{\eta_u}{B_u}\right)
\lambda[1-Q_{0,t}(\lambda)]\\
&\leq
\frac1\delta
\left(\sup_u\frac{\eta_u}{B_u}\right)\lambda.
\end{align*}
Integration against \(\mu_m^{\mathcal K}\), using its first moment, proves
\cref{eq:td_lrs_joint_de_row_mass_certificate}.  Stability also gives
\(\mathcal F_{\boldsymbol\eta,\boldsymbol B}(t,m)
\leq\mu_m^{\mathcal F}([0,\infty))
=(\bm\theta^\star)^{\!\top}\bm\Lambda\bm\theta^\star\).
Finite-horizon suprema in \cref{eq:td_lrs_joint_de_volterra} prove
\cref{eq:td_lrs_joint_de_uniform_bound}.
\end{proof}

Setting \(B_s\equiv B\) and then \(\eta_s\equiv\eta\) in
\cref{eq:td_lrs_joint_de_forcing,eq:td_lrs_joint_de_kernel,eq:td_lrs_joint_de_volterra}
gives the fixed-batch specialization and then the constant-schedule
convolution recursion in \cref{eq:noisy_43_clean_risk_volterra}.

\subsection{Stable schedules and spectral-mode decay}
\label{subsec:td_lrs_common_schedule_setup}

The exact finite-width and deterministic-equivalent risk results below require
control of each update and of accumulated feedback.  We first state these
conditions and prove common bounds describing how errors in each spectral mode
decay with elapsed intrinsic time.

\begin{assumption}[Joint-schedule stability]
\label{ass:td_lrs_joint_schedule_stability}
A deterministic schedule satisfies joint-schedule stability if there are
constants \(\delta>0\) and \(\kappa<1\), uniform over the width and horizon,
such that:
\begin{enumerate}[label=(\alph*)]
\item \emph{Pointwise contraction.}
\[
\sup_s\eta_s\left(1+\frac1{B_s}\right)
\lambda_{\max}(\widehat{\bm H})
\leq2-\delta.
\]

\item \emph{Row stability.}
\[
\sup_{t\geq1}\sum_{s<t}K_{t,s}\leq\kappa.
\]
\end{enumerate}
For the DE recursion, pointwise contraction is imposed uniformly over
\(\lambda\in\operatorname{supp}\mu_m^{\mathcal K}\), and row stability
replaces \(K_{t,s}\) by \(\mathcal K_{t,s}\).
\end{assumption}

Here a kernel row fixes the current time \(t\):
\(\sum_{s<t}K_{t,s}\) adds the weights of all past risks entering the risk
at \(t\).  Keeping every such total uniformly below one makes repeated
feedback a contraction on bounded risk paths, so its amplification remains
controlled.

\begin{lemma}[Empirical row-stability certificate]
\label{lem:td_lrs_joint_row_mass_certificate}
Under part \emph{(a)} of
\cref{ass:td_lrs_joint_schedule_stability},
\[
\sup_{t\geq1}\sum_{s<t}K_{t,s}
\leq
\frac{\operatorname{tr}(\widehat{\bm H})}{\delta}
\sup_s\frac{\eta_s}{B_s}.
\]
Consequently, row stability holds whenever
\begin{equation}
\label{eq:td_lrs_joint_row_mass_certificate}
\frac{\operatorname{tr}(\widehat{\bm H})}{\delta}
\sup_s\frac{\eta_s}{B_s}<1.
\end{equation}
\end{lemma}

\begin{proof}
For each empirical eigenvalue, the identity
\[
\frac{(\eta_s^2/B_s)\widehat\lambda_j^2}
{1-q_s(\widehat\lambda_j)}
=
\frac{\eta_s\widehat\lambda_j/B_s}
{2-(1+1/B_s)\eta_s\widehat\lambda_j}
\]
and pointwise stability bound the right-hand side by
\(\eta_s\widehat\lambda_j/(B_s\delta)\).  Summing in \(s\) and using the
telescoping identity for
\([1-q_s(\widehat\lambda_j)]Q_{s+1,t}(\widehat\lambda_j)\), then summing in
\(j\), proves the stated bound.
\end{proof}

\begin{assumption}[Regular spectral-mode decay]
\label{ass:td_lrs_regular_spectral_mode_decay}
A deterministic schedule satisfies the required one-step regularity if either:
\begin{enumerate}[label=(\alph*)]
\item \(\sup_{s,z\in\operatorname{spec}(\widehat{\bm H})}\eta_sz<1\);

\item for some \(B_{\max}<\infty\),
\(\sup_sB_s\leq B_{\max}\).
\end{enumerate}
For DE statements, condition \emph{(a)} takes the supremum over
\(\operatorname{supp}\mu_m^{\mathcal K}\) instead of
\(\operatorname{spec}(\widehat{\bm H})\).
\end{assumption}

\begin{lemma}[Intrinsic-time decay of individual spectral modes]
\label{lem:td_lrs_propagator}
Assume part \emph{(a)} of
\cref{ass:td_lrs_joint_schedule_stability} and
\cref{ass:td_lrs_regular_spectral_mode_decay}.  Then, for constants
\(0<c<C<\infty\),
\[
e^{-Cz(T_t-T_s)}\leq Q_{s,t}(z)\leq e^{-cz(T_t-T_s)}.
\]
If \(z(T_t-T_s)=O(1)\) and
\(\max_{s\leq u<t}\eta_u/(T_t-T_s)\to0\), then
\[
Q_{s,t}(z)=\exp\{-2z(T_t-T_s)+o(1)\}.
\]
\end{lemma}

\begin{proof}
The joint one-step factorization is
\(q_s(z)=(1-\eta_sz)^2+\eta_s^2z^2/B_s\), with
\(1-q_s(z)=\eta_sz[2-(1+1/B_s)\eta_sz]\).
Pointwise stability gives \(0\leq q_s(z)\leq1\).  Under condition \emph{(a)},
choose \(\rho<1\) such that \(\eta_sz\leq\rho\) uniformly; then
\(q_s(z)\geq(1-\rho)^2\).  A uniform batch bound also keeps the minimum of
this quadratic uniformly positive.  Thus \(-\log q_s(z)\) is
comparable to \(1-q_s(z)\), and summation proves the two exponential bounds.
Taylor expansion gives
\[
\log Q_{s,t}(z)
=-2z(T_t-T_s)+O\left(z^2\sum_{u=s}^{t-1}\eta_u^2\right),
\]
so the relative-mesh condition proves the critical-scale limit.
\end{proof}

\begin{remark}[Fixed-batch specialization]
For \(B_t\equiv B\), the uniform batch bound applies and
\(q_s=q_{\eta_s}\geq1/(B+1)\).  Thus both spectral-mode decay conclusions hold
without a separate joint log-regularity assumption.
\end{remark}

Thus elapsed intrinsic time determines spectral-mode decay in the fine-mesh
limit; at finite steps, the exact decay still depends on the individual steps.

\subsection{PLRF regimes and admissible schedule families}
\label{subsec:td_lrs_plrf_regimes_and_admissibility}

We now specialize the preceding stability and spectral-mode decay requirements to the
PLRF parameter regimes used below.

The eight open PLRF propagation subregimes below comprise three families:
finite bulk (FB), long memory (LM), and integrable memory (IM).  This taxonomy
is PLRF-specific.

\begin{paperdefinition}[Open PLRF propagation subregimes]
\label{def:td_lrs_plrf_open_subregimes}
Write \(p:=2\alpha+2\beta-1\) and assume \(p>0\).  Set
\(q_{\mathcal F}:=p/(2\alpha)\); when \(\alpha>1/4\), also write
\(q_{\mathcal K}:=2-1/(2\alpha)\).  The eight open subregimes are
\[
\begin{array}{c|l}
0<\alpha<\frac14
&\mathrm{FB}_1:q_{\mathcal F}<1,\qquad
 \mathrm{FB}_2:q_{\mathcal F}>1,\\[0.25em]
\frac14<\alpha<\frac12
&\mathrm{LM}_1:q_{\mathcal F}<q_{\mathcal K},\qquad
 \mathrm{LM}_2:q_{\mathcal K}<q_{\mathcal F}<1,\qquad
 \mathrm{LM}_3:q_{\mathcal F}>1,\\[0.25em]
\alpha>\frac12
&\mathrm{IM}_1:q_{\mathcal F}<1,\qquad
 \mathrm{IM}_2:1<q_{\mathcal F}<q_{\mathcal K},\qquad
 \mathrm{IM}_3:q_{\mathcal F}>q_{\mathcal K}.
\end{array}
\]
\end{paperdefinition}

Within subregime \(\mathrm{IM}_1\), the stability-limited hard branch is
\begin{equation*}
\alpha>\frac12,
\qquad
\frac12-\alpha<\beta<0,
\qquad
0<p<2\alpha-1.
\end{equation*}
The lower \(\beta\) bound is precisely finite target energy; feature
distortion occurs in \(\mathrm{IM}_{2,3}\).  We exclude the propagation
boundaries \(\alpha\in\{1/4,1/2\}\), \(\beta\in\{1/2,\alpha\}\), and
\(2\alpha+2\beta=1\).  The excluded optimizer crossover \(\beta=0\) splits
the two \(\mathrm{IM}_1\) branches, while \(\alpha=1\) splits interior and
source-saturated compute branches within \(\mathrm{IM}_{2,3}\); neither is a
propagation boundary.

\begin{paperdefinition}[DE-admissible integer schedule family]
\label{def:admissible_integer_schedules}
Fix a model, open subregime, \(0<C_0<\infty\), and small \(c_\star>0\).  A
deterministic, sample-independent family (\(\eta_s>0,B_s\in\mathbb N\)) is
DE-admissible along width--horizon sequences satisfying \(T_t\to\infty\) and
\(T_t\leq C_0m^{2\alpha}\) if it satisfies the DE versions of part
\emph{(a)} of \cref{ass:td_lrs_joint_schedule_stability} and
\cref{ass:td_lrs_regular_spectral_mode_decay} uniformly, its intrinsic mesh
is uniformly bounded, and
\(\sup_s r_s^{-1}\leq c_\star m^{2\alpha-1}\) or \(c_\star\) for
\(0<\alpha<1/2\), \(\alpha\neq1/4\), or \(\alpha>1/2\), respectively.
All class constants are budget-independent, with no finite upper ratio cap; see
\cref{lem:td_lrs_de_volterra_reduction,prop:td_lrs_joint_de_finite_bulk_kernel}.
\end{paperdefinition}

\subsection{Finite-bulk empirical spectrum, kernel, and gap}
\label{app:noisy_43_finite_bulk_empirical}

When \(\alpha<1/4\), the squared spectral mass diverges in the
infinite-spectrum limit, so a width-independent memory profile no longer
applies.  At finite width, the empirical spectrum is cut off at scale
\(m^{-2\alpha}\).  The next result locates this spectral band; the
following kernel and gap results show how it controls exact conditional SGD.

\begin{lemma}[Finite-bulk empirical spectral event]
\label{lem:noisy_43_finite_bulk_spectral_event}
Suppose \(\lambda_j=j^{-2\alpha}\), \(0<\alpha<1/4\), and
\(d/m\to c\in(1,\infty)\).  There are constants
\(a_-,a_+,b_0,C_{\rm tr}>0\) and events \(\mathcal E_m\), depending only on
\(\bm W\), with \(\mathbb P_{\bm W}(\mathcal E_m)\to1\), such that on
\(\mathcal E_m\),
\begin{equation}
\operatorname{tr}(\widehat{\bm H}^{\,2})
\leq C_{\rm tr}m^{1-4\alpha},
\qquad
\#\left\{j:
a_-m^{-2\alpha}\leq\widehat\lambda_j\leq a_+m^{-2\alpha}
\right\}\geq b_0m.
\label{eq:noisy_43_empirical_bulk_band}
\end{equation}
\end{lemma}

\begin{proof}
The \(m\) nonzero eigenvalues of \(\widehat{\bm H}\) equal those of
\(\bm W^{\!\top}\bm\Lambda\bm W\).  Choose
\(0<\varepsilon<\min\{1,c-1\}\), set \(J=\lfloor\varepsilon m\rfloor\), and
restrict \((\bm\Lambda,\bm W)\) to rows \(J<i\leq d\).  Then
\(\bm\Lambda_{\rm tail}\asymp m^{-2\alpha}\bm I\), while Gaussian
singular-value bounds give
\(\bm W_{\rm tail}^{\!\top}\bm W_{\rm tail}\asymp\bm I_m\) with probability
tending to one.  Hence the tail Gram matrix has all its eigenvalues in a fixed
multiple of \(m^{-2\alpha}\).  The positive semidefinite head has rank at most
\(J\), so rank interlacing leaves at least \(m-J\asymp m\) eigenvalues in the
band in \cref{eq:noisy_43_empirical_bulk_band}.

For the trace upper bound, separating equal and unequal Gaussian-column
indices gives
\[
\mathbb E\operatorname{tr}(\widehat{\bm H}^{\,2})
=\left(1+\frac1m\right)\operatorname{tr}(\bm\Lambda^2)
+\frac{(\operatorname{tr}\bm\Lambda)^2}{m}
\asymp m^{1-4\alpha}.
\]
Gaussian Poincar\'e applied to this trace, whose gradient with respect to
\(\bm W\) is
\(4\bm\Lambda\bm W(\bm W^{\!\top}\bm\Lambda\bm W)\), yields
\[
\operatorname{Var}\!\left[\operatorname{tr}(\widehat{\bm H}^{\,2})\right]
\leq
\frac{C}{m}\mathbb E\!\left[
\|\bm W\|_{\rm op}^4
\operatorname{tr}(\bm W^{\!\top}\bm\Lambda^2\bm W)
\right]
\leq C\frac{\operatorname{tr}(\bm\Lambda^2)}{m}
=o\!\left(m^{2-8\alpha}\right).
\]
Here we used Gaussian operator-norm moments, \(d\asymp m\), and
\(\operatorname{tr}(\bm\Lambda^2)\asymp m^{1-4\alpha}\).
Concentration and the spectral-band event give the claimed common event.
\end{proof}

\begin{lemma}[Finite-bulk kernel for subregimes \(\mathrm{FB}_1\) and \(\mathrm{FB}_2\)]
\label{lem:td_lrs_fb_bulk_kernel}
Suppose \(\lambda_j=j^{-2\alpha}\), \(0<\alpha<1/4\), and \(d/m\to c>1\).
Fix \(0<C_0<\infty\).  Up to the horizon \(t\), assume common margins in
part \emph{(a)} of \cref{ass:td_lrs_joint_schedule_stability} and in
\cref{ass:td_lrs_regular_spectral_mode_decay}, with the suprema in those
conditions restricted to \(s<t\).  On the events
\(\mathcal E_m\) in
\cref{lem:noisy_43_finite_bulk_spectral_event}, the following holds
simultaneously for every deterministic, sample-independent schedule satisfying
these margins and every \(0\leq s<t\) with \(T_t\leq C_0m^{2\alpha}\):
\[
K_{t,s}
\asymp
\frac{\Delta T_s}{r_s}m^{1-4\alpha},
\qquad
\sup_{u\leq t}\sum_{s<u}K_{u,s}
\lesssim m^{1-4\alpha}T_t\sup_{s<t}\frac1{r_s}.
\]
The first expression is zero at a zero step.  Under
\(B_s\equiv B\) and
\(\sup_{s<t}\eta_s\lesssim m^{2\alpha-1}\), these statements reduce to
\begin{equation}
\label{eq:td_lrs_fb_bulk_kernel}
K_{t,s}\asymp \frac{\eta_s^2}{B}m^{1-4\alpha},
\qquad
\sup_{u\leq t}\sum_{s<u}K_{u,s}
\lesssim m^{1-4\alpha}\left(\sup_{s<t}\frac{\eta_s}{B}\right)T_t=o(1),
\end{equation}
whenever \(T_t=o(m^{2\alpha})\).  The result is target-independent and
applies to both finite-bulk subregimes \(\mathrm{FB}_1\) and \(\mathrm{FB}_2\).
\end{lemma}

\begin{proof}
On \(\mathcal E_m\), \cref{eq:noisy_43_empirical_bulk_band} supplies the
trace upper bound and \(\Theta(m)\) eigenvalues with
\(\widehat\lambda\asymp m^{-2\alpha}\).  The upper bound uses
\(0\leq Q_{s+1,t}\leq1\).  Apply the spectral-mode decay lemma after extending
the finite-horizon schedule by zero steps.  On this bulk band, its lower bound
and \(T_t-T_{s+1}\leq C_0m^{2\alpha}\) give
\(Q_{s+1,t}(\widehat\lambda)\gtrsim1\), so these \(\Theta(m)\) modes give
the matching lower bound after multiplication by
\(\eta_s^2/B_s=\Delta T_s/r_s\).  Finally,
\[
\sum_{s<u}\frac{\Delta T_s}{r_s}
\leq T_u\sup_{s<t}\frac1{r_s},
\]
which proves the joint row-mass estimate.  Under the fixed-batch
substitution, \(\Delta T_s/r_s=\eta_s^2/B\) and
\[
\frac1B\sum_{s<u}\eta_s^2
\leq\left(\sup_{s<t}\frac{\eta_s}{B}\right)T_u.
\]
This proves \cref{eq:td_lrs_fb_bulk_kernel}; the
peak bound and strict window make the latter \(o(1)\).
\end{proof}

\begin{corollary}[Exact finite-bulk noisy--clean gap]
\label{cor:td_lrs_exact_finite_bulk_gap}
Retain the model, source window, and common schedule margins of
\cref{lem:td_lrs_fb_bulk_kernel}.  Fix \(\sigma^2>0\), a target with finite
population energy, and \(\kappa\in(0,1)\).  There is a sufficiently small
\(c_\star>0\) such that, on the same \(\bm W\)-only events, simultaneously for
every eligible schedule and finite horizon satisfying
\[
T_t\leq C_0m^{2\alpha},
\qquad
\sup_{s<t}\frac1{r_s}\leq c_\star m^{2\alpha-1},
\]
the exact conditional risks obey
\[
R_{\sigma,t}-R_{0,t}
\asymp
\sigma^2m^{1-4\alpha}\sum_{s<t}\frac{\eta_s^2}{B_s}
=
\sigma^2m^{1-4\alpha}\int_0^{T_t}\frac{\mathrm du}{r(u)}.
\]
The comparison constants are independent of width, horizon, schedule, target,
and \(\sigma^2\); they may depend on the fixed \(\alpha,c,C_0,\kappa\) and the
stability and propagation margins.  On the strict window
\(0<T_t=o(m^{2\alpha})\),
\[
R_{\sigma,t}-R_{0,t}
=
\sigma^2\sum_{s<t}K_{t,s}
\left[1+O\!\left(\frac{T_t}{m^{2\alpha}}\right)\right]
\]
uniformly over the same schedule class.
\end{corollary}

\begin{proof}
Fix one of the events in \cref{lem:td_lrs_fb_bulk_kernel} and write
\(Z_v:=R_{\sigma,v}-R_{0,v}\).  Its row bound gives, for
\(\kappa_t:=\max_{v\leq t}\sum_{s<v}K_{v,s}\),
\[
\kappa_t
\lesssim m^{1-4\alpha}T_t\sup_{s<t}\frac1{r_s}
\lesssim c_\star\frac{T_t}{m^{2\alpha}}
\leq Cc_\star C_0.
\]
Choose \(c_\star\) so that \(Cc_\star C_0\leq\kappa<1\).  Subtracting the
clean and noisy exact recursions in \cref{eq:td_lrs_joint_volterra} yields
\[
Z_v=\sum_{s<v}K_{v,s}(\sigma^2+Z_s).
\]
Positivity and forward induction give
\(0\leq Z_v\leq\sigma^2\kappa_t/(1-\kappa_t)\) for every \(v\leq t\).
Substitution at the terminal row therefore gives the pointwise sandwich
\[
\sigma^2\sum_{s<t}K_{t,s}
\leq Z_t\leq
\frac{\sigma^2}{1-\kappa_t}\sum_{s<t}K_{t,s}.
\]
The two-sided kernel bound in \cref{lem:td_lrs_fb_bulk_kernel}, together with
\(\eta_s^2/B_s=(T_{s+1}-T_s)/r_s\), proves the asserted order and integral
identity.  On the strict window,
\(\kappa_t=O(T_t/m^{2\alpha})\), so the same sandwich gives the relative
refinement.  All implications hold on the common \(\bm W\)-only event, proving
the stated simultaneous high-probability conclusion.
\end{proof}

The result holds for both finite-bulk target subregimes because the clean
forcing cancels from the gap.  It therefore requires only finite target energy,
not a canonical source profile.

\paragraph{A concrete schedule family.}
Fix \(B_0\in\mathbb N\), \(\vartheta\geq0\), and \(c_\eta>0\) with
\(c_\eta/B_0\leq c_\star\), and set
\[
\eta_s\equiv c_\eta m^{2\alpha-1},
\qquad
B_s=\left\lceil B_0(1+T_s)^\vartheta\right\rceil.
\]
On \(\mathcal E_m\),
\(\lambda_{\max}(\widehat{\bm H})\leq
\sqrt{C_{\rm tr}}m^{1/2-2\alpha}\), so the pointwise and logarithmic margins
hold for all sufficiently large \(m\).  The chosen \(c_\eta/B_0\) also gives
the required ratio peak.  Since
\(B_s\asymp B_0(1+T_s)^\vartheta\) and \(\eta_s\to0\), cellwise comparison
gives
\[
\sum_{s<t}\frac{\eta_s^2}{B_s}
\asymp
\frac{c_\eta m^{2\alpha-1}}{B_0}
\int_0^{T_t}(1+u)^{-\vartheta}\,\mathrm du.
\]
Consequently, uniformly as
\(T_t\to\infty\) with \(T_t\leq C_0m^{2\alpha}\),
\[
R_{\sigma,t}-R_{0,t}
\asymp \frac{\sigma^2c_\eta}{B_0}m^{-2\alpha}
\begin{cases}
T_t^{1-\vartheta},&0\leq\vartheta<1,\\
\log T_t,&\vartheta=1,\\
1,&\vartheta>1.
\end{cases}
\]

\paragraph{Boundary case.}
The case \(\alpha=1/4\) is excluded because
\(\operatorname{tr}(\bm\Lambda^2)\asymp\log m\), which introduces
logarithmic corrections to the trace and kernel estimates.

\subsection{Long-memory and integrable-memory deterministic-equivalent risk asymptotics}
\label{subsec:td_lrs_joint_de_ordinary}

\paragraph{DE forcing and memory filters in the LM and IM regimes.}
We now establish the spectral estimates for the PLRF deterministic
equivalent, beginning with forcing and then memory.  Recall
\(q_{\mathcal K}=2-1/(2\alpha)\), and put
\[
h_\alpha(u):=(1+u)^{-q_{\mathcal K}},
\qquad
A_{\alpha,m}:=
\begin{cases}
m^{1-2\alpha},&1/4<\alpha<1/2,\\
1,&\alpha>1/2.
\end{cases}
\]

For \(\alpha>1/4\), DE spectral measures and spectral-mode decay bounds give the
LM/IM memory weight \(h_\alpha\), with schedule dependence through
intrinsic time and \(\eta_s^2/B_s=\Delta T_s/r_s\).  The class-uniform DE risk bound
and explicit PLRF forcing asymptotic are
\cref{thm:td_lrs_joint_fsl,cor:td_lrs_joint_de_explicit_forcing}.

\begin{lemma}[Direct low-source PLRF DE forcing filter]
\label{lem:td_lrs_low_source_de_forcing}
Fix \(0<\alpha<1/2\), \(\beta<1/2\), and
\(p=2\alpha+2\beta-1>0\).  Suppose
\(d/m\in[c_-,c_+]\Subset(1,\infty)\).  For every fixed
\(0<c_0\leq c\leq c_1<\infty\) and \(0<C_0<\infty\),
\begin{equation}
\label{eq:td_lrs_low_source_de_forcing_filter}
\int e^{-cu\lambda}\,\mu_m^{\mathcal F}(\mathrm d\lambda)
\asymp
m^{-p}+(1+u)^{-p/(2\alpha)},
\qquad 0\leq u\leq C_0m^{2\alpha},
\end{equation}
uniformly in \(m,u,c\).
\end{lemma}

\begin{proof}
Write \(A_m(x):=\mu_m^{\mathcal F}([0,x])\),
\(u_x:=\mathfrak m(-x)\), and
\(s_m:=\operatorname{tr}(\bm\Lambda)/m\).
For \(x>0\), the fixed-point equation at \(-x\) and the Stieltjes
representation in \cref{lem:noisy_43_de_positive_spectral_measure} give
\begin{equation*}
1=u_x+\frac1m\sum_{j=1}^d
\frac{\lambda_ju_x}{\lambda_ju_x+x},
\qquad
\frac1{\lambda_ju_x+x}
=\int\frac{\mu_{j,m}(\mathrm d\lambda)}{\lambda+x}.
\end{equation*}
Since \(s_m\lesssim m^{-2\alpha}\), the first identity implies
\(u_x\geq c>0\) whenever \(x\geq m^{-2\alpha}\).  The second identity then
gives
\[
\mu_{j,m}([0,x])
\leq \frac{2x}{\lambda_ju_x+x}
\lesssim \min\!\left\{1,\frac{x}{\lambda_j}\right\}.
\]
For \(m^{-2\alpha}\leq x\leq1\), put
\(J=x^{-1/(2\alpha)}\leq m<d\).  Because
\(\lambda_j(\theta_j^\star)^2=j^{-2\alpha-2\beta}\), splitting at \(J\)
yields
\begin{align*}
A_m(x)
&\lesssim x\sum_{j\leq J}j^{-2\beta}
+\sum_{j>J}j^{-2\alpha-2\beta}
\lesssim x^{p/(2\alpha)}.
\end{align*}
For \(x<m^{-2\alpha}\), monotonicity gives \(A_m(x)\lesssim m^{-p}\);
for \(x>1\),
\(A_m(x)\leq\mu_m^{\mathcal F}([0,\infty))\asymp1
\leq x^{p/(2\alpha)}\).  Thus, for every \(x>0\),
\begin{equation}
\label{eq:td_lrs_low_source_tail_upper}
A_m(x)\lesssim m^{-p}+x^{p/(2\alpha)}.
\end{equation}

We next obtain a matching lower mass on each population scale.  The moment
identities in \cref{eq:noisy_43_de_matrix_measure_moments} imply
\[
\mathbb E X_{j,m}=\lambda_j,
\qquad
\mathbb E X_{j,m}^2=\lambda_j^2+s_m\lambda_j,
\]
for \(X_{j,m}\sim\mu_{j,m}\).  Since
\(s_m/\lambda_j\leq s_m/\lambda_d\lesssim1\), Paley--Zygmund followed by
Markov gives constants \(0<b_-<b_+<\infty\) and \(b_0>0\), independent of
\(j,m\), such that
\begin{equation}
\label{eq:td_lrs_coordinate_scale_band}
\mu_{j,m}([b_-\lambda_j,b_+\lambda_j])\geq b_0,
\qquad 1\leq j\leq d.
\end{equation}

Let \(1\leq u\leq C_0m^{2\alpha}\) and
\(J_u:=\min\{d,\lfloor u^{1/(2\alpha)}\rfloor\}\).  Summing
\cref{eq:td_lrs_coordinate_scale_band} over
\(\lceil J_u/2\rceil\leq j\leq J_u\) gives, uniformly for
\(c\in[c_0,c_1]\),
\[
\int e^{-cu\lambda}\,\mu_m^{\mathcal F}(\mathrm d\lambda)
\gtrsim J_u^{-p}.
\]
If \(J_u<d\), then \(J_u^{-p}\asymp u^{-p/(2\alpha)}\) and
\(m^{-p}\lesssim u^{-p/(2\alpha)}\); if \(J_u=d\), then
\(J_u^{-p}\asymp m^{-p}\) and
\(u^{-p/(2\alpha)}\lesssim m^{-p}\).  This proves the required lower bound
for \(u\geq1\).  For \(0\leq u\leq1\), the upper bound follows from the
total target mass, while the \(j=1\) instance of
\cref{eq:td_lrs_coordinate_scale_band} gives a uniform positive lower bound.

For the remaining upper bound with \(u\geq1\), decompose the positive axis
into \([0,u^{-1}]\) and
\((2^ku^{-1},2^{k+1}u^{-1}]\), \(k\geq0\).  By
\cref{eq:td_lrs_low_source_tail_upper},
\[
\int e^{-cu\lambda}\,\mu_m^{\mathcal F}(\mathrm d\lambda)
\leq A_m(u^{-1})
+\sum_{k\geq0}e^{-c_0 2^k}A_m(2^{k+1}u^{-1})
\lesssim m^{-p}+u^{-p/(2\alpha)}.
\]
Combining the two bounds proves \cref{eq:td_lrs_low_source_de_forcing_filter},
including the cutoff endpoint; its constants may depend on \(C_0\) and the
fixed aspect-ratio interval.
\end{proof}

For \(p=2\alpha+2\beta-1\), define the appendix forcing template
\begin{equation}
\label{eq:td_lrs_explicit_ordinary_forcing}
\Phi_{\alpha,\beta}(u,m):=
\begin{cases}
m^{-p}+(1+u)^{-p/(2\alpha)},
&\mathrm{FB}_1,\mathrm{LM}_1,\mathrm{LM}_2,\mathrm{IM}_1,\\[0.3em]
m^{-2\alpha}+(1+u)^{-p/(2\alpha)},
&\mathrm{FB}_2,\mathrm{LM}_3,\\[0.3em]
m^{-2\alpha}+(1+u)^{-p/(2\alpha)}
+m^{-1}(1+u)^{-1+1/(2\alpha)},
&\mathrm{IM}_2,\mathrm{IM}_3.
\end{cases}
\end{equation}
The main text records the large-clock order of this bounded template as
\(\mathcal F(T,m)\); the uniform statements use the \((1+u)\)-regularized
representative above.

\begin{lemma}[Uniform PLRF DE forcing filters in the proved source range]
\label{lem:td_lrs_de_forcing_filters}
Fix \(\alpha>0\), \(\alpha\notin\{1/4,1/2\}\), an admissible PLRF
dimension sequence, and \(0<C_0<\infty\).  Suppose
\(p=2\alpha+2\beta-1>0\), \(\beta<1+2\alpha\), and
\((\alpha,\beta)\) lies off the PLRF critical lines.  For every fixed
\(0<c_-\leq c\leq c_+<\infty\),
\begin{equation}
\label{eq:td_lrs_de_forcing_laplace_filter}
\int e^{-cu\lambda}\,\mu_m^{\mathcal F}(\mathrm d\lambda)
\asymp
\Phi_{\alpha,\beta}(u,m),
\qquad 0\leq u\leq C_0m^{2\alpha},
\end{equation}
The constants may depend on the fixed branch, \(C_0,c_-,c_+\), and the
admissible dimension limit, but not on \(m\) or \(u\).
\end{lemma}

\begin{lemma}[Uniform DE memory filters in the LM and IM regimes]
\label{lem:td_lrs_rough_spectral_transfer}
\label{lem:td_lrs_de_spectral_filters}
Fix \(\alpha>1/4\), \(\alpha\neq1/2\), an admissible PLRF dimension
sequence, and \(0<C_0<\infty\).  For every fixed
\(0<c_-\leq c\leq c_+<\infty\),
\begin{equation}
\label{eq:td_lrs_de_kernel_laplace_filter}
\int\lambda^2e^{-cu\lambda}\,
\mu_m^{\mathcal K}(\mathrm d\lambda)
\asymp h_\alpha(u),
\qquad 0\leq u\leq C_0m^{2\alpha},
\end{equation}
uniformly in \(m,u,c\).  Moreover,
\begin{equation}
\label{eq:td_lrs_de_kernel_l1_scale}
\int_0^{C_0m^{2\alpha}}h_\alpha(u)\,\mathrm du
\asymp A_{\alpha,m}.
\end{equation}
The constants may depend on the fixed branch, \(C_0,c_-,c_+\), and the
admissible dimension limit, but not on \(m\) or \(u\).
\end{lemma}

\begin{proof}[Proof of \cref{lem:td_lrs_de_forcing_filters,lem:td_lrs_de_spectral_filters}]
\emph{Step 1: a constant-schedule DE input.}
Retain the weighted memory measure \(\nu_m^{\mathcal K}\) from
\cref{eq:noisy_43_de_weighted_memory_measure}.  This memory measure and the forcing measure
\(\mu_m^{\mathcal F}\) are the resolvent-DE analogues of the full empirical
measures \(\nu_{\bm W}^{\mathcal K}\) and \(\nu_{\bm W}^{\mathcal F}\) from
\cref{iclr:sec:foundations}.
The fixed-point stability and far-field estimates of
\citet[Propositions~E.1--E.2]{paquette20244+}, together with their
large-\(z\) expansion \citep[Proposition~E.6]{paquette20244+}, give a
common finite enclosure
\begin{equation}
\label{eq:td_lrs_uniform_de_support}
\Lambda_\star
:=\sup_m\sup\bigl(\operatorname{supp}\mu_m^{\mathcal K}\bigr)<\infty
\end{equation}
along every admissible dimension sequence.  For the constant reference
schedule, set \(B=1\) and
\[
\eta_m:=\eta_\star
\begin{cases}
m^{2\alpha-1},&0<\alpha<1/2,\\
1,&\alpha>1/2,
\end{cases}
\]
where \(\eta_\star>0\) is small enough that, uniformly in \(m\),
\[
\sup_m\sup_{\lambda\in\operatorname{supp}\mu_m^{\mathcal K}}
\eta_m\lambda<1,
\]
and the reference schedule lies a fixed distance below the pointwise and
kernel-norm convergence thresholds used in
\citet[Appendices~F--H]{paquette20244+}.  This choice works on both sides of
\(\alpha=1/4\).  Indeed, pointwise stability follows from
\cref{eq:td_lrs_uniform_de_support}, while at batch size \(B=1\)
\[
\sum_{r\geq0}\mathscr K(r)
=\frac{\eta_m}{2}
\int\frac{\lambda}{1-\eta_m\lambda}\,
\mu_m^{\mathcal K}(\mathrm d\lambda)
\lesssim \eta_m\operatorname{tr}\bm\Lambda
\lesssim\eta_\star.
\]
The last bound uses
\(\operatorname{tr}\bm\Lambda\asymp m^{1-2\alpha}\) for
\(\alpha<1/2\) and summability for \(\alpha>1/2\); choosing
\(\eta_\star\) sufficiently small gives a common kernel-norm margin.
For this proof, abbreviate
\(q_m(\lambda):=1-2\eta_m\lambda+2\eta_m^2\lambda^2\).

The precise parameter map to \citet{paquette20244+} is
\begin{equation*}
d_{4+3}=m,
\qquad
v_{4+3}=d,
\qquad
B_{4+3}=1,
\qquad
\gamma_{4+3}=\eta_m.
\end{equation*}
Equivalently, before setting \(B_{4+3}=1\), the averaged-batch convention
used here satisfies
\(\gamma_{4+3}B_{4+3}=\eta_m\) and
\(\gamma_{4+3}^2B_{4+3}=\eta_m^2/B_{4+3}\); the second identity matches the
kernel prefactor.  For the memory lemma, where \(\alpha>1/4\),
\cref{eq:noisy_43_de_functional_calculus} identifies the complete
constant-schedule DE kernel in their notation as
\[
\mathscr K(r)
=\eta_m^2
\int q_m(\lambda)^r\nu_m^{\mathcal K}(\mathrm d\lambda),
\]
not merely its pure-point surrogate.  Their complete-kernel comparison
\citep[Proposition~G.1]{paquette20244+} gives
\(\mathscr K(r)\asymp\mathscr K_{pp}(r)\) uniformly whenever
\(\eta_mr\leq C_{\rm win}m^{2\alpha}\), for each fixed
\(C_{\rm win}<\infty\).  Using their displayed definition of
\(\mathscr K_{pp}\), whose large-time asymptotic is recorded in
\citet[Proposition~H.5]{paquette20244+}, and dividing by \(\eta_m^2\)
yields
\begin{equation}
\label{eq:td_lrs_constant_de_filter_input}
\int q_m(\lambda)^r\nu_m^{\mathcal K}(\mathrm d\lambda)
\asymp
\frac1{2\alpha}\int_0^1
x^{1-1/(2\alpha)}e^{-2\eta_mrx}\,\mathrm dx
\asymp h_\alpha(\eta_mr).
\end{equation}
The first comparison is the essential full-kernel input; the separate
\(\mathscr K_{pp}\) asymptotic alone would not imply it.  For the second
comparison, use the elementary fact
\begin{equation}
\label{eq:td_lrs_incomplete_gamma_comparison}
\int_0^1x^ae^{-2\tau x}\,\mathrm dx
\asymp(1+\tau)^{-a-1},
\qquad a>-1,\quad \tau\geq0.
\end{equation}
Indeed, the integral is bounded above and below when \(\tau\leq1\); when
\(\tau\geq1\), the change of variables \(y=\tau x\), followed by integration
over \(y\in[0,1]\) for the lower bound and over \(y\in[0,\infty)\) for the
upper bound, proves the claim.  Here
\(a=1-1/(2\alpha)>-1\) is exactly \(\alpha>1/4\).

We record the forcing input at the same level of precision.  In the
sub-trace low-source range \(0<\alpha<1/2\), \(\beta<1/2\), the desired
Laplace comparison is already
\cref{lem:td_lrs_low_source_de_forcing}; below we use the external contour
input only for the remaining branches.  The same parameter map and
functional-calculus identity identify the complete forcing
\(\mathscr F(r)\) of \citet{paquette20244+} with
\(\int q_m(\lambda)^r\mu_m^{\mathcal F}(\mathrm d\lambda)\).
Their zero-, cap-, and central-contour estimates
\citep[Propositions~F.1--F.3]{paquette20244+}, combined in the
complete-forcing comparison \citep[Corollary~F.1]{paquette20244+}, give
\begin{equation}
\label{eq:td_lrs_constant_de_forcing_decomposition}
\int q_m(\lambda)^r\mu_m^{\mathcal F}(\mathrm d\lambda)
\asymp
\mathscr F_0(m)+\mathscr F_{pp}(\eta_mr)
+\mathscr F_{ac}(\eta_mr,m)
\end{equation}
on the same fixed-multiple cutoff window.  Proposition~F.3 of
\citet{paquette20244+} is the step requiring \(\beta<1+2\alpha\): this
condition makes the \(u\log(1/u)\) remainder negligible.  Accordingly, the
forcing comparison below assumes \(\beta<1+2\alpha\).
The gap filling there applies to the complete forcing: by the uniform
pointwise margin above, \(0\leq q_m\leq1\), so its positive
measure representation makes the left-hand side of
\cref{eq:td_lrs_constant_de_forcing_decomposition} nonincreasing in \(r\).

The component asymptotics are recorded in
\citet[Propositions~H.2--H.4]{paquette20244+}; for completeness, we evaluate
the corresponding integrals directly below.  The pure-point integral and
\cref{eq:td_lrs_incomplete_gamma_comparison} give
\[
\mathscr F_{pp}(\tau)
=\frac1{2\alpha}\int_0^1
x^{(2\beta-1)/(2\alpha)}e^{-2\tau x}\,\mathrm dx
\asymp(1+\tau)^{-p/(2\alpha)},
\]
because \(p>0\).  The zero-mode proposition gives
\begin{equation}
\label{eq:td_lrs_forcing_zero_orders}
\mathscr F_0(m)\asymp
\begin{cases}
m^{-p},&\beta<1/2,\\
m^{-2\alpha},&\beta>1/2.
\end{cases}
\end{equation}
When \(\beta<1/2\), the absolutely-continuous component is absent.  When
\(\beta>1/2\) and \(\alpha<1/2\), its defining integral gives
\(\mathscr F_{ac}\lesssim m^{-2\alpha}\), so it is absorbed by
\cref{eq:td_lrs_forcing_zero_orders}.  Finally, when
\(\alpha>1/2\) and \(\beta>1/2\), put
\(b_\alpha:=1-1/(2\alpha)>0\).  For
\(0\leq\tau\leq C_{\rm win}m^{2\alpha}\),
\[
\mathscr F_{ac}(\tau,m)
\asymp
\frac1m\int_{m^{-2\alpha}}^1
x^{-1/(2\alpha)}e^{-2\tau x}\,\mathrm dx
\asymp m^{-1}(1+\tau)^{-b_\alpha}.
\]
For \(\tau\leq1\), this follows by dropping the exponential up to fixed
constants.  For \(\tau\geq1\), the change of variables \(y=\tau x\) writes
the integral as
\[
\tau^{-b_\alpha}
\int_{\tau m^{-2\alpha}}^\tau
y^{b_\alpha-1}e^{-2y}\,\mathrm dy.
\]
The last integral is bounded above by the complete gamma integral.  For the
lower bound, if \(\tau m^{-2\alpha}\leq1/2\), integrate over
\([1/2,1]\).  Otherwise its lower endpoint lies in
\([1/2,C_{\rm win}]\), and, for all sufficiently large widths along the
admissible sequence, the integration interval contains a fixed-length
interval starting there.  The integrand has a positive minimum on
\([1/2,C_{\rm win}+1]\), so the lower bound is uniform in the stated window.
Thus the sum of the three components is uniformly comparable to
\(\Phi_{\alpha,\beta}(\tau,m)\) in the corresponding open branch.  Outside
the sub-trace low-source range, we have proved
\begin{equation}
\label{eq:td_lrs_constant_de_forcing_input}
\int q_m(\lambda)^r\mu_m^{\mathcal F}(\mathrm d\lambda)
\asymp\Phi_{\alpha,\beta}(\eta_mr,m),
\qquad
0\leq\eta_mr\leq C_{\rm win}m^{2\alpha}.
\end{equation}
In particular, the cutoff endpoint is included.  For the memory input,
Proposition~G.1 of \citet{paquette20244+} is uniform for
\(\gamma Br<Md^{2\alpha}\) for every fixed \(M\); choosing
\(M>C_{\rm win}\) covers
\(\eta_mr=C_{\rm win}m^{2\alpha}\).  For the forcing input,
Corollary~F.1 gives comparability for the complete forcing throughout the
transition and uses its monotonicity to fill the intervening cutoff gaps.

\emph{Step 2: discrete-to-Laplace transfer.}
On the uniform stability interval, the ratio
\(-\log q_m(\lambda)/(\eta_m\lambda)\), with continuous value \(2\) at
\(\lambda=0\), is bounded above and below by positive constants.  Hence there
are \(0<a_0\leq a_1<\infty\), independent of \(m,r\), and \(\lambda\), such
that, for every integer \(r\geq0\),
\[
e^{-a_1\eta_mr\lambda}
\leq q_m(\lambda)^r
\leq e^{-a_0\eta_mr\lambda}.
\]

In the sub-trace low-source range, integrating this sandwich and applying
\cref{lem:td_lrs_low_source_de_forcing} at the two fixed time rescalings
proves \cref{eq:td_lrs_constant_de_forcing_input} without any contour input.

For either \(\nu=\nu_m^{\mathcal K}\) or
\(\nu=\mu_m^{\mathcal F}\), fix \(c\in[c_-,c_+]\) and \(u>0\), and choose
the proof-local integers
\[
\ell_-:=
\left\lfloor\frac{cu}{a_1\eta_m}\right\rfloor,
\qquad
\ell_+:=
\left\lceil\frac{cu}{a_0\eta_m}\right\rceil.
\]
Then \(a_1\eta_m\ell_-\leq cu\leq a_0\eta_m\ell_+\), and the preceding
sandwich gives
\[
q_m(\lambda)^{\ell_+}
\leq e^{-cu\lambda}
\leq q_m(\lambda)^{\ell_-}.
\]
Positivity therefore gives
\[
\int q_m(\lambda)^{\ell_+}\nu(\mathrm d\lambda)
\leq \int e^{-cu\lambda}\nu(\mathrm d\lambda)
\leq \int q_m(\lambda)^{\ell_-}\nu(\mathrm d\lambda).
\]

We first handle a fixed small-time interval.  By
\cref{eq:noisy_43_de_scalar_measure_moments},
\[
\nu_m^{\mathcal K}([0,\infty))
=\operatorname{tr}(\bm\Lambda^2)
+\frac{(\operatorname{tr}\bm\Lambda)^2}{m}
\asymp1,
\qquad
\mu_m^{\mathcal F}([0,\infty))
=\sum_{j=1}^d j^{-2(\alpha+\beta)}\asymp1.
\]
The first comparison uses \(\alpha>1/4\); under the additional forcing
hypotheses, the second uses \(p>0\).
Together with the common support bound
\cref{eq:td_lrs_uniform_de_support}, these identities show, for each fixed
\(u_0<\infty\),
\[
\int e^{-cu\lambda}\nu_m^{\mathcal K}(\mathrm d\lambda)
\asymp1\asymp h_\alpha(u),
\qquad
\int e^{-cu\lambda}\mu_m^{\mathcal F}(\mathrm d\lambda)
\asymp1
\asymp\Phi_{\alpha,\beta}(u,m)
\]
uniformly over \(0\leq u\leq u_0\) and \(c\in[c_-,c_+]\).

Choose \(u_0\geq\max\{1,2a_1\eta_\star/c_-\}\).
For \(u\geq u_0\), the definitions of \(\ell_-\) and \(\ell_+\) give fixed
positive constants
\(b_-,b_+\), depending only on \(a_0,a_1,c_-,c_+\), and \(u_0\), such that
\[
b_-u\leq\eta_m\ell_-
\leq\eta_m\ell_+\leq b_+u.
\]
Both \(h_\alpha\) and every nonzero summand in
\(\Phi_{\alpha,\beta}\) are stable, uniformly, under a time rescaling in
the fixed interval \([b_-,b_+]\).  Consequently,
\[
h_\alpha(\eta_m\ell_\pm)\asymp h_\alpha(u),
\qquad
\Phi_{\alpha,\beta}(\eta_m\ell_\pm,m)
\asymp\Phi_{\alpha,\beta}(u,m).
\]

It remains to check that rounding does not leave the source window.  If
\(u\leq C_0m^{2\alpha}\), then
\[
\eta_m\ell_+
\leq\frac{c_+}{a_0}C_0m^{2\alpha}+\eta_\star.
\]
Choose once and for all
\(C_{\rm win}>(c_+/a_0)C_0+\eta_\star\).
Then both bracketing integers lie in the enlarged discrete window used in
\cref{eq:td_lrs_constant_de_filter_input,eq:td_lrs_constant_de_forcing_input}.
Applying those two inputs to the preceding integral bracket and clock
comparison proves
\cref{eq:td_lrs_de_kernel_laplace_filter,eq:td_lrs_de_forcing_laplace_filter}
uniformly through \(u=C_0m^{2\alpha}\).  Thus the fixed exponent rescaling
and the ceiling error are absorbed before the cutoff endpoint is invoked.

\emph{Step 3: kernel mass.}
Since \(q_{\mathcal K}=2-1/(2\alpha)\), direct integration gives
\[
\int_0^{C_0m^{2\alpha}}h_\alpha(u)\,\mathrm du
\asymp
\begin{cases}
m^{1-2\alpha},&1/4<\alpha<1/2,\\
1,&\alpha>1/2,
\end{cases}
=A_{\alpha,m},
\]
which is \cref{eq:td_lrs_de_kernel_l1_scale}.  The excluded boundary
\(\alpha=1/2\) would instead give a logarithm.
\end{proof}

\begin{lemma}[Class-uniform nonconvolution Volterra reduction]
\label{lem:td_lrs_common_volterra_reduction}
\label{lem:td_lrs_de_volterra_reduction}
Fix the setting of \cref{lem:td_lrs_de_spectral_filters}.  Consider any
deterministic schedule with \(T_t\leq C_0m^{2\alpha}\), uniformly bounded
intrinsic mesh, and satisfying the DE versions of part \emph{(a)} of
\cref{ass:td_lrs_joint_schedule_stability} and
\cref{ass:td_lrs_regular_spectral_mode_decay}.
Assume \(A_{\alpha,m}\sup_s r_s^{-1}\leq\varepsilon_0\).
For the joint DE kernel, define
\[
(\mathcal K^2)_{t,s}
:=\sum_{u=s+1}^{t-1}\mathcal K_{t,u}\mathcal K_{u,s}.
\]
There is a constant \(C_\star\), independent of the
schedule, width, and horizon, such that
\begin{equation}
\label{eq:td_lrs_de_row_and_square_bounds}
\sup_t\sum_{s<t}\mathcal K_{t,s}
\leq C_\star\varepsilon_0,
\qquad
(\mathcal K^2)_{t,s}
\leq C_\star\varepsilon_0\,\mathcal K_{t,s}.
\end{equation}
Consequently, after choosing \(\varepsilon_0\) so that
\(C_\star\varepsilon_0<1\), every nonnegative sequence \(v\) satisfies
\begin{equation}
\label{eq:td_lrs_neumann_reduction}
v+\mathcal Kv
\leq
(I-\mathcal K)^{-1}v
\leq
v+\frac1{1-C_\star\varepsilon_0}\mathcal Kv.
\end{equation}
\end{lemma}

\begin{proof}
The one-step identity
\(q_s(\lambda)=(1-\eta_s\lambda)^2+\eta_s^2\lambda^2/B_s\)
and the stated margins give constants \(0<c<C<\infty\), common to the
schedule class, for which
\[
e^{-C\lambda(T_t-T_s)}
\leq Q_{s,t}(\lambda)
\leq e^{-c\lambda(T_t-T_s)}.
\]
Thus \cref{lem:td_lrs_de_spectral_filters} yields, with
\(a_s:=r_s^{-1}\Delta T_s=\eta_s^2/B_s\),
\begin{equation}
\label{eq:td_lrs_de_kernel_entry_comparison}
\mathcal K_{t,s}
\asymp
a_sh_\alpha(T_t-T_{s+1}).
\end{equation}
Bounded mesh lets endpoint sums be compared with integrals.  Hence
\[
\sum_{s<t}\mathcal K_{t,s}
\lesssim
(\sup_s r_s^{-1})\int_0^{T_t}h_\alpha(u)\,\mathrm du
\lesssim
A_{\alpha,m}\sup_s r_s^{-1}.
\]
For \(V:=T_t-T_{s+1}\), positivity and the same comparison give
\begin{align*}
(\mathcal K^2)_{t,s}
&\lesssim
a_s(\sup_\ell r_\ell^{-1})
\int_0^Vh_\alpha(V-u)h_\alpha(u)\,\mathrm du\\
&\lesssim
a_s(\sup_\ell r_\ell^{-1})A_{\alpha,m}h_\alpha(V)\\
&\lesssim
A_{\alpha,m}(\sup_\ell r_\ell^{-1})\mathcal K_{t,s}.
\end{align*}
The middle inequality follows by splitting at \(V/2\): on each half one
factor is at most a fixed multiple of \(h_\alpha(V)\), while the other
integrates to at most \(A_{\alpha,m}\).  This proves
\cref{eq:td_lrs_de_row_and_square_bounds}.

Choose \(\varepsilon_0\) so that \(C_\star\varepsilon_0<1\).  Positivity
gives inductively
\(\mathcal K^n\leq(C_\star\varepsilon_0)^{n-1}\mathcal K\) for
\(n\geq1\).  Summing the
finite-horizon Neumann series proves
\cref{eq:td_lrs_neumann_reduction}.
\end{proof}

For the joint-schedule DE quantities in
\cref{prop:td_lrs_joint_de_recursion}, suppress the schedule arguments and
write \(\mathcal F_t\), \(\mathcal K_{t,s}\), and
\(\mathcal R_{\sigma,t}\).

\begin{theorem}[Class-uniform joint-schedule risk bounds in the LM and IM PLRF regimes]
\label{thm:td_lrs_joint_fsl}
Fix \(\alpha>1/4\), \(\alpha\neq1/2\), an admissible PLRF dimension
sequence, and a deterministic schedule family satisfying
\cref{lem:td_lrs_de_volterra_reduction}, with its ratio-peak constant chosen
so that \(C_\star\varepsilon_0<1\).  Then, uniformly over the schedule
family, widths, horizons, and \(0\leq s<t\),
\begin{equation}
\label{eq:td_lrs_joint_de_kernel_law}
\mathcal K_{t,s}
\asymp
\frac{\Delta T_s}{r_s}
(1+T_t-T_{s+1})^{-2+1/(2\alpha)},
\end{equation}
and
\begin{equation}
\label{eq:td_lrs_joint_de_risk_law}
\mathcal R_{\sigma,t}\asymp
\mathcal F_t+\sum_{s<t}\frac{\Delta T_s}{r_s}
(1+T_t-T_{s+1})^{-2+1/(2\alpha)}
\bigl[\mathcal F_s+\sigma^2\bigr].
\end{equation}
The comparison constants may
depend on the fixed model and common schedule-family constants, but not on the
schedule, width, or horizon.
\end{theorem}

\begin{proof}
The DE recursion gives
\[
\mathcal R_\sigma
=(I-\mathcal K)^{-1}
\bigl(\mathcal F+\sigma^2\mathcal K\bm1\bigr).
\]
Apply \cref{eq:td_lrs_neumann_reduction} directly with
\(v=\mathcal F+\sigma^2\mathcal K\bm1\).  Positivity, together with
\(\mathcal K^2\bm1\leq
C_\star\varepsilon_0\mathcal K\bm1\), gives
\[
\mathcal F+\mathcal K\mathcal F+\sigma^2\mathcal K\bm1
\leq
\mathcal R_\sigma
\lesssim
\mathcal F+\mathcal K\mathcal F+\sigma^2\mathcal K\bm1,
\]
and inserting \cref{eq:td_lrs_de_kernel_entry_comparison} proves both
displays.
\end{proof}

\begin{corollary}[Explicit LM/IM PLRF DE forcing and risk asymptotics]
\label{cor:td_lrs_joint_de_explicit_forcing}
Under \cref{thm:td_lrs_joint_fsl}, suppose additionally that
\(p=2\alpha+2\beta-1>0\), \(\beta<1+2\alpha\), and
\((\alpha,\beta)\) lies in one of the six open
\(\mathrm{LM}/\mathrm{IM}\) subregimes.  Then, uniformly over the same
schedule family and source window,
\begin{equation}
\label{eq:td_lrs_joint_de_explicit_forcing}
\mathcal F_t
\asymp
\Phi_{\alpha,\beta}(T_t,m),
\end{equation}
and hence
\[
\mathcal R_{\sigma,t}\asymp
\Phi_{\alpha,\beta}(T_t,m)
+\sum_{s<t}\frac{\Delta T_s}{r_s}
(1+T_t-T_{s+1})^{-2+1/(2\alpha)}
\bigl[\Phi_{\alpha,\beta}(T_s,m)+\sigma^2\bigr].
\]
The comparison remains valid at \(T_t\asymp m^{2\alpha}\).
\end{corollary}

\begin{proof}
The two-sided spectral-mode decay bounds in \cref{lem:td_lrs_propagator} and
\cref{eq:td_lrs_de_forcing_laplace_filter} give
\cref{eq:td_lrs_joint_de_explicit_forcing}; substitute this comparison in
\cref{eq:td_lrs_joint_de_risk_law}.
\end{proof}

\subsection{Finite-bulk deterministic-equivalent risk asymptotics}
\label{subsec:app_td_lrs_joint_schedules}

Unlike the width-independent memory tails in the LM and IM regimes,
FB memory is generated by a
width-scale band supplied by the PLRF DE measure.  Consequently, the kernel
and the noisy--clean gap require no target-forcing assumption.
\Cref{prop:td_lrs_joint_de_finite_bulk} proves the DE total-risk asymptotic for
\(\alpha<1/4\) and \(\beta<1+2\alpha\) from
\cref{lem:noisy_43_de_finite_bulk_band_mass,lem:td_lrs_de_forcing_filters}.

\begin{lemma}[Finite-bulk DE band mass]
\label{lem:noisy_43_de_finite_bulk_band_mass}
Suppose
\[
0<\alpha<\frac12,
\qquad
1<c_-\leq \frac dm\leq c_+<\infty,
\qquad
\lambda_j=j^{-2\alpha}.
\]
There exist constants \(0<a_-<a_+<\infty\) and \(0<c<C<\infty\),
depending only on \(\alpha,c_-,c_+\), such that
\begin{equation}
\label{eq:noisy_43_de_finite_bulk_band_mass}
c m
\leq
\mu_m^{\mathcal K}
\bigl([a_-m^{-2\alpha},a_+m^{-2\alpha}]\bigr)
\leq
C m .
\end{equation}
\end{lemma}

\begin{proof}
Each index \(j\in\{m,\ldots,d\}\) contributes a fixed positive amount of mass
to the same width-scale band, and there are order \(m\) such indices.  To show
this, set \(s_m:=\operatorname{tr}(\bm\Lambda)/m\).  Since \(d\asymp m\) and
\(\alpha<1/2\), we have
\(s_m=m^{-1}\sum_{i=1}^d i^{-2\alpha}\leq C_0m^{-2\alpha}\).
Let \(X_{j,m}\sim\mu_{j,m}\), the \(j\)-th diagonal measure.  By
\cref{eq:noisy_43_de_matrix_measure_moments},
\(\mathbb E X_{j,m}=\lambda_j\) and
\(\mathbb E X_{j,m}^2=\lambda_j^2+s_m\lambda_j\).  For
\(m\leq j\leq d\),
\(c_+^{-2\alpha}m^{-2\alpha}\leq\lambda_j\leq m^{-2\alpha}\), and therefore,
uniformly over these indices,
\(\mathbb E X_{j,m}^2/(\mathbb E X_{j,m})^2
=1+s_m/\lambda_j\leq C_1\).
Thus the first two moments are uniformly comparable on these indices, and
Paley--Zygmund yields
\[
\mathbb P\!\left(X_{j,m}\geq\frac12\lambda_j\right)
\geq
\frac14\frac{\lambda_j^2}{\mathbb E X_{j,m}^2}
\geq c_1>0.
\]
A matching upper cutoff follows from Markov's inequality: for every \(A>0\),
\(\mathbb P(X_{j,m}>A m^{-2\alpha})
\leq\lambda_j/(A m^{-2\alpha})\leq1/A\).
Choose \(A\) so that \(A^{-1}\leq c_1/2\), and set
\(a_-:=c_+^{-2\alpha}/2\) and \(a_+:=A\).  Then
\[
\mu_{j,m}\bigl([a_-m^{-2\alpha},a_+m^{-2\alpha}]\bigr)
\geq \frac{c_1}{2},
\qquad m\leq j\leq d.
\]
Thus every selected coordinate contributes at least \(c_1/2\) mass to the
same band.  Summing over the at least \((c_--1)m\) indices \(j=m,\ldots,d\)
gives the lower bound in \cref{eq:noisy_43_de_finite_bulk_band_mass}; the total
mass \(\mu_m^{\mathcal K}([0,\infty))=d\leq c_+m\) gives the upper bound.
\end{proof}

\begin{proposition}[Finite-bulk DE kernel and noisy--clean gap]
\label{prop:td_lrs_joint_de_finite_bulk_kernel}
Fix \(0<\alpha<1/4\) and an admissible PLRF dimension sequence with aspect
ratios in a compact subset of \((1,\infty)\), and fix
\(0<C_0<\infty\).  Suppose the schedule family
satisfies the DE versions of part \emph{(a)} of
\cref{ass:td_lrs_joint_schedule_stability} and
\cref{ass:td_lrs_regular_spectral_mode_decay}, has uniformly bounded
intrinsic mesh, and
\begin{equation*}
\sup_s\frac1{r_s}\leq c_\star m^{2\alpha-1}
\end{equation*}
for a sufficiently small common \(c_\star\), which may depend on \(C_0\)
and the common admissibility constants.  Then, uniformly for
\(0\leq s<t\) with \(T_t\leq C_0m^{2\alpha}\),
\begin{equation}
\label{eq:td_lrs_joint_de_finite_bulk_kernel}
\mathcal K_{\boldsymbol\eta,\boldsymbol B}(t,s,m)
\asymp
m^{1-4\alpha}\frac{\Delta T_s}{r_s}.
\end{equation}
Moreover, after decreasing \(c_\star\) if necessary, there is one
\(\rho<1\), common to the schedule family, such that
\begin{equation}
\label{eq:td_lrs_joint_de_finite_bulk_square}
(\mathcal K^2)_{t,s}\leq\rho\mathcal K_{t,s}.
\end{equation}
For every fixed \(\sigma^2>0\), on the same window,
\begin{equation}
\label{eq:td_lrs_joint_de_finite_bulk_gap}
\mathcal R_{\sigma,\boldsymbol\eta,\boldsymbol B}(t,m)
-\mathcal R_{0,\boldsymbol\eta,\boldsymbol B}(t,m)
\asymp
\sigma^2m^{1-4\alpha}
\sum_{s<t}\frac{\Delta T_s}{r_s}.
\end{equation}
\end{proposition}

\begin{proof}
The two-sided spectral-mode decay bounds in \cref{lem:td_lrs_propagator} and the DE band-mass
bound in \cref{lem:noisy_43_de_finite_bulk_band_mass} give the lower bound in
\cref{eq:td_lrs_joint_de_finite_bulk_kernel}: on the band
\(\lambda\asymp m^{-2\alpha}\), the age condition
\(T_t-T_{s+1}\leq C_0m^{2\alpha}\) keeps \(Q_{s+1,t}(\lambda)\) bounded
below by a positive constant.  The upper bound follows from
\(Q_{s+1,t}\leq1\) and the DE moment identity
\[
\int\lambda^2\,\mu_m^{\mathcal K}(\mathrm d\lambda)
\asymp m^{1-4\alpha}.
\]
Consequently,
\[
\sup_{u\leq t}\sum_{s<u}\mathcal K_{u,s}
\lesssim
m^{1-4\alpha}T_t\sup_s r_s^{-1}.
\]
This is at most a fixed constant smaller than one in \(\mathrm{FB}_1\), after
choosing \(c_\star\), and is \(o(1)\) in the strict \(\mathrm{FB}_2\)
window.  The two-sided kernel comparison also gives, for \(s<t\),
\[
(\mathcal K^2)_{t,s}
=\sum_{r=s+1}^{t-1}\mathcal K_{t,r}\mathcal K_{r,s}
\lesssim
\mathcal K_{t,s}
m^{1-4\alpha}\sum_{r=s+1}^{t-1}\frac{\Delta T_r}{r_r}
\leq\rho\mathcal K_{t,s},
\]
with one common \(\rho<1\) after decreasing \(c_\star\).  Thus the same
two-term Neumann reduction as in
\cref{eq:td_lrs_neumann_reduction} applies.

Set \(\mathcal Z_t:=\mathcal R_{\sigma,t}-\mathcal R_{0,t}\).  Subtracting the two
DE recursions gives
\(\mathcal Z=(I-\mathcal K)^{-1}\sigma^2\mathcal K\bm1\).
Positivity and \cref{eq:td_lrs_joint_de_finite_bulk_square} therefore imply
\[
\sigma^2\mathcal K\bm1
\leq \mathcal Z
\leq\frac{\sigma^2}{1-\rho}\mathcal K\bm1.
\]
Summing \cref{eq:td_lrs_joint_de_finite_bulk_kernel} over \(s<t\) proves
\cref{eq:td_lrs_joint_de_finite_bulk_gap}.
\end{proof}

The canonical PLRF target forcing is also controlled throughout the source
range in which the central-contour remainder is negligible.  This closes the
total-risk asymptotic without an additional finite-bulk forcing assumption.

\begin{proposition}[Finite-bulk joint-schedule DE risk asymptotics in the proved source range]
\label{prop:td_lrs_joint_de_finite_bulk}
\label{iclr:prop:finite_bulk_de}
Fix an open finite-bulk subregime, \(p=2\alpha+2\beta-1>0\), and
\(\beta<1+2\alpha\), away from the PLRF critical lines.  Retain the
dimension-sequence and schedule-family hypotheses of
\cref{prop:td_lrs_joint_de_finite_bulk_kernel}.  For fixed \(\sigma^2>0\),
uniformly for \(1\ll T_t\leq C_0m^{2\alpha}\) in \(\mathrm{FB}_1\),
\begin{equation}
\label{eq:td_lrs_joint_fb1_fsl}
\mathcal R_{\sigma,\boldsymbol\eta,\boldsymbol B}(t,m)
\asymp
m^{-p}+(1+T_t)^{-p/(2\alpha)}
+\sigma^2m^{1-4\alpha}
\sum_{s<t}\frac{\Delta T_s}{r_s}.
\end{equation}
Uniformly for \(1\ll T_t=o(m^{2\alpha})\) in \(\mathrm{FB}_2\),
\begin{equation}
\label{eq:td_lrs_joint_fb2_fsl}
\mathcal R_{\sigma,\boldsymbol\eta,\boldsymbol B}(t,m)
\asymp
m^{-2\alpha}+(1+T_t)^{-p/(2\alpha)}
+\sigma^2m^{1-4\alpha}
\sum_{s<t}\frac{\Delta T_s}{r_s}.
\end{equation}
All constants are uniform over the stated family.
\end{proposition}

\begin{proof}
The two-sided spectral-mode decay bounds in \cref{lem:td_lrs_propagator} and
\cref{eq:td_lrs_de_forcing_laplace_filter} give, in the two finite-bulk
branches,
\[
\mathcal F_t\asymp
\begin{cases}
m^{-p}+(1+T_t)^{-p/(2\alpha)},&\mathrm{FB}_1,\\
m^{-2\alpha}+(1+T_t)^{-p/(2\alpha)},&\mathrm{FB}_2.
\end{cases}
\]

Write \(q=p/(2\alpha)\).  In \(\mathrm{FB}_1\), \(0<q<1\), and
\[
(\mathcal K\mathcal F)_t
\lesssim
m^{1-4\alpha}(\sup_s r_s^{-1})
\int_0^{T_t}\left[m^{-p}+(1+u)^{-q}\right] \,\mathrm du
\lesssim
m^{-p}+(1+T_t)^{-q}
\asymp\mathcal F_t.
\]
In \(\mathrm{FB}_2\), \(q>1\); using
\(\int_0^\infty(1+u)^{-q}\,\mathrm du<\infty\) gives
\[
(\mathcal K\mathcal F)_t
\lesssim m^{-2\alpha}
\bigl[m^{-2\alpha}T_t+1\bigr]
\lesssim m^{-2\alpha}\lesssim\mathcal F_t
\]
on the strict window.  Positivity and the
two-term Neumann reduction from
\cref{eq:td_lrs_joint_de_finite_bulk_square} now give
\[
\mathcal F_t\leq\mathcal R_{0,t}
\leq\mathcal F_t+\frac{1}{1-\rho}(\mathcal K\mathcal F)_t
\lesssim\mathcal F_t.
\]
Adding the
noisy--clean gap in
\cref{eq:td_lrs_joint_de_finite_bulk_gap} proves
\cref{eq:td_lrs_joint_fb1_fsl,eq:td_lrs_joint_fb2_fsl}.
\end{proof}

\begin{assumption}[High-source \(\mathrm{FB}_2\) DE forcing closure]
\label{ass:td_lrs_high_source_fb2_forcing}
Fix \(0<\alpha<1/4\), \(\beta>1+2\alpha\), and aspect ratios in a compact
subset of \((1,\infty)\).  Uniformly over the finite-bulk admissible schedule
family,
\begin{equation*}
\mathcal F_{\boldsymbol\eta,\boldsymbol B}(t,m)
\asymp
m^{-2\alpha}+(1+T_t)^{-p/(2\alpha)}
\end{equation*}
on the strict \(\mathrm{FB}_2\) window \(T_t=o(m^{2\alpha})\).  The boundary
\(\beta=1+2\alpha\) is excluded; logarithmic corrections may occur there.
\end{assumption}

\begin{corollary}[Conditional high-source \(\mathrm{FB}_2\) extension]
\label{cor:td_lrs_joint_de_high_source_fb2}
Fix \(0<\alpha<1/4\), \(\beta>1+2\alpha\), and \(\sigma^2>0\), and let a
schedule family satisfy the hypotheses of
\cref{prop:td_lrs_joint_de_finite_bulk_kernel}.  Under the residual forcing
input in \cref{ass:td_lrs_high_source_fb2_forcing},
\cref{eq:td_lrs_joint_fb2_fsl} holds uniformly for
\(1\ll T_t=o(m^{2\alpha})\).  The boundary \(\beta=1+2\alpha\) is not
included in this pure-power statement.
\end{corollary}

\begin{proof}
The residual assumption supplies the \(\mathrm{FB}_2\) forcing comparison.
The same bound \(\mathcal K\mathcal F\lesssim\mathcal F\), followed by the
two-term Neumann reduction and the target-free noisy--clean gap, is identical
to the \(\mathrm{FB}_2\) part of the preceding proof.
\end{proof}
\subsection{Conditional transfer from DE risk asymptotics to realized SGD}
\label{subsec:td_lrs_joint_empirical_transfer}

The preceding risk asymptotics concern the deterministic equivalent.  Comparing their
total-risk orders, or the resource infima below, with realized SGD requires
the empirical-to-DE conditions stated here.

For the transfer statements below, suppress the fixed schedule and width by
writing
\[
F_t^{\bm W}:=F_t,
\qquad
K_{t,s}^{\bm W}:=K_{t,s},
\qquad
\mathcal F_t:=\mathcal F_{\boldsymbol\eta,\boldsymbol B}(t,m),
\qquad
\mathcal K_{t,s}:=\mathcal K_{\boldsymbol\eta,\boldsymbol B}(t,s,m),
\]
and set
\[
R_\sigma^{\bm W}:=(R_{\sigma,t})_{t\geq0},
\qquad
\mathcal R_\sigma:=
\bigl(\mathcal R_{\sigma,\boldsymbol\eta,\boldsymbol B}(t,m)\bigr)_{t\geq0}.
\]

For a deterministic schedule, \emph{fixed-sequence transfer} means that, for
each prescribed width--horizon--schedule sequence, there are events of probability
tending to one and constants \(0<c<C<\infty\), possibly depending on that
sequence but not on \(m,s,t\), such that
\begin{equation*}
c\mathcal F_t\leq F_t^{\bm W}\leq C\mathcal F_t,
\qquad
c\mathcal K_{t,s}\leq K_{t,s}^{\bm W}\leq C\mathcal K_{t,s}
\end{equation*}
uniformly on its source window.  On the same events, for one common
\(\rho<1\),
\begin{equation*}
(K^{\bm W})^2_{t,s}\leq\rho K^{\bm W}_{t,s},
\qquad
\mathcal K^2_{t,s}\leq\rho\mathcal K_{t,s}.
\end{equation*}
A \emph{class-uniform transfer} means that the PLRF operators are coupled
triangularly across resource levels and that there are events \(E_N\), with
\(\mathbb P(E_N)\to1\), and common \(c,C,\rho\) such that, on \(E_N\), the
comparisons hold simultaneously for every admissible width, horizon, and
schedule whose data or feature-compute budget is at most \(N\).

\begin{paperassumption}[Conditional empirical-to-DE transfer]
\label{ass:uniform_subregime_plrf}
Assume fixed-sequence transfer for each prescribed
width--horizon--schedule sequence under consideration.  Whenever comparisons
are made simultaneously over an admissible resource-bounded family, assume
the class-uniform transfer property instead.
\end{paperassumption}

\begin{proposition}[Conditional transfer from joint DE to realized SGD]
\label{prop:td_lrs_joint_empirical_de_transfer}
Under fixed-sequence transfer, for every fixed \(\sigma^2\geq0\),
\[
R_{\sigma,t}
\asymp
\mathcal R_{\sigma,\boldsymbol\eta,\boldsymbol B}(t,m)
\]
uniformly along the prescribed sequence.  Under class-uniform transfer, the
same comparison is uniform over the admissible family; consequently the
exact conditional and DE infima at fixed data or feature compute are
comparable.
\end{proposition}

\begin{proof}
Let \(K\) denote either the exact or the DE positive kernel.  The bound
\(K^2\leq\rho K\) gives, for \(v\geq0\),
\[
v+Kv\leq(I-K)^{-1}v\leq v+\frac1{1-\rho}Kv.
\]
Applied to forcing and label-noise inputs, componentwise comparison gives
\[
R_{\sigma}^{\bm W}
\asymp F^{\bm W}+K^{\bm W}(F^{\bm W}+\sigma^2\bm1)
\asymp\mathcal F+\mathcal K(\mathcal F+\sigma^2\bm1)
\asymp\mathcal R_\sigma.
\]
Class uniformity permits taking infima on the same event.
\end{proof}

\paragraph{Finite-width Volterra hierarchy.}
Across the displayed \(\mathrm{LM}\), \(\mathrm{IM}\), and \(\mathrm{FB}\)
systems, true online SGD closely tracks both the exact conditional finite-\(W\)
Volterra recursion and its resolvent-DE reduction; see
\cref{fig:volterra_three_layer_bridge,fig:fb_volterra_three_layer_bridge}.
At fixed \(\bm W\), \cref{eq:td_lrs_joint_volterra} gives the exact
conditional population risk, and replacing its forcing and kernel by their
resolvent deterministic equivalents gives the second reduction.

\begin{figure}[t]
\centering
\includegraphics[width=\linewidth]{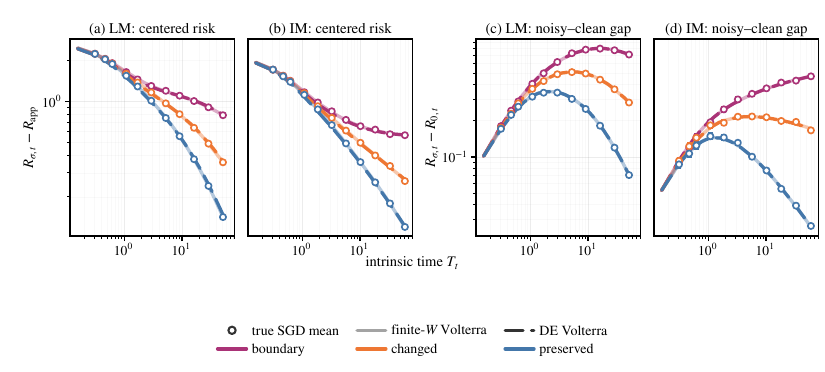}
\caption{Online SGD, the exact finite-width recursion, and its DE reduction
for centered risk and the noisy--clean gap in long and integrable memory.
Their agreement supports both approximation steps and the predicted boundary,
changed, and preserved schedule responses.}
\label{fig:volterra_three_layer_bridge}
\end{figure}

\begin{figure}[t]
\centering
\includegraphics[width=\linewidth]{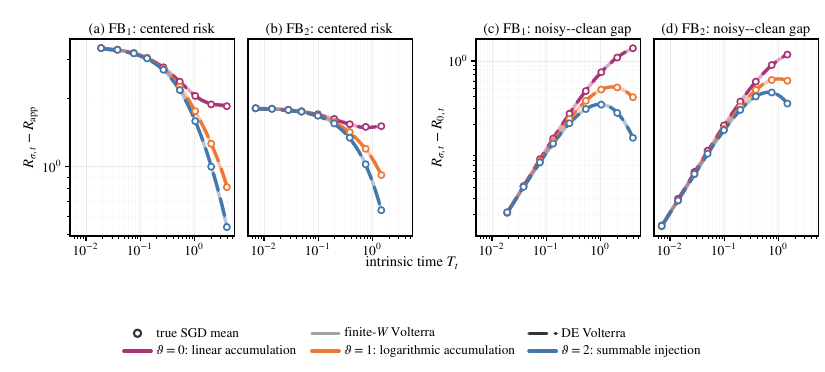}
\caption{Online SGD, the exact finite-width recursion, and its DE reduction
for the three injection responses in \(\mathrm{FB}_1\) and \(\mathrm{FB}_2\).
Their agreement supports the finite-bulk Volterra hierarchy and accumulation
laws.}
\label{fig:fb_volterra_three_layer_bridge}
\end{figure}

\subsection{Optimal ratio and schedule realizations}
\label{subsec:td_lrs_ratio_control_integration}

Represent the discrete ratio as the piecewise-constant intrinsic-time control
\(r(u):=r_s\) for \(T_s\leq u<T_{s+1}\),
so the injection weights remain exact.  If
\(\sup_s\eta_s\lesssim1\) and \(\sup_s r_s^{-1}\lesssim1\), their continuum
order is
\begin{equation}
\label{eq:td_lrs_joint_ratio_functional}
\int_0^{T_t}
\frac1{r(u)}
\left(1+T_t-u\right)^{-2+1/(2\alpha)}
\left[\mathcal F(u,m)+\sigma^2\right] \,\mathrm du,
\end{equation}
with finite-bulk analogue \(\int_0^{T_t}r(u)^{-1}\,\mathrm du\).  It also
records exact sample count:
\begin{equation}
\label{eq:td_lrs_joint_data_identity}
D_t:=\sum_{s<t}B_s
=
\sum_{s<t}r_s\Delta T_s
=
\int_0^{T_t}r(u)\,\mathrm du.
\end{equation}
Use \(\mathfrak f\asymp mD_t\) as the FLOP proxy.

Setting
\begin{equation}
\label{eq:td_lrs_joint_fixed_batch_recovery}
B_s\equiv B
\end{equation}
gives \(T_t=\sum_{u<t}\eta_u\), \(r_s=B/\eta_s\), and
\(\Delta T_s/r_s=\eta_s^2/B\),
recovering every fixed-\(B\) risk formula.

The identities \(\eta_s^2/B_s=\Delta T_s/r_s\) and
\(B_s=r_s\Delta T_s\) are algebraic.  Bounded-mesh endpoint quadrature, and its vanishing-mesh
strengthening, prove respectively the comparison and convergence in
\cref{eq:td_lrs_joint_ratio_functional}; summing the second identity proves
\cref{eq:td_lrs_joint_data_identity}.  Under
\cref{eq:td_lrs_joint_fixed_batch_recovery},
\(q_s(z)=1-2\eta_sz+(1+1/B)\eta_s^2z^2\) and
\(\Delta T_s/r_s=\eta_s^2/B\)
recover every fixed-\(B\) coefficient.

Fix terminal horizon \(T:=T_t\) and budget \(D:=D_t\).  For every LM or IM
branch of \cref{thm:td_lrs_joint_fsl}, define the surviving injection weight
\begin{equation}
\label{eq:td_lrs_ratio_weight}
w_{T,m}(u)
:=
(1+T-u)^{-2+1/(2\alpha)}
\left[\mathcal F(u,m)+\sigma^2\right],
\qquad 0\leq u\leq T.
\end{equation}
The schedule-dependent loss and its budget constraint are then
\begin{equation}
\label{eq:td_lrs_ratio_proxy}
\mathcal J_{T,m}(r)
:=
\int_0^T\frac{w_{T,m}(u)}{r(u)}\,\mathrm du,
\qquad
\int_0^T r(u)\,\mathrm du=D.
\end{equation}
The optimizer is proportional to \(\sqrt{w_{T,m}}\), allocating more samples
where the weighted noise contribution to the terminal risk is larger; see
\cref{thm:td_lrs_ratio_control}.
For the box-constrained form, use
\(\operatorname{clip}(x,a,b):=\min\{b,\max\{a,x\}\}\).

\begin{theorem}[Fixed-horizon continuum ratio control]
\label{thm:td_lrs_ratio_control}
Suppose \(w_{T,m}\) is integrable and positive almost everywhere and
\(0<\int_0^T\sqrt{w_{T,m}(u)}\,\mathrm du<\infty\).
Among positive controls satisfying the data constraint in
\cref{eq:td_lrs_ratio_proxy}, the unique minimizer up to null sets and its
minimum value are
\begin{equation}
\label{eq:td_lrs_ratio_unconstrained}
r_T^\star(u)
=
\frac{D\sqrt{w_{T,m}(u)}}
{\displaystyle\int_0^T\sqrt{w_{T,m}(v)}\,\mathrm dv},
\qquad
\min_r\mathcal J_{T,m}(r)
=
\frac1D
\left(
\int_0^T\sqrt{w_{T,m}(u)}\,\mathrm du
\right)^2.
\end{equation}
With pointwise bounds \(0<r_-<r_+<\infty\), the admissible set is nonempty
exactly when \(r_-T\leq D\leq r_+T\).  At \(D=r_\pm T\), the unique feasible
control is \(r_T^\star\equiv r_\pm\).  Under strict feasibility the unique
control is
\begin{equation}
\label{eq:td_lrs_ratio_clipped}
r_T^\star(u)
=
\operatorname{clip}
\left(
\sqrt{\frac{w_{T,m}(u)}{\mu}},
r_-,r_+
\right),
\qquad
\int_0^T r_T^\star(u)\,\mathrm du=D,
\end{equation}
for a multiplier \(\mu>0\) chosen to satisfy the data constraint.
\end{theorem}

\begin{proof}[Proof of \cref{thm:td_lrs_ratio_control}]
Write \(w=w_{T,m}\).  For every positive control with
\(\int_0^T r=D\), Cauchy--Schwarz gives
\begin{equation*}
\mathcal J_{T,m}(r)D
=
\left(\int_0^T\frac{w(u)}{r(u)}\,\mathrm du\right)
\left(\int_0^T r(u)\,\mathrm du\right)
\geq
\left(\int_0^T\sqrt{w(u)}\,\mathrm du\right)^2.
\end{equation*}
Equality gives \cref{eq:td_lrs_ratio_unconstrained}; strict convexity of
\(x\mapsto w(u)/x\) almost everywhere gives uniqueness up to null sets.

For the bounded problem, \(r_-T\leq D\leq r_+T\) is necessary by integration
and sufficient by \(r\equiv D/T\).  A multiplier \(\mu>0\) solves
\begin{equation*}
\int_0^T
\operatorname{clip}
\left(
\sqrt{\frac{w(u)}{\mu}},
r_-,r_+
\right)\,\mathrm du
=D.
\end{equation*}
Its left side is continuous and nonincreasing from \(r_+T\) to \(r_-T\), so
strict feasibility gives a solution.  The clipped control in
\cref{eq:td_lrs_ratio_clipped} pointwise minimizes
\(w(u)/x+\mu x\) on \([r_-,r_+]\), whence
\[
\frac{w(u)}{r(u)}+\mu r(u)
\geq
\frac{w(u)}{r_T^\star(u)}+\mu r_T^\star(u).
\]
Integration cancels the resource terms, proving global optimality, while
strict convexity gives primal uniqueness; at either resource endpoint the
box constraint forces the stated constant control.  The multiplier is unique
exactly when its resource equation is.  Its left side is strictly decreasing
through a solution iff the free set
\[
\left\{u:
r_-<\sqrt{w(u)/\mu}<r_+
\right\}
\]
has positive measure; otherwise it may have a plateau despite primal
uniqueness.
\end{proof}

Because \(\sqrt{\mathcal F(u,m)+\sigma^2}\) accounts for clean and label
noise, the full optimizer need not be monotone.  When fixed nonzero label
noise leads, it reduces to
\begin{equation}
\label{eq:td_lrs_ratio_label_noise_profile}
r_T^\star(u)
=
\operatorname{clip}
\left(
c(1+T-u)^{-1+1/(4\alpha)},
r_-,r_+
\right),
\end{equation}
where \(c\) enforces the data constraint.  In \(\mathrm{FB}_{1,2}\), instead,
\begin{equation}
\label{eq:td_lrs_ratio_low_regular}
\sigma^2m^{1-4\alpha}
\int_0^T\frac{\mathrm du}{r(u)},
\qquad
r_T^\star(u)\equiv\frac DT,
\end{equation}
when feasible, with the same \(\mathrm{FB}_{1,2}\)
proved-source and high-source qualifications stated above.

\paragraph{Two continuum schedule realizations.}
The positive optimizer fixes \(r\), not its learning-rate/batch factorization.
Let \(\tau\) denote continuum iteration time, so
\(\mathrm du/\mathrm d\tau=\eta(\tau)\) and
\(r(u(\tau))=B(\tau)/\eta(\tau)\).  For a fixed learning rate \(\eta_0\), set
\begin{equation}
\label{eq:td_lrs_ratio_jinbo_factorization}
\eta(\tau)\equiv\eta_0,
\qquad
u(\tau)=\eta_0\tau,
\qquad
B(\tau)=\eta_0r_T^\star(\eta_0\tau),
\qquad
0\leq \tau\leq\frac{T}{\eta_0}.
\end{equation}
This is the fixed-learning-rate realization of \citet{wang2026fast}.  For a
fixed continuum batch \(B_0>0\), let \(u(0)=0\) solve
\begin{equation}
\label{eq:td_lrs_ratio_blake_factorization}
\frac{\mathrm du}{\mathrm d\tau}
=\frac{B_0}{r_T^\star(u(\tau))},
\qquad
B(\tau)\equiv B_0,
\qquad
\eta(\tau)=\frac{B_0}{r_T^\star(u(\tau))},
\qquad
0\leq\tau\leq\frac{D}{B_0}.
\end{equation}
Both forms process \(D\) samples and reach intrinsic time \(T\); the second is
the fixed-batch realization of \citet{bordelon2026theory}.  If, on an unclipped
terminal segment,
\[
r_T^\star(u)=c(1+T-u)^{-1+1/(4\alpha)},
\]
then its fixed-batch realization on the corresponding terminal iteration-time
segment is
\begin{equation}
\label{eq:td_lrs_ratio_physical_decay}
\eta(\tau)
=
\frac{B_0}{c}
\left[
1+\frac{D-B_0\tau}{4\alpha c}
\right]^{4\alpha-1}.
\end{equation}
Thus the intrinsic square-root profile becomes iteration-time decay of exponent
\(4\alpha-1\), while clipping becomes learning-rate or batch plateaus.
With \(p=2\alpha+2\beta-1\), the corresponding terminal-decay, peak, risk,
and source exponents are, respectively,
\begin{equation}
\label{eq:td_lrs_ratio_external_map}
4\alpha-1,
\qquad
\frac{\beta}{\alpha+\beta},
\qquad
\frac{p}{p+1},
\qquad
\frac{p}{2\alpha}.
\end{equation}
When fixed label noise leads, \cref{eq:td_lrs_ratio_label_noise_profile}
coincides, after reparameterization by intrinsic time, with the joint
learning-rate--batch-size optimizer of \citet{bordelon2026theory}.
Thus, when fixed label noise leads, the \(\beta>0\) branch recovers the
power-decay shape and the peak and risk exponents of
\citet{bordelon2026theory} and \citet{li2026optimal}.  For \(\beta<0\), lower-ratio
clipping becomes a maximum-learning-rate plateau followed by the same terminal
power decay, recovering their WSD structure and hard-branch risk exponent.
The easy--hard boundary is \(\beta=0\).  These comparisons are theorem-level
only in the trace-class regime \(\alpha>1/2\), away from the excluded boundary;
below trace class the correspondence is algebraic.  The constant-learning-rate
factorization similarly matches \citet{wang2026fast}.  Related work optimizes batch
given learning rate \citep{li2026towards} or proves their finite-sample
equivalence \citep{meterez2026seesaw}; here ratio optimization follows after
the spectral forcing and memory asymptotics of
\cref{thm:noisy_43_effective_spectral_criterion}.

\paragraph{From continuum control to integer SGD.}
Here \emph{admissible} means the full class of
\cref{def:admissible_integer_schedules}, with common constants fixed as
budgets grow and no finite upper ratio cap in the scaling infima.  Equal-sample
quantiles discretize the continuum profile, using \(B_s\equiv1\) for every
upper bound.  Let \(\mathcal R_\sigma^\star(D)\) be the terminal DE-risk
infimum subject to
\(\sum_{s<t}B_s\leq D\), and define
\(\mathcal R_\sigma^\star(\mathfrak f)\) analogously under
\(m\sum_{s<t}B_s\leq\mathfrak f\).  Integer realization loses no exponent and
no admissible schedule improves the orders, as proved below.

\paragraph{Proof of the phasewise optimal rates.}
We now prove \cref{cor:td_lrs_subregime_exponents} and the rates summarized
in \cref{tab:td_lrs_ratio_optimal_exponents}.  The propagation subregime
fixes the clean-risk and memory scalings, the ratio path \(r\) allocates
stochastic error, and the resource constraint determines the trade-off
between width and data.  Equal-sample quantiles yield admissible integer
realizations with \(B_s\equiv1\), and the lower objectives below show that no
admissible schedule improves the displayed orders.

At \(\alpha=1\), the two \(\mathrm{IM}_{2,3}\) feature-compute expressions
agree at \(\mathfrak f^{-1/2}\).  For \(\alpha>1\), the optimizer reaches
\(T\asymp m^{2\alpha}\), which is a source-window saturation rather than a
new propagation regime.  The lower-ratio floor is active in
\(\mathrm{IM}_1\) when \(\beta<0\).  At \(\beta=0\), the neighboring
exponents agree, although the open-branch statements exclude the crossover.
The high-source \(\mathrm{FB}_2\) rate retains the conditional qualification
in \cref{ass:td_lrs_high_source_fb2_forcing,%
cor:td_lrs_joint_de_high_source_fb2}.

The subregime fixes the clean-risk and memory scalings, \(r\) allocates noise, and the
compute budget determines the trade-off between width and data.  Unit-batch
schedules attain every proved order.
At \(\alpha=1\), the two \(\mathrm{IM}_{2,3}\) compute expressions agree at
\(\mathfrak f^{-1/2}\).  For \(\alpha>1\), the optimizer reaches
\(T\asymp m^{2\alpha}\), a source-window saturation rather than a new
subregime.  The lower-ratio floor is active in \(\mathrm{IM}_1\) for
\(\beta<0\); at fixed data its \(\beta>0\) and \(\beta<0\) branches have peaks
\[
D^{\,1-\frac1{2(p+1)}}
\quad\text{and}\quad
D^{\,\frac12(1+p/(2\alpha))},
\]
respectively, so any cap must exceed the relevant order.  At \(\beta=0\) the
neighboring exponents agree, but open-branch statements exclude the crossover.

We now prove \cref{cor:td_lrs_subregime_exponents} and the consequences in
\cref{subsec:td_lrs_ratio_control_integration}: specialize the continuum
solution, realize its profiles with integer \(B_s=1\) schedules, and match the
DE lower objectives under data and FLOP budgets.

\paragraph{Subregime profiles.}
Substitution of \cref{eq:td_lrs_ratio_weight} into the square-root optimizer gives,
on the LM and IM branches,
\[
r_T^\star(u)
\propto
(1+T-u)^{-1+1/(4\alpha)}
\sqrt{\mathcal F(u,m)+\sigma^2}.
\]
Its fixed nonzero label-noise component yields
\cref{eq:td_lrs_ratio_label_noise_profile}.  In \(\mathrm{FB}_1\) and
\(\mathrm{FB}_2\),
\[
\left(\int_0^T\frac{\mathrm du}{r(u)}\right)
\left(\int_0^T r(u)\,\mathrm du\right)\geq T^2
\]
instead gives the constant optimizer \cref{eq:td_lrs_ratio_low_regular}.

The next lemma turns the required ratio profiles into genuine SGD schedules;
it is also the achievability step in
\cref{cor:td_lrs_subregime_exponents}.

For integers \(D\geq1\), a horizon \(T>0\), and \(\alpha>1/4\), define
\[
I_\alpha(T):=\int_0^T(1+T-u)^{-1+1/(4\alpha)}\,\mathrm du,
\qquad
\bar r(u):=\frac{D}{I_\alpha(T)}(1+T-u)^{-1+1/(4\alpha)}.
\]
Let \(0=u_0<\cdots<u_D=T\) be the data-quantile partition determined by
\(\int_0^{u_s}\bar r(v)\,\mathrm dv=s\), and set \(B_s:=1\) and
\(\eta_s:=u_{s+1}-u_s\).

\begin{lemma}[Integer realization of the subregime optimizers]
\label{lem:td_lrs_integer_reachability}
Let \(D\to\infty\) through integers and let \(T\to\infty\) with
\(T/D\lesssim1\).  Fix \(\alpha>0\), \(\alpha\neq1/4\).  For
\(\alpha>1/4\), the preceding deterministic integer-batch schedule uses
exactly \(D\) samples, reaches intrinsic time \(T\), and satisfies
\[
\sum_{s<D}\eta_s^2
\bigl(1+T-u_{s+1}\bigr)^{-2+1/(2\alpha)}
\asymp
\frac{T^{1/(2\alpha)}}D.
\]
For \(1/4<\alpha<1/2\), the same order is attained by the simpler constant
choice \(B_s=1\), \(\eta_s=T/D\).  For \(0<\alpha<1/4\), that constant
choice satisfies \(\sum_{s<D}\eta_s^2=T^2/D\).

In the \(\beta<0\) branch of subregime \(\mathrm{IM}_1\), there is a horizon
\(T\asymp D\) and an admissible unit-batch schedule for which
\[
\sum_{s<D}\eta_s^2
\bigl(1+T-u_{s+1}\bigr)^{-2+1/(2\alpha)}
=o\!\left(D^{-p/(2\alpha)}\right).
\]
These schedules satisfy the peak and bounded-batch conditions at the
subregime-wise horizons and widths used in the resource bounds.
\end{lemma}

\begin{proof}
Assume first \(\alpha>1/4\).  The data-quantile definition gives
\[
\sum_{s<D}B_s=D,
\qquad
\sum_{s<D}\eta_s=T,
\qquad
\frac1{\eta_s}
=\frac1{u_{s+1}-u_s}\int_{u_s}^{u_{s+1}}\bar r(v)\,\mathrm dv.
\]
Thus the discrete ratio \(B_s/\eta_s\) is the cell average of \(\bar r\).
Moreover, \(\inf_u\bar r(u)\asymp D/T\), so \(\eta_s\lesssim T/D\).
The shifted powers \(\bar r(u)\) and
\((1+T-u)^{-2+1/(2\alpha)}\) vary by only constant factors on each
unit-data cell.  Consequently the discrete injection sum is comparable to
\[
\int_0^T
\frac{(1+T-u)^{-2+1/(2\alpha)}}{\bar r(u)}\,\mathrm du
\asymp
\frac{I_\alpha(T)^2}{D}
\asymp
\frac{T^{1/(2\alpha)}}D.
\]
For a constant step, direct summation gives the same long-memory order
when \(1/4<\alpha<1/2\).

If \(0<\alpha<1/4\), take \(B_s=1\) and \(\eta_s=T/D\) directly.  Then
\[
\sum_{s<D}B_s=D,
\qquad
\sum_{s<D}\eta_s=T,
\qquad
\sum_{s<D}\eta_s^2=\frac{T^2}{D},
\]
which proves the finite-bulk assertion.

For the \(\beta<0\) branch of subregime \(\mathrm{IM}_1\), put
\[
s_0:=\frac{p}{2\alpha},
\qquad h:=1-\frac1{2\alpha},
\qquad L:=D^{(1+s_0)/(1+h)},
\]
choose a sufficiently large fixed \(r_0\), and set
\[
\bar r(u):=r_0\max\!\left\{
1,\left(\frac{1+L}{1+T-u}\right)^{1-1/(4\alpha)}
\right\},
\qquad
\int_0^T\bar r(u)\,\mathrm du=D.
\]
Then \(T\asymp D\), and the same quantile construction has
\[
\sum_{s<D}\eta_s^2
\bigl(1+T-u_{s+1}\bigr)^{-2+1/(2\alpha)}
\lesssim L^{-h}=o\!\left(D^{-p/(2\alpha)}\right),
\]
because \(s_0<h\) and
\(h(1+s_0)/(1+h)>s_0\).  Here \(\eta_s\leq1/r_0\), so the clipped profile
and shifted kernel are uniformly comparable on every cell; hence the
continuum \(L^{-h}\) bound transfers to the displayed discrete sum.  Finally,
\(B_s\equiv1\) supplies the bounded-batch propagation alternative.

It remains to check admissibility at these scales.  Put
\(x=T^{1/(2\alpha)}\).  The largest step is \(O(T/D)\), except that the
\(\beta<0\) \(\mathrm{IM}_1\) construction directly has
\(\eta_s\leq1/r_0\).  We have \(x\asymp m\) in both \(\mathrm{IM}_1\)
branches, \(\mathrm{FB}_1\), \(\mathrm{LM}_1\), \(\mathrm{LM}_2\), and the
\(\alpha>1\) source-saturated \(\mathrm{IM}_2\)--\(\mathrm{IM}_3\) branch; whereas
\(x/m\to0\) in \(\mathrm{FB}_2\), \(\mathrm{LM}_3\), and the
\(1/2<\alpha\leq1\) interior \(\mathrm{IM}_2\)--\(\mathrm{IM}_3\) branch.  Thus every
source window holds, strictly in \(\mathrm{FB}_2\) because \(p>2\alpha\).
For \(\alpha<1/2\), the \(\mathrm{FB}_1,\mathrm{LM}_1,\mathrm{LM}_2\)
choices satisfy
\[
\frac{T/D}{m^{2\alpha-1}}\asymp x^{-p}\longrightarrow0,
\]
while the corresponding ratios in \(\mathrm{FB}_2\) and \(\mathrm{LM}_3\)
are \(x^{-2\alpha}\to0\) and vanish by the defining inequalities,
respectively.  Since the normalized peak tends to zero, the peak condition
holds for all sufficiently large widths.  The row-mass
certificate \cref{eq:td_lrs_joint_row_mass_certificate} for \(\alpha>1/2\)
and \cref{lem:td_lrs_common_volterra_reduction} below trace class give the
Volterra margin; the strict \(\mathrm{FB}_2\) window has vanishing row mass.  Rounding
\(m,D\) changes only fixed factors.
\end{proof}

\paragraph{Two continuum schedule parameterizations.}
For fixed learning rate, \(u(\tau)=\eta_0\tau\) and
\(B(\tau)=\eta_0r_T^\star(\eta_0\tau)\), so
\[
\int_0^{T/\eta_0}B(\tau)\,\mathrm d\tau
=\int_0^Tr_T^\star(u)\,\mathrm du=D.
\]
This proves \cref{eq:td_lrs_ratio_jinbo_factorization}.  For fixed batch,
\cref{eq:td_lrs_ratio_blake_factorization} gives
\(\mathrm d\tau/\mathrm du=r_T^\star(u)/B_0\).  Hence the terminal iteration
time is \(D/B_0\), and
\(\int_0^{D/B_0}B_0\,\mathrm d\tau=D\).

On a terminal, subsequently unclipped segment with
\(r_T^\star(u)=c(1+T-u)^{-1+1/(4\alpha)}\), put \(x=T-u\).  Throughout its
constant-batch-size time coordinate,
\[
\frac{D}{B_0}-\tau
=
\frac1{B_0}\int_u^T r_T^\star(v)\,\mathrm dv
=
\frac{4\alpha c}{B_0}
\left[
(1+x)^{1/(4\alpha)}-1
\right].
\]
Solving for \(x\) and substituting into
\(\eta=B_0/r_T^\star\) gives
\cref{eq:td_lrs_ratio_physical_decay} with decay exponent \(4\alpha-1\).
In other clipped configurations we assert only the plateau correspondence:
lower/upper ratio clipping becomes maximum/minimum learning-rate plateaus in
the constant-batch-size parameterization and minimum/maximum batch plateaus
in the constant-learning-rate parameterization.  Direct substitution gives
\cref{eq:td_lrs_ratio_external_map}.  In the external conventions, the spectral
parameter is \(2\alpha\), the source parameter of \citet{li2026optimal} is
\(p/(2\alpha)\), and the target-decay parameter of
\citet{bordelon2026theory} is \(p+1\).  Hence their decay, peak, and risk
exponents match the displayed native exponents, while both trace-class
conditions reduce to \(\alpha>1/2\).

\paragraph{Matching lower bounds.}
Let \(D_t:=\sum_{s<t}B_s\leq D\).  Positivity of the DE Volterra recursion and
\cref{eq:td_lrs_joint_de_kernel_law} give, for \(\alpha>1/4\),
\[
\mathcal R_{\sigma,t}\gtrsim
\Phi_{\alpha,\beta}(T,m)
+\sigma^2\sum_{s<t}\frac{\eta_s^2}{B_s}
\bigl(1+T-T_{s+1}\bigr)^{-2+1/(2\alpha)}.
\]
Discrete Cauchy--Schwarz and the right Riemann sum imply
\begin{align*}
D_t\sum_{s<t}\frac{\eta_s^2}{B_s}
\bigl(1+T-T_{s+1}\bigr)^{-2+1/(2\alpha)}
&\geq
\left[
\sum_{s<t}\eta_s
\bigl(1+T-T_{s+1}\bigr)^{-1+1/(4\alpha)}
\right]^2\\
&\gtrsim T^{1/(2\alpha)}.
\end{align*}
For \(0<\alpha<1/4\), the target-free gap bound in
\cref{prop:td_lrs_joint_de_finite_bulk_kernel}, together with the proved
forcing comparison when \(p>0\), \(\beta<1+2\alpha\), and the parameters
are off the critical lines, instead gives
\[
\mathcal R_{\sigma,t}\gtrsim
\Phi_{\alpha,\beta}(T,m)
+\sigma^2m^{1-4\alpha}\sum_{s<t}\frac{\eta_s^2}{B_s}
\gtrsim
\Phi_{\alpha,\beta}(T,m)
+\sigma^2m^{1-4\alpha}\frac{T^2}{D}.
\]
These bounds use the integer sample count, not a continuum factorization.  In
the \(\beta<0\) branch of subregime \(\mathrm{IM}_1\), the lower ratio floor
also gives \(T\lesssim D\).

For compactness put \(x:=T^{1/(2\alpha)}\), so that the source window is
\(x\lesssim m\).  Ignoring fixed constants, the preceding inequalities give
the following lower objectives:
\begin{align*}
&m^{-p}+x^{-p}+x/D
&&\text{in subregimes \(\mathrm{LM}_1\), \(\mathrm{LM}_2\), and \(\mathrm{IM}_1\)},\\
&m^{-2\alpha}+x^{-p}+x/D
&&\text{in subregime \(\mathrm{LM}_3\)},\\
&m^{-2\alpha}+x^{-p}+m^{-1}x^{-(2\alpha-1)}+x/D
&&\text{in subregimes \(\mathrm{IM}_2\) and \(\mathrm{IM}_3\)},\\
&m^{-p}+x^{-p}+m^{1-4\alpha}x^{4\alpha}/D
&&\text{in subregime \(\mathrm{FB}_1\)},\\
&m^{-2\alpha}+x^{-p}+m^{1-4\alpha}x^{4\alpha}/D
&&\text{in subregime \(\mathrm{FB}_2\)}.
\end{align*}

The only compute balance that changes inside the \(\mathrm{IM}\) region is
the following one.  We isolate it here because both the unrestricted schedule
problem and the fixed-batch schedule comparison use the same four terms.

\begin{lemma}[Joint \(\mathrm{IM}_2\)--\(\mathrm{IM}_3\) compute balance]
\label{lem:td_lrs_im23_compute_balance}
Let \(\alpha>1/2\), \(\beta>1/2\),
\(p=2\alpha+2\beta-1\), and \(F\to\infty\).  Then
\[
\inf_{1\ll x\lesssim m}
\left\{
x^{-p}+m^{-2\alpha}+m^{-1}x^{-(2\alpha-1)}+\frac{mx}{F}
\right\}
\asymp
\begin{cases}
F^{-p/(1+p+2\beta)},&1/2<\alpha\leq1,\\[0.2em]
F^{-\alpha/(1+\alpha)},&\alpha>1.
\end{cases}
\]
For \(1/2<\alpha\leq1\), one attaining choice is
\[
x\asymp F^{1/(1+p+2\beta)},
\qquad m\asymp x^{2\beta}.
\]
For \(\alpha>1\), the optimum reaches the source endpoint:
\[
x\asymp m\asymp F^{1/[2(1+\alpha)]}.
\]
At \(\alpha=1\), the rate is \(F^{-1/2}\), but the allocation is not unique:
every choice satisfying
\[
mx\asymp F^{1/2},
\qquad
F^{1/(2p)}\lesssim x\lesssim F^{1/4}
\]
has the same order.
\end{lemma}

\begin{proof}
For \(1/2<\alpha\leq1\), apply weighted AM--GM to the forcing,
label-noise, and feature-distortion terms with weights
\[
\frac{2(1-\alpha)}{1+p+2\beta},
\qquad
\frac{p}{1+p+2\beta},
\qquad
\frac{p}{1+p+2\beta}.
\]
They are nonnegative, sum to one, and their weighted product is
\(F^{-p/(1+p+2\beta)}\).  At the displayed choice of \(x\) and \(m\),
those three terms match.  Moreover,
\(4\alpha\beta-p=(2\alpha-1)(2\beta-1)>0\), so the width floor is lower
order, and \(x\lesssim m\).

For \(\alpha>1\), use instead the label-noise, feature-distortion, and
width-floor weights
\[
\frac{\alpha}{1+\alpha},
\qquad
\frac{\alpha}{(2\alpha-1)(1+\alpha)},
\qquad
\frac{\alpha-1}{(2\alpha-1)(1+\alpha)}.
\]
Their weighted product is \(F^{-\alpha/(1+\alpha)}\).  Taking
\(x\asymp m\asymp F^{1/[2(1+\alpha)]}\) matches these three terms, while
\(p>2\alpha\) makes \(x^{-p}\) lower order.  When \(\alpha=1\), direct
substitution gives the stated family of attaining allocations.
\end{proof}

\paragraph{Attaining scales.}
Because \(\sigma^2>0\) is fixed and the clean forcing is bounded, the upper
FSL has the same order as these label-noise objectives for the schedules in
\cref{lem:td_lrs_integer_reachability}.  At fixed data, the matching choices
are
\[
\begin{aligned}
(\mathrm{IM}_1,\ \beta<0):\quad
&m\asymp D^{1/(2\alpha)}, &&T\asymp D;\\
(\mathrm{IM}_1,\ \beta>0),\ \mathrm{FB}_1,\ \mathrm{LM}_1,\ \mathrm{LM}_2:\quad
&m\asymp D^{1/(1+p)}, &&T\asymp D^{2\alpha/(1+p)};\\
\mathrm{FB}_2:\quad
&m\asymp D^{\frac{p}{p(1-2\alpha)+8\alpha^2}},
&&T\asymp D^{\frac{4\alpha^2}{p(1-2\alpha)+8\alpha^2}};\\
\mathrm{LM}_3:\quad
&m\asymp D^{\frac{p}{2\alpha(1+p)}},
&&T\asymp D^{2\alpha/(1+p)};\\
\mathrm{IM}_2,\ \mathrm{IM}_3:\quad
&m\asymp D^{2\beta/(1+p)},
&&T\asymp D^{2\alpha/(1+p)}.
\end{aligned}
\]
The first uses the plateau--ramp construction; the second and last use the
data-quantile square-root schedule in \(\mathrm{IM}\), while all listed
\(\mathrm{FB}\) and \(\mathrm{LM}\) cases use the constant \(B_s=1\)
construction.  Direct substitution gives the fixed-data column of
\cref{tab:td_lrs_ratio_optimal_exponents}.

On \(mD\leq\mathfrak f\), the matching choices for, respectively,
\((\mathrm{IM}_1,\beta<0)\); the shared
\((\mathrm{IM}_1,\beta>0),\mathrm{FB}_1,\mathrm{LM}_1,\mathrm{LM}_2\) class;
\(\mathrm{FB}_2\); and \(\mathrm{LM}_3\), are
\[
\begin{aligned}
m&\asymp\mathfrak f^{1/(1+2\alpha)},
&D&\asymp\mathfrak f^{2\alpha/(1+2\alpha)},
&T&\asymp D;\\
m&\asymp\mathfrak f^{1/(2+p)},
&D&\asymp\mathfrak f^{(1+p)/(2+p)},
&T&\asymp\mathfrak f^{2\alpha/(2+p)};\\
m&\asymp\mathfrak f^{\frac{p}{p(2-2\alpha)+8\alpha^2}},
&D&\asymp\mathfrak f^{\frac{p(1-2\alpha)+8\alpha^2}
{p(2-2\alpha)+8\alpha^2}},
&T&\asymp\mathfrak f^{\frac{4\alpha^2}{p(2-2\alpha)+8\alpha^2}};\\
m&\asymp\mathfrak f^{\frac{p}{p+2\alpha(1+p)}},
&D&\asymp\mathfrak f^{\frac{2\alpha(1+p)}{p+2\alpha(1+p)}},
&T&\asymp\mathfrak f^{\frac{4\alpha^2}{p+2\alpha(1+p)}}.
\end{aligned}
\]
Balancing proves the corresponding table entries.  Indeed,
\(x\lesssim m\) and \(1-4\alpha>0\) lower-bound the \(\mathrm{FB}_1\) noise
by \(x/D\), explaining its shared rate, while \(p>2\alpha\) makes the
\(\mathrm{FB}_2\) source window strict.  If \(\varepsilon\) bounds all three
\(\mathrm{FB}_2\) terms, then \(m\gtrsim \varepsilon^{-1/(2\alpha)}\) and
\(x\gtrsim \varepsilon^{-1/p}\); substitution into the noise term gives
\[
\varepsilon\gtrsim
D^{-\frac{2\alpha p}{p(1-2\alpha)+8\alpha^2}}
\quad\text{or}\quad
\varepsilon\gtrsim
\mathfrak f^{-\frac{2\alpha p}{p(2-2\alpha)+8\alpha^2}},
\]
under the data and compute budgets, respectively.

For \(\mathrm{IM}_2\) and \(\mathrm{IM}_3\), apply
\cref{lem:td_lrs_im23_compute_balance} with \(F=\mathfrak f\).  Together
with \(D=\mathfrak f/m\) and \(T=x^{2\alpha}\), its two attaining choices
give exactly the two compute branches in
\cref{tab:td_lrs_ratio_optimal_exponents}.  Both use the data-quantile
\(B_s=1\) schedule from \cref{lem:td_lrs_integer_reachability}; at
\(\alpha=1\) the rates agree at \(\mathfrak f^{-1/2}\), with no forced
logarithmic correction and no unique width allocation.

All upper constructions have \(B_s\equiv1\), so batch integrality and the
bounded-batch propagation alternative are exact; rounding \(m,D\) preserves
orders.  The uniform FSL, these lower bounds, and
\cref{lem:td_lrs_integer_reachability} prove
\cref{cor:td_lrs_subregime_exponents,tab:td_lrs_ratio_optimal_exponents}
over the fixed admissible class and the stated source range.  The high-source
\(\mathrm{FB}_2\) continuation separately invokes
\cref{ass:td_lrs_high_source_fb2_forcing}.

\endgroup

\section{Sharp transfer for regularly varying joint schedules}
\label{app:rv_joint_schedule_proof}

This appendix gives the full regular-variation version of the schedule results
stated in the main text.  Appendix~\ref{app:rv_spectral_memory_transfer}
converts the low-spectrum tail into a two-time memory kernel.
Appendix~\ref{app:condition_examples} uses examples and counterexamples to show
what the spectral, uniformity, and stability assumptions permit and exclude.
Appendix~\ref{app:rv_schedule_transfer} derives the noisy--clean response in
long memory, including the effect of feedback, while
Appendix~\ref{app:rv_integrable_memory} treats the two endpoint contributions
that arise under integrable memory.  Appendix~\ref{app:rv_schedule_identification}
proves the converse that recovers \(B/\eta\) from the observed gap below the
memory ceiling.  Appendix~\ref{app:rv_total_loss} combines the gap with the
clean risk to obtain the preserve--change--destroy classification, and
Appendix~\ref{app:rv_physical_time} translates the intrinsic-time results to
polynomial schedules in iteration time.

We first state the regular-variation condition shared by these results.

\begin{paperdefinition}[Width--time triangular sequence]
\label{def:rv_width_time_triangular_sequence}
A width--time triangular sequence is a family indexed by the width \(m\).  At
width \(m\), fix a finite system \(\bm W_m\), a deterministic schedule
\(\{(\eta_{m,s},B_{m,s})\}_{s<t_m}\), and a terminal index \(t_m\), and set
\[
T_m:=\sum_{s<t_m}\eta_{m,s}.
\]
Both the finite system and the observation horizon may vary with \(m\), and
the limit is taken with \(m,t_m,T_m\to\infty\).  Uniform statements along the
sequence use constants and bounds independent of \(m\) and \(t_m\).  Below,
we write \(t\) and \(T_t\) for the terminal pair \(t_m\) and \(T_m\) when no
confusion can arise.
\end{paperdefinition}

By contrast, a fixed infinite-spectrum system keeps the dynamics and schedule
fixed and sends only the terminal time to infinity.

\begin{assumption}[Regularly varying schedule path]
\label{ass:rv_schedule_path}
Consider either one fixed infinite-spectrum system or a width--time triangular
sequence in \cref{def:rv_width_time_triangular_sequence}; omit the width index
in the fixed system.
A schedule path satisfies the required regularity if:
\begin{enumerate}[label=(\alph*)]
\item For some \(\vartheta\in\mathbb R\), \(1/r_m\) is eventually positive
and monotone, and
\begin{equation}
\frac{r_m(T_t)}{r_m(xT_t)}\longrightarrow x^{-\vartheta}
\label{eq:rv_uniform_schedule}
\end{equation}
for every \(x>0\).

\item For every sufficiently small \(\epsilon>0\), there is a finite
constant \(C_\epsilon\), independent of the width and horizon, such that
eventually
\begin{equation}
\frac{r_m(T_t)}{r_m(xT_t)}
\leq C_\epsilon
\max\{x^{-\vartheta-\epsilon},x^{-\vartheta+\epsilon}\},
\qquad x>0.
\label{eq:rv_uniform_schedule_potter}
\end{equation}
In a triangular sequence, the convergence in part \emph{(a)} is locally
uniform and eventual monotonicity begins beyond a common intrinsic time.
\end{enumerate}
These conditions permit width-dependent slow variation but exclude hidden
oscillations or bursts; polynomial width amplitudes must be tracked
separately.
\end{assumption}

\begin{lemma}[Uniform Karamata consequence]
\label{lem:rv_uniform_karmata}
Under \cref{ass:rv_schedule_path}, if \(\vartheta<1\), then
\begin{equation*}
\int_0^{xT_t}\frac{\mathrm du}{r_m(u)}
\sim \frac{x^{1-\vartheta}}{1-\vartheta}
\frac{T_t}{r_m(T_t)},
\end{equation*}
locally uniformly for \(x>0\).
\end{lemma}

\begin{proof}
Choose \(\epsilon<1-\vartheta\) in
\cref{eq:rv_uniform_schedule_potter}.  After the change of variables
\(u=yT_t\), the Potter bound supplies an integrable envelope, so dominated
convergence gives the result.
\end{proof}

We use the following two-time notation throughout this section.  For the
schedule coordinates in \cref{eq:joint_sgd_schedule_coordinates}, let
\begin{equation*}
H_{t,s}:=T_t-T_{s+1},
\qquad
\Psi_{t,s}:=
\sum_j\widehat\lambda_j^2Q_{s+1,t}(\widehat\lambda_j).
\end{equation*}
Here \(H_{t,s}\) is the age of the injection at time \(s\), while
\(\Psi_{t,s}\) is its surviving variance before weighting by the injected
noise; thus \(K_{t,s}=(T_{s+1}-T_s)\Psi_{t,s}/r_s\).
We use the piecewise-constant interpolation \(r(u)=r_s\) for
\(T_s\leq u<T_{s+1}\), and write
\(h_t:=\max_{s<t}(T_{s+1}-T_s)\) for the largest intrinsic-time step.

\subsection{From the empirical spectral tail to two-time memory}
\label{app:rv_spectral_memory_transfer}

The schedule-transfer results use the following two-time kernel condition.
The theorem below derives it from the empirical low-spectrum tail.

\begin{assumption}[Long-memory two-time kernel control]
\label{ass:rv_long_memory_kernel_control}
Consider either one fixed infinite-spectrum system or a width--time triangular
sequence in \cref{def:rv_width_time_triangular_sequence}; omit the width index
in the fixed system.
The two-time kernel has the following properties for some
\(0<q_{\mathcal K}<1\):
\begin{enumerate}[label=(\alph*)]
\item The reference profile satisfies
\begin{equation}
k(v)\sim v^{-q_{\mathcal K}}L_{\mathcal K}(v),
\label{eq:rv_kernel_rv}
\end{equation}
where \(L_{\mathcal K}\) is eventually positive and slowly varying.

\item For every fixed \(\varepsilon>0\), the bulk-age kernel satisfies
\begin{equation}
\sup_{\substack{s<t:\ H_{t,s}\geq\varepsilon T_t}}
\left|\frac{\Psi_{t,s}}{k(H_{t,s})}-1\right|
\longrightarrow0.
\label{eq:rv_sharp_bulk_kernel}
\end{equation}

\item There are nonincreasing envelopes \(G_m\) such that
\begin{equation}
0\leq\Psi_{t,s}\leq G_m(H_{t,s}),
\qquad
\sup_mG_m(0)<\infty,
\label{eq:rv_kernel_envelope}
\end{equation}
and, for every sufficiently small \(\xi>0\), there is a finite constant
\(C_\xi\), independent of \(m,t\), such that
\begin{equation}
\limsup_{m,t\to\infty}
\frac{\int_0^{\varepsilon T_t}G_m(v)\,\mathrm dv}
{T_tk(T_t)}
\leq C_\xi\varepsilon^{1-q_{\mathcal K}-\xi},
\qquad 0<\varepsilon<\frac12.
\label{eq:rv_integrated_envelope}
\end{equation}

\item The intrinsic mesh satisfies
\begin{equation}
h_t=o\!\left(T_tk(T_t)\right).
\label{eq:rv_mesh_condition}
\end{equation}
\end{enumerate}
\end{assumption}

Since \(Tk(T)\) has positive index \(1-q_{\mathcal K}\), it diverges, so a
uniformly bounded learning rate satisfies \cref{eq:rv_mesh_condition}.  These
are the main text's precise fine-mesh long-memory conditions.

The disjoint age intervals and monotonicity of \(G_m\) give the repeatedly used
recent-endpoint bound
\begin{equation}
\sum_{H_{t,s}\leq x}(T_{s+1}-T_s)G_m(H_{t,s})
\leq h_tG_m(0)+\int_0^xG_m(v)\,\mathrm dv.
\label{eq:rv_recent_cell_envelope}
\end{equation}

The next conditional, forward theorem transfers one-time empirical tails to
two-time memory; global Potter control is needed at recent injection.

\begin{theorem}[Sharp empirical spectral transfer to two-time memory]
\label{thm:rv_spectral_memory_transfer}
Condition on a deterministic empirical width--time sequence indexed by \(m\),
with terminal index \(t=t_m\) and \(T_t\to\infty\).  Fix
\(0<q_{\mathcal K}<1\) and an eventually positive slowly varying function
\(L_{\mathcal K}\).  Suppose, locally uniformly for \(c>0\),
\begin{equation}
\sum_{0<\widehat\lambda_j\leq c/T_t}\widehat\lambda_j^2
\sim
\frac{2^{q_{\mathcal K}}}{\Gamma(q_{\mathcal K}+1)}
c^{q_{\mathcal K}}T_t^{-q_{\mathcal K}}L_{\mathcal K}(T_t),
\label{eq:rv_spectral_memory_input}
\end{equation}
and assume the corresponding global Potter/no-escape bound
\cref{eq:noisy_43_uniform_potter_envelope} for every sufficiently small
slack \(\xi\in(0,\min\{q_{\mathcal K},1-q_{\mathcal K}\})\), with \(T=T_t\)
and the normalization in \cref{eq:rv_spectral_memory_input}.  Under part
\emph{(a)} of \cref{ass:td_lrs_joint_schedule_stability} and
\cref{eq:rv_mesh_condition}, for every fixed
\(\varepsilon>0\),
\begin{equation}
\sup_{\substack{s<t:\ H_{t,s}\geq\varepsilon T_t}}
\left|
\frac{\Psi_{t,s}}
{H_{t,s}^{-q_{\mathcal K}}L_{\mathcal K}(H_{t,s})}-1
\right|
\longrightarrow0.
\label{eq:rv_spectral_memory_transfer}
\end{equation}
If, in addition,
\(\sup_m\sum_j\widehat\lambda_j^2<\infty\) on the conditioned sequence,
then there are envelopes satisfying
\cref{eq:rv_kernel_envelope,eq:rv_integrated_envelope} for the profile in
\cref{eq:rv_kernel_rv}.  Hence
\cref{ass:rv_long_memory_kernel_control} holds under the additional trace
bound.
\end{theorem}

\begin{proof}
Transfer the cutoff over bulk ages, then construct the recent-endpoint
envelope.  Write \(H=H_{t,s}\).  Since
\(H/T_t\in[\varepsilon,1]\),
\cref{eq:rv_spectral_memory_input}, the uniform convergence theorem for
slowly varying functions, and the global Potter bound give, locally uniformly
for \(c>0\),
\[
\sum_{0<\widehat\lambda_j\leq c/H}\widehat\lambda_j^2
\sim
\frac{2^{q_{\mathcal K}}}{\Gamma(q_{\mathcal K}+1)}
c^{q_{\mathcal K}}H^{-q_{\mathcal K}}L_{\mathcal K}(H),
\]
uniformly over all such ages.

Joint pointwise stability gives
\(0\leq Q_{s+1,t}(\lambda)\leq e^{-\delta H\lambda}\) for its fixed margin
\(\delta>0\).
Moreover, \cref{eq:rv_mesh_condition} and \(k(T)\to0\) imply
\(h_t/T_t\to0\).  On every compact critical band
\(a\leq H\lambda\leq b\), expanding the one-step factors gives
\(Q_{s+1,t}(\lambda)=e^{-2H\lambda}(1+o(1))\) uniformly.  If the supremum in
\cref{eq:rv_spectral_memory_transfer} did not
converge to zero, there would be a sequence of injection ages along which the
normalized error remained bounded away from zero.  Along this sequence, the
critical-band and cutoff limits give the Laplace--Stieltjes limit on compact
rescaled intervals, while Potter control and the exponential bound remove
both endpoints as in \cref{thm:noisy_43_uniform_tauberian}.  Hence
\[
\frac{\Psi_{t,s}}{H^{-q_{\mathcal K}}L_{\mathcal K}(H)}
\longrightarrow
\frac{2^{q_{\mathcal K}}}{\Gamma(q_{\mathcal K}+1)}
\Gamma(q_{\mathcal K}+1)2^{-q_{\mathcal K}}=1,
\]
a contradiction, proving \cref{eq:rv_spectral_memory_transfer} and, with
\cref{eq:rv_kernel_rv}, \cref{eq:rv_sharp_bulk_kernel}.

For the endpoint, take
\(G_m(v):=\sum_j\widehat\lambda_j^2e^{-\delta v\widehat\lambda_j}\).
It is nonincreasing, bounds \(\Psi_{t,s}\), and is uniformly bounded at zero.
The global Potter envelope yields
\(G_m(v)\lesssim k(T)(T/v)^{q_{\mathcal K}+\xi}\) for
\(0<v\leq\varepsilon T\): integration by parts writes
\(G_m(v)=\delta v\int_0^\infty e^{-\delta vx}A_m(x)\,\mathrm dx\), where
\(A_m(x)=\sum_{0<\widehat\lambda_j\leq x}\widehat\lambda_j^2\), and the
two branches of the Potter bound give the displayed estimate.  Integration,
using
\(q_{\mathcal K}+\xi<1\), yields
\[
\int_0^{\varepsilon T}G_m(v)\,\mathrm dv
\lesssim \varepsilon^{1-q_{\mathcal K}-\xi}Tk(T).
\]
This is \cref{eq:rv_integrated_envelope}; the remaining envelope and mesh
conditions were assumed.
\end{proof}

For PLRF with \(d/m\to c>1\) and \(1/4<\alpha<1/2\), the endpoint trace is
tight because
\[
\mathbb E_{\bm W}\operatorname{tr}(\widehat{\bm H}^{\,2})
=
\left(1+\frac1m\right)\operatorname{tr}(\bm\Lambda^2)
+\frac{(\operatorname{tr}\bm\Lambda)^2}{m}=O(1).
\]
Thus \(\operatorname{tr}(\widehat{\bm H}^{\,2})=O_{\mathbb P}(1)\), excluding
polynomial endpoint blowup.

\subsection{Examples and counterexamples for the transfer assumptions}
\label{app:condition_examples}

The spectral criterion concerns weighted mass near zero rather than a
polynomial formula for individual eigenvalues.  The following table collects
two component-power examples and two elementary failures.

\noindent

\paragraph{Stretched-exponential spectra.}
We examine how stretched-exponential spectral decay affects the forcing and
memory component asymptotics.  The calligraphic components below are
specialization-specific resolvent-DE quantities.

\begin{paperdefinition}[Stretched-exponential spectral specialization]
\label{def:stretched_exponential_specialization}
Fix \(0<\gamma<1\), \(c>1\), and \(\beta\in\mathbb{R}\), let
\(d=\lceil cm\rceil\), and set
\[
\lambda_j=e^{-j^\gamma},
\qquad
\lvert\theta_j^\star\rvert\asymp j^{-\beta}.
\]
\end{paperdefinition}

The eigenvalue sequence in
\cref{def:stretched_exponential_specialization} is not regularly varying in
the index \(j\).
Since \(\gamma<1\), the ratio of successive eigenvalues satisfies
\(\lambda_{j+1}/\lambda_j\to1\).  Its small-eigenvalue cumulative energies have
the same order as regularly varying reference functions, so training retains
power laws with logarithmic corrections.  For the constant schedule below,
write \(T:=\eta t\).

\begin{theorem}[Stretched-exponential forcing and memory asymptotics]
\label{thm:stretched_exponential_scaling}
Under \cref{def:stretched_exponential_specialization}, let the constant
schedule \(\eta_t\equiv\eta\), \(B_t\equiv B\) satisfy
\cref{ass:noisy_43_constant_schedule_stability}, with \(\eta\asymp1\).
Uniformly on every strict intermediate subwindow
\[
1\ll T,
\qquad
\log T\leq(1-\varepsilon)m^\gamma
\]
with fixed \(\varepsilon>0\), the aligned pure-point components satisfy
\begin{align}
\mathcal F_{pp}(t,m)
&\asymp
T^{-1}(\log T)^{(1-\gamma-2\beta)/\gamma},
\label{eq:stretched_exponential_Fpp}\\
\frac1{\eta}\mathcal K_{pp}(t,m)
&\asymp
\frac{\eta}{B}T^{-2}(\log T)^{(1-\gamma)/\gamma},
\label{eq:stretched_exponential_Kpp}\\
\sum_{s=t}^{\infty}\mathcal K_{pp}(s,m)
&\asymp
\frac{\eta}{B}T^{-1}(\log T)^{(1-\gamma)/\gamma}.
\label{eq:stretched_exponential_kernel_tail}
\end{align}
\end{theorem}

The proof uses the following tail and Laplace-sum estimates.

\paragraph{Stretched-exponential sum calculus.}

\begin{lemma}[Tail and Laplace sums]
\label{lem:stretched_exponential_sums}
For fixed \(q>0\) and \(a\in\mathbb{R}\),
\begin{align*}
\sum_{j\geq n}j^{-a}e^{-qj^\gamma}
&\sim
\frac{1}{\gamma q}n^{1-\gamma-a}e^{-qn^\gamma}.
\\
\sum_{j\geq1}j^{-a}e^{-qj^\gamma}
\exp\{-2u e^{-j^\gamma}\}
&\sim
\frac{\Gamma(q)}{\gamma(2u)^q}
\bigl(\log(2u)\bigr)^{(1-a)/\gamma-1}.
\end{align*}
\end{lemma}

\begin{proof}
For the tail, eventual monotonicity and variation scale
\(n^{1-\gamma}\to\infty\) permit sum--integral comparison; \(y=qx^\gamma\)
gives
\[
\int_n^\infty x^{-a}e^{-qx^\gamma}\,dx
=
\frac{q^{(a-1)/\gamma}}{\gamma}
\int_{qn^\gamma}^\infty y^{(1-a)/\gamma-1}e^{-y}\,dy.
\]
The incomplete-Gamma tail expansion proves the first claim.

For the second, the summand is concentrated at
\(x\asymp(\log u)^{1/\gamma}\) on a window of width
\(\asymp(\log u)^{(1-\gamma)/\gamma}\to\infty\), so sum--integral comparison
applies.  With \(x^\gamma=z\) and \(y=2ue^{-z}\), the continuous proxy becomes
\begin{align*}
\int_1^\infty x^{-a}e^{-qx^\gamma}e^{-2ue^{-x^\gamma}}\,\mathrm dx
&=
\frac{(2u)^{-q}}{\gamma}
\int_0^{2u/e}
y^{q-1}e^{-y}
\left(\log\frac{2u}{y}\right)^{(1-a)/\gamma-1}\,dy.
\end{align*}
After division by \((\log(2u))^{(1-a)/\gamma-1}\), the logarithmic factor
tends to one on compact subsets of \(y>0\).  Near zero, \(q>0\) absorbs the
remaining logarithmic factor.  On \(1\leq y\leq\sqrt u\), the normalized
factor is uniformly bounded, while on
\([\sqrt u,2u/e]\) the exponential tail dominates its at-most-polylogarithmic
growth.  Dominated convergence therefore yields \(\Gamma(q)\).
\end{proof}

The condition \(\gamma<1\) both makes the local index width
\(j^{1-\gamma}\) diverge, validating the integral approximation, and gives
\(\lambda_{j+1}/\lambda_j
=\exp\{-[(j+1)^\gamma-j^\gamma]\}\to1\).
At \(\gamma=1\), geometric lacunarity can introduce log-periodic corrections
and destroy regular variation.

\begin{proof}[Proof of \cref{thm:stretched_exponential_scaling}]
We first compute the cumulative forcing and memory energies, then apply the
Laplace-sum estimate on the active spectral shell.  Finally, the finite-width
cutoff verifies uniformity on the stated intermediate subwindow.

\medskip
\noindent\textit{Cumulative energies and component transforms.}

Let \(x\downarrow0\) and set
\(n_x=(\log(1/x))^{1/\gamma}\).  Applying the first part of
\cref{lem:stretched_exponential_sums} gives
\begin{align*}
\sum_{\lambda_j\leq x}\lambda_j(\theta_j^\star)^2
&\asymp
x\left(\log\frac1x\right)^{(1-\gamma-2\beta)/\gamma}
\\
\frac{\eta^2}{B}\sum_{\lambda_j\leq x}\lambda_j^2
&\sim
\frac{\eta^2}{2\gamma B}x^2
\left(\log\frac1x\right)^{(1-\gamma)/\gamma}.
\end{align*}
Thus the kernel cumulative energy is regularly varying at zero with index
\(2\).  Under the stated two-sided target comparison, the forcing cumulative
energy is only comparable to an index-one regularly varying reference; it is
itself regularly varying when
\(\theta_j^\star\sim C_\theta j^{-\beta}\), in which case its first display
above strengthens to
\[
\sum_{\lambda_j\leq x}\lambda_j(\theta_j^\star)^2
\sim
\frac{C_\theta^2}{\gamma}x
\left(\log\frac1x\right)^{(1-\gamma-2\beta)/\gamma}.
\]
Moreover,
\[
\sum_j\lambda_j<\infty,
\qquad
\sum_j\lambda_j^2<\infty,
\qquad
\sum_j\lambda_j(\theta_j^\star)^2<\infty
\]
for every fixed \(\beta\).  Hence the covariance is trace class and
\(\eta\asymp1\), whereas \(\bm\theta^\star\in\ell^2\) iff \(\beta>1/2\).

On the main shell \(\lambda\asymp T^{-1}\), the remainder in
\(t\log q_\eta(\lambda)\) is \(O(T\lambda^2)=o(1)\), so
\(q_\eta(\lambda)^t=\exp\{-2T\lambda\}(1+o(1))\).
In \cref{lem:stretched_exponential_sums},
\((a,q)=(2\beta,1)\) proves \cref{eq:stretched_exponential_Fpp} and
\((a,q)=(0,2)\) proves \cref{eq:stretched_exponential_Kpp}.
If \(\theta_j^\star\sim C_\theta j^{-\beta}\), the corresponding leading
constants are
\begin{align*}
\mathcal F_{pp}(t,m)
&\sim
\frac{C_\theta^2}{2\gamma}T^{-1}
\bigl(\log(2T)\bigr)^{(1-\gamma-2\beta)/\gamma},\\
\mathcal K_{pp}(t,m)
&\sim
\frac{\eta^2}{4\gamma B}T^{-2}
\bigl(\log(2T)\bigr)^{(1-\gamma)/\gamma}.
\end{align*}

The modal geometric sum gives the terminal kernel tail directly:
\[
\sum_{s=t}^\infty\mathcal K_{pp}(s,m)
=
\frac{\eta}{B}\sum_j
\frac{\lambda_j}{2-\eta(1+1/B)\lambda_j}
q_\eta(\lambda_j)^t.
\]
The denominator tends to \(2\) on the active shell, so the Laplace-sum
formula yields
\begin{equation*}
\sum_{s=t}^\infty\mathcal K_{pp}(s,m)
\sim
\frac{\eta}{4\gamma B}T^{-1}
\bigl(\log(2T)\bigr)^{(1-\gamma)/\gamma},
\end{equation*}
which proves \cref{eq:stretched_exponential_kernel_tail}.

\medskip
\noindent\textit{Finite-width cutoff.}

The critical learned index solves
\(T e^{-j_T^\gamma}\asymp1\).  Thus
\begin{align*}
j_T&\asymp(\log T)^{1/\gamma},&
1\ll j_T\ll m&\Longleftrightarrow
1\ll T,\ \log T\ll m^\gamma,\\
T_{\rm term}&\asymp\lambda_m^{-1}=e^{m^\gamma},&
\log T&\leq(1-\varepsilon)m^\gamma
\quad\text{on a uniform strict subwindow}.
\end{align*}

For completeness, for fixed \(q>0\) and \(a\in\mathbb{R}\), the truncated
continuous proxy satisfies
\begin{align*}
\int_1^m x^{-a}e^{-qx^\gamma}e^{-2Te^{-x^\gamma}}\,\mathrm dx
&=
\frac{(2T)^{-q}}{\gamma}
\int_{2T e^{-m^\gamma}}^{2T/e}
y^{q-1}e^{-y}
\left(\log\frac{2T}{y}\right)^{(1-a)/\gamma-1}\,dy.
\end{align*}
The logarithmic factor is kept on this real-valued domain.  After it is
extracted at the required asymptotic precision, the exponential tail permits
the remaining Gamma-type integral to be extended to infinity.
Thus \(2T e^{-m^\gamma}\to0\), \(=\Theta(1)\), and \(\to\infty\) give,
respectively, the intermediate power law, finite-width crossover, and exit
from the regularly varying window.  At fixed \(m\), eventual convergence is
geometric.

On the strict subwindow in the theorem,
\(T e^{-m^\gamma}\leq e^{-\varepsilon m^\gamma}\to0\).  Hence the preceding
sum asymptotics hold uniformly there, and
\cref{eq:stretched_exponential_Fpp,eq:stretched_exponential_Kpp,eq:stretched_exponential_kernel_tail} complete the proof.
\end{proof}

\paragraph{Irregular targets with the canonical forcing and memory exponents.}
Fix \(\alpha>1/4\), let \(p:=2\alpha+2\beta-1>0\), and set
\(\lambda_j=j^{-2\alpha}\) and
\(\lvert\theta_j^\star\rvert^2=j^{-2\beta}c_j\).  Suppose \(c_j\geq0\)
and, with \(C_n:=\sum_{j\leq n}c_j\), assume \(C_n/n\to1\).  For
\(N_x:=\lceil x^{-1/(2\alpha)}\rceil\), summation by parts gives
\[
\sum_{\lambda_j\leq x}\lambda_j\lvert\theta_j^\star\rvert^2
=\sum_{j\geq N_x}j^{-(p+1)}c_j
\sim \frac{1}{p}N_x^{-p}
\sim \frac{1}{p}x^{p/(2\alpha)}.
\]
Indeed, for \(a_j=j^{-(p+1)}\),
\[
\sum_{j=N}^{\infty}c_j a_j
=-C_{N-1}a_N+\sum_{j=N}^{\infty}C_j(a_j-a_{j+1})
\sim -N^{-p}+\frac{p+1}{p}N^{-p}
=\frac{1}{p}N^{-p}.
\]
Thus the forcing tail has the same leading power and constant as for the
canonical choice \(c_j\equiv1\), although the target coordinates need not
obey any fixed power envelope.  Three examples satisfy the averaging
condition:
\[
\begin{array}{lll}
c_{2k}=2,\quad c_{2k-1}=(2k-1)^{-2\varepsilon}, & \varepsilon>0,
& \text{alternating two-rate};\\
c_{k^2}=2k-1,\quad c_j=0\ \text{otherwise},
&& \text{mass-compensated sparse};\\
c_j=Z_j^2,\quad Z_j\overset{\mathrm{i.i.d.}}{\sim}\mathcal N(0,1),
&& \text{frozen random amplitude}.
\end{array}
\]
For the first example, the even terms contribute \(n+O(1)\) to \(C_n\)
and the odd terms contribute \(o(n)\).  For the second,
\(C_n=\lfloor\sqrt n\rfloor^2\).  For the third, \(C_n/n\to1\) almost
surely by the strong law; the infinite sequence is drawn once and then held
fixed.  The first construction has different powers on its even and odd
subsequences, the second vanishes off the squares, and in the third the
amplitudes are almost surely neither bounded above nor bounded away from zero.
Consequently, each construction violates
\(\lvert\theta_j^\star\rvert\asymp j^{-\rho}\) for every fixed \(\rho\),
almost surely in the random case.

These population tails can also be realized as exact empirical tails by a
deterministic aligned truncation.  Take \(d=2m\),
\(\bm W_m=(I_m,0)^{\!\top}\), and truncate one infinite target sequence to
its first \(2m\) coordinates.  Then
\(\widehat{\bm H}_m=\operatorname{diag}(\lambda_1,\ldots,\lambda_m,
0,\ldots,0)\).  If \(T\to\infty\) with \(T=o(m^{2\alpha})\), set
\(N_T:=\lceil T^{1/(2\alpha)}\rceil\).  Then
\[
\sum_{N_T\leq j\leq m}j^{-(p+1)}c_j
\sim \frac{1}{p}T^{-p/(2\alpha)},
\qquad
\sum_{N_T\leq j\leq m}j^{-4\alpha}
\sim \frac{1}{4\alpha-1}T^{-(4\alpha-1)/(2\alpha)}.
\]
The omitted tails beyond \(m\) are lower order, and the same estimates hold
locally uniformly at constant-factor cutoffs.  Regular variation of the
infinite tails gives the Potter envelope, and truncation can only decrease the
cutoff mass.  Hence, for a fixed constant schedule and some fixed
\(\delta>0\), if \(\eta(1+1/B)\leq2-\delta\), these tails satisfy the stable
no-escape hypotheses of
\cref{thm:noisy_43_effective_spectral_criterion}.  The resulting forcing
asymptotic matches the canonical diagonal choice \(c_j\equiv1\); the memory response is
unchanged because \(K_{\bm W_m}\) is target-independent.  The random-case
claims hold almost surely with respect to the single frozen target draw.

\begin{proposition}[Power-law spectrum with faster-than-polynomial forcing]
\label{prop:rapid_target_counterexample}
Fix \(\alpha>1/4\) and let
\[
\lambda_j=j^{-2\alpha},
\qquad
\lvert\theta_j^\star\rvert^2=e^{-j}.
\]
The target-weighted spectral tail satisfies
\[
\sum_{\lambda_j\leq x}
\lambda_j\lvert\theta_j^\star\rvert^2
\asymp
x\exp\!\left(-x^{-1/(2\alpha)}\right),
\qquad x\downarrow0.
\]

For the deterministic aligned truncations, take \(d=2m\),
\(\bm W_m=(I_m,0)^{\!\top}\), and truncate the displayed spectrum and target
to their first \(2m\) coordinates.  Fix a constant schedule
\(\eta_t\equiv\eta\), \(B_t\equiv B\geq1\), satisfying
\[
\eta\left(1+\frac1B\right)\leq2-\delta
\]
for some \(\delta>0\).  Along any joint limit \(m,t\to\infty\) such that
\(T=\eta t\to\infty\) and \(T=o(m^{2\alpha+1})\), the exact learnable
forcing obeys
\[
-\log F_{\bm W_m,>0}(t)
\sim
\frac{2\alpha+1}{2\alpha}
(4\alpha)^{1/(2\alpha+1)}T^{1/(2\alpha+1)}.
\]
Consequently,
\[
F_{\bm W_m,>0}(t)=o(T^{-q})
\qquad\text{for every }q>0.
\]
\end{proposition}

\begin{proof}
If \(N_x:=\lceil x^{-1/(2\alpha)}\rceil\), then
\(\sum_{\lambda_j\leq x}\lambda_j\lvert\theta_j^\star\rvert^2
=\sum_{j\geq N_x}j^{-2\alpha}e^{-j}\).
Successive summands have ratio at most \(e^{-1}\), so this tail is comparable
to its first term:
\[
\sum_{j\geq N_x}j^{-2\alpha}e^{-j}
\asymp N_x^{-2\alpha}e^{-N_x}
\asymp x e^{-x^{-1/(2\alpha)}}.
\]

For the aligned truncation,
\[
F_{\bm W_m,>0}(t)
=\sum_{j=1}^{m}j^{-2\alpha}e^{-j}q_\eta(j^{-2\alpha})^t.
\]
Since the schedule is fixed,
\((-\log q_\eta(\lambda))/(\eta\lambda)\to2\) as \(\lambda\downarrow0\).
Thus, for every \(0<\varepsilon<2\), there is a fixed
\(J_\varepsilon\) such that, for \(j\geq J_\varepsilon\),
\[
(2-\varepsilon)\eta j^{-2\alpha}
\leq-\log q_\eta(j^{-2\alpha})
\leq(2+\varepsilon)\eta j^{-2\alpha}.
\]
The stability margin makes the finitely many terms \(j<J_\varepsilon\)
exponentially small in \(T\).

For \(A>0\), an elementary discrete Laplace estimate gives
\[
-\log\sum_{j=J_\varepsilon}^{m}
j^{-2\alpha}\exp\!\left(-j-ATj^{-2\alpha}\right)
\sim
\frac{2\alpha+1}{2\alpha}
(2\alpha A)^{1/(2\alpha+1)}T^{1/(2\alpha+1)}
\]
whenever \(T^{1/(2\alpha+1)}=o(m)\).  Indeed, the exponent
\(x+ATx^{-2\alpha}\) is minimized at
\(x=(2\alpha AT)^{1/(2\alpha+1)}\), where its value is
\[
\frac{2\alpha+1}{2\alpha}
(2\alpha A)^{1/(2\alpha+1)}T^{1/(2\alpha+1)}.
\]
A nearest integer gives the matching lower bound, while splitting the sum
below, near, and above this scale gives the upper bound; the polynomial factor
contributes only \(O(\log T)\) to the logarithm.

Applying this estimate with \(A=2+\varepsilon\) and
\(A=2-\varepsilon\) yields
\[
\frac{2\alpha+1}{2\alpha}
\bigl(2\alpha(2-\varepsilon)\bigr)^{1/(2\alpha+1)}
\leq
\liminf\frac{-\log F_{\bm W_m,>0}(t)}{T^{1/(2\alpha+1)}}
\]
and
\[
\limsup\frac{-\log F_{\bm W_m,>0}(t)}{T^{1/(2\alpha+1)}}
\leq
\frac{2\alpha+1}{2\alpha}
\bigl(2\alpha(2+\varepsilon)\bigr)^{1/(2\alpha+1)}.
\]
Letting \(\varepsilon\downarrow0\) proves the stated constant.  Since
\(T^{1/(2\alpha+1)}/\log T\to\infty\), the forcing is smaller than every
inverse power of \(T\).  The final memory statement follows directly from the
definition of \(K_{\bm W_m}\).
\end{proof}

The stretched-exponential row is understood on the strict intermediate window
of \cref{thm:stretched_exponential_scaling}.  Width-dependent no-escape is a
separate issue: target-weighted mass can move toward zero with width and change
the forcing-transform constant, as shown in
\cref{prop:three_spectral_boundaries}.  Beyond frozen dynamics,
\citet{barkeshli2026origin} find dataset- and model-size scaling laws on
Erd\H{o}s--R\'enyi random-walk data without explicit power-law structure,
showing that loss scaling need not directly inherit a raw-data power law.

The following deterministic empirical operators separate the spectral and
feedback hypotheses.

For \(1/2<a<1\), \(b>1\), and width \(m\), consider the deterministic
spectrum \(\widehat\lambda_j=j^{-a}\), \(1\leq j\leq m\), and denote its
target-alignment weights by
\(w_j:=|\langle\widehat{\bm u}_j,\bm\Lambda^{1/2}\bm\theta^\star\rangle|^2
=j^{-b}\).
For this example, write
\[
q_{\mathcal K}:=2-\frac1a\in(0,1),
\qquad
c_{\mathcal K}:=
\frac{\Gamma(q_{\mathcal K}+1)}
{2^{q_{\mathcal K}}(2a-1)},
\qquad
L_{\mathcal K}(T)\equiv c_{\mathcal K}.
\]

\begin{proposition}[A deterministic long-memory example]
\label{prop:diagonal_long_memory_example}
Along every sequence \(T=T_m\to\infty\) with \(T=o(m^a)\), locally uniformly
for \(c>0\),
\begin{align*}
\sum_{0<\widehat\lambda_j\leq c/T}w_j
&\sim
\frac{c^{(b-1)/a}}{b-1}T^{-(b-1)/a},\\
\sum_{0<\widehat\lambda_j\leq c/T}\widehat\lambda_j^2
&\sim
\frac{c^{2-1/a}}{2a-1}T^{-(2-1/a)}.
\end{align*}
Both tails satisfy the uniform Potter/no-escape envelope.  Consequently,
bounded pointwise-stable steps satisfying \cref{eq:rv_mesh_condition} verify
the pure-power kernel asymptotic and the two-time conditions of
\cref{thm:rv_discrete_transfer} through
\cref{thm:rv_spectral_memory_transfer}, with the preceding
\(q_{\mathcal K}\) and \(L_{\mathcal K}\).
By contrast, a constant schedule with \(\eta/B\) bounded away from zero has
no row-stability margin uniform in \(m\).
\end{proposition}

\begin{proof}
Integral comparison, with the active index
\(n_x=\lceil x^{-1/a}\rceil\), gives
\[
\sum_{j\geq n_x}j^{-b}
\sim\frac{x^{(b-1)/a}}{b-1},
\qquad
\sum_{j\geq n_x}j^{-2a}
\sim\frac{x^{2-1/a}}{2a-1}.
\]
Since \(T^{1/a}=o(m)\), tails beyond \(m\) are negligible at \(x=c/T\)
uniformly on compact \(c\)-sets.  The same integral comparisons give
\cref{eq:noisy_43_uniform_potter_envelope} and
\(\sup_m\sum_{j\leq m}\widehat\lambda_j^2<\infty\).

For the last claim, let
\(q(\lambda)=1-2\eta\lambda+(1+1/B)\eta^2\lambda^2\).  Under constant
pointwise-stable \(\eta,B\), the limiting row mass satisfies
\begin{align*}
\lim_{t\to\infty}\sum_{s<t}K_{t,s}
&=
\frac{\eta^2}{B}\sum_{j\leq m}
\frac{\widehat\lambda_j^2}{1-q(\widehat\lambda_j)}
\geq
\frac{\eta}{2B}\sum_{j\leq m}\widehat\lambda_j
\asymp m^{1-a}.
\end{align*}
Thus two-time long memory does not imply constant-ratio row stability; one
must additionally shrink \(\eta/B\) or raise the ratio path.
\end{proof}

\begin{proposition}[Stable no-escape for a non-power spectrum]
\label{prop:stretched_exponential_no_escape_example}
Let
\[
\widehat{\bm H}_m=\operatorname{diag}_{1\leq j\leq m}(e^{-\sqrt j}),
\qquad
w_j=e^{-\sqrt j},
\qquad B=1,
\qquad \eta=\frac14.
\]
If \(T_m=\eta t_m\to\infty\) and
\(\log T_m\leq(1-\rho)\sqrt m\) for fixed \(\rho\in(0,1)\), then, locally
uniformly for \(c>0\),
\begin{align}
\sum_{\substack{j\leq m\\e^{-\sqrt j}\leq c/T_m}}w_j
&\sim 2cT_m^{-1}\log T_m,
\label{eq:stretched_example_forcing_tail}\\
\sum_{\substack{j\leq m\\e^{-\sqrt j}\leq c/T_m}}e^{-2\sqrt j}
&\sim c^2T_m^{-2}\log T_m.
\label{eq:stretched_example_memory_tail}
\end{align}
Both branches satisfy stable no-escape, and
\[
F_{\bm W_m,>0}(t_m)
\sim T_m^{-1}\log T_m,
\qquad
K_{\bm W_m}(t_m)
\sim\frac{\eta^2}{2}T_m^{-2}\log T_m.
\]
The pointwise and row-stability margins are uniform in \(m\).
\end{proposition}

\begin{proof}
The tail estimates
\[
\sum_{j\geq n}e^{-\sqrt j}\sim2\sqrt n\,e^{-\sqrt n},
\qquad
\sum_{j\geq n}e^{-2\sqrt j}\sim\sqrt n\,e^{-2\sqrt n}
\]
are the \(\gamma=1/2\) cases of
\cref{lem:stretched_exponential_sums}.  Substituting
\(n=(\log(T_m/c))^2\) gives
\cref{eq:stretched_example_forcing_tail,eq:stretched_example_memory_tail}.
The strict window makes mass beyond \(m\) negligible, since its two
omitted-to-leading ratios are bounded by
\[
\frac{T_m\sqrt m\,e^{-\sqrt m}}{\log T_m}
\quad\text{and}\quad
\frac{T_m^2\sqrt m\,e^{-2\sqrt m}}{\log T_m},
\]
and vanish.  Infinite-tail Potter control and bounded mass away from zero give
finite-width no-escape; transform constants follow from
\cref{thm:noisy_43_uniform_tauberian}.

Finally, \(\lambda_{\max}=e^{-1}\), so the pointwise condition holds with
\(\delta=1\).  Monotone integral comparison yields
\(\operatorname{tr}(\widehat{\bm H}_m)
\leq e^{-1}+\int_1^\infty e^{-\sqrt x}\,\mathrm dx=5/e\).
The certificate \cref{eq:td_lrs_joint_row_mass_certificate} is therefore at
most \(5/(4e)<1\), uniformly in \(m\).
\end{proof}

The next proposition separates three spectral-transfer failures.

\begin{proposition}[Three distinct spectral boundaries]
\label{prop:three_spectral_boundaries}
The following conditioned systems isolate three failures.
\begin{enumerate}[label=(\roman*)]
\item If every positive eigenvalue is at least \(\lambda_\star>0\), fixed
\(\eta>0,B\) with a pointwise margin give exponentially decaying forcing and
memory, rather than a nondegenerate low-spectrum power law.

\item For \(\lambda_j=e^{-j}\) and target-weighted mass
\(w_j=(1-e^{-\tau})e^{-\tau j}\), \(\tau>0\), the cutoff tail is bounded
above and below by constant multiples of \(x^\tau\), but is not regularly
varying.

\item There is a triangular target-weighted spectral array for the forcing
branch in which every fixed constant-factor cutoff has the candidate
power-law limit, yet a moving packet violates the no-escape envelope and
changes the training-transform constant.
\end{enumerate}
\end{proposition}

\begin{proof}
For (i), pointwise stability gives
\[
q_\eta(\lambda)
=1-\eta\lambda\left[2-\left(1+\frac1B\right)\eta\lambda\right]
\leq1-\delta\eta\lambda_\star=:\varrho<1.
\]
Thus, when the corresponding total weighted masses are bounded,
\(F_{\bm W,>0}(t)\lesssim\varrho^t\) and
\(K_{\bm W}(t)\lesssim\varrho^t\), whereas both cutoff masses
vanish for \(x<\lambda_\star\).  A gap therefore violates the nondegenerate
tail asymptotic, not the no-escape upper envelope.

For (ii), for all sufficiently small \(x\),
\[
A(x):=\sum_{\lambda_j\leq x}w_j
=e^{-\tau\lceil\log(1/x)\rceil},
\qquad
e^{-\tau}x^\tau\leq A(x)\leq x^\tau.
\]
But, with \(T_n=e^{n+\phi}\) and \(0<\phi<1\),
\[
\frac{A(1/T_n)}{T_n^{-\tau}}=e^{-\tau(1-\phi)}.
\]
The limit depends on the logarithmic phase \(\phi\), so
\(T^\tau A(1/T)\) is log-periodic rather than slowly varying.  This is the
geometric-lacunarity boundary \(\gamma=1\) of
\cref{thm:stretched_exponential_scaling}.

For (iii), fix \(\tau>0\), set \(T_n=n\), \(a_n=n^{-\tau}\), and define
background eigenvalues and target-alignment weights
\[
\lambda_{n,k}=\frac{k}{n^2},
\qquad
w_{n,k}=\left(\frac{k}{n^2}\right)^\tau
-\left(\frac{k-1}{n^2}\right)^\tau,
\qquad 1\leq k\leq n^2.
\]
Add one mode at
\[
y_n=\frac{\tau}{4}\log n,
\qquad
\lambda_n^\star=\frac{y_n}{n},
\qquad
w_n^\star=n^{-\tau/2}.
\]
Writing \(A_n\) for the cumulative weighted mass, every compact
\(K\subset(0,\infty)\), for all sufficiently large \(n\), satisfies
\[
\sup_{c\in K}
\left|\frac{A_n(c/n)}{a_n}-c^\tau\right|
=
\sup_{c\in K}
\left|\left(\frac{\lfloor cn\rfloor}{n}\right)^\tau-c^\tau\right|
\longrightarrow0,
\]
because the moving mode lies above \(c/n\).  At its own scale, however,
\[
\frac{A_n(y_n/n)}{a_n}\geq n^{\tau/2},
\]
which exceeds every bound allowed by the Potter envelope,
\(Cy_n^{\tau+\epsilon}=O((\log n)^{\tau+\epsilon})\).  Moreover,
\[
\frac1{a_n}\int e^{-2n\lambda}\,\mathrm dA_n(\lambda)
\longrightarrow
\Gamma(\tau+1)2^{-\tau}+1:
\]
the background gives the Tauberian constant, while the packet contributes
\(a_n^{-1}w_n^\star e^{-2y_n}=1\).  For the exact stable-SGD version, take
\(B=1\), \(\eta=1/2\), and \(t_n=2n\), so \(T_n=\eta t_n=n\) and the
pointwise margin can be chosen uniformly.  The background constant follows
from \cref{thm:noisy_43_uniform_tauberian}; for the packet,
\[
q_\eta(\lambda_n^\star)^{t_n}
=e^{-2n h_\eta(\lambda_n^\star)},
\qquad
h_\eta(\lambda):=-\frac{\log q_\eta(\lambda)}{2\eta}
=\lambda+O(\lambda^2),
\]
and \(nh_\eta(\lambda_n^\star)=y_n+o(1)\), again giving the extra \(1\).
\end{proof}

\begin{proposition}[Pointwise contraction does not imply row stability]
\label{prop:pointwise_not_row_stable}
Even a one-mode constant schedule can satisfy the pointwise condition with a
fixed margin while its Volterra rows exceed one.  In particular, take
\[
B=1,
\qquad \widehat\lambda=1,
\qquad \eta=\frac7{10}.
\]
Then part \emph{(a)} of \cref{ass:td_lrs_joint_schedule_stability} holds with
\(\delta=1/2\), but
\[
\sum_{s<t}K_{t,s}\longrightarrow\frac76>1,
\]
and, for every \(\sigma^2>0\), the exact noisy risk grows exponentially.
\end{proposition}

\begin{proof}
Writing \(x=\eta\widehat\lambda\), the one-mode decay factor and limiting row
mass are
\[
q=1-2x+\left(1+\frac1B\right)x^2,
\qquad
\lim_{t\to\infty}\sum_{s<t}K_{t,s}
=\frac{x/B}{2-(1+1/B)x}.
\]
Pointwise contraction requires \(x<2B/(B+1)\), whereas row stability requires
the strictly stronger \(x<2B/(B+2)\).  At the displayed numerical choice,
\[
q=\frac{29}{50},
\qquad
\eta\left(1+\frac1B\right)\widehat\lambda
=\frac75\leq\frac32=2-\frac12,
\qquad
\sum_{s<t}K_{t,s}
=\frac76\left[1-\left(\frac{29}{50}\right)^t\right].
\]
The exact recursion \cref{eq:td_lrs_joint_mode_recursion} reduces to
\[
R_{\sigma,t+1}
=\frac{107}{100}R_{\sigma,t}+\frac{49}{100}\sigma^2,
\]
so the full second moment diverges whenever
\(R_{\sigma,0}+\sigma^2>0\), although every homogeneous mode contracts.
\end{proof}

Thus a gap removes the low-spectrum mechanism, lacunarity removes regular
variation, and spectral mass that shifts with width can invalidate the uniform
comparison in the joint width--time limit; pointwise and row stability
respectively control modes and accumulated feedback.

\paragraph{Finite-width realizations of the six constructions.}
We complement the preceding componentwise results with literal finite-width
minibatch-SGD trajectories on fresh Gaussian covariates.  The left column of
each figure reports the population excess risk, evaluated exactly from the
parameter error, while the middle and right columns report the corresponding
finite-spectrum forcing and one-injection memory together with their cumulative
spectral masses.  We use intrinsic time \(T=\eta t\) and normalize every
target so that
\(\sum_j\lambda_j\lvert\theta_j^\star\rvert^2=1\).  Multiple widths in the
first two rows expose the finite-width cutoff, and multiple batch sizes in the
rapid-target row probe the memory-driven tail.  At any fixed width, all
positive modes eventually relax exponentially; the power-law descriptions
below concern the displayed pre-cutoff width--time window.

For the mass-compensated sparse target, set \(c_{k^2}=2k-1\) for
\(k\geq1\) and \(c_j=0\) otherwise.  The four rows of
\cref{fig:spectral_six_constructions_powerlike} have the displayed risk forms
\begin{align*}
\lambda_j=j^{-0.8},\quad
\lvert\theta_j^\star\rvert^2\propto j^{-0.6}
&\quad\Longrightarrow\quad R(T)\asymp T^{-1/2},\\
\lambda_j=j^{-0.8},\quad
\lvert\theta_j^\star\rvert^2\propto j^{-0.6}c_j
&\quad\Longrightarrow\quad R(T)\asymp T^{-1/2},\\
\lambda_j=e^{1-\sqrt j},\quad
\lvert\theta_j^\star\rvert^2\propto j^{-3/2}
&\quad\Longrightarrow\quad
R(T)\asymp T^{-1}(\log T)^{-2},\\
\lambda_j=e^{-1.4(j-1)},\quad
\lvert\theta_j^\star\rvert^2\propto\lambda_j^{0.75}
&\quad\Longrightarrow\quad
R(T)=T^{-7/4}\Phi(\log T),
\end{align*}
where \(\Phi\) is bounded, nonconstant, and log-periodic.  Thus all four risk
curves show an overall power-law decrease on the resolved finite-width window,
although geometric spacing prevents the last row from having a pure-power
asymptotic.

\begin{figure}[p]
\centering
\includegraphics[width=\linewidth]{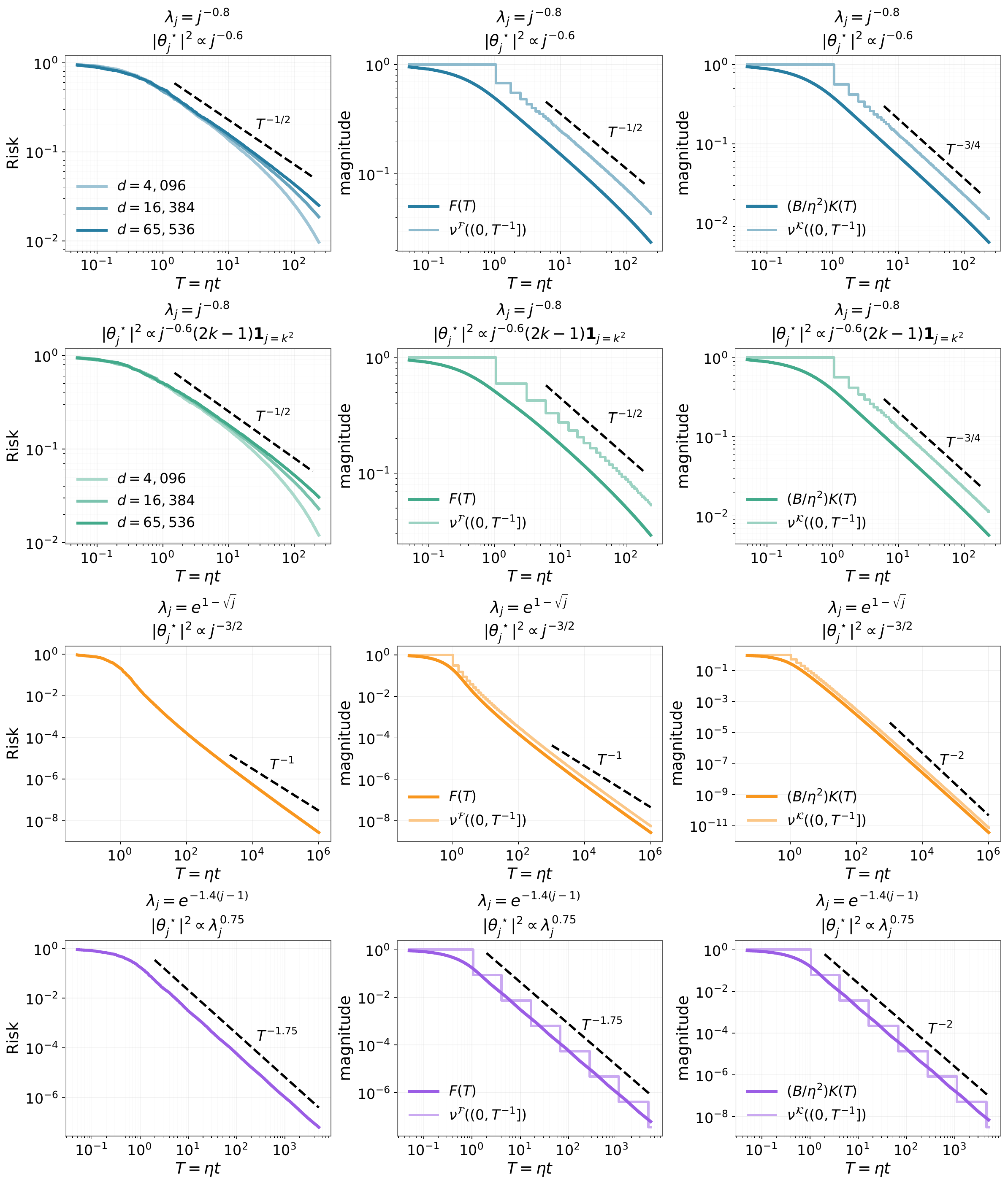}
\caption{Finite-width realizations of minibatch-SGD excess risk.  Across the
four constructions, the risk curves exhibit an overall power-law decay over
the displayed finite-width window, including logarithmic corrections and
log-periodic modulation.}
\label{fig:spectral_six_constructions_powerlike}
\end{figure}

The two rows of \cref{fig:spectral_six_constructions_boundaries} isolate the
remaining boundary mechanisms:
\[
\lambda_j=j^{-0.8},\quad
\lvert\theta_j^\star\rvert^2\propto e^{-j}
\quad\Longrightarrow\quad
F(T)=\exp\{-\Theta(T^{5/9})\},\qquad
R(T)\asymp T^{-3/4}.
\]
For the ideal gapped construction, set
\[
(\lambda_j,\lvert\theta_j^\star\rvert^2)
=
\begin{cases}
\bigl(10^{-(j-1)/63},Cj^{-1}\bigr), & 1\leq j\leq64,\\
(0,0), & j>64.
\end{cases}
\]
Its risk satisfies \(R(T)=O(e^{-cT})\) for some \(c>0\).  The first row
removes slow-direction target energy without
removing the target-independent memory tail, whereas the second removes the
low-spectrum mechanism itself.  Together, the two figures illustrate that the
cumulative weighted spectral masses in \cref{iclr:thm:spectral_iff}, rather
than the pointwise form of either the spectrum or the target alone, determine
the long-time curve shape.

\begin{figure}[htbp]
\centering
\includegraphics[width=\linewidth]{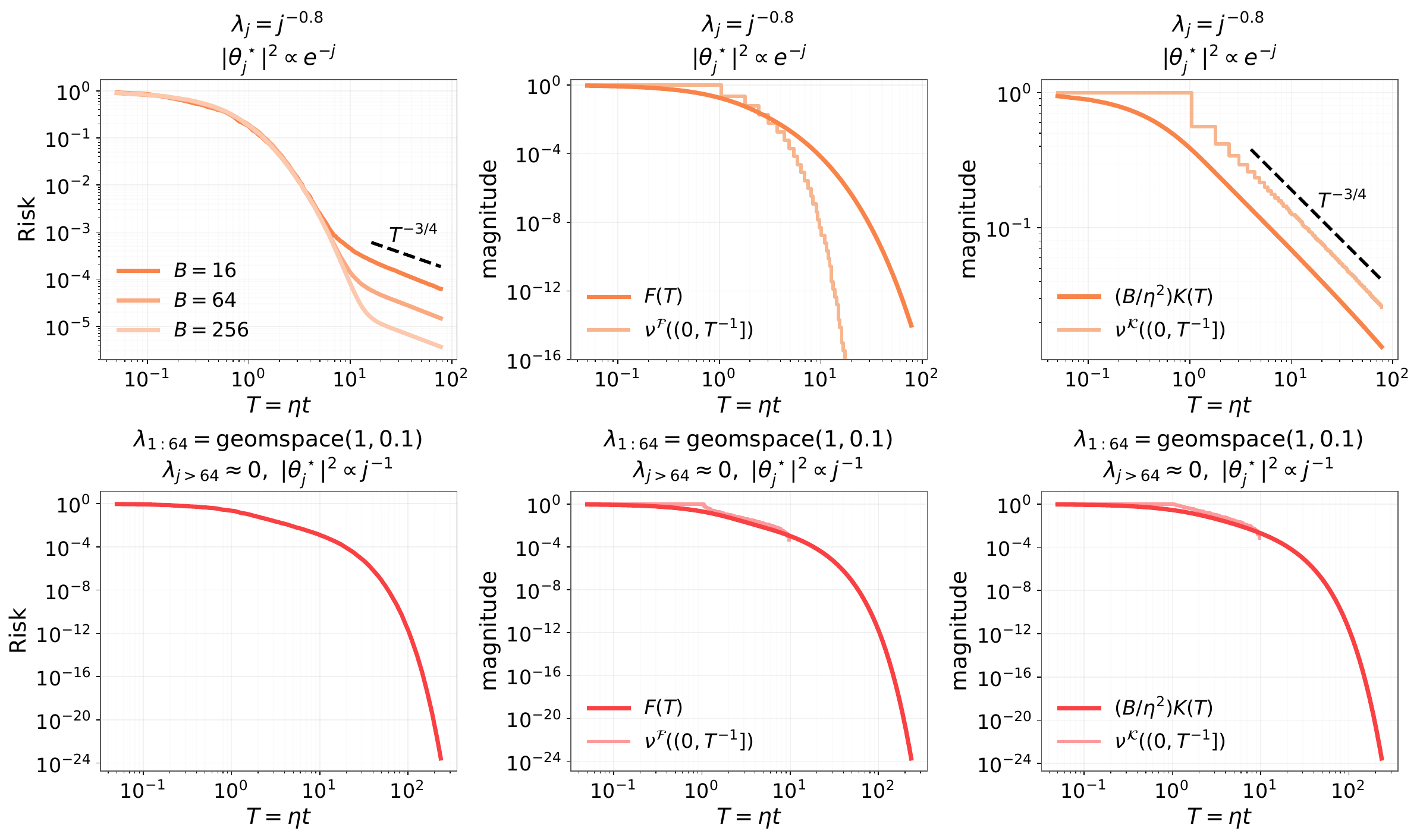}
\caption{Finite-width realizations of minibatch-SGD excess risk.  The rapidly
decaying target develops a memory-driven \(T^{-3/4}\) tail, whereas the
spectrally gapped construction relaxes exponentially.}
\label{fig:spectral_six_constructions_boundaries}
\end{figure}

Together, \cref{fig:spectral_six_constructions_powerlike,%
fig:spectral_six_constructions_boundaries} isolate the component-level laws
for \(F\) and \(K\); the next subsection shows how schedule-dependent
injection and Volterra feedback transform these component laws into the
observed risk.

\subsection{Long-memory schedule transfer with feedback}
\label{app:rv_schedule_transfer}

Having separated the spectral and stability assumptions, we now derive the
response to label noise.

Define the first label-noise injection and the exact label-noise excess by
\(N_t:=\sigma^2\sum_{s<t}K_{t,s}\) and
\(Z_t:=R_{\sigma,t}-R_{0,t}\), respectively.
Subtracting the clean and noisy versions of
\cref{eq:td_lrs_joint_volterra} gives the exact identity
\begin{equation}
Z_t=N_t+\sum_{s<t}K_{t,s}Z_s.
\label{eq:rv_noise_excess_volterra}
\end{equation}

For the continuum comparison used below, define the counterpart of \(N_t\) by
\begin{equation}
N_{\mathrm c}(T):=\sigma^2\int_0^T
\frac{k(T-u)}{r(u)}\,\mathrm du.
\label{eq:rv_continuum_comparison}
\end{equation}

First isolate a feedback estimate shared by forward transfer and converse.

\begin{lemma}[Lower-order feedback below the memory ceiling]
\label{lem:rv_feedback_invisibility}
Assume \cref{eq:rv_kernel_rv,eq:rv_sharp_bulk_kernel,eq:rv_noise_excess_volterra}
with \(K_{t,s}\geq0\).
For \(0<q<q_{\mathcal K}\), if either \(N_t\) or \(Z_t\) is asymptotic to
\(T_t^{-q}L(T_t)\), where \(L\) is eventually positive and slowly varying,
then
\begin{equation*}
Z_t\sim N_t\sim T_t^{-q}L(T_t).
\end{equation*}
\end{lemma}

\begin{proof}
Positivity in \cref{eq:rv_noise_excess_volterra} gives
\(0\leq N_t\leq Z_t\).  For fixed \(S\), sharp bulk convergence and regular
variation of \(k\) give
\(\sum_{s<S}K_{t,s}Z_s=O_S(k(T_t))\).
The remaining feedback satisfies
\begin{equation}
0\leq Z_t-N_t
\leq O_S(k(T_t))
+\frac{\sup_{s\geq S}Z_s}{\sigma^2}N_t.
\label{eq:rv_feedback_prefix_tail}
\end{equation}

If \(Z_t\sim T_t^{-q}L(T_t)\), then \(Z_t\to0\), \(k(T_t)=o(Z_t)\), and
\(N_t\leq Z_t\).  Divide \cref{eq:rv_feedback_prefix_tail} by \(Z_t\), then
send \(t\to\infty\) and \(S\to\infty\), obtaining \(Z_t-N_t=o(Z_t)\).

If instead the equivalent holds for \(N_t\), its row sum tends to zero; after
a finite prefix it is at most \(1/2\), so induction bounds \(Z\), and
\(Z_t\leq N_t+(N_t/\sigma^2)\sup_s Z_s=O(N_t)\).
Thus \(Z_t\to0\), and division of \cref{eq:rv_feedback_prefix_tail} by
\(N_t\), using \(k(T_t)=o(N_t)\), finishes the proof.
\end{proof}

For an eventually positive ratio \(r\) that is regularly varying with
exponent \(\vartheta\), write its slowly varying factor and cumulative
injection as \(L_r(u):=u^{-\vartheta}r(u)\),
\(V(T):=\int_0^T r(u)^{-1}\,\mathrm du\).  For the fixed-system endpoint
branches below, also write \(V_\infty:=\lim_{T\to\infty}V(T)\) and denote the
extended-valued endpoint amplitude by
\[
C:=
\sum_{s\geq0}\frac{T_{s+1}-T_s}{r_s}
\bigl(\sigma^2+Z_s\bigr)\in[0,\infty].
\]

\begin{theorem}[Sharp discrete schedule transfer]
\label{thm:rv_discrete_transfer}
Fix \(\sigma^2>0\), assume
\cref{ass:rv_schedule_path,ass:rv_long_memory_kernel_control}, and retain
\(r(u)=u^\vartheta L_r(u)\) and the preceding \(V\).  In a fixed
infinite-spectrum system, also retain \(V_\infty\) and \(C\).

If \(\vartheta<1\), then
\begin{equation}
N_t
\sim
\sigma^2\operatorname{Beta}(1-\vartheta,1-q_{\mathcal K})
\frac{T_tk(T_t)}{r(T_t)}
\sim
\sigma^2\operatorname{Beta}(1-\vartheta,1-q_{\mathcal K})
T_t^{-(\vartheta+q_{\mathcal K}-1)}
\frac{L_{\mathcal K}(T_t)}{L_r(T_t)}.
\label{eq:rv_direct_beta_law}
\end{equation}
For a fixed infinite-spectrum system, if \(\vartheta=1\) and
\(V(T)\to\infty\), then
\begin{equation}
N_t\sim\sigma^2k(T_t)V(T_t).
\label{eq:rv_theta_one_direct}
\end{equation}
For such a fixed system, if instead \(V_\infty<\infty\), including every
\(\vartheta>1\),
\begin{equation}
N_t\sim\sigma^2V_\infty k(T_t).
\label{eq:rv_integrable_direct}
\end{equation}

For the full response, assume in addition part \emph{(b)} of
\cref{ass:td_lrs_joint_schedule_stability}.
If \(1-q_{\mathcal K}<\vartheta<1\), then
\begin{equation}
Z_t\sim N_t.
\label{eq:rv_full_beta_law}
\end{equation}
For a fixed infinite-spectrum system, the same conclusion holds when
\(\vartheta=1\) and \(V(T)\to\infty\).  For such a fixed system, if
\(V_\infty<\infty\), including every \(\vartheta>1\), then
\begin{equation}
Z_t\sim C k(T_t),
\qquad C<\infty.
\label{eq:rv_saturation_constant}
\end{equation}
Thus every fixed-system branch with \(\vartheta\geq1\) has noise exponent
\(q_{\mathcal K}\), although at \(\vartheta=1\) its slowly varying factor
depends on whether \(V\) converges.
\end{theorem}

\begin{proof}
For \(\vartheta<1\), split \(N_t\) at \(T=T_t\) into
\[
T_{s+1}\leq\varepsilon T,
\qquad
\varepsilon T<T_{s+1}<(1-\varepsilon)T,
\qquad
T_s\geq(1-\varepsilon)T.
\]
On the middle region, sharp bulk transfer, schedule convergence, and fine mesh
give
\begin{equation}
\frac{r(T)}{Tk(T)}
\sum_{\varepsilon T<T_{s+1}<(1-\varepsilon)T}
\frac{T_{s+1}-T_s}{r_s}\Psi_{t,s}
\longrightarrow
\int_\varepsilon^{1-\varepsilon}
x^{-\vartheta}(1-x)^{-q_{\mathcal K}}\,\mathrm dx.
\label{eq:rv_bulk_riemann_limit}
\end{equation}

For the early region, \cref{eq:rv_sharp_bulk_kernel} and the uniform Potter
bounds give
\[
\sum_{T_{s+1}\leq\varepsilon T}
\frac{T_{s+1}-T_s}{r_s}\Psi_{t,s}
\lesssim
k(T)\int_0^{\varepsilon T}\frac{\mathrm du}{r(u)}.
\]
\Cref{lem:rv_uniform_karmata} yields, for
every sufficiently small \(\xi>0\),
\begin{equation}
\limsup_{t\to\infty}
\frac{r(T)\,\textnormal{early}}{Tk(T)}
\lesssim\varepsilon^{1-\vartheta-\xi}.
\label{eq:rv_early_endpoint}
\end{equation}

On the recent region, uniform regular variation and eventual monotonicity give
\(r_s^{-1}\lesssim r(T)^{-1}\).  Applying
\cref{eq:rv_recent_cell_envelope} with \(x=\varepsilon T\) therefore yields
\begin{equation}
\limsup_{t\to\infty}
\frac{r(T)\,\textnormal{recent}}{Tk(T)}
\lesssim\varepsilon^{1-q_{\mathcal K}-\xi},
\label{eq:rv_recent_endpoint}
\end{equation}
with the mesh contribution vanishing by \cref{eq:rv_mesh_condition}.  Sending
first \(t\to\infty\), then \(\varepsilon\downarrow0\), in
\cref{eq:rv_bulk_riemann_limit,eq:rv_early_endpoint,eq:rv_recent_endpoint}
gives
\[
\int_0^1x^{-\vartheta}(1-x)^{-q_{\mathcal K}}\,\mathrm dx
=\operatorname{Beta}(1-\vartheta,1-q_{\mathcal K}),
\]
proves \cref{eq:rv_direct_beta_law}.

For \(1-q_{\mathcal K}<\vartheta<1\),
\cref{eq:rv_direct_beta_law} has exponent
\(-q\), \(q=\vartheta+q_{\mathcal K}-1\in(0,q_{\mathcal K})\), so
\cref{lem:rv_feedback_invisibility} gives \cref{eq:rv_full_beta_law}.

For the remaining fixed-system branches, first let \(\vartheta=1\).  The cumulative injection
\(V(T)=\int_0^T\mathrm du/r(u)\) is slowly varying.  If \(V(T)\to\infty\), boundary
Karamata theory gives
\[
\frac{T}{r(T)V(T)}\longrightarrow0,
\qquad
\frac{V(\varepsilon T)}{V(T)}\longrightarrow1.
\]
In the same split, the early part normalized by \(k(T)V(T)\) tends to one as
\(\varepsilon\downarrow0\); bulk and recent parts vanish because
\(T/[r(T)V(T)]\to0\).  This proves \cref{eq:rv_theta_one_direct}.  The
boundedness argument of \cref{lem:rv_feedback_invisibility}, followed by
\cref{eq:rv_feedback_prefix_tail}, gives \(Z_t\to0\) and \(Z_t\sim N_t\)
because \(k(T_t)=o(N_t)\).

If \(V_\infty<\infty\), the following prefix--tail argument with \(Z_s=0\)
gives \cref{eq:rv_integrable_direct}.  Row stability bounds \(Z\), so
\((T_{s+1}-T_s)(\sigma^2+Z_s)/r_s\) is summable.  For every fixed
\(s\), \cref{eq:rv_sharp_bulk_kernel} gives
\(\Psi_{t,s}/k(T_t)\to1\).  For the remaining tail, first keep
\(T_{s+1}\leq(1-\varepsilon)T_t\), where Potter bounds multiply its small
mass by an \(\varepsilon\)-dependent constant; the recent endpoint is
\(o(k(T_t))\) since \(T/r(T)\to0\).  Prefix--tail limits followed by
\(\varepsilon\downarrow0\) prove \cref{eq:rv_saturation_constant}, also for
\(\vartheta>1\).
\end{proof}

At \(\vartheta=1-q_{\mathcal K}\), \cref{eq:rv_direct_beta_law} is slowly
varying and dominates positive-power clean decay unless a triangular
normalization vanishes.  Below it, the row mass \(N_t/\sigma^2\) grows
polynomially, contradicting part \emph{(b)} of
\cref{ass:td_lrs_joint_schedule_stability}; this is inadmissible, not a stable
growing-loss regime.

\subsection{Integrable-memory schedule transfer}
\label{app:rv_integrable_memory}

For \(q_{\mathcal K}>1\), recent injections see total memory mass while early
ones see the tail of \(k\).  We prove the sharp continuum endpoint asymptotic, then
transfer its order to the exact gap.  The discrete step requires all-age
memory because low-spectrum tails do not determine bounded-age mass; these are
the quantified conditions underlying the integrable-memory part of
\cref{thm:rv_joint_schedule}.

For a nonnegative locally bounded memory profile \(k\) and injection density
\(1/r\), write their total masses as
\(\bar k:=\int_0^\infty k(v)\,\mathrm dv\) and
\(V_\infty:=\int_0^\infty\mathrm du/r(u)\),
allowing the value \(+\infty\) before integrability is assumed.
The proposition gives exact constants for \(N_{\mathrm c}\); the discrete
theorem uses only its order consequence.

\begin{proposition}[Two-endpoint direct response under integrable memory]
\label{prop:rv_integrable_convolution}
Let \(k\geq0\) be locally bounded and eventually positive, and let \(r\) be
positive with \(1/r\) locally bounded and eventually positive.  Suppose \(k\)
is eventually nonincreasing,
\[
k(T)\sim T^{-q_{\mathcal K}}L_{\mathcal K}(T),
\qquad q_{\mathcal K}>1,
\qquad \bar k\in(0,\infty),
\]
and suppose \(1/r\) is eventually monotone and regularly varying with
exponent \(-\vartheta\).  Then
\begin{equation}
N_{\mathrm c}(T)\sim\sigma^2
\begin{cases}
\bar k/r(T),&\vartheta<q_{\mathcal K},\\[0.25em]
\bar k/r(T)+V_\infty k(T),&\vartheta=q_{\mathcal K},\\[0.25em]
V_\infty k(T),&\vartheta>q_{\mathcal K},
\end{cases}
\label{eq:rv_integrable_direct_law}
\end{equation}
where \(V_\infty<\infty\) in the last two branches.  In particular,
\begin{equation*}
N_{\mathrm c}(T)\asymp
\frac{1}{r(T)}+k(T)\int_0^T\frac{\mathrm du}{r(u)}.
\end{equation*}
Thus the direct continuum response has power exponent
\(\min\{\vartheta,q_{\mathcal K}\}\).  At \(\vartheta=0\), the direct response has
zero power exponent.  If \(\vartheta<0\), it grows with a positive power.
\end{proposition}

\begin{proof}
Split the direct convolution into recent and old endpoints.  In age
\(v=T-u\), fix \(A>0\); uniform
regular variation gives
\[
\int_0^A \frac{k(v)}{r(T-v)}\,\mathrm dv
\sim \frac{\bar k_A}{r(T)},
\qquad
\bar k_A:=\int_0^A k(v)\,\mathrm dv.
\]
On \(A\leq v\leq T/2\), Potter bounds and integrability of \(k\) give an
upper bound of order
\(r(T)^{-1}\int_A^\infty k(v)\,\mathrm dv\).  On the old-injection half,
regular variation of \(k\) gives
\[
\int_0^{T/2}\frac{k(T-u)}{r(u)}\,\mathrm du
\lesssim k(T)\int_0^{T/2}\frac{\mathrm du}{r(u)}.
\]
For \(\vartheta<1\), Karamata and \(Tk(T)\to0\) make this
\(o(1/r(T))\).  At \(\vartheta=1\), it is
\(T^{1-q_{\mathcal K}}\) times a slow factor and still vanishes.  For
\(1<\vartheta<q_{\mathcal K}\), \(V_\infty<\infty\) and
\(k(T)=o(1/r(T))\).  Letting \(A\to\infty\) proves the first branch of
\cref{eq:rv_integrable_direct_law}.

For \(\vartheta>q_{\mathcal K}\), exchanging \((u,1/r,V_\infty)\) with
\((v,k,\bar k)\) gives \(V_\infty k(T)\).  At equality both integrable
endpoints remain, while Potter bounds remove the middle, giving
\(\bar k/r(T)+V_\infty k(T)\).  Multiplication by \(\sigma^2\) proves
\cref{eq:rv_integrable_direct_law}; the order-level display follows by the
same three cases.

At \(\vartheta=0\), \cref{eq:rv_integrable_direct_law} gives zero power
exponent, while for \(\vartheta<0\) it makes
\(N_{\mathrm c}(T)/\sigma^2\), the continuum row mass, diverge.
\end{proof}

We now quantify the discrete integrable-memory part of
\cref{thm:rv_joint_schedule}, using
\((H_{t,s},\Psi_{t,s},N_t,Z_t)\) defined above.
For an eventually positive regularly varying ratio of index \(\vartheta\),
retain the slowly varying factor \(L_r(u):=u^{-\vartheta}r(u)\).

\begin{assumption}[Integrable-memory all-age kernel control]
\label{ass:rv_integrable_memory_kernel_control}
Consider either one fixed limiting dynamics or a width--time triangular
sequence in \cref{def:rv_width_time_triangular_sequence}.  The two-time kernel
has the following properties:
\begin{enumerate}[label=(\alph*)]
\item For some \(q_{\mathcal K}>1\), the reference profile
\(k:[0,\infty)\to(0,\infty)\) is locally bounded and nonincreasing, with
\[
k(T)\sim T^{-q_{\mathcal K}}L_{\mathcal K}(T),
\qquad
0<\int_0^\infty k(v)\,\mathrm dv<\infty.
\]
Here \(L_{\mathcal K}\) is eventually positive and slowly varying.

\item For constants \(0<c<C<\infty\),
\begin{equation}
c\,k(H_{t,s})
\leq \Psi_{t,s}\leq
C\,k(H_{t,s}),
\qquad 0\leq s<t.
\label{eq:rv_integrable_all_age_comparison}
\end{equation}

\item The intrinsic mesh is bounded:
\begin{equation}
h_t=\max_{s<t}(T_{s+1}-T_s)\leq h_\star<\infty.
\label{eq:rv_integrable_bounded_mesh}
\end{equation}
\end{enumerate}
In a triangular sequence, \(c,C\), and \(h_\star\) are uniform over the width
and horizon.
\end{assumption}

\begin{theorem}[Exact discrete transfer under integrable memory]
\label{thm:rv_discrete_integrable_transfer}
Fix \(\sigma^2>0\), assume
\cref{ass:rv_schedule_path,ass:rv_integrable_memory_kernel_control}, and use
the slowly varying factor \(L_r\) defined immediately above.  Assume also that
\(1/r\) is locally bounded.  For the exponent conclusions below,
additionally require
\[
\frac{\log(r_m(T_t)^{-1})}{\log T_t}\longrightarrow-\vartheta,
\qquad
\frac{\log\!\int_0^{T_t}\mathrm du/r_m(u)}{\log T_t}
\longrightarrow\max\{1-\vartheta,0\}
\]
along a triangular sequence; no extra normalization is needed for a fixed
system.  Finally assume part \emph{(b)} of
\cref{ass:td_lrs_joint_schedule_stability}, with its constant \(\kappa<1\).
Then, with all comparisons uniform in the triangular setting,
\begin{equation}
Z_t
\asymp N_t
\asymp
\sigma^2\left[
\frac{1}{r(T_t)}
+k(T_t)\int_0^{T_t}\frac{\mathrm du}{r(u)}
\right]
\asymp
\sigma^2\left[
\frac{T_t^{-\vartheta}}{L_r(T_t)}
+T_t^{-q_{\mathcal K}}L_{\mathcal K}(T_t)
\int_0^{T_t}\frac{\mathrm du}{r(u)}
\right].
\label{eq:rv_discrete_integrable_order}
\end{equation}
Consequently, for a fixed system, and along a triangular sequence under the
displayed normalization conditions,
\[
Z_t=T_t^{-\min\{\vartheta,q_{\mathcal K}\}+o(1)},
\qquad \vartheta>0.
\]
At \(\vartheta=0\) the exact noisy--clean gap has no positive
power exponent.  For a fixed system, \(\vartheta<0\) makes the row mass grow and
contradicts part \emph{(b)} of
\cref{ass:td_lrs_joint_schedule_stability}; the same holds along a triangular
sequence whenever its width-dependent normalization does not cancel that
growth.
\end{theorem}

\begin{proof}
Compare the cell sum with the continuum convolution, sandwich the exact gap,
then read off the exponent.  Regular variation and positivity of nonincreasing
\(k\) give \(C_\star<\infty\) such that
\(k(v)\leq C_\star k(v+h)\) for \(v\geq0\) and
\(0\leq h\leq h_\star\).
because the ratio tends uniformly to one at infinity and positivity controls
compact intervals.  For \(u\in[T_s,T_{s+1})\), age \(T_t-u\) lies between
\(H_{t,s}\) and \(H_{t,s}+h_\star\).  Since \(1/r\) is constant on this
cell, summing the resulting upper and lower bounds gives
\begin{align*}
\sum_{s<t}\frac{T_{s+1}-T_s}{r_s}k(H_{t,s})
&\asymp
\int_0^{T_t}\frac{k(T_t-u)}{r(u)}\,\mathrm du.
\end{align*}
Together with \cref{eq:rv_integrable_all_age_comparison}, this proves
\begin{equation*}
N_t
\asymp
\sigma^2\int_0^{T_t}\frac{k(T_t-u)}{r(u)}\,\mathrm du.
\end{equation*}
For a fixed system, \cref{prop:rv_integrable_convolution} gives the second
comparison in \cref{eq:rv_discrete_integrable_order}.  Uniformly in the
triangular setting, splitting at \(T_t/2\) and applying common Potter bounds
gives
\begin{equation*}
\int_0^{T_t}\frac{k(T_t-u)}{r(u)}\,\mathrm du
\lesssim
\frac{1}{r(T_t)}+k(T_t)\int_0^{T_t}\frac{\mathrm du}{r(u)}.
\end{equation*}
A fixed recent-age interval gives the \(1/r(T_t)\) lower bound.  If
\(1/r\) decreases, the old half gives the second lower bound since
\(\int_0^{T_t/2}\mathrm du/r(u)\gtrsim\int_0^{T_t}\mathrm du/r(u)\).  If it increases, then
\(\vartheta\leq0\), and \cref{lem:rv_uniform_karmata} with \(Tk(T)\to0\)
absorbs the second term into the first.  Constants are width-independent.

For the exact gap, \cref{eq:rv_noise_excess_volterra} and positivity give
\(Z_t\geq N_t\).  Let
\(C_0:=\kappa\sigma^2/(1-\kappa)\).
An induction using part \emph{(b)} of
\cref{ass:td_lrs_joint_schedule_stability} gives
\(0\leq Z_t\leq C_0\) for every \(t\).  Hence
\begin{equation*}
Z_t
=\sum_{s<t}K_{t,s}(\sigma^2+Z_s)
\leq
(\sigma^2+C_0)\sum_{s<t}K_{t,s}
=\frac1{1-\kappa}N_t.
\end{equation*}
Thus \(N_t\leq Z_t\leq N_t/(1-\kappa)\), proving the first comparison.

For \(\vartheta>0\), recent injection leads below
\(q_{\mathcal K}\), old injection above it, and both remain at equality,
giving exponent \(\min\{\vartheta,q_{\mathcal K}\}\).  At \(\vartheta=0\)
the recent term excludes positive-power decay; below zero it grows and
contradicts the row bound under the stated normalization.
\end{proof}

If the centered clean loss has exponent \(q_0>0\), positivity of
\(R_\sigma-R_{\mathrm{app}}=(R_0-R_{\mathrm{app}})+Z\) and
\cref{eq:rv_discrete_integrable_order} give total exponent
\(\min\{q_0,q_{\mathcal N}(\vartheta)\}\).  This is the integrable-memory
input used in \cref{thm:rv_preserve_change_destroy}.

Under the corresponding row-stable fixed-dynamics pure-power hypotheses, the
following table collects the schedule boundaries.
\par\noindent

For the PLRF \(\mathrm{IM}\) regime, where \(\alpha>1/2\), the DE theorem for
the LM and IM regimes supplies the all-age profile
\(k(v)=(1+v)^{-2+1/(2\alpha)}\)
at the resolvent-DE level.  Applying the exact finite-width result additionally
requires the fixed-sequence empirical-to-DE transfer in
\cref{ass:uniform_subregime_plrf}, together with bounded mesh and empirical row
stability.

\subsection{From the exact noisy--clean gap back to
\texorpdfstring{\(B/\eta\)}{B/eta}}
\label{app:rv_schedule_identification}

Return to \(0<q_{\mathcal K}<1\) and remove schedule regular variation from
the reverse direction.

If needed, extend \(k\) nonnegatively and locally integrably over bounded ages,
without changing its tail, and use \cref{eq:rv_continuum_comparison}.

With \cref{lem:rv_feedback_invisibility}, the next output-driven lemma supplies
the discrete step, without assuming regular variation of \(1/r\).

\begin{lemma}[Discrete-to-continuous convolution comparison from the output asymptotics]
\label{lem:rv_output_driven_transparency}
Work with a fixed limiting sequence and assume
\cref{ass:rv_long_memory_kernel_control}, with \(T_t\uparrow\infty\).
Let \(r\) be positive and suppose its step interpolation \(1/r\) is
eventually nonincreasing.  If, for some
\(0<q<q_{\mathcal K}\) and eventually positive slowly varying \(L\),
\begin{equation}
N_t\sim T_t^{-q}L(T_t),
\label{eq:rv_discrete_output_assumption}
\end{equation}
then the continuum comparison in \cref{eq:rv_continuum_comparison} satisfies
\begin{equation}
N_{\mathrm c}(T)\sim T^{-q}L(T)
\qquad(T\to\infty).
\label{eq:rv_continuum_from_discrete}
\end{equation}
\end{lemma}

\begin{proof}
Use the output to bound recent schedule height, remove both recent endpoints,
compare old cells with the continuum integral, then interpolate.  Write
\(\widetilde N_t=N_t/\sigma^2\), \(T=T_t\); since \(k(T)\to0\), the mesh condition gives
\(h_t/T_t\to0\).  For \(0<\varepsilon<1/8\), cells with
\(T_s\in[T/4,T/2]\) have length \(T/4+o(T)\).  There,
\cref{eq:rv_sharp_bulk_kernel} and the uniform convergence theorem for
regular variation give \(\Psi_{t,s}\geq c k(T)\) for some \(c>0\).
Eventual monotonicity yields the output-driven bound
\begin{equation}
\frac{T k(T)}{r((1-2\varepsilon)T)}\leq C_0 \widetilde N_t.
\label{eq:rv_output_schedule_bound}
\end{equation}

Let \(\widetilde N_t^{\mathrm{rec}}\) be the contribution of the exact cells with
\(H_{t,s}<\varepsilon T\).  Their left endpoints eventually exceed
\((1-2\varepsilon)T\), so
\cref{eq:rv_recent_cell_envelope} with \(x=2\varepsilon T\) gives
\[
\widetilde N_t^{\mathrm{rec}}
\leq\frac{1}{r((1-2\varepsilon)T)}
\left\{h_tG_m(0)+\int_0^{2\varepsilon T}G_m(v)\,\mathrm dv\right\}.
\]
Combining \cref{eq:rv_output_schedule_bound,eq:rv_integrated_envelope,%
eq:rv_mesh_condition}, for every sufficiently small \(\xi>0\), gives
\begin{equation*}
\limsup_{t\to\infty}\frac{\widetilde N_t^{\mathrm{rec}}}{\widetilde N_t}
\leq C_\xi\varepsilon^{1-q_{\mathcal K}-\xi}.
\end{equation*}
Karamata integration of \(k\), together with
\cref{eq:rv_output_schedule_bound}, gives the same estimate for the recent
part of the continuum comparison:
\begin{equation}
\limsup_{t\to\infty}
\frac{\int_{(1-\varepsilon)T}^{T}k(T-u)\,\mathrm du/r(u)}{\widetilde N_t}
\leq C_\xi\varepsilon^{1-q_{\mathcal K}-\xi}.
\label{eq:rv_continuum_recent_small}
\end{equation}

For old parts, sharp bulk convergence gives
\(\Psi_{t,s}/k(H_{t,s})\to1\) uniformly when
\(H_{t,s}\geq\varepsilon T\), and
\begin{equation*}
\sup_{H_{t,s}\geq\varepsilon T}
\left|
\frac{\int_{T_s}^{T_{s+1}}k(T-u)\,\mathrm du}
{(T_{s+1}-T_s)k(H_{t,s})}-1
\right|\longrightarrow0,
\end{equation*}
because \(h_t/T\to0\) and convergence is uniform away from zero age.
Multiplication by \(1/r_s\geq0\) makes the old-part difference
\(o(\widetilde N_t)\), except at most one crossing cell.  The remaining
contribution from this cell is bounded by
\[
\frac{h_t}{r((1-2\varepsilon)T)}
\sup_{v\in[\varepsilon T-h_t,\varepsilon T+h_t]}k(v)
=o(\widetilde N_t)
\]
for fixed \(\varepsilon\), by the output bound, mesh condition, and regular
variation.  The two recent estimates followed by
\(\varepsilon\downarrow0\) prove
\begin{equation}
N_{\mathrm c}(T_t)\sim N_t.
\label{eq:rv_grid_convolution_transparency}
\end{equation}

For \(T_t\leq U<T_{t+1}\), the mesh condition gives
\[
T_{t+1}-T_t\leq h_{t+1}
=o\!\left(T_{t+1}k(T_{t+1})\right)=o(T_{t+1}),
\]
so \(U/T_t\to1\); on ages at least \(\varepsilon T_t\),
\(k(U-u)/k(T_t-u)\to1\) uniformly.  The new interval is bounded by
\[
\frac{1}{r_t}\int_0^{U-T_t}k(v)\,\mathrm dv
\leq
\frac{1}{r((1-2\varepsilon)T_t)}
\int_0^{2\varepsilon T_t}k(v)\,\mathrm dv
\leq C_\xi\varepsilon^{1-q_{\mathcal K}-\xi}\widetilde N_t,
\]
using \cref{eq:rv_output_schedule_bound}.  The other recent pieces obey the
same estimate as \cref{eq:rv_continuum_recent_small}.  Therefore
\[
\lim_{t\to\infty}\ \sup_{T_t\leq U<T_{t+1}}
\left|\frac{N_{\mathrm c}(U)}{N_{\mathrm c}(T_t)}-1\right|=0.
\]
Combining this with \cref{eq:rv_discrete_output_assumption,%
eq:rv_grid_convolution_transparency} and uniform convergence of the slowly
varying factor proves \cref{eq:rv_continuum_from_discrete}.
\end{proof}

\begin{proposition}[Gap-to-ratio Tauberian converse]
\label{prop:rv_continuum_converse}
Let \(k\geq0\) be locally integrable, have finite Laplace transform for every
positive argument, and satisfy
\(k(T)\sim T^{-q_{\mathcal K}}L_{\mathcal K}(T)\), where \(0<q_{\mathcal K}<1\) and \(L_{\mathcal K}\) is
eventually positive and slowly varying.  Let \(r\) be positive, with \(1/r\)
locally integrable and eventually nonincreasing, and suppose \(1/r\) has a
finite Laplace transform for every positive argument.  Retain the continuum response
\(N_{\mathrm c}\) from \cref{eq:rv_continuum_comparison}.
For \(0<q<q_{\mathcal K}\) and eventually positive slowly varying \(L\),
\begin{equation*}
N_{\mathrm c}(T)\sim T^{-q}L(T)
\end{equation*}
if and only if
\begin{equation}
\frac{1}{r(T)}
\sim
\frac{1}{\sigma^2\operatorname{Beta}(q_{\mathcal K}-q,1-q_{\mathcal K})}
T^{q_{\mathcal K}-1-q}\frac{L(T)}{L_{\mathcal K}(T)}.
\label{eq:rv_continuum_inverse}
\end{equation}
Equivalently,
\begin{equation}
r(T)
\sim
\sigma^2\operatorname{Beta}(q_{\mathcal K}-q,1-q_{\mathcal K})
T^{1-q_{\mathcal K}+q}\frac{L_{\mathcal K}(T)}{L(T)}.
\label{eq:rv_continuum_ratio_inverse}
\end{equation}
At the ceiling \(q=q_{\mathcal K}\), the converse fails.
\end{proposition}

\begin{proof}
The reverse implication is beta convolution.  Conversely, positivity and
Tonelli give, with hats denoting Laplace transforms,
\[
\widehat N_{\mathrm c}(s)
=\sigma^2\widehat{(1/r)}(s)\widehat k(s).
\]
Karamata's Laplace theorem gives, as \(s\downarrow0\),
\begin{align*}
\widehat N_{\mathrm c}(s)
&\sim \Gamma(1-q)s^{q-1}L(1/s),\\
\widehat k(s)
&\sim \Gamma(1-q_{\mathcal K})s^{q_{\mathcal K}-1}L_{\mathcal K}(1/s).
\end{align*}
Consequently,
\[
\widehat{(1/r)}(s)
\sim
\frac{\Gamma(1-q)}{\sigma^2\Gamma(1-q_{\mathcal K})}
s^{q-q_{\mathcal K}}\frac{L(1/s)}{L_{\mathcal K}(1/s)}.
\]
For \(V(T)=\int_0^T\mathrm du/r(u)\), the
Laplace--Stieltjes Tauberian theorem yields
\[
V(T)
\sim
\frac{\Gamma(1-q)}
{\sigma^2\Gamma(1-q_{\mathcal K})\Gamma(1+q_{\mathcal K}-q)}
T^{q_{\mathcal K}-q}\frac{L(T)}{L_{\mathcal K}(T)}.
\]
The monotone-density theorem then gives
\cref{eq:rv_continuum_inverse}; the beta identity gives
\cref{eq:rv_continuum_ratio_inverse}.  These standard regular-variation
results are collected in
\citet[Theorems~1.7.1--1.7.2]{bingham1989regular}.

At \(q=q_{\mathcal K}\), the family
\(r_a(u)^{-1}=(a-1)(1+u)^{-a}\), \(a>1\), has unit total injection and
distinct tails but \(N_{\mathrm c,a}(T)\sim\sigma^2k(T)\); the ceiling is not
identifiable.
\end{proof}

\begin{theorem}[Exact fixed-limit discrete gap-to-ratio identification]
\label{thm:rv_exact_discrete_identification}
Fix one infinite-spectrum limiting sequence with
\(T_t\uparrow\infty\), \(0<T_{s+1}-T_s<\infty\), and \(\sigma^2>0\).  Assume
\cref{ass:rv_long_memory_kernel_control,eq:rv_noise_excess_volterra},
with the width index suppressed and all limits taken as \(t\to\infty\).
Assume \(K_{t,s}\geq0\), and let \(r\) be a positive step interpolation such
that \(1/r\) is finite, locally integrable, and eventually nonincreasing.  For
\(0<q<q_{\mathcal K}\) and eventually positive slowly varying \(L\),
\begin{equation}
Z_t\sim T_t^{-q}L(T_t)
\quad\Longleftrightarrow\quad
\frac{1}{r(T)}\sim
\frac{1}{\sigma^2\operatorname{Beta}(q_{\mathcal K}-q,1-q_{\mathcal K})}
T^{q_{\mathcal K}-1-q}\frac{L(T)}{L_{\mathcal K}(T)}.
\label{eq:rv_exact_discrete_inverse}
\end{equation}
\end{theorem}

\begin{proof}
If the equivalent holds for \(Z_t\),
\cref{lem:rv_feedback_invisibility,lem:rv_output_driven_transparency} transfer
it first to \(N_t\), then to \(N_{\mathrm c}(T)\).  Local integrability, the
kernel tail, and eventual schedule monotonicity give finite Laplace transforms,
so \cref{prop:rv_continuum_converse} yields the displayed asymptotic for \(1/r\).

Conversely, this asymptotic makes \(r\) regularly varying with exponent
\(1-q_{\mathcal K}+q\in(1-q_{\mathcal K},1)\), supplying the fixed-sequence
uniform convergence, Potter, and Karamata inputs to
\cref{thm:rv_discrete_transfer}.  Its direct part and the beta identity give
\(N_t\sim T_t^{-q}L(T_t)\), and
\cref{lem:rv_feedback_invisibility} gives the same for \(Z_t\).
\end{proof}

\begin{remark}[Why gap-to-ratio identification is regime-qualified]
\label{rem:rv_inverse_qualifiers}
Eventual monotonicity is essential: sparse bounded spikes can destroy
pointwise regular variation of \(1/r\) while adding only
\(O(k(T))=o(T^{-q})\).  So is the fixed limit: a triangular diagonal
\(r_m(u)^{-1}=a_m f(u/T_m)\), with arbitrary positive decreasing \(f\), can be
normalized to the same terminal response.  Triangular identification needs locally
uniform output at all constant-factor subhorizons.
\end{remark}

\subsection{Full preserve--change--destroy classification}
\label{app:rv_total_loss}

Combining the gap estimates with the clean-risk asymptotics gives the
preserve/change/destroy classification.

For the classification below, abbreviate \(T_t\) by \(T\).  For a regularly
varying ratio of index \(\vartheta\), retain
\(r(T)=T^\vartheta L_r(T)\) and
\(V(T)=\int_0^T r(u)^{-1}\,\mathrm du\) from the preceding notation.

\begin{theorem}[Full regular-variation preserve--change--destroy classification]
\label{thm:rv_full_preserve_change_destroy}
Fix one infinite-spectrum dynamics and suppose
\[
R_0(T)-R_{\mathrm{app}}\sim T^{-q_0}L_0(T),
\qquad q_0>0,
\]
where \(L_0\) is eventually positive and slowly varying.

In long memory, assume the hypotheses of
\cref{thm:rv_discrete_transfer}.  For
\(1-q_{\mathcal K}<\vartheta<1\),
\[
R_\sigma(T)-R_0(T)
\sim
\sigma^2\operatorname{Beta}(1-\vartheta,1-q_{\mathcal K})
T^{-(\vartheta+q_{\mathcal K}-1)}
\frac{L_{\mathcal K}(T)}{L_r(T)}.
\]
For \(\vartheta=1\), the gap is asymptotic to
\(\sigma^2k(T)V(T)\) if \(V(T)\to\infty\), and to \(Ck(T)\), with the
positive constant \(C\) from \cref{eq:rv_saturation_constant}, if
\(V_\infty<\infty\); the latter also holds for \(\vartheta>1\).  Thus
\[
q_{\mathcal N}(\vartheta)=
\begin{cases}
\vartheta+q_{\mathcal K}-1,
&1-q_{\mathcal K}<\vartheta<1,\\
q_{\mathcal K},&\vartheta\geq1.
\end{cases}
\]

In integrable memory, assume the hypotheses of
\cref{thm:rv_discrete_integrable_transfer}.  For \(\vartheta>0\),
\[
R_\sigma(T)-R_0(T)
\asymp
\sigma^2\left[
\frac1{r(T)}+k(T)\int_0^T\frac{\mathrm du}{r(u)}
\right],
\qquad
q_{\mathcal N}(\vartheta)=\min\{\vartheta,q_{\mathcal K}\}.
\]

In every decaying branch,
\[
R_\sigma(T)-R_{\mathrm{app}}
=T^{-\min\{q_0,q_{\mathcal N}(\vartheta)\}+o(1)}.
\]
Hence the schedule changes, preserves, or is critical according as
\(q_{\mathcal N}(\vartheta)\) is below, above, or equal to \(q_0\).
At equality positivity prevents cancellation; in branches with exact
equivalents, the slow factors add.  At LM-interior equality
\(q_0=q_{\mathcal N}(\vartheta)\), for example,
\[
R_\sigma(T)-R_{\mathrm{app}}
\sim T^{-q_0}\left[
L_0(T)
+\sigma^2\operatorname{Beta}(1-\vartheta,1-q_{\mathcal K})
\frac{L_{\mathcal K}(T)}{L_r(T)}
\right].
\]
At an IM equality, the available order-level statement is
\[
R_\sigma(T)-R_{\mathrm{app}}
\asymp T^{-q_0}L_0(T)
+\sigma^2\left[
\frac1{r(T)}+k(T)\int_0^T\frac{\mathrm du}{r(u)}
\right].
\]
At the LM edge \(\vartheta=1-q_{\mathcal K}\) and the IM edge
\(\vartheta=0\), no positive-power gap decay is possible, so positive-power
convergence of the total centered loss to \(R_{\mathrm{app}}\) is destroyed.
Below either edge, row stability fails.
\end{theorem}

The long-memory interior equivalent and the integrable-memory order comparison
extend to width--time triangular sequences under the uniform conditions of
\cref{thm:rv_discrete_transfer,thm:rv_discrete_integrable_transfer}.  The
long-memory endpoint equivalents for \(\vartheta\geq1\) are fixed-system
statements.  An exponent-only triangular classification additionally requires
subpower diagonal normalization of the clean factor and of the relevant slow
factor; without it, width-dependent amplitudes must remain explicit.

\begin{proof}[Proof of \cref{thm:rv_full_preserve_change_destroy}]

Write \(Z(T_t):=Z_t\), with all asymptotics on the grid \(T=T_t\).  For common
clean/noisy floor \(R_{\mathrm{app}}\), suppose
\[
R_0(T)-R_{\mathrm{app}}\sim T^{-q_0}L_0(T),
\qquad q_0>0,
\]
with \(L_0\) eventually positive and slowly varying.  The positive decomposition
\[
R_\sigma(T)-R_{\mathrm{app}}=R_0(T)-R_{\mathrm{app}}+Z(T),
\]
reduces the result to exponent comparison and the stability boundary.

For \(1-q_{\mathcal K}<\vartheta<1\), write
\(r(u)=u^\vartheta L_r(u)\), and set
\(q_{\mathcal N}(\vartheta):=\vartheta+q_{\mathcal K}-1\).  By
\cref{eq:rv_direct_beta_law,eq:rv_full_beta_law},
\[
Z(T)
\sim
\sigma^2\operatorname{Beta}(1-\vartheta,1-q_{\mathcal K})
T^{-q_{\mathcal N}(\vartheta)}\frac{L_{\mathcal K}(T)}{L_r(T)}.
\]
Thus the noisy exponent is \(\min\{q_0,q_{\mathcal N}(\vartheta)\}\): it
changes below \(q_0\), preserves the clean equivalent above, and at equality
\[
R_\sigma(T)-R_{\mathrm{app}}
\sim T^{-q_0}
\left[
L_0(T)
+\sigma^2\operatorname{Beta}(1-\vartheta,1-q_{\mathcal K})
\frac{L_{\mathcal K}(T)}{L_r(T)}
\right].
\]
preserves the exponent, with a renormalized slow factor if the gap matters.

For \(\vartheta=1\), put \(V(T)=\int_0^T r(u)^{-1}\,\mathrm du\).  The two
branches are \(Z(T)\sim\sigma^2k(T)V(T)\) if \(V(T)\to\infty\), and
\(Z(T)\sim Ck(T)\) if \(V_\infty<\infty\), by
\cref{eq:rv_theta_one_direct,eq:rv_full_beta_law,eq:rv_saturation_constant}.
Both have exponent \(q_{\mathcal K}\); e.g.\ \(r(u)\sim cu\) gives
\(V(T)\sim c^{-1}\log T\), changing only normalization.  For
\(\vartheta>1\), \cref{eq:rv_saturation_constant} gives the same exponent.
Thus throughout \(\vartheta\geq1\), the centered total exponent is
\(\min\{q_0,q_{\mathcal K}\}\), giving change, criticality, or preservation
according as \(q_{\mathcal K}<q_0\), \(q_{\mathcal K}=q_0\), or
\(q_{\mathcal K}>q_0\).

At \(\vartheta=1-q_{\mathcal K}\), \cref{eq:rv_direct_beta_law} becomes
\[
N_t
\sim
\sigma^2\operatorname{Beta}(q_{\mathcal K},1-q_{\mathcal K})
\frac{L_{\mathcal K}(T_t)}{L_r(T_t)}.
\]
Since \(Z_t\geq N_t\), this slow lower bound excludes positive-power decay;
below the boundary, polynomial growth of
\(N_t/\sigma^2=\sum_{s<t}K_{t,s}\) contradicts row stability for the fixed
dynamics considered here.  By the positive decomposition, the same lower bound
rules out positive-power convergence of the total centered loss to
\(R_{\mathrm{app}}\); this is the destroy boundary.

For \(q_{\mathcal K}>1\), \cref{eq:rv_discrete_integrable_order} gives
\[
Z(T)=T^{-q_{\mathcal N}(\vartheta)+o(1)},
\qquad
q_{\mathcal N}(\vartheta)=\min\{\vartheta,q_{\mathcal K}\},
\qquad \vartheta>0.
\]
The same positive decomposition therefore gives total exponent
\(\min\{q_0,q_{\mathcal N}(\vartheta)\}\), and hence the stated change,
criticality, and preservation alternatives.  At \(\vartheta=0\), recent
injection rules out positive-power decay; for \(\vartheta<0\), the row mass
grows and contradicts row stability.  The positive decomposition transfers
this obstruction to the total centered loss.  Together with the long-memory
argument above, this proves \cref{thm:rv_full_preserve_change_destroy} and its
pure-power main-text specialization
\cref{thm:rv_preserve_change_destroy}.
\end{proof}

\Cref{fig:rv_preserve_change_destroy_continuum} illustrates the
classification in a finite-width analytic-spectrum continuum: increasing
\(\vartheta\) moves the response from the unstable side of the destroy
boundary, through schedule-controlled change, and toward preservation of the
clean-loss decay.

\begin{figure}[t]
\centering
\includegraphics[width=\linewidth]{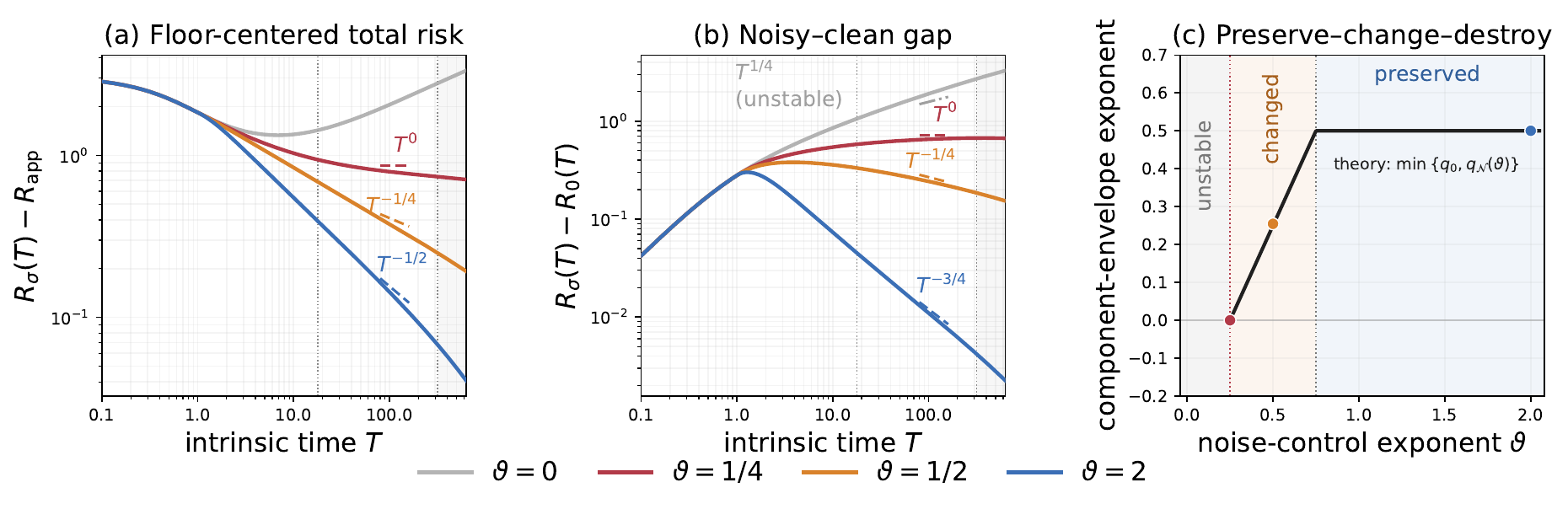}
\caption{Floor-centered risk, noisy--clean gap, and predicted exponents across
four noise-control schedules.  The transition from growth through destruction
and change to clean-law preservation supports
\cref{thm:rv_preserve_change_destroy}.}
\label{fig:rv_preserve_change_destroy_continuum}
\end{figure}

\begin{proposition}[Clean masking and memory saturation]
\label{prop:rv_total_loss_sharpness}
Work in the fixed-limit setting above, and suppose the schedules being
compared have centered clean losses with the same leading asymptotic
\(R_{0,i}(T)-R_{\mathrm{app}}\sim T^{-q_0}L_0(T)\), where \(L_0\) is eventually
positive and slowly varying.  The gap-based converse in
\cref{thm:rv_exact_discrete_identification} transfers to total loss without a
clean baseline only in the noise-dominated range \(q<q_0\).  This restriction
is necessary for a uniform converse based on the full leading total-loss
equivalent, while \(q<q_{\mathcal K}\) is necessary for exponent-level
schedule identification.
\begin{enumerate}[label=(\roman*)]
\item If \(q_0<q_{\mathcal K}\), distinct gap exponents in
\((q_0,q_{\mathcal K})\) can be generated by distinct admissible schedules, while
their corresponding clean and total losses share the same leading clean-loss asymptotic.
\item If the clean loss has exponent \(q_0<q_{\mathcal K}\) and
\(Z(T)\sim T^{-q_0}L(T)\) for an eventually positive slowly varying \(L\),
then
\begin{equation}
R_\sigma(T)-R_{\mathrm{app}}
\sim T^{-q_0}\bigl(L_0(T)+L(T)\bigr).
\label{eq:rv_equal_order_total}
\end{equation}
\item At the memory ceiling, integrable regularly varying injections with
different tail exponents all produce stable full gaps with exponent
\(q_{\mathcal K}\).
\end{enumerate}
\end{proposition}

\begin{proof}
For (i), choose distinct \(q_i\in(q_0,q_{\mathcal K})\).
Applying \cref{eq:rv_exact_discrete_inverse} with constant slow factors gives ratios of
power \(1-q_{\mathcal K}+q_i\) and gaps
\(Z_i(T)\sim a_iT^{-q_i}\), with \(a_i>0\) chosen small enough for uniform
stability.  Since \(q_i>q_0\), \(Z_i=o(R_{0,i}-R_{\mathrm{app}})\), hence
\(R_{\sigma,i}-R_{\mathrm{app}}\sim R_{0,i}-R_{\mathrm{app}}\sim
T^{-q_0}L_0(T)\) for both schedules.

Adding the positive equivalents proves \cref{eq:rv_equal_order_total} and the
three cases in (ii).  For example, when \(L_0\equiv1\),
\(L_i(T)=(\log T)^{-i}\), \(i=1,2\), yield different ratio tails through
\cref{eq:rv_exact_discrete_inverse}, but both leading total-risk asymptotics equal the
clean-loss asymptotic by the common clean-asymptotic assumption.

For (iii), choose distinct integrable tail exponents
\(\vartheta_i>1\) with row-stable amplitudes.
From \cref{eq:rv_saturation_constant}, we have \(Z_i(T)\sim C_i k(T)\), so both gaps
have exponent \(q_{\mathcal K}\).  The unit-mass family in
\cref{prop:rv_continuum_converse} more strongly gives distinct integrable
tails with the same direct equivalent.
\end{proof}

\subsection{Polynomial schedules in iteration time}
\label{app:rv_physical_time}

\begin{corollary}[Polynomial schedules in iteration time]
\label{cor:rv_physical_regime_map}
Assume the corresponding fixed-system hypotheses of
\cref{thm:rv_full_preserve_change_destroy}, including its regularly varying
clean-loss asymptotic with exponent \(q_0>0\), and impose its row-stability condition only
on the decaying branches.  For \(n\geq1\), let
\[
\eta_n=c_\eta n^{-\ell},
\qquad
B_n=\lceil c_Bn^b\rceil,
\qquad
c_\eta,c_B>0,\quad 0\leq\ell<1,\quad b\geq0.
\]
Then
\[
T_n\asymp n^{1-\ell},
\qquad
r(T_n)=\frac{B_n}{\eta_n}\asymp T_n^\vartheta,
\qquad
\vartheta=\frac{\ell+b}{1-\ell}.
\]
Consequently, on every decaying branch, the noisy--clean gap exponent with
respect to iteration index is
\[
q_{\mathcal N}^{\mathrm{it}}(\ell,b)
:=(1-\ell)
q_{\mathcal N}\!\left(\frac{\ell+b}{1-\ell}\right).
\]

For the PLRF long-memory specialization
\[
\frac14<\alpha<\frac12,
\qquad
q_{\mathcal K}=2-\frac{1}{2\alpha},
\]
the following alternatives hold:
\begin{enumerate}
\item If \(\ell+2\alpha b<1-2\alpha\), the row-stability bound cannot hold.

\item If \(\ell+2\alpha b=1-2\alpha\), the noisy--clean gap has no
positive-power decay.

\item If
\[
\ell+2\alpha b>1-2\alpha,
\qquad
2\ell+b<1,
\]
then
\[
q_{\mathcal N}^{\mathrm{it}}(\ell,b)
=
\frac{\ell+2\alpha b-(1-2\alpha)}{2\alpha}.
\]

\item If \(2\ell+b\geq1\), then
\[
q_{\mathcal N}^{\mathrm{it}}(\ell,b)
=(1-\ell)q_{\mathcal K}.
\]
At \(2\ell+b=1\), the gap additionally carries the logarithmic correction
inherited from the long-memory ceiling.
\end{enumerate}

In integrable memory, where \(q_{\mathcal K}>1\), if \(\ell+b>0\), then
\[
q_{\mathcal N}^{\mathrm{it}}(\ell,b)
=\min\{\ell+b,(1-\ell)q_{\mathcal K}\}.
\]
If \(\ell=b=0\), the noisy--clean gap has no positive-power decay.

The clean-risk asymptotic in \cref{thm:rv_full_preserve_change_destroy} gives
\[
R_0(T_n)-R_{\mathrm{app}}
=n^{-(1-\ell)q_0+o(1)}.
\]
Hence, on every decaying branch,
\[
R_\sigma(T_n)-R_{\mathrm{app}}
=n^{-\min\{(1-\ell)q_0,q_{\mathcal N}^{\mathrm{it}}(\ell,b)\}+o(1)}.
\]
\end{corollary}

\begin{proof}
Changing finitely many initial schedule values does not affect the stated
asymptotic exponents, and discrete Karamata gives
\[
T_n
=T_1+\sum_{s=1}^{n-1}c_\eta s^{-\ell}
\sim \frac{c_\eta}{1-\ell}n^{1-\ell}.
\]
Moreover, \(B_n\sim\bar c_Bn^b\), where \(\bar c_B=c_B\) for \(b>0\) and
\(\bar c_B=\lceil c_B\rceil\) for \(b=0\).  Therefore
\[
r(T_n)=\frac{B_n}{\eta_n}
\sim\frac{\bar c_B}{c_\eta}n^{b+\ell}.
\]
Since \(T_{n+1}/T_n\to1\), the piecewise-constant intrinsic-time
interpolation of \(r\) is eventually nondecreasing and regularly varying with
index
\[
\vartheta=\frac{\ell+b}{1-\ell}.
\]

In PLRF long memory,
\[
1-q_{\mathcal K}=\frac{1-2\alpha}{2\alpha}.
\]
Thus
\[
\vartheta=1-q_{\mathcal K}
\quad\Longleftrightarrow\quad
\ell+2\alpha b=1-2\alpha,
\]
which gives the stability/no-decay boundary.  Likewise,
\[
\vartheta=1
\quad\Longleftrightarrow\quad
2\ell+b=1,
\]
which gives the memory ceiling.

On the interior long-memory branch
\(1-q_{\mathcal K}<\vartheta<1\),
\cref{thm:rv_full_preserve_change_destroy} gives intrinsic-time gap exponent
\(\vartheta+q_{\mathcal K}-1\).  Composing with
\(T_n\asymp n^{1-\ell}\) yields
\[
(1-\ell)(\vartheta+q_{\mathcal K}-1)
=(1-\ell)\left(
\frac{\ell+b}{1-\ell}+1-\frac{1}{2\alpha}
\right)
=\frac{\ell+2\alpha b-(1-2\alpha)}{2\alpha}.
\]
For \(\vartheta>1\), the intrinsic-time exponent saturates at
\(q_{\mathcal K}\), giving iteration-time exponent
\((1-\ell)q_{\mathcal K}\).  At \(\vartheta=1\), the polynomial schedule
gives \(r(u)\sim cu\) for some \(c>0\), so
\(V(T)\sim c^{-1}\log T\); the long-memory boundary asymptotic in
\cref{thm:rv_full_preserve_change_destroy} therefore gives the stated ceiling
logarithm.

In integrable memory, the intrinsic-time exponent is
\(\min\{\vartheta,q_{\mathcal K}\}\) whenever \(\vartheta>0\).  Hence
\[
(1-\ell)\min\{\vartheta,q_{\mathcal K}\}
=\min\{\ell+b,(1-\ell)q_{\mathcal K}\}.
\]
When \(\ell=b=0\), one has \(\vartheta=0\), and the gap has no
positive-power decay.

Finally, the regularly varying clean-risk asymptotic in
\cref{thm:rv_full_preserve_change_destroy} composed with
\(T_n\asymp n^{1-\ell}\) gives clean iteration-time exponent
\((1-\ell)q_0\).  Positivity then gives the displayed minimum of the clean and
gap exponents and the change/critical/preserve alternatives.
\end{proof}

\begingroup

\section{Constant learning-rate and batch-size schedules}
\label{sec:constant_schedule_noisy_43}
\label{app:constant_schedule_noisy_proofs}

This appendix analyzes constant learning-rate and batch-size schedules under a
feature-compute budget.  Appendix~\ref{subsec:noisy_43_constant_scaling_inputs}
collects the forcing, memory, and localization inputs needed to formulate the
optimization problem.  Appendix~\ref{subsec:noisy_43_constant_discrete_optimum_proof}
studies the exact discrete-time infinite-width problem and establishes the
scale of its fixed-noise optimum.
Appendix~\ref{subsec:noisy_43_constant_fixed_noise_optima} derives the
source-window-optimal rates below trace class and the conditional plateau
within a fixed finite-time window above trace class, and
Appendix~\ref{subsec:noisy_43_constant_optimum_plateau_proof} proves these
results.  Appendix~\ref{subsec:noisy_43_constant_crossover_proof} compares the
clean and noisy optima and derives the compute and noise scales at which the
transition occurs.

Within the analyzed source and finite-time windows, label noise shifts compute
toward wider models and fewer iterations.

Unless a statement explicitly invokes the exact conditional recursion and
uses Roman risk notation, every calligraphic forcing, kernel, and risk in this
section is a resolvent deterministic equivalent (DE).

Write
\(P_{\sigma,\mathrm{const}}^\star
:=\lim_{\mathfrak f\to\infty}
\mathcal R_{\sigma,\mathrm{const}}^\star(\mathfrak f)\)
for the constant-schedule compute-optimal DE plateau, distinct from the
infinite-width DE curve \(\mathcal R_{\sigma,\infty}(t)\).

\subsection{Scaling inputs and optimization setup}
\label{subsec:noisy_43_constant_scaling_inputs}

Write \(p:=2\alpha+2\beta-1>0\).

Here \(T=\eta t\), and stability gives
\(q_\eta(\lambda)^t\asymp e^{-cT\lambda}\).  Thus
\(\lambda_j=j^{-2\alpha}\) has cutoff
\[
j_T\asymp T^{1/(2\alpha)}:
\qquad
j\ll j_T\ \hbox{is substantially learned},
\qquad
j\gg j_T\ \hbox{is substantially unlearned}.
\]
The power-law window \(1\ll\eta t\lesssim m^{2\alpha}\) is therefore past
the initial transient but before resolution of the smallest feature modes.

The four clean-loss components used below have the following scales.  On the
LM and IM DE branches, assume \(\alpha>1/4\), \(\beta<1+2\alpha\), and stay
away from critical lines; the FB total-risk asymptotic uses the same source
restriction with \(0<\alpha<1/4\) and \(p>0\).  In the window
\(1\ll\eta t\lesssim m^{2\alpha}\), absent components are set to zero and
\begin{align}
\mathcal F_{pp}(t,m)
&\asymp
(\eta t)^{-p/(2\alpha)},
\label{eq:noisy_43_Fpp_scaling}\\
\mathcal F_0(m)
&\asymp
\begin{cases}
m^{-p},&\beta<\frac12,\\
m^{-2\alpha},&\beta>\frac12,
\end{cases}
\notag\\
\mathcal F_{ac}(t,m)
&\asymp
m^{-1}(\eta t)^{-1+1/(2\alpha)}
\quad
\left(\alpha>\frac12,\ \beta>\frac12\right),
\label{eq:noisy_43_Fac_scaling}\\
\frac1{\eta}\mathcal K_{pp}(t,m)
&\asymp
\frac{\eta}{B}(\eta t)^{-2+1/(2\alpha)}
\quad
\left(\alpha>\frac14\right).
\label{eq:noisy_43_Kpp_scaling}
\end{align}
Here \(\mathcal F_{pp}\) is time-limited unlearned aligned-target error,
\(\mathcal F_0\) is width-limited permanent null-space error, and
\(\mathcal F_{ac}\) and \(\mathcal K_{pp}/\eta\) depend on both width and time,
representing finite-width spectral distortion and aligned spectral memory from
transient SGD sampling noise, respectively.  The next corollary shows that
these four terms recover the clean \(4+3\) decomposition.  This closure is not
an additional assumption: the forcing and kernel estimates of
\citet[Appendices~G--H]{paquette20244+}, together with their Volterra
approximation theorem and subexponential kernel estimate, establish the
decomposition in the canonical PLRF model.  We record its translation to the
present notation.

\begin{corollary}[Recovery of the \(4+3\) spectral decomposition]
\label{cor:noisy_43_43_spectral_separation}
Suppose \(\alpha>1/4\), \(2\alpha+2\beta>1\),
\(\beta<1+2\alpha\), with parameters away from the critical lines.  Let the constant
schedule satisfy \cref{ass:noisy_43_constant_schedule_stability}.  Under the
notation map
\[
d_{4+3}=m,
\qquad
v_{4+3}=d,
\qquad
B_{4+3}=B,
\qquad
\gamma_{4+3}=\frac{\eta}{B},
\]
uniformly for \(1\ll\eta t\lesssim m^{2\alpha}\),
\begin{equation}
\label{eq:noisy_43_clean_scaling_decomposition}
\mathcal R_0(t,m)
\asymp
\mathcal F_0(m)+\mathcal F_{pp}(t,m)+\mathcal F_{ac}(t,m)
+\frac1{\eta}\mathcal K_{pp}(t,m).
\end{equation}
\end{corollary}

\begin{proof}
Under the displayed map,
\(\gamma_{4+3}B_{4+3}=\eta\) and
\(\gamma_{4+3}^2B_{4+3}=\eta^2/B\), matching both the intrinsic clock and
the kernel prefactor.  The forcing estimates of
\citet[Appendix~H]{paquette20244+} give
\(\mathcal F\asymp\mathcal F_0+\mathcal F_{pp}+\mathcal F_{ac}\), including
their error control.  Their complete-kernel comparison and Volterra
approximation theorem give the remaining
\(\eta^{-1}\mathcal K_{pp}\) contribution, proving the claim.
\end{proof}

When \(\alpha<1/4\), the infinite pure-point kernel is not summable.  The
finite-bulk DE kernel and noisy--clean gap are nevertheless controlled by
\cref{prop:td_lrs_joint_de_finite_bulk_kernel}, and the canonical PLRF
constant-schedule risk asymptotic is the specialization of
\cref{prop:td_lrs_joint_de_finite_bulk} for \(p>0\),
\(\beta<1+2\alpha\), off the critical lines.  For comparison, write
\(F_{ac}^{\bm W}\) and \(F_0^{\bm W}\) for the absolutely-continuous and
zero-mode components of the exact empirical forcing.  The corresponding
finite-dimensional proposition of \citet{paquette20244+} gives
\[
F_{ac}^{\bm W}(t,m)=0
\quad\left(\beta<\frac12\right),
\qquad
0\leq F_{ac}^{\bm W}(t,m)
\leq C_{\alpha,\beta,c}F_0^{\bm W}(m)
\quad\left(\beta>\frac12,\ \alpha<\frac12\right).
\]
Here \(c=\lim d/m\in(1,\infty)\), uniformly on strict localized windows.

Label noise is controlled by the cumulative kernel.  For \(\alpha>1/4\), the
PLRF mode estimates and \cref{eq:noisy_43_learning_rate} give
\begin{equation}
\label{eq:noisy_43_cumulative_kernel}
\mathcal S_t(m)
\asymp
\frac{\eta}{B}\sum_{j\leq m}j^{-2\alpha}
\left(1-\exp\{-c_1\eta tj^{-2\alpha}\}\right),
\end{equation}
for a constant \(c_1>0\).  In the trace-class region \(\alpha>1/2\), write
\(\mathcal S_\infty(m):=\lim_{t\to\infty}\mathcal S_t(m)\in(0,\kappa]\).
Then
\begin{equation}
\label{eq:noisy_43_cumulative_kernel_regimes}
\mathcal S_t(m)=
\begin{cases}
\mathcal S_\infty(m)-\Theta\!\left(t^{-1+1/(2\alpha)}\right),
&\alpha>\frac12,\\[0.5em]
\Theta\!\left(
\left(\dfrac{t}{Bm}\right)^{(1-2\alpha)/(2\alpha)}
\right),
&\frac14<\alpha<\frac12,\\[0.9em]
\Theta\!\left(\dfrac{t}{Bm}\right),
&0<\alpha<\frac14,
\end{cases}.
\end{equation}
The last line is the target-independent finite-bulk
DE kernel estimate from \cref{prop:td_lrs_joint_de_finite_bulk_kernel}; the exact empirical analogue is
\cref{lem:td_lrs_fb_bulk_kernel}.  Above trace class, accumulated noise
saturates.  Below trace class, the width-dependent learning rate can drive it
to zero.

One iteration costs order \(Bm\), so total compute is \(\mathfrak f:=tBm\).
For reference, define the full-budget optimum
\begin{equation*}
\mathcal R_{\sigma,\mathrm{const}}^\star(\mathfrak f)
:=
\min_{1\leq m\leq\mathfrak f/B}
\mathcal R_\sigma\left(\frac{\mathfrak f}{Bm},m\right),
\end{equation*}
and let \(m_{\sigma,\mathrm{const}}^\star(\mathfrak f)\) be any minimizer.
To separate a
possible noise floor from the part that still improves, define
\begin{align*}
P_{\sigma,\mathrm{const}}^\star
&:=
\lim_{\mathfrak f\to\infty}\mathcal R_{\sigma,\mathrm{const}}^\star(\mathfrak f),
\\
\Delta_{\sigma,\mathrm{const}}^\star(\mathfrak f)
&:=
\mathcal R_{\sigma,\mathrm{const}}^\star(\mathfrak f)
-P_{\sigma,\mathrm{const}}^\star .
\end{align*}

Below trace class, the power laws hold only in the source window.  We therefore
optimize only over the concrete interior class
\[
\mathcal M_{\mathrm{src}}(\mathfrak f)
:=\left\{1\leq m\leq\frac{\mathfrak f}{B}:
\log\mathfrak f\leq\frac{\eta\mathfrak f}{Bm}
\leq\frac{m^{2\alpha}}{\log\mathfrak f}\right\}.
\]
For sufficiently large \(\mathfrak f\), define
\[
\mathcal R_{\sigma,\mathrm{src}}^\star(\mathfrak f)
:=\min_{m\in\mathcal M_{\mathrm{src}}(\mathfrak f)}
\mathcal R_\sigma\!\left(\frac{\mathfrak f}{Bm},m\right),
\qquad
m_{\sigma,\mathrm{src}}^\star(\mathfrak f)
\in\operatorname*{argmin}_{m\in\mathcal M_{\mathrm{src}}(\mathfrak f)}
\mathcal R_\sigma\!\left(\frac{\mathfrak f}{Bm},m\right).
\]
Set
\(t_{\sigma,\mathrm{src},\mathfrak f}^\star
:=\mathfrak f/(B m_{\sigma,\mathrm{src}}^\star(\mathfrak f))\).
The logarithmic margins select a strict source window; the attaining scales
below lie polynomially far from both boundaries.

In the trace-class region \(\alpha>1/2\), define
\(\mathcal R_{\sigma,\infty}(t):=\lim_{m\to\infty}\mathcal R_\sigma(t,m)\)
and
\begin{equation}
\label{eq:noisy_43_a_sigma}
a_\sigma=
\begin{cases}
p,&\beta<\frac12,\\
1,&\beta>\frac12.
\end{cases}
\end{equation}

\begin{assumption}[Trace-class local optimum and DE expansion]
\label{ass:noisy_43_trace_local_expansion}
Let \(U\subset(1,\infty)\) be compact.
\begin{enumerate}[label=(\alph*)]
\item The continuously interpolated curve \(\mathcal R_{\sigma,\infty}\) has
a unique nondegenerate minimizer
\(t_{\sigma,\mathrm{const}}^\star\) in the interior of \(U\).
\item Uniformly for \(t\in U\),
\begin{equation}
\label{eq:noisy_43_fixed_time_expansion}
\mathcal R_\sigma\!\left(t,\frac{\mathfrak f}{Bt}\right)
=
\mathcal R_{\sigma,\infty}(t)
+\left(\frac{\mathfrak f}{Bt}\right)^{-a_\sigma}G_\sigma(t)
+o(\mathfrak f^{-a_\sigma}),
\qquad
G_\sigma(t_{\sigma,\mathrm{const}}^\star)>0.
\end{equation}
\end{enumerate}
\end{assumption}

Under \cref{ass:noisy_43_trace_local_expansion}, define
\[
\mathcal R_{\sigma,U}^\star(\mathfrak f)
:=\min_{t\in U}\mathcal R_\sigma\!\left(t,\frac{\mathfrak f}{Bt}\right),
\qquad
t_{\sigma,U,\mathfrak f}^\star
\in\operatorname*{argmin}_{t\in U}
\mathcal R_\sigma\!\left(t,\frac{\mathfrak f}{Bt}\right),
\]
and set \(m_{\sigma,U,\mathfrak f}^\star
:=\mathfrak f/(B t_{\sigma,U,\mathfrak f}^\star)\),
\(P_{\sigma,U}^\star
:=\mathcal R_{\sigma,\infty}(t_{\sigma,\mathrm{const}}^\star)\), and
\(\Delta_{\sigma,U}^\star(\mathfrak f)
:=\mathcal R_{\sigma,U}^\star(\mathfrak f)-P_{\sigma,U}^\star\).
\subsection{Discrete infinite-width fixed-noise optimum}
\label{subsec:noisy_43_constant_discrete_optimum_proof}

Before optimizing over continuous compute diagonals, introduce the exact
integer-time objects at infinite width.  For trace-class parameters, let
\[
\kappa_\infty
:=
\frac{\eta}{B}\sum_{j\geq1}
\frac{j^{-2\alpha}}{2-\eta\left(1+\frac1B\right)j^{-2\alpha}}.
\]
For \(|w|<1\), set
\begin{align*}
\widetilde{\mathcal F}_\infty(w)
&:=
\sum_{j\geq1}
\frac{j^{-2(\alpha+\beta)}}{1-q_\eta(j^{-2\alpha})w},
&
\widetilde{\mathcal K}_\infty(w)
&:=
\frac{\eta^2}{B}\sum_{j\geq1}
\frac{j^{-4\alpha}}{1-q_\eta(j^{-2\alpha})w}.
\end{align*}
Thus \(\widetilde{\mathcal K}_\infty(1)=\kappa_\infty\).  Write
\([w^t]\) for formal coefficient extraction and, for \(\sigma^2>0\), define
\begin{align*}
\mathcal R_{\sigma,\infty}^{\mathrm{disc}}(t)
&:=
[w^t]\,
\frac{\widetilde{\mathcal F}_\infty(w)
+\sigma^2w\widetilde{\mathcal K}_\infty(w)/(1-w)}
{1-w\widetilde{\mathcal K}_\infty(w)},
&
\mathcal R_{\sigma,\infty}^{\star,\mathrm{disc}}
&:=
\inf_{t\in\mathbb N,\ t\geq1}
\mathcal R_{\sigma,\infty}^{\mathrm{disc}}(t).
\end{align*}
The next proposition checks this discrete problem.  It shows that its best
fixed-noise risk is uniformly of order \(\sigma^2\).  The proof is in
\cref{subsec:noisy_43_constant_discrete_optimum_proof}.

\begin{proposition}[Uniform scale of the discrete trace-class optimum]
\label{prop:noisy_43_discrete_trace_optimum}
Fix \(\sigma^2>0\), \(\delta\in(0,1)\), \(\underline\eta>0\), and
\(\overline\kappa\in(0,1)\), and restrict to
\[
\alpha\geq\tfrac12+\delta,\qquad
\eta\geq\underline\eta,\qquad
\eta\left(1+\frac1B\right)\leq2-\delta,\qquad
\kappa_\infty\leq\overline\kappa.
\]
Then, uniformly on this strict region,
\[
\sigma^2\frac{\eta^2}{B}\zeta(4\alpha)
\leq
\mathcal R_{\sigma,\infty}^{\star,\mathrm{disc}}
\leq
\frac{\sigma^2\kappa_\infty}{1-\kappa_\infty},
\qquad
\mathcal R_{\sigma,\infty}^{\star,\mathrm{disc}}
\asymp
\sigma^2\frac{\eta^2}{B}\zeta(4\alpha)
\asymp\sigma^2.
\]
\end{proposition}

For genuinely empirical finite-step statements, the Gaussian propagation in
\cref{subsec:td_lrs_joint_exact_dynamics} specializes to
\begin{equation}
\label{eq:noisy_43_mode_recursion}
\mathbb E_t[\rho_j^2(t+1)]
=
q_\eta(\widehat\lambda_j)\rho_j^2(t)
+\frac{\eta^2}{B}\widehat\lambda_j^2
\left(\|\bm e_t\|_2^2+\sigma^2\right).
\end{equation}

\begin{proof}[Proof of
\cref{prop:noisy_43_discrete_trace_optimum}]
For fixed integer time,
\[
\mathbb E\|\widehat{\bm H}-\bm\Lambda\|_{\mathrm{HS}}^2
=
\frac{(\operatorname{tr}\bm\Lambda)^2
+\operatorname{tr}(\bm\Lambda^2)}{m}
=O(m^{-1}).
\]
Finite induction in \cref{eq:noisy_43_mode_recursion} gives the stated
infinite-width series, while positivity gives
\[
\mathcal R_{\sigma,\infty}^{\mathrm{disc}}(t)
\geq\sigma^2\widetilde{\mathcal K}_\infty(0)
=\sigma^2\frac{\eta^2}{B}\zeta(4\alpha).
\]
The kernel geometric series gives \(\kappa_\infty\); hence the positive
Volterra resolvent converges to
\(\sigma^2\kappa_\infty/(1-\kappa_\infty)\), which bounds the integer-time
infimum from above.  Therefore
\[
1\leq
\frac{\mathcal R_{\sigma,\infty}^{\star,\mathrm{disc}}}
{\sigma^2\frac{\eta^2}{B}\zeta(4\alpha)}
\leq
\frac{\overline\kappa}
{(1-\overline\kappa)\underline\eta^2},
\quad\text{and}\quad
\underline\eta^2
\leq\frac{\eta^2}{B}\zeta(4\alpha)
\leq\zeta(2+4\delta),
\]
proving both uniform comparisons.
\end{proof}

\subsection{Source-window rates and trace-class plateau}
\label{subsec:noisy_43_constant_fixed_noise_optima}

The next theorem gives two distinct optimization statements across the eight
open PLRF propagation subregimes of
\cref{def:td_lrs_plrf_open_subregimes}.  Above trace class, it concerns the
optimum within the fixed finite-time window \(U\).  Below trace class, it
concerns the source-window-restricted optimum, which balances
unlearned target error against accumulated label noise.  The proof is in
\cref{subsec:noisy_43_constant_optimum_plateau_proof}; the plateau
consequence is stated separately in
\cref{prop:noisy_43_noise_plateau}.

For fixed \(\sigma^2>0\), denote the relevant loss and width exponents by
\((\bar\rho_\sigma,\bar\xi_\sigma)\), using
\[
\begin{cases}
\mathcal R_{\sigma,\mathrm{src}}^\star(\mathfrak f)
\asymp\mathfrak f^{-\bar\rho_\sigma},\quad
m_{\sigma,\mathrm{src}}^\star(\mathfrak f)
\asymp\mathfrak f^{\bar\xi_\sigma},&\alpha<1/2,\\
\Delta_{\sigma,U}^\star(\mathfrak f)
\asymp\mathfrak f^{-\bar\rho_\sigma},\quad
m_{\sigma,U,\mathfrak f}^\star
\asymp\mathfrak f^{\bar\xi_\sigma},&\alpha>1/2.
\end{cases}
\]
On the trace-class branch, \(t_{\sigma,U,\mathfrak f}^\star\) denotes the
optimizing iteration within \(U\).

\begin{theorem}[Fixed-noise compute rates]
\label{thm:noisy_43_centered_compute_optimal}
Fix \(\sigma^2>0\), suppose \(2\alpha+2\beta>1\), and assume
\((\alpha,\beta)\) lies in an open PLRF propagation subregime.  Assume
\cref{ass:noisy_43_constant_schedule_stability} and the cumulative-memory
asymptotics in \cref{eq:noisy_43_cumulative_kernel_regimes}.

On the LM and IM branches, assume \(\beta<1+2\alpha\), stay off the critical
lines, and use \cref{cor:noisy_43_43_spectral_separation}.  On the FB
branches, assume \(\beta<1+2\alpha\), stay off the critical lines, and use
\cref{prop:td_lrs_joint_de_finite_bulk}; the high-source
\(\mathrm{FB}_2\) extension additionally uses
\cref{ass:td_lrs_high_source_fb2_forcing,cor:td_lrs_joint_de_high_source_fb2}.

For \(\alpha<1/2\), optimize over
\(\mathcal M_{\mathrm{src}}(\mathfrak f)\).  For \(\alpha>1/2\), optimize
over \(U\) and assume \cref{ass:noisy_43_trace_local_expansion}.
As \(\mathfrak f\to\infty\) with fixed \(\sigma^2>0\), these exponents
are well defined and satisfy
\begin{equation}
\label{eq:noisy_43_exponents}
(\bar\rho_\sigma,\bar\xi_\sigma)
=
\begin{cases}
\left(p,1\right),
&\alpha>\frac12,\ \beta<\frac12,\\[0.4em]
(1,1),
&\alpha>\frac12,\ \beta>\frac12,\\[0.4em]
\left(
\dfrac{(1-2\alpha)p}
{2[\alpha(1-2\alpha)+2\beta(1-\alpha)]},
\dfrac{\beta}
{\alpha(1-2\alpha)+2\beta(1-\alpha)}
\right),
&\frac14<\alpha<\frac12,\\[1em]
\left(
\dfrac{\alpha p}
{p(1-\alpha)+2\alpha},
\dfrac{p+2\alpha}
{2[p(1-\alpha)+2\alpha]}
\right),
&0<\alpha<\frac14.
\end{cases}
\end{equation}
For the five sub-trace subregimes, the corresponding loss and width scalings are
\begin{equation}
\label{eq:noisy_43_full_subtrace_laws}
\begin{aligned}
\mathcal R_{\sigma,\mathrm{src}}^\star(\mathfrak f)
&\asymp
\begin{cases}
(\sigma^2)^{\frac{(1-\alpha)p}
{\alpha(1-2\alpha)+2\beta(1-\alpha)}}
\mathfrak f^{-\frac{(1-2\alpha)p}
{2[\alpha(1-2\alpha)+2\beta(1-\alpha)]}},
&\frac14<\alpha<\frac12,\\[0.6em]
(\sigma^2)^{\frac{(1-\alpha)p}{p(1-\alpha)+2\alpha}}
\mathfrak f^{-\frac{\alpha p}{p(1-\alpha)+2\alpha}},
&0<\alpha<\frac14,
\end{cases}\\[1em]
m_{\sigma,\mathrm{src}}^\star(\mathfrak f)
&\asymp
\begin{cases}
(\sigma^2)^{\frac{\alpha}
{\alpha(1-2\alpha)+2\beta(1-\alpha)}}
\mathfrak f^{\frac{\beta}
{\alpha(1-2\alpha)+2\beta(1-\alpha)}},
&\frac14<\alpha<\frac12,\\[1em]
(\sigma^2)^{\frac{\alpha}{p(1-\alpha)+2\alpha}}
\mathfrak f^{\frac{p+2\alpha}{2[p(1-\alpha)+2\alpha]}},
&0<\alpha<\frac14.
\end{cases}
\end{aligned}
\end{equation}
In the three \(\mathrm{IM}\) subregimes, every such choice satisfies
\[
t_{\sigma,U,\mathfrak f}^\star
\longrightarrow t_{\sigma,\mathrm{const}}^\star=\Theta(1),
\qquad
m_{\sigma,U,\mathfrak f}^\star\longrightarrow\infty.
\]
\end{theorem}
In the \(\mathrm{IM}\) regime, the optimal number of iterations remains finite.
Within the source window, the five \(\mathrm{FB}\)--\(\mathrm{LM}\)
subregimes balance
\(\mathcal F_{pp}\asymp\sigma^2\mathcal S_t\).

\Cref{fig:noisy_43_im_early_stopping} visualizes the finite-time stopping
mechanism in the three \(\mathrm{IM}\) subregimes using finite-width
deterministic-equivalent Volterra dynamics with \(B=1\), \(\eta=0.05\), and
\(d/m=2\).

\begin{figure}[t]
\centering
\includegraphics[width=\linewidth]{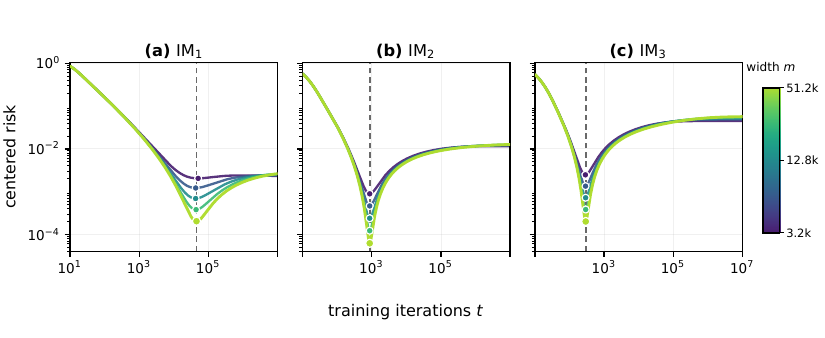}
\caption{Centered finite-width DE risk across widths in the three
integrable-memory subregimes.  The stabilization of the minimizing iteration supports
the finite-time local optimum in
\cref{thm:noisy_43_centered_compute_optimal}.}
\label{fig:noisy_43_im_early_stopping}
\end{figure}
To compare with clean training, define
\[
\mathcal R_{0,\mathrm{const}}^\star(\mathfrak f)
:=
\min_{1\leq m\leq\mathfrak f/B}
\mathcal R_0\left(\frac{\mathfrak f}{Bm},m\right),
\qquad
\mathcal R_{0,\mathrm{const}}^\star(\mathfrak f)
\asymp\mathfrak f^{-\rho_0},
\qquad
m_{0,\mathrm{const}}^\star(\mathfrak f)\asymp\mathfrak f^{\xi_0},
\]
where \(m_{0,\mathrm{const}}^\star(\mathfrak f)\) is any minimizer in the
definition of \(\mathcal R_{0,\mathrm{const}}^\star\).  The clean \(4+3\)
calculation gives the reusable
baseline
\begin{equation}
\label{eq:noisy_43_clean_compute_exponents}
(\rho_0,\xi_0)
=
\begin{cases}
\left(\dfrac{p}{2\alpha+1},\dfrac{1}{2\alpha+1}\right),
&\text{4+3 Phase Ia},\\[0.8em]
\left(\dfrac{p}{2(\alpha+\beta)},\dfrac{\beta}{\alpha+\beta}\right),
&\text{4+3 Phase II},\\[0.8em]
\left(\dfrac{4\alpha-1}{4\alpha},\dfrac12\right),
&\text{4+3 Phase III},\\[0.8em]
\left(\dfrac p2,\dfrac12\right),
&\text{4+3 Phase Ib},\\[0.8em]
\left(
\dfrac{\alpha p}{2\beta-1+3\alpha-2\alpha\beta},
\dfrac{p}{2(2\beta-1+3\alpha-2\alpha\beta)}
\right),
&\text{4+3 Phase Ic},\\[1em]
\left(\alpha,\dfrac12\right),
&\text{4+3 Phase IVa},\\[0.8em]
\left(
\dfrac{(1-2\alpha)p}{2[2\beta(1-\alpha)-\alpha]},
\dfrac{\beta-\alpha}{2\beta(1-\alpha)-\alpha}
\right),
&\text{4+3 Phase IVb}.
\end{cases}
\end{equation}
Combining \cref{eq:noisy_43_clean_compute_exponents,eq:noisy_43_exponents}
gives the following comparisons, with the noisy exponents interpreted as
source-window-restricted below trace class and \(U\)-restricted above trace
class:
\[
\begin{gathered}
\bar\rho_\sigma
\begin{cases}
>\rho_0,&\text{in 4+3 Phases Ia, II, and III},\\
<\rho_0,&\text{in 4+3 Phases Ib, Ic, IVa, and IVb},
\end{cases},\\[0.4em]
\bar\xi_\sigma>\xi_0
\quad\text{in all seven open 4+3 comparison phases}.
\end{gathered}
\]
The larger exponent in 4+3 Phases Ia, II, and III does not mean that noise improves the
uncentered risk.  Clean and noisy losses are centered at different limits.
The larger exponent only describes the large-width DE correction around the
finite-time local optimum.
Since \(\mathfrak f=Bmt\),
\[
t_{0,\mathrm{const},\mathfrak f}^\star
\asymp\mathfrak f^{1-\xi_0},
\qquad
t_{\sigma,\mathrm{src},\mathfrak f}^\star
\asymp\mathfrak f^{1-\bar\xi_\sigma}
\quad(\alpha<1/2).
\]
Thus the displayed noisy branch allocates a larger compute exponent to width
than the clean benchmark.
Above trace class the shift is complete, with
\(t_{\sigma,\mathrm{const}}^\star=\Theta(1)\) within \(U\).  Below trace
class, both width and iteration count still diverge, but noise favors earlier
stopping.

Fixed noise merges the eight propagation subregimes into four exponent regions.
Within the source window below trace class, the balance is
\(\mathcal F_{pp}\asymp\sigma^2\mathcal S_t\); above it the \(U\)-restricted
optimizer stops at finite time.  The following proposition separates its
plateau from the decaying centered correction.

\begin{proposition}[Source-window decay and trace-class plateau]
\label{prop:noisy_43_noise_plateau}
Under the hypotheses of \cref{thm:noisy_43_centered_compute_optimal},
if \(\alpha<1/2\), then
\[
\mathcal R_{\sigma,\mathrm{src}}^\star(\mathfrak f)\longrightarrow0.
\]
If \(\alpha>1/2\), then
\[
P_{\sigma,U}^\star
=\mathcal R_{\sigma,\infty}(t_{\sigma,\mathrm{const}}^\star)>0,
\]
and
\[
\mathcal R_{\sigma,U}^\star(\mathfrak f)
=
P_{\sigma,U}^\star
+\Theta(\mathfrak f^{-a_\sigma}),
\qquad
t_{\sigma,\mathrm{const}}^\star=\Theta(1).
\]
\end{proposition}

\subsection{Proof of \cref{thm:noisy_43_centered_compute_optimal}}
\label{subsec:noisy_43_constant_optimum_plateau_proof}

\begin{proof}[Proof of \cref{thm:noisy_43_centered_compute_optimal}]
\textit{Step 1: noisy Volterra reduction.}
Set \(\eta_s\equiv\eta\), \(B_s\equiv B\) in
\cref{prop:td_lrs_joint_de_recursion}.  In the convolution form
\cref{eq:noisy_43_clean_risk_volterra}, the forcing and kernel from
\cref{eq:noisy_43_forcing_definition,eq:noisy_43_kernel_definition}
propagate initialization and one unit of clean-target DE risk, while labels contribute
\(\sigma^2\mathcal S\).  Subtracting
\cref{eq:noisy_43_clean_volterra_definition}, with \(\bm1=(1,1,\ldots)\), gives
\[
\mathcal R_\sigma-\mathcal R_0
=
\sigma^2\sum_{\ell\geq1}\mathcal K^{*\ell}*\bm1 .
\]
Its first term is \(\mathcal S\), so positivity and
\(\|\mathcal K\|_{\ell^1}\leq\kappa\) give
\[
\mathcal S_t
\leq
\sum_{\ell\geq1}(\mathcal K^{*\ell}*\bm1)(t)
\leq
\mathcal S_t\sum_{\ell\geq0}\kappa^\ell
=\frac{\mathcal S_t}{1-\kappa}.
\]
This proves \cref{eq:noisy_43_scaling_reduction}.

\textit{Step 2: cumulative-kernel asymptotics.}
For \(\alpha>1/4\), \cref{lem:td_lrs_de_spectral_filters} gives, modewise and
uniformly below the stability threshold,
\[
\sum_{s=0}^{t-1}
\frac{\eta^2}{B}\lambda^2q_\eta(\lambda)^s
=
\frac{\eta^2}{B}\lambda^2
\frac{1-q_\eta(\lambda)^t}{1-q_\eta(\lambda)}
\asymp
\eta\lambda
\left(1-e^{-c_2\eta t\lambda}\right)
\]
for some \(c_2>0\).  Inserting \(\lambda_j\asymp j^{-2\alpha}\) proves
\cref{eq:noisy_43_cumulative_kernel}; for \(1/4<\alpha<1/2\), splitting at
\(j_t\asymp(\eta t)^{1/(2\alpha)}\) gives the same order on both sides:
\[
\mathcal S_t
\asymp
\eta(\eta t)^{(1-2\alpha)/(2\alpha)}
\asymp
\left(\frac{t}{Bm}\right)^{(1-2\alpha)/(2\alpha)}.
\]
For \(0<\alpha<1/4\), the LM/IM filter fails at the lower edge, while the
constant-schedule target-independent FB kernel estimate in
\cref{prop:td_lrs_joint_de_finite_bulk_kernel} gives
\[
\mathcal S_t\asymp\frac{t}{Bm}.
\]
For \(\alpha>1/2\), summability gives
\[
\mathcal S_\infty(m)-\mathcal S_t(m)
\asymp
\sum_{j\geq1}j^{-2\alpha}
e^{-c_2\eta tj^{-2\alpha}}
\asymp t^{-1+1/(2\alpha)}.
\]
Together these three cases prove
\cref{eq:noisy_43_cumulative_kernel_regimes}.

\textit{Step 3: optimization below the trace-class line.}
Suppose \(\alpha<1/2\).  We identify the unique corner within
\(\mathcal M_{\mathrm{src}}(\mathfrak f)\).  With
\(t=\mathfrak f/(Bm)\),
\cref{eq:noisy_43_scaling_reduction,eq:noisy_43_Fpp_scaling} and, for
\(1/4<\alpha<1/2\), the noisy contribution give
\begin{align}
\label{eq:noisy_43_Fpp_compute_form}
\mathcal F_{pp}
&\asymp
\mathfrak f^{-(2\alpha+2\beta-1)/(2\alpha)}
m^{(2\alpha+2\beta-1)(1-\alpha)/\alpha},\\
\label{eq:noisy_43_noise_compute_form_mid}
\sigma^2\mathcal S_t
&\asymp
\sigma^2
\mathfrak f^{(1-2\alpha)/(2\alpha)}
m^{-(1-2\alpha)/\alpha}.
\end{align}
Equating these terms and simplifying the two exponents gives
\begin{align*}
m^{[\alpha(1-2\alpha)+2\beta(1-\alpha)]/\alpha}
&\asymp\sigma^2\mathfrak f^{\beta/\alpha},\\
p(1-\alpha)+1-2\alpha
&=\alpha(1-2\alpha)+2\beta(1-\alpha),
\qquad p+1-2\alpha=2\beta,\\
m_{\sigma,\mathrm{src}}^\star
&\asymp
(\sigma^2)^{\frac{\alpha}
{\alpha(1-2\alpha)+2\beta(1-\alpha)}}
\mathfrak f^{\frac{\beta}
{\alpha(1-2\alpha)+2\beta(1-\alpha)}},\\
\mathcal R_{\sigma,\mathrm{src}}^\star
&\asymp
(\sigma^2)^{\frac{(1-\alpha)p}
{\alpha(1-2\alpha)+2\beta(1-\alpha)}}
\mathfrak f^{-\frac{(1-2\alpha)p}
{2[\alpha(1-2\alpha)+2\beta(1-\alpha)]}}.
\end{align*}
The second line is the algebraic simplification; the first relation gives
the third, and substitution into either balanced term gives the fourth.

For \(0<\alpha<1/4\),
\cref{eq:noisy_43_cumulative_kernel_regimes} instead gives the three-step
calculation
\begin{align*}
\sigma^2\mathcal S_t&\asymp\sigma^2\mathfrak f m^{-2},
&m^{[p(1-\alpha)+2\alpha]/\alpha}&\asymp
\sigma^2\mathfrak f^{(p+2\alpha)/(2\alpha)},\\
m_{\sigma,\mathrm{src}}^\star&\asymp
(\sigma^2)^{\frac{\alpha}{p(1-\alpha)+2\alpha}}
\mathfrak f^{\frac{p+2\alpha}{2[p(1-\alpha)+2\alpha]}},
&\mathcal R_{\sigma,\mathrm{src}}^\star&\asymp
(\sigma^2)^{\frac{(1-\alpha)p}{p(1-\alpha)+2\alpha}}
\mathfrak f^{-\frac{\alpha p}{p(1-\alpha)+2\alpha}}.
\end{align*}
The second relation balances the first with
\cref{eq:noisy_43_Fpp_compute_form}; substitution gives the last line.

\textit{Step 4: validity of the sub-trace corners.}
For \(1/4<\alpha<1/2\), the candidate has
\[
\begin{aligned}
m&\asymp
\mathfrak f^{\frac{\beta}
{\alpha(1-2\alpha)+2\beta(1-\alpha)}},
&\qquad
\eta t&\asymp
\mathfrak f^{\frac{\alpha(1-2\alpha)}
{\alpha(1-2\alpha)+2\beta(1-\alpha)}},\\
\frac{m^{2\alpha}}{\eta t}
&\asymp
\mathfrak f^{\frac{\alpha p}
{\alpha(1-2\alpha)+2\beta(1-\alpha)}},
&
t&\asymp
\mathfrak f^{\frac{(1-2\alpha)(\alpha+\beta)}
{\alpha(1-2\alpha)+2\beta(1-\alpha)}}.
\end{aligned}
\]
All exponents are positive, so the candidate lies polynomially inside
\(\mathcal M_{\mathrm{src}}(\mathfrak f)\).  The
\(\mathcal F_0\)-exponent gap is
\(p^2/[2(\alpha(1-2\alpha)+2\beta(1-\alpha))]>0\) in
subregimes~\(\mathrm{LM}_{1,2}\) and
\([2\alpha\beta-\tfrac12(1-2\alpha)p]/
[\alpha(1-2\alpha)+2\beta(1-\alpha)]>0\)
in subregime~\(\mathrm{LM}_3\); its numerator equals \(2\alpha^2\) at
\(\beta=1/2\) and has derivative \(4\alpha-1>0\).  In all three subregimes,
\cref{eq:noisy_43_Kpp_scaling} gives the positive kernel gap
\(\alpha(1-2\alpha)/
[\alpha(1-2\alpha)+2\beta(1-\alpha)]\).

For \(0<\alpha<1/4\), the corresponding candidate satisfies
\[
\begin{aligned}
m&\asymp
\mathfrak f^{\frac{p+2\alpha}{2[p(1-\alpha)+2\alpha]}},
&\qquad
t&\asymp
\mathfrak f^{\frac{p(1-2\alpha)+2\alpha}
{2[p(1-\alpha)+2\alpha]}},\\
\eta t&\asymp
\mathfrak f^{\frac{2\alpha^2}{p(1-\alpha)+2\alpha}},
&
\frac{m^{2\alpha}}{\eta t}
&\asymp
\mathfrak f^{\frac{\alpha p}{p(1-\alpha)+2\alpha}}.
\end{aligned}
\]
All exponents are again positive, so this candidate also lies polynomially
inside \(\mathcal M_{\mathrm{src}}(\mathfrak f)\).  The width-floor gap is
\(p^2/[2(p(1-\alpha)+2\alpha)]>0\) in subregime~\(\mathrm{FB}_1\) and
\(2\alpha^2/[p(1-\alpha)+2\alpha]>0\) in subregime~\(\mathrm{FB}_2\).
In the proved source range of \cref{prop:td_lrs_joint_de_finite_bulk}, these
gaps make the displayed forcing--noise corner govern the finite-bulk DE
formulas.  On the high-source \(\mathrm{FB}_2\) branch, the same deduction
additionally requires
\cref{ass:td_lrs_high_source_fb2_forcing,cor:td_lrs_joint_de_high_source_fb2}.

On the LM branches, \cref{cor:td_lrs_joint_de_explicit_forcing} applies for
\(\beta<1+2\alpha\) off the critical lines.  Positivity gives the matching
in-window lower bound \(\mathcal F_{pp}+\sigma^2\mathcal S_t\); evaluation at
the admissible candidate gives the upper bound.  Thus the
source-window-restricted optima in subregimes~\(\mathrm{LM}_{1,2,3}\)
satisfy
\[
\mathcal F_{pp}
\asymp
\sigma^2\mathcal S_t,
\]
whereas both \(\mathrm{FB}\) subregimes in the proved source range use
\cref{prop:td_lrs_joint_de_finite_bulk}.
Their balanced values tend to zero, proving
\(\mathcal R_{\sigma,\mathrm{src}}^\star(\mathfrak f)\to0\).

\textit{Step 5: optimization above the trace-class line.}
Suppose \(\alpha>1/2\), and let
\(t_{\sigma,U,\mathfrak f}^\star\) minimize the compute-constrained diagonal
family over \(U\).  Its associated width
\(m_{\sigma,U,\mathfrak f}^\star
:=\mathfrak f/(B t_{\sigma,U,\mathfrak f}^\star)\)
tends to infinity uniformly because \(U\subset(1,\infty)\) is compact.
Uniform convergence on \(U\) makes every subsequential
limit minimize \(\mathcal R_{\sigma,\infty}\); uniqueness gives
\(t_{\sigma,U,\mathfrak f}^\star
\to t_{\sigma,\mathrm{const}}^\star\),
with no optimizer-location rate claimed.

For the upper bound, evaluate \cref{eq:noisy_43_fixed_time_expansion} at
\(t_{\sigma,\mathrm{const}}^\star\); for the lower, use nonnegativity of the
infinite-width value difference, convergence of the minimizers, and
\(G_\sigma(t_{\sigma,\mathrm{const}}^\star)>0\).  Thus
\[
\Delta_{\sigma,U}^\star(\mathfrak f)
\asymp
\left(m_{\sigma,U,\mathfrak f}^\star\right)^{-a_\sigma}
\asymp\mathfrak f^{-a_\sigma}.
\]
By \cref{eq:noisy_43_a_sigma}, the exponent is \(2\alpha+2\beta-1\) in
\(\mathrm{IM}_1\) and \(1\) in \(\mathrm{IM}_{2,3}\), proving the
finite-time conclusion over \(U\).

The sub-trace argument gives the last two rows of
\cref{eq:noisy_43_exponents} and
\cref{eq:noisy_43_full_subtrace_laws}.  The trace-class argument gives the
first two rows of \cref{eq:noisy_43_exponents} and the finite-time conclusion,
completing the proof.
\end{proof}

\begin{proof}[Proof of \cref{prop:noisy_43_noise_plateau}]
If \(\alpha<1/2\), Steps~3--4 above show that the source-window optimum has a
strictly positive decay exponent, and hence
\(\mathcal R_{\sigma,\mathrm{src}}^\star(\mathfrak f)\to0\).  If
\(\alpha>1/2\), Step~5 identifies the finite limiting optimizer
\(t_{\sigma,\mathrm{const}}^\star\), and
\cref{eq:noisy_43_fixed_time_expansion} gives
\[
\mathcal R_{\sigma,U}^\star(\mathfrak f)
=\mathcal R_{\sigma,\infty}(t_{\sigma,\mathrm{const}}^\star)
+\Theta(\mathfrak f^{-a_\sigma}).
\]
The positive stable noise resolvent makes the limiting value strictly
positive.  This is the claimed trace-class plateau.
\end{proof}

\Cref{fig:noisy_43_compute_optimal_regime_diagram} places representative
constant-schedule Volterra trajectories from \(\mathrm{FB}_2\) and
\(\mathrm{IM}_1\) on the PLRF propagation map.  Their contrast supports the
predicted continuing forcing--noise balance below trace class and finite-time
noisy stopping within \(U\) above trace class.

\begin{figure}[t]
\centering
\includegraphics[width=\linewidth]{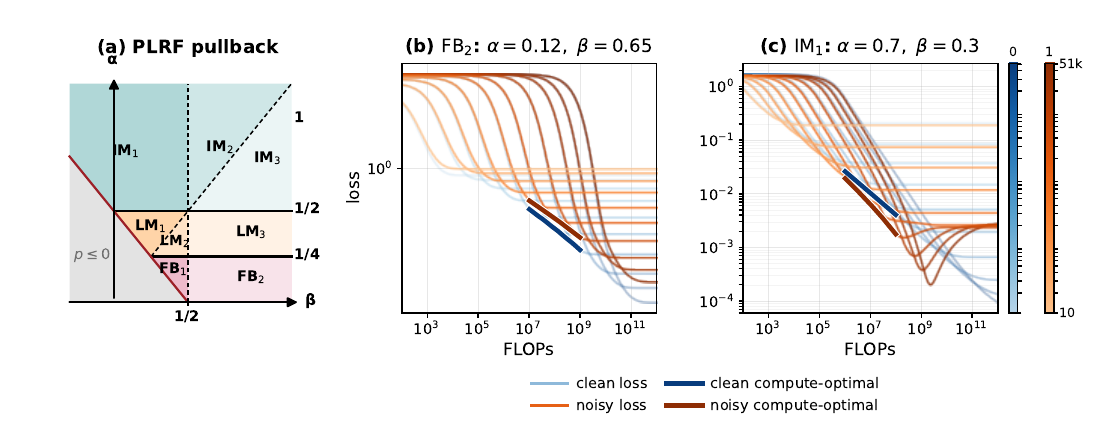}
\caption{The PLRF propagation map and representative centered-risk compute
frontiers in \(\mathrm{FB}_2\) and \(\mathrm{IM}_1\).  Their continuing
forcing--noise balance and finite-time stopping illustrate the phase-dependent
shift of compute from iterations to width.}
\label{fig:noisy_43_compute_optimal_regime_diagram}
\end{figure}

\subsection{Clean-to-noisy crossover scales}
\label{subsec:noisy_43_constant_crossover_proof}

\paragraph{Weak-noise early-stopping prediction in the \(\mathrm{IM}\) regime.}
For the crossover comparison below, the formal large-iteration matching convention
compares the clean infinite-width transient with the unsaturated
terminal-noise deficit as \(\sigma^2\downarrow0\).  With \(B\) and a stable
\(\eta\) held fixed, it predicts
\begin{equation}
\label{eq:noisy_43_weak_noise_early_stopping}
t_{\sigma,\mathrm{const}}^\star
\asymp
\begin{cases}
(\sigma^2)^{-\alpha/\beta},
&\mathrm{IM}_1,\ \beta>0,\\[0.4em]
(\sigma^2)^{-\alpha/\beta},
&\mathrm{IM}_2,\\[0.4em]
(\sigma^2)^{-1},
&\mathrm{IM}_3.
\end{cases}
\end{equation}
For subregime~\(\mathrm{IM}_1\) with \(\beta\leq0\), this balance yields no
diverging weak-noise scale.
These scales are formal order predictions from the weak-noise balance.

\paragraph{Scope of the weak-noise matching.}
For comparison only, first suppose \(\beta>0\) and set
\[
q_{\mathcal S}:=1-\frac{1}{2\alpha}>0,
\qquad
q_{\mathrm{tail}}:=q_{\mathcal S}+\min\left\{\frac{\beta}{\alpha},1\right\}
>q_{\mathcal S}.
\]
The clean infinite-width transient is \(O(t^{-q_{\mathrm{tail}}})\), so the
large-\(t\) inputs and bounded positive Volterra resolvent give
\[
\mathcal R_{\sigma,\infty}(t)
=
\mathcal R_{\sigma,\infty}(\infty)
-c_\sigma\sigma^2t^{-q_{\mathcal S}}
+O(t^{-q_{\mathrm{tail}}}).
\]
The big-\(O\) clean term has no specified leading coefficient or sign.
Locating a weak-noise minimizer would require a uniform signed clean-tail
expansion along a joint limit with \(\sigma^2\downarrow0\).

For \(\beta=0\), the decay powers coincide; for \(\beta<0\),
\(q_{\mathrm{tail}}=q_{\mathcal S}+\beta/\alpha<q_{\mathcal S}\).  Without a
signed clean-tail expansion, neither identifies a diverging weak-noise
minimizer.  Hence the \(\mathrm{IM}_1\) row of
\cref{eq:noisy_43_weak_noise_early_stopping} and the 4+3 Phase Ia crossover
entry in \cref{prop:noisy_43_matched_crossover_compute} are restricted to
\(\beta>0\); the fixed-noise plateau conclusion above is unchanged.

Under the same numerical setting, \cref{fig:noisy_43_im_stopping_dynamics}
varies noise and width to reveal how the finite early-stopping region lies
between terminal training and no training in each \(\mathrm{IM}\) subregime.

\begin{figure}[t]
\centering
\includegraphics[width=\linewidth]{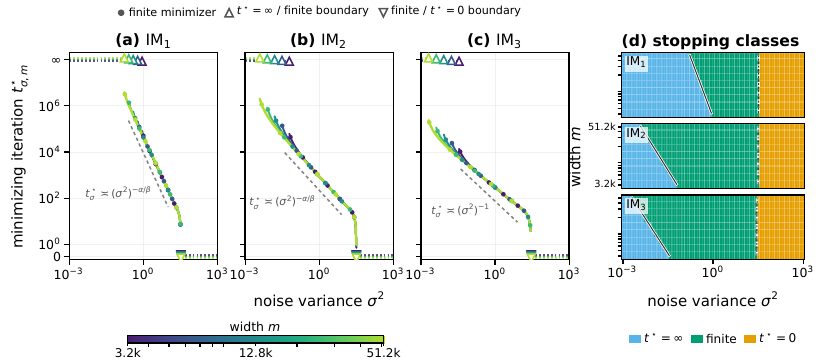}
\caption{The minimizing iteration and stopping class across noise levels and
widths in the three integrable-memory subregimes.  The terminal--finite and
finite--no-training transitions support the predicted subregime dependence of
early stopping.}
\label{fig:noisy_43_im_stopping_dynamics}
\end{figure}
At small but fixed \(\sigma^2\), when does the noisy optimum replace the clean
optimum?  Equivalently, at fixed compute, how large can the label-noise variance
be before the fixed-noise optimum takes over?  The next proposition retains the
clean \(4+3\) indexing because the
clean \(4+3\) crossover rates distinguish branches that merge under the
propagation taxonomy.  These are
crossover scales, not sharp transition points.  The proof is in
\cref{subsec:noisy_43_constant_crossover_proof}.

For any matched compute scale \(\mathfrak f_{\mathrm{cross}}(\sigma)\),
denote by \(\sigma_{\mathrm{cross}}^2(\mathfrak f)\) a fixed-compute inverse
scale satisfying
\[
\mathfrak f_{\mathrm{cross}}(\sigma_{\mathrm{cross}})
\asymp\mathfrak f.
\]

\begin{proposition}[Matched clean-to-noisy compute scales]
\label{prop:noisy_43_matched_crossover_compute}
Fix a small noise variance \(\sigma^2>0\).  Below trace class, match the noisy
term at the clean-optimal width to the clean-optimal loss.  Above trace class,
use the weak-noise large-iteration prediction in
\cref{eq:noisy_43_weak_noise_early_stopping}, and match its noisy optimal
iteration to the clean one.  In 4+3 Phase Ia (our
subregime~\(\mathrm{IM}_1\)), this convention applies only when
\(\beta>0\); for \(\beta\leq0\), it gives no diverging weak-noise minimizer.
The resulting matched compute scales in the clean \(4+3\) comparison
coordinates are
\begin{equation}
\label{eq:noisy_43_matched_crossover_compute}
\mathfrak f_{\mathrm{cross}}(\sigma)
\asymp
\begin{cases}
(\sigma^2)^{-(2\alpha+1)/(2\beta)},
&\text{4+3 Phase Ia},\ \beta>0,\\[0.3em]
(\sigma^2)^{-(\alpha+\beta)/\beta},
&\text{4+3 Phase II},\\[0.3em]
(\sigma^2)^{-2},
&\text{4+3 Phase III},\\[0.3em]
(\sigma^2)^{-2/p},
&\text{4+3 Phase Ib},\\[0.3em]
(\sigma^2)^{-\frac{2\beta-1+3\alpha-2\alpha\beta}{2\alpha^2}},
&\text{4+3 Phase Ic},\\[0.6em]
(\sigma^2)^{-\frac{2\beta(1-\alpha)-\alpha}{\alpha(1-2\alpha)}},
&\text{4+3 Phase IVb}.
\end{cases}
\end{equation}
Equivalently, within the same source/weak-noise scope, any branch written in
\cref{eq:noisy_43_matched_crossover_compute} as
\(\mathfrak f_{\mathrm{cross}}(\sigma)\asymp(\sigma^2)^{-\gamma}\)
has \(\sigma_{\mathrm{cross}}^2(\mathfrak f)\asymp\mathfrak f^{-1/\gamma}\),
with the Phase Ia restriction \(\beta>0\).  Within the source window,
\(1\ll\mathfrak f\ll\mathfrak f_{\mathrm{cross}}(\sigma)\), equivalently
\(\sigma^2\ll\sigma_{\mathrm{cross}}^2(\mathfrak f)\), gives clean-optimal
scaling; reversing both inequalities gives noisy-optimal scaling.
The 4+3 Phase IVa branch within our subregime~\(\mathrm{LM}_3\) has two matched compute scales,
\begin{equation}
\label{eq:noisy_43_IVa_crossover_compute}
\mathfrak f_{\mathrm{onset}}(\sigma)
\asymp
(\sigma^2)^{-1/\alpha},
\qquad
\mathfrak f_{\mathrm{full}}(\sigma)
\asymp
(\sigma^2)^{-
\frac{2\beta(1-\alpha)-\alpha}{\alpha(1-2\alpha)}}.
\end{equation}
Their fixed-compute inverses are
\[
\sigma_{\mathrm{onset}}^2(\mathfrak f)
\asymp \mathfrak f^{-\alpha},
\qquad
\sigma_{\mathrm{full}}^2(\mathfrak f)
\asymp
\mathfrak f^{-\alpha(1-2\alpha)/[2\beta(1-\alpha)-\alpha]}.
\]
Because \(\beta>1/2\),
\(2\beta(1-\alpha)-\alpha>1-2\alpha\), so
\(\sigma_{\mathrm{onset}}^2\ll\sigma_{\mathrm{full}}^2\).  Hence
\[
\begin{aligned}
\sigma^2\ll\sigma_{\mathrm{onset}}^2(\mathfrak f)
&\quad\Longrightarrow\quad \text{clean-optimal scaling},\\
\sigma_{\mathrm{onset}}^2(\mathfrak f)\ll\sigma^2
\ll\sigma_{\mathrm{full}}^2(\mathfrak f)
&\quad\Longrightarrow\quad \text{intermediate kernel/noise balance},\\
\sigma^2\gg\sigma_{\mathrm{full}}^2(\mathfrak f)
&\quad\Longrightarrow\quad
\text{full }\mathcal F_{pp}\text{--noise balance}.
\end{aligned}
\]
All locations are order predictions, with constants possibly depending on
\(\alpha,\beta,B,\eta,c\).
\end{proposition}

The displayed noise ranges invert \cref{eq:noisy_43_IVa_crossover_compute}.

\begin{proof}[Proof of
\cref{prop:noisy_43_matched_crossover_compute}]
For \(\alpha>1/2\), combine
\cref{eq:noisy_43_clean_compute_exponents,eq:noisy_43_weak_noise_early_stopping}
with \(t_{0,\mathrm{const}}^\star=\mathfrak f/(Bm_{0,\mathrm{const}}^\star)\).
Equating clean and weak-noise iteration scales gives
\[
\begin{array}{c|c|c|c}
\text{4+3 phase}&\text{clean }t^\star&\text{weak-noise }t^\star
&\sigma^2_{\rm match}\\ \hline
\mathrm{Ia}&\mathfrak f^{2\alpha/(2\alpha+1)}
&(\sigma^2)^{-\alpha/\beta}&\mathfrak f^{-2\beta/(2\alpha+1)}\\
\mathrm{II}&\mathfrak f^{\alpha/(\alpha+\beta)}
&(\sigma^2)^{-\alpha/\beta}&\mathfrak f^{-\beta/(\alpha+\beta)}\\
\mathrm{III}&\mathfrak f^{1/2}&(\sigma^2)^{-1}&\mathfrak f^{-1/2}.
\end{array}
\]
Inverting the last column gives the three trace-class branches of
\cref{eq:noisy_43_matched_crossover_compute}.

Below trace class, recall \(p:=2\alpha+2\beta-1\) and write
\(\mathfrak d_{\mathrm{Ic}}:=2\beta-1+3\alpha-2\alpha\beta\) and
\(\mathfrak d_{\mathrm{IVb}}:=2\beta(1-\alpha)-\alpha\).
For \(1/4<\alpha<1/2\), evaluating
\cref{eq:noisy_43_noise_compute_form_mid} at the clean-optimal width gives
\[
\sigma^2\mathcal S_t
\asymp
\sigma^2
\mathfrak f^{(1-2\alpha)/(2\alpha)}
\bigl(m_{0,\mathrm{const}}^\star(\mathfrak f)\bigr)^{-(1-2\alpha)/\alpha}.
\]
For \(0<\alpha<1/4\), use instead
\(\sigma^2\mathcal S_t\asymp\sigma^2\mathfrak f m^{-2}\).  Substituting the
clean branches from \cref{eq:noisy_43_clean_compute_exponents} gives, in
Phase Ib for either range,
\(\sigma^2\mathcal S_t\asymp\sigma^2\),
\(\mathcal R_{0,\mathrm{const}}^\star\asymp\mathfrak f^{-p/2}\), and hence
\(\sigma^2\asymp\mathfrak f^{-p/2}\).  In Phase IVa, the same noise order
matched with \(\mathfrak f^{-\alpha}\) gives
\(\sigma^2\asymp\mathfrak f^{-\alpha}\).  The remaining substitutions are
\begin{align*}
\text{4+3 Phase IVb}:\quad
\sigma^2\mathcal S_t
&\asymp\sigma^2\mathfrak f^{-(1-2\alpha)(2\beta-1)/(2\mathfrak d_{\mathrm{IVb}})},
&\mathcal R_{0,\mathrm{const}}^\star
&\asymp\mathfrak f^{-(1-2\alpha)p/(2\mathfrak d_{\mathrm{IVb}})},\\[-0.2em]
&&\sigma^2
&\asymp\mathfrak f^{-\alpha(1-2\alpha)/\mathfrak d_{\mathrm{IVb}}},\\
\text{4+3 Phase Ic}:\quad
\sigma^2\mathcal S_t
&\asymp\sigma^2\mathfrak f^{1-p/\mathfrak d_{\mathrm{Ic}}},
&\mathcal R_{0,\mathrm{const}}^\star
&\asymp\mathfrak f^{-\alpha p/\mathfrak d_{\mathrm{Ic}}},\\[-0.2em]
&&\sigma^2
&\asymp\mathfrak f^{-2\alpha^2/\mathfrak d_{\mathrm{Ic}}}.
\end{align*}
These are the Phase Ib, Ic, IVb, and IVa-onset variance thresholds; inversion
gives their compute thresholds.

For the full-noise scale on the Phase IVa branch of our
subregime~\(\mathrm{LM}_3\), evaluate the \(\mathcal F_{pp}\)--noise corner in
\cref{eq:noisy_43_full_subtrace_laws}, direct substitution into
\cref{eq:noisy_43_Kpp_scaling} gives
\[
\frac{\mathcal K_{pp}/\eta}{\mathcal F_{pp}}
\asymp
(\sigma^2)^{-
\frac{2\beta(1-\alpha)-\alpha}
{\alpha(1-2\alpha)+2\beta(1-\alpha)}}
\mathfrak f^{-
\frac{\alpha(1-2\alpha)}
{\alpha(1-2\alpha)+2\beta(1-\alpha)}}.
\]
This corner becomes self-consistent at
\[
\sigma^2
\asymp
\mathfrak f^{-
\frac{\alpha(1-2\alpha)}{2\beta(1-\alpha)-\alpha}},
\]
which is \(\sigma_{\mathrm{full}}^2(\mathfrak f)\); inversion gives
\(\mathfrak f_{\mathrm{full}}(\sigma)\).  Since \(\beta>1/2\)
implies \(\mathfrak d_{\mathrm{IVb}}>1-2\alpha\), as
\(\sigma^2\downarrow0\),
\(\mathfrak f_{\mathrm{onset}}(\sigma)\ll
\mathfrak f_{\mathrm{full}}(\sigma)\); equivalently, for large \(\mathfrak f\),
\(\sigma_{\mathrm{onset}}^2(\mathfrak f)\ll
\sigma_{\mathrm{full}}^2(\mathfrak f)\), leaving the intermediate kernel/noise
range.  This proves both crossover parametrizations, with all source-window
and order-level qualifications unchanged.
\end{proof}

Constant schedules therefore expose the same mechanism as the general theory:
the spectrum fixes memory, memory accumulates label noise, and the resulting
balance moves compute from iterations to width.  The schedule sections ask how
changing \(B_t/\eta_t\) can control this balance.

\endgroup

\section{End-to-end language-model experiments}
\label{app:end_to_end_lm_experiments}

We use three experiments to test increasingly stronger consequences of the schedule coordinates predicted by our theory.

\noindent {\bf Roadmap:} First, we fix the intrinsic-time horizon and vary the ratio path \(r(T)=B(T)/\eta(T)\), testing
whether different ratio schedules produce different loss trajectories.
Second, we fix both the intrinsic clock and the ratio path, but realize the
same path either by varying learning rate at fixed batch size or by varying
batch size at fixed learning rate.  This tests whether the loss is
approximately invariant to the particular learning-rate--batch-size
factorization.  Third, we adapt the practical seven-parameter functional scaling law of \citet{li2025functional} but fit a forcing-memory surrogate on a single
fixed-batch 8-1-1 trajectory and, without refitting, use it to predict the
alternative factorization and unseen WSD schedules. 

Together, the three experiments test \emph{ratio-path separation}, \emph{factorization collapse}, and \emph{cross-schedule prediction}.  The plain-SGD experiments directly test the coordinates \(T=\sum_t\eta_t\) and \(r=B/\eta\) derived by our theory; the Muon experiment uses separately calibrated coordinates and serves only as an optimizer-specific external-validity check. Code and configurations for reproducing these experiments are available in our \href{https://github.com/yichenblue/spectra-to-joint-schedules-in-pretraining}{GitHub repository}.

\subsection{Testing ratio-path effects at matched intrinsic time}
\label{subsec:llm_power_law_ratio_sweep}

Our theory assigns different roles to two schedule coordinates: intrinsic time
\(T_t=\sum_{s<t}\eta_s\) measures optimization progress, while
\(r_t=B_t/\eta_t\) controls the accumulation of stochastic error.  This
experiment tests whether the ratio path affects end-to-end LLM loss after the
intrinsic-time horizon is fixed.  If intrinsic time were the only relevant
coordinate, different ratio schedules would produce similar loss curves.
The theory instead predicts systematic differences, with diminishing
improvement once the memory ceiling is reached.

We train a 30M-parameter nanoGPT model \citep{Karpathy2022} on
OpenWebText \citep{Gokaslan2019OpenWeb}, with six layers, six
attention heads, embedding width \(384\), and context length \(256\).
Training uses plain SGD in FP32, without momentum, weight decay, or gradient
clipping.  We first train one model for \(196{,}608\) updates at a constant
learning rate.  We then copy this checkpoint into eleven runs and continue
training them on the same ordered data stream.  All runs use batch size \(B=8\);
only the learning-rate schedule differs.

Let \(T_{\mathrm{fork}}\) be the intrinsic time at the shared checkpoint.  For
each tail, the learning rate decays at
\[
  \eta(T)=\eta_0X^{-\vartheta},
  \qquad
  X=1+\frac{T-T_{\mathrm{fork}}}{\tau_s},
  \qquad
  \tau_s=1024\eta_0,
\]
and hence
\[
  \frac{B}{\eta(T)}
  =\frac{B}{\eta_0}X^\vartheta.
\]
Larger \(\vartheta\) therefore means faster learning-rate decay and faster
growth of \(B/\eta\).  We use
\[
\vartheta\in
\{0,0.125,0.25,0.375,0.5,0.75,1,1.25,1.5,1.75,2\}
\]
and stop every run at \(X=15\) to ensure that all runs have the same intrinsic-time horizon, \(T-T_{\mathrm{fork}}=14\tau_s\), but different numbers of updates
and tokens.  We evaluate each trajectory at \(410\) points on the same fixed
set of \(1{,}024\) validation contexts. Note that, at fixed model, starting the same checkpoint, batch size, and intrinsic-time horizon, but not at fixed updates or tokens, the validation loss changes systematically with the ratio-growth exponent.

\begin{figure}[t]
  \centering
  \includegraphics[width=0.5467\linewidth]
  {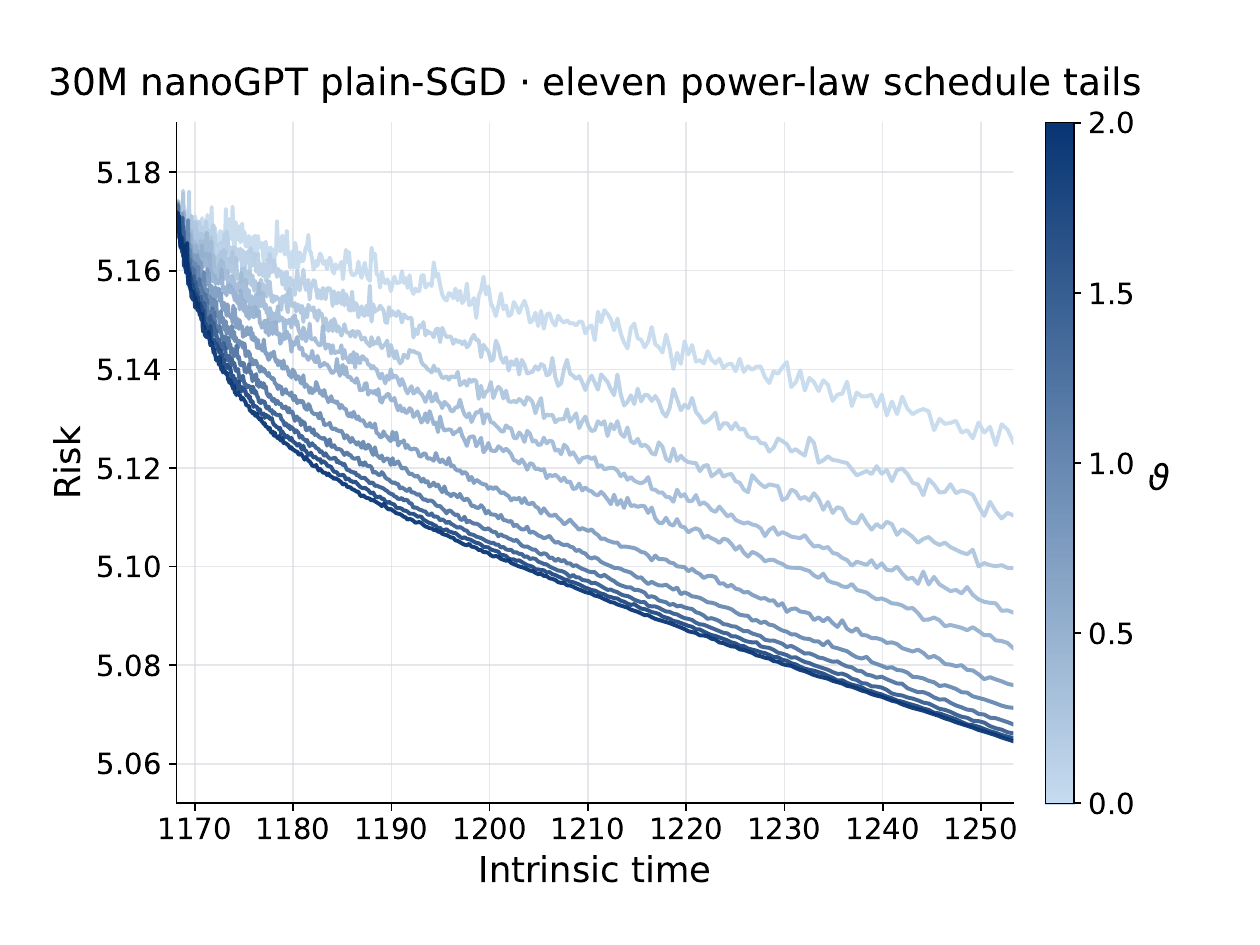}
  \caption{Effect of the ratio path at matched intrinsic time.  All runs start
  from the same checkpoint and stop at the same intrinsic-time horizon, but
  use different numbers of updates and tokens.  Larger \(\vartheta\) means
  faster growth of \(B/\eta\), and darker curves correspond to larger
  \(\vartheta\).  Validation loss decreases with \(\vartheta\), with
  diminishing improvement at large \(\vartheta\).}
  \label{fig:nanogpt30m_theta_schedule_response}
\end{figure}

\Cref{fig:nanogpt30m_theta_schedule_response} shows that the validation curves
are systematically ordered by \(\vartheta\).  At the same intrinsic time,
larger \(\vartheta\) gives lower validation loss, confirming that the ratio
path \(B/\eta\), rather than intrinsic time alone, organizes the schedule
response.  The improvement diminishes at large \(\vartheta\), as expected from
a memory ceiling.

This experiment tests these predictions only qualitatively.  It uses one seed,
measures total validation cross-entropy rather than a noisy--clean gap, and
does not independently estimate \(q_{\mathcal K}\).  It therefore does not
identify the LM/IM boundary or verify the piecewise asymptotic exponent in
\cref{thm:rv_joint_schedule}.

\subsection{Testing factorization invariance at matched ratio paths}

%This experiment tests whether learning rate and batch size affect LLM loss mainly through the intrinsic clock and their ratio path, rather than through their individual values.

This experiment tests whether an end-to-end validation-loss trajectory is
governed mainly by intrinsic time and the ratio path, rather than by a
particular factorization into learning rate and batch size.  We compare two
ratio-path shapes, named by their fixed-batch learning-rate realizations.  WSD
80/20 is constant for the first 80\% of training and decays
exponentially over the final 20\%, whereas 8-1-1 uses the same 80\% prefix
followed by two successive 10\% decay stages.  For each ratio path, one
realization fixes batch size and varies learning rate, while the other fixes
learning rate and varies batch size.  Each pair starts from the
same mature checkpoint, traverses the same future training tape, and is
evaluated on the same fixed probe of \(1024\) validation contexts.

Both experiments use the same \(300\)M-parameter GPT-2-tokenized nanoGPT
architecture \citep{Karpathy2022}: \(20\) layers, \(16\) attention heads,
width \(1024\), and context length \(256\).  Each trajectory is trained on a
frozen OpenWebText stream \citep{Gokaslan2019OpenWeb} for \(6.5\)B tokens.
The first \(5.2\)B tokens form the shared prefix, and the remaining \(1.3\)B
tokens form one schedule tail.  The left pair in
\cref{fig:nanogpt_factorization_collapse} uses plain SGD, while the right pair
uses hybrid Muon.

For plain SGD, the matched clock and ratio are
\[
  T_t=\sum_{s<t}\eta_s,
  \qquad
  r(T_t)=\frac{B_t}{\eta_t}.
\]
For a prescribed ratio path \(r(T)\), the two factorizations are
\[
  \text{fixed batch:}\quad
  B_t=B_0,\quad \eta_t=\frac{B_0}{r(T_t)},
  \qquad
  \text{fixed learning rate:}\quad
  \eta_t=\eta_0,\quad B_t=\eta_0 r(T_t).
\]
The \(300\)M plain-SGD run realizes this match exactly on an integer macro
grid.  At each grid cell, \(\Delta T=0.005\), and the integer schedule value is
\(g\in\{1,\ldots,10\}\) for WSD and
\(g\in\{3,4,9,10\}\) for 8-1-1:
\begin{align*}
  \text{fixed batch:}\quad
  &g\ \text{updates with } B=16,\quad \eta=0.005/g,\\
  \text{fixed learning rate:}\quad
  &1\ \text{update with } B=16g,\quad \eta=0.005.
\end{align*}
Both realizations consume exactly \(16g\) training contexts from the same
contiguous segment of the frozen future tape, advance intrinsic time by the
same \(\Delta T\), and realize \(B/\eta=3200g\) at every completed grid cell.
Each source-schedule value is repeated for \(16\) such cells before the next
value is used.  Hence the match is exact at grid boundaries even though the
optimizer-update counts differ.

For Muon \citep{jordan2024muon,liu2025muon}, a separate short calibration selects the empirical coordinates
\[
  \widetilde T_t=\sum_{s<t}\eta_s^2,
  \qquad
  \widetilde r(\widetilde T_t)=\frac{B_t}{\eta_t^2},
\]
with factorizations
\[
  \text{fixed batch:}\quad
  B_t=B_0,\quad
  \eta_t=\sqrt{\frac{B_0}{\widetilde r(\widetilde T_t)}},
  \qquad
  \text{fixed learning rate:}\quad
  \eta_t=\eta_0,\quad
  B_t=\eta_0^2\widetilde r(\widetilde T_t).
\]
The Muon coordinate is empirical and is not identified with the plain-SGD
ratio in our theory.

\Cref{fig:nanogpt_factorization_collapse} separates the ratio path from its
factorization.  On the optimizer-step axes, paired factorizations follow
different trajectories because they require different numbers of updates to
process the same data.  On their respective intrinsic-time axes, the two
factorizations of each WSD or 8-1-1 ratio path nearly coincide.

\begin{figure}[t]
  \centering
  \includegraphics[width=\linewidth]{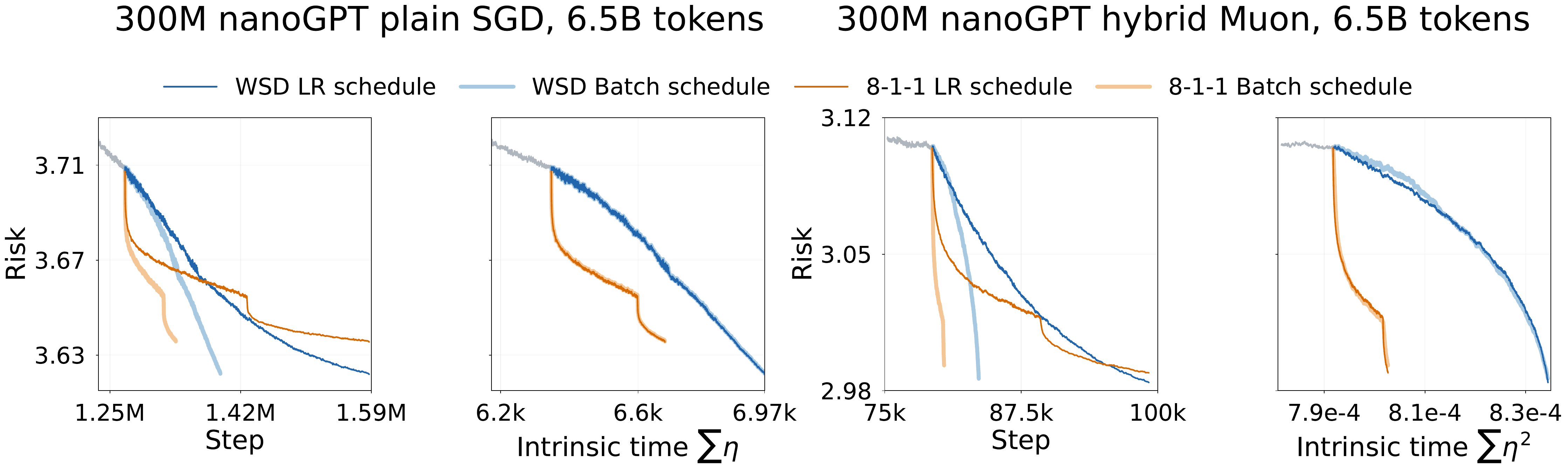}
  \caption{\textbf{Factorization collapse under matched ratio schedules.}
  Fixed-probe validation cross-entropy for the same \(300\)M nanoGPT
  architecture trained for \(6.5\)B OpenWebText tokens with plain SGD (left
  pair) and hybrid Muon (right pair),
  shown against optimizer step and the corresponding intrinsic clock.
  Within each WSD or 8-1-1 schedule family, the LR schedule varies
  \(\eta_t\) at fixed \(B_t\), whereas the batch-size schedule varies \(B_t\)
  at fixed \(\eta_t\).  For plain SGD, these two realizations follow the same
  ratio path \(r_t=B_t/\eta_t\) under \(T_t=\sum_{s<t}\eta_s\); for Muon, they
  follow the same empirically calibrated path
  \(\widetilde r_t=B_t/\eta_t^2\) under
  \(\widetilde T_t=\sum_{s<t}\eta_s^2\).  The paired trajectories differ in
  optimizer step but nearly coincide in the corresponding intrinsic time.
  Blue denotes WSD, orange denotes 8-1-1, and gray denotes the shared prefix.
  }
  \label{fig:nanogpt_factorization_collapse}
\end{figure}

\subsection{A forcing-memory surrogate predicts unseen schedules}

We finally test whether the forcing--memory mechanism is expressive enough to
describe the complete validation-risk trajectory, rather than only the
factorization collapse.  Our starting point is the practical seven-parameter
FSL ansatz of \citet{li2025functional}.  Their fixed-batch parameterization
combines an intrinsic-time clean power law with a learning-rate-decay
correction driven by the increments \(\eta_{i-1}-\eta_i\); they fit its seven
parameters on 8-1-1 and predict unseen WSD and cosine learning-rate schedules.
We retain the intrinsic-time clean-plus-memory structure and the same
fit-then-transfer protocol, but make the stochastic injection depend directly
on the joint ratio \(r(T)=B(T)/\eta(T)\).  This gives the finite-window
surrogate
\[
  \widehat L(T)
  =L_\infty+A_{\mathcal F}(1+T)^{-q_{\mathcal F}}
  +\int_0^T
  \frac{A_0+A_1(1+u)^{-q_{\mathcal F}}}{r(u)}
  \bigl(1+c_{\mathcal K}(T-u)\bigr)^{-q_{\mathcal K}}\,\mathrm{d}u \,.
\]
Its seven parameters
\((L_\infty,A_{\mathcal F},A_0,A_1,c_{\mathcal K},q_{\mathcal F},q_{\mathcal K})\)
replace the seven empirical FSL parameters of \citet{li2025functional} with a
forcing--kernel parameterization adapted to joint schedules.  Their clean
exponent \(s\) corresponds to our \(q_{\mathcal F}\), whereas their
\(\gamma\) parameterizes an integrated learning-rate-decay response rather
than the memory kernel itself; for a pure power-law kernel, the corresponding
exponents satisfy \(\gamma=q_{\mathcal K}-1\).  The ratio form also applies
when learning rate is fixed and batch size varies, a factorization not covered
by the fixed-batch LRS ansatz.

For the OpenWebText run, the seven parameters are fitted only to the
fixed-batch 8-1-1 validation trajectory and are then frozen.  The same
surrogate tracks the alternative
8-1-1 factorization and predicts both WSD trajectories without refitting
(\cref{fig:nanogpt124m_surrogate_transfer}).  Thus it captures both the
long-horizon loss decay and the change induced by a different schedule.  The
fitted exponents are \(q_{\mathcal K}=1.017\) and
\(q_{\mathcal F}=0.374\): the effective memory exponent lies close to the
\(q_{\mathcal K}=1\) boundary, while the forcing exponent \(q_{\mathcal F}<1\).

\begin{figure}[t]
  \centering
  \includegraphics[width=\linewidth]{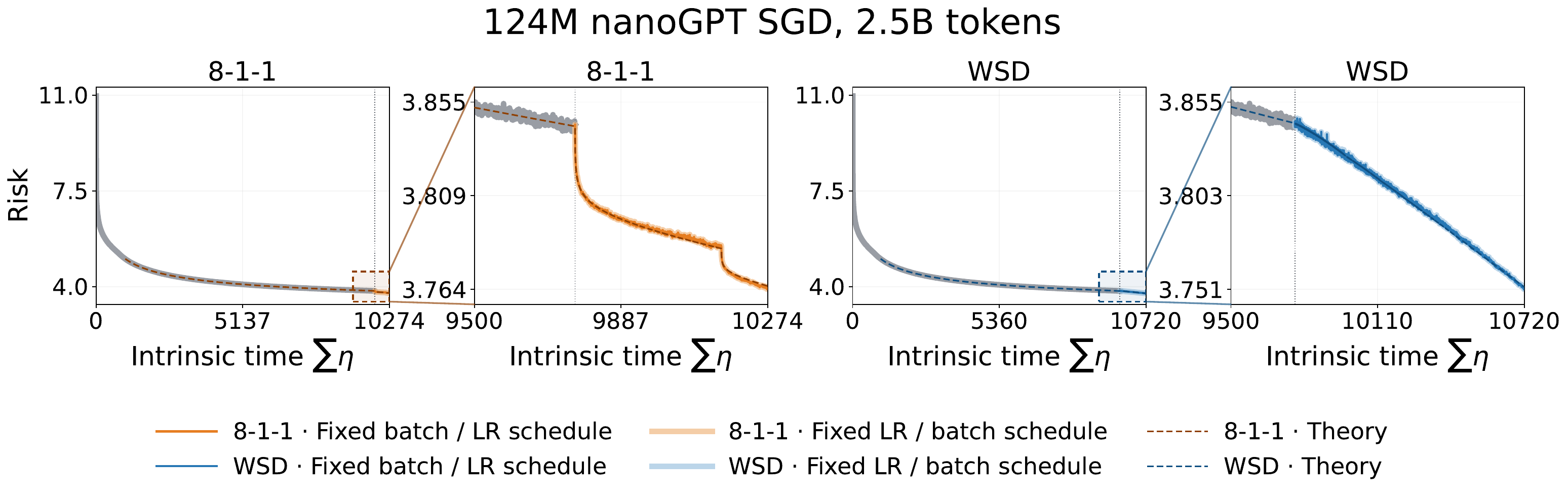}
  \caption{\textbf{A joint-schedule adaptation of the functional scaling law
  fits and transfers across LLM schedules.}
  Following the fit-then-transfer protocol of \citet{li2025functional}, the
  ratio-based surrogate is fitted only to the fixed-batch 8-1-1 validation
  trajectory, after which all parameters are frozen.  It tracks the alternative
  8-1-1 factorization and predicts both WSD trajectories without refitting,
  showing that the same forcing--memory parameterization captures long-horizon
  LLM loss decay across learning-rate and batch-size schedules.  The fitted
  exponents are \(q_{\mathcal K}=1.017\) and \(q_{\mathcal F}=0.374\), placing
  the effective memory response close to the \(q_{\mathcal K}=1\) boundary.
  Risk denotes fixed-probe validation cross-entropy and
  \(T=\sum_t\eta_t\).}
  \label{fig:nanogpt124m_surrogate_transfer}
\end{figure}

To test whether this fitted coordinate is specific to one dataset, we repeat
the identical fit-then-transfer analysis with separately trained \(124\)M
models on \(2.5\)B-token subsets of the \texttt{sample-10BT} configuration of
FineWeb \citep{penedo2024fineweb} and the peS2o V2 \texttt{s2orc} full-text
corpus \citep{soldaini2023pes2o}.  The fits give
\((q_{\mathcal K},q_{\mathcal F})=(0.952,0.405)\) on FineWeb and
\((0.995,0.365)\) on peS2o, compared with \((1.017,0.374)\) on OpenWebText
above.  Across web text and scientific full text, all three independently
fitted values of \(q_{\mathcal K}\) lie within \(0.05\) of one.  This
three-corpus agreement supports \(q_{\mathcal K}\approx1\) as a
dataset-robust effective response coordinate, rather than an artifact of a
particular pretraining corpus.  In both new datasets, the frozen parameters
also track the alternative 8-1-1 factorization and both WSD trajectories
without refitting
(\cref{fig:nanogpt124m_cross_dataset_surrogate_transfer}).

\paragraph{Protocol corresponding to the main-text experiment.}
This protocol uses a separately trained \(300\)M plain-SGD nanoGPT model on a
\(6.5\)B-token OpenWebText corpus.  For
\cref{fig:llm_surrogate_fit_transfer_profile}, we fix
\(q_{\mathcal K}=1\) and fit the remaining six parameters in
\cref{eq:llm_forcing_memory_surrogate} only to the raw fixed-batch 8-1-1
trajectory.  These six parameters are then frozen when predicting WSD.
For the profile in \cref{fig:llm_qk_prediction_error}, we repeat this procedure
for each fixed \(q_{\mathcal K}\): the remaining six parameters are fitted only
on 8-1-1, after which the post-fork WSD prediction error is evaluated without
refitting.  This is a transfer-sensitivity profile rather than an independent
estimate of the LLM memory exponent.

\paragraph{Unrestricted-fit robustness check.}
As a complementary check, we also allow all seven surrogate parameters to vary
when fitting the same raw fixed-batch 8-1-1 trajectory.  This unrestricted fit
gives \(q_{\mathcal K}=0.989\) and \(q_{\mathcal F}=0.291\).  With all seven
parameters frozen, it also predicts the alternative 8-1-1 factorization and
both WSD trajectories without refitting
(\cref{fig:nanogpt300m_sgd_surrogate_transfer}).  The agreement between the
restricted \(q_{\mathcal K}=1\) analysis and the unrestricted estimate
\(q_{\mathcal K}=0.989\) supports using one as an effective finite-window
coordinate, but does not independently identify an asymptotic memory exponent.

\begin{figure}[t]
  \centering
  \includegraphics[width=\linewidth]{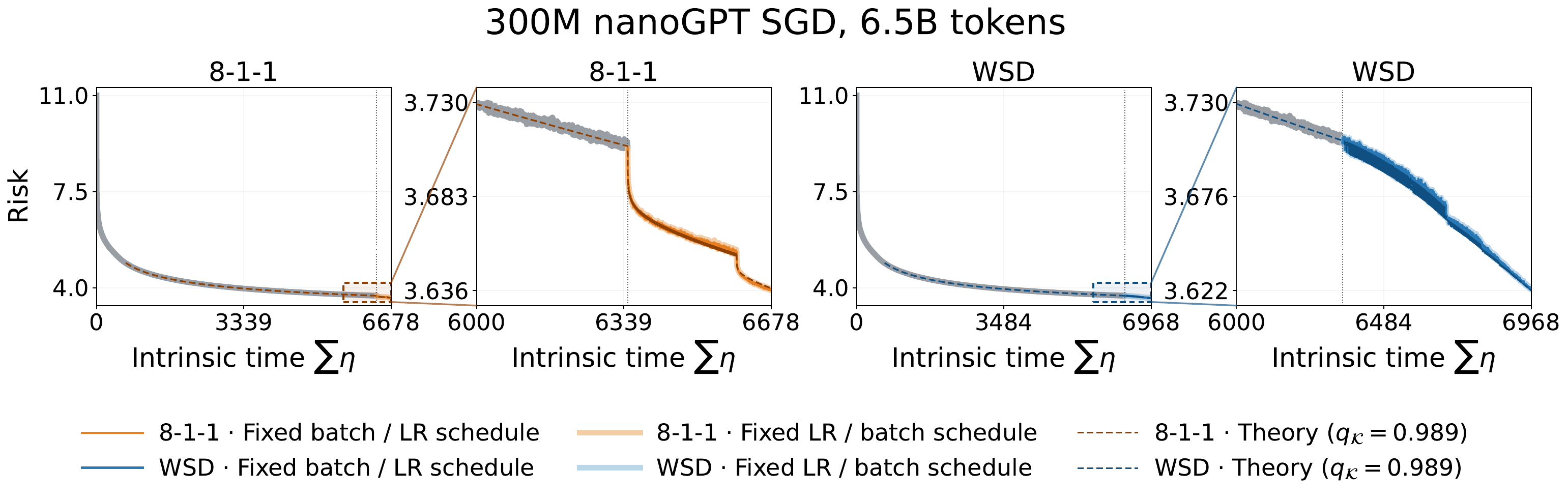}
  \caption{\textbf{The joint-schedule surrogate transfers at \(300\)M scale.}
  The seven parameters are fitted only to the raw fixed-batch 8-1-1 validation
  trajectory.  With those parameters frozen, the surrogate tracks the
  alternative 8-1-1 factorization and predicts both WSD trajectories without
  refitting.  The fitted exponents are \(q_{\mathcal K}=0.989\) and
  \(q_{\mathcal F}=0.291\), again placing the effective surrogate coordinate
  close to \(q_{\mathcal K}=1\).  Risk denotes fixed-probe validation
  cross-entropy and \(T=\sum_t\eta_t\).}
  \label{fig:nanogpt300m_sgd_surrogate_transfer}
\end{figure}

\paragraph{Summary.}
Together, the above three experiments in this section support a simple description of schedule-dependent LLM loss.  At matched intrinsic time, changing the ratio path changes the validation trajectory; when the ratio path is fixed, different
learning-rate--batch-size factorizations produce nearly the same trajectory
after intrinsic-time alignment; and a forcing-memory surrogate fitted on one
ratio path predicts unseen schedules without refitting.  Thus, for the tested
plain-SGD runs, the intrinsic clock \(T=\sum_t\eta_t\) and the ratio path
\(r(T)=B(T)/\eta(T)\) organize the loss more consistently than optimizer step,
learning rate, or batch size alone.  The successful transfer across schedules,
datasets, and model scales provides response-level evidence that the
forcing--memory mechanism remains useful beyond the random-feature proxy.
Across datasets and model scales, the fitted exponents consistently satisfy
\(q_{\mathcal K}\approx1\) and \(q_{\mathcal F}<1\), placing the tested LLMs
near the \(\mathrm{LM}_1/\mathrm{IM}_1\) boundary of our phase map.  The Muon
results similarly suggest optimizer-specific
schedule coordinates, but are not implied by our plain-SGD theory.

\end{document}